%% file: main.tex
\documentclass[twoside,11pt]{article}

\usepackage{jmlr2e}
\input{0_preamble}

\jmlrheading{}{2022}{1--84}{08/22}{}{}{}

\ShortHeadings{Consistent Algorithms for Complex Multiclass Problems}{Narasimhan, Ramaswamy, Tavker, Khurana, Netrapalli, Agarwal}
\firstpageno{1}

\begin{document}

\title{Consistent Multiclass Algorithms for\\Complex  Metrics and Constraints}

\author{
  \name 
  \hspace{-3pt}Harikrishna Narasimhan \email hnarasimhan@google.com\\
  \addr Google Research, Mountain View, USA\AND
  \name Harish G. Ramaswamy \email hariguru@cse.iitm.ac.in\\
  \addr Indian Institute of Technology Madras, Chennai, India\AND  
  \name Shiv Kumar Tavker\footnotemark
  \email tavker@amazon.com\\
  \addr  Amazon Inc., Bengaluru, India\AND
   \name Drona Khurana \email dronakhurana1294@gmail.com\\
  \addr  RBCDSAI, Indian Institute of Technology Madras, Chennai, India\AND
  \name Praneeth Netrapalli \email pnetrapalli@google.com\\
  \addr Google Research India, Bengaluru, India\AND
  \name Shivani Agarwal \email ashivani@seas.upenn.edu\\
  \addr University of Pennsylvania, Philadelphia, USA
}

\editor{}

\maketitle


\begin{abstract}
We present
 consistent algorithms
for multiclass learning with complex performance metrics and constraints, where the objective and constraints are defined by arbitrary functions
of the confusion matrix. This setting includes 
many common 
performance metrics such as the multiclass G-mean and micro $F_1$-measure, and  constraints such as those on the
classifier's precision and recall and more recent measures of fairness discrepancy. We give a
general framework for designing consistent algorithms for such complex design goals
by viewing the learning problem as an optimization problem over the set of feasible confusion matrices.
We provide multiple instantiations of our framework under 
different assumptions on the performance metrics and constraints, and in each case show rates of convergence  to the optimal (feasible) classifier (and thus asymptotic consistency).
Experiments on a variety of multiclass classification tasks and fairness constrained problems show that our algorithms compare favorably to the state-of-the-art
 baselines.
\end{abstract}


\input{1_intro}
\input{2_prelim}

\input{3_bayes}

\input{4_unconstrained}

\input{5_constrained}
\input{6_lmo}
\input{7_fairness}
\input{8_experiments}
\input{9_conclusions}
\vskip 0.2in
\bibliography{jmlr-20-paper,generalized-rates}
\input{appendix}

\end{document}

%% file: 0_preamble.tex
\usepackage{times}
\usepackage{helvet}
\usepackage{courier}
\usepackage{amssymb}
\usepackage{amsmath}

\usepackage{amsthm}
\usepackage{latexsym}
\usepackage{graphicx}
\usepackage{wrapfig}
\usepackage{color,colortbl}
\usepackage{algorithm,algorithmic}
\usepackage{url}
\usepackage{subcaption}
\usepackage{multirow}
\usepackage{enumitem}
\usepackage[symbol]{footmisc}

\newtheorem{thm}{Theorem}
\newtheorem{lem}[thm]{Lemma}
\newtheorem{prop}[thm]{Proposition}
\newtheorem{cor}[thm]{Corollary}
\newtheorem{defn}[thm]{Definition}
\newtheorem{assump}{Assumption}
\newtheorem{exmp}{Example}

\newenvironment{pf}{{\noindent\sc Proof. }}{\qed}

\newtheorem*{thm*}{Theorem}
\newtheorem*{prop*}{Proposition}
\newtheorem*{lem*}{Lemma}
\newtheorem*{exmp*}{Example}

\newcommand{\Tab}[1]{Table~\ref{#1}}

\newcommand{\Lem}[1]{Lemma~\ref{#1}}

\newcommand{\Prop}[1]{Proposition~\ref{#1}}

\definecolor{Gray}{gray}{0.8}
\renewcommand{\hat}{\widehat}
\renewcommand{\tilde}{\widetilde}
\renewcommand{\>}{{\rightarrow}}

\newcommand{\grad}{\nabla}
\newcommand{\argmax}{\operatorname{argmax}}
\newcommand{\argmin}{\operatorname{argmin}}
\newcommand{\conv}{\operatorname{conv}}

\newcommand{\R}{{\mathbb R}}
\newcommand{\Z}{{\mathbb Z}}
\newcommand{\N}{{\mathbb N}}
\renewcommand{\P}{{\mathbf P}}
\newcommand{\E}{{\mathbf E}}

\newcommand{\I}{{\mathbf I}}
\newcommand{\1}{{\mathbf 1}}
\newcommand{\0}{{\mathbf 0}}
\newcommand{\A}{{\mathbf A}}

\newcommand{\B}{{\mathcal B}}
\newcommand{\cC}{{\mathcal C}}
\newcommand{\cF}{{\mathcal F}}
\newcommand{\cE}{{\mathcal E}}
\newcommand{\C}{{\mathbf C}}
\newcommand{\F}{{\mathbf F}}

\renewcommand{\H}{{\mathcal H}}

\renewcommand{\L}{{\mathbf L}}

\newcommand{\cO}{{\mathcal O}}

\renewcommand{\B}{{\mathbf B}}

\newcommand{\W}{{\mathcal W}}
\newcommand{\X}{{\mathcal X}}
\newcommand{\Y}{{\mathcal Y}}

\newcommand{\cR}{{\mathcal R}}

\renewcommand{\c}{{\mathbf c}}

\newcommand{\h}{{h}}

\newcommand{\q}{{\mathbf q}}
\renewcommand{\r}{{\mathbf r}}

\newcommand{\w}{{\mathbf w}}
\newcommand{\x}{{\mathbf x}}

\newcommand{\zo}{\textup{\textrm{0-1}}}
\newcommand{\bal}{\textup{\textrm{bal}}}
\newcommand{\ord}{\textup{\textrm{ord}}}
\newcommand{\FN}{\textup{\textrm{FN}}}
\newcommand{\FP}{\textup{\textrm{FP}}}
\newcommand{\TN}{\textup{\textrm{TN}}}
\newcommand{\TP}{\textup{\textrm{TP}}}

\newcommand{\GM}{\textup{\textrm{GM}}}
\newcommand{\HM}{\textup{\textrm{HM}}}
\newcommand{\micro}{\textup{\textrm{micro}}}

\newcommand{\balpha}{{\boldsymbol \alpha}}
\newcommand{\bxi}{{\boldsymbol \xi}}

\newcommand{\bell}{{\boldsymbol \ell}}

\newcommand{\bphi}{{\boldsymbol \phi}}

\newcommand{\bbeta}{{\boldsymbol \beta}}
\newcommand{\seta}{{\boldsymbol \eta}}
\newcommand{\bGamma}{{\boldsymbol \Gamma}}
\newcommand{\boldeta}{{\boldsymbol \eta}}
\newcommand{\rl}{{}}

\newcommand{\bmu}{{\boldsymbol \mu}}

\newcommand{\svmp}{$\text{SVM}^\text{perf~}$}
\newcommand{\JLE}{\textup{\textrm{JLE}}}

\newcommand{\Zeta}{Z}

\newcommand{\vol}{\textup{\textrm{vol}}}

\newcommand{\g}{{\mathbf g}}

\newcommand{\performance}{\Psi}
\newcommand{\constraint}{\Phi}

\newcommand{\BER}{\textup{\textrm{BER}}}
\renewcommand{\prec}{\textup{\textrm{prec}}}
\newcommand{\cov}{\textup{\textrm{cov}}}
\newcommand{\eff}{\textup{\textrm{eff}}}
\newcommand{\KLD}{\textup{\textrm{KLD}}}
\newcommand{\DP}{\textup{\textrm{DP}}}
\newcommand{\EOpp}{\textup{\textrm{EOpp}}}

\newcommand{\optutag}{OP1}
\newcommand{\optctag}{OP2}
\newcommand{\usubscript}{\textup{U}}
\newcommand{\csubscript}{\textup{\textrm{C}}}

\newcommand{\blambda}{\boldsymbol{\lambda}}
\newcommand{\cL}{\mathcal{L}}
\newcommand{\aug}{\textup{\textrm{aug}}}
\newcommand{\con}{\textup{\textrm{con}}}
\newcommand{\ba}{\mathbf{a}}
\newcommand{\bb}{\mathbf{b}}

\renewcommand{\vec}{\textup{\textrm{vec}}}
\newcommand{\gen}{\textup{\textrm{gen}}}
\newcommand{\ocC}{\cC}


%% file: 1_intro.tex
\section{Introduction}
\label{sec:intro}
\footnotetext{$^*$Part of this work was done while SKT was a master's student at the Indian Institute of Technology Madras, India.}
In many real-world  machine learning tasks, the performance metric used to evaluate the performance of a classifier takes a complex form, and is not simply the expectation or sum of a loss on individual examples. Indeed, this is the case with the G-mean, H-mean and Q-mean performance metric used in class imbalance settings \citep{Lawrence+98,Sun+06,Kennedy+09,WangYao12,Kim+13}, the micro and macro $F_1$-measure used in information retrieval (IR) applications \citep{Lewis91}, the worst-case error used in robust classification tasks \citep{Vincent94, chen2017robust}, and many others.  Unlike linear performance metrics, which are simply linear functions (defined by a loss matrix) of the confusion matrix of a classifier, these complex performance metrics are defined by general functions of the confusion matrix. 
In this paper, we seek to
design \emph{consistent} learning algorithms for such complex performance metrics, i.e.\ algorithms that converge in the limit of infinite data to the optimal classifier for the metrics.


More generally, it is common for a classifier to be evaluated on more than one performance metric, and in such cases, a desirable goal could be to optimize the classifier's performance on one metric while constraining the others to be within an acceptable range. These constrained classification problems  commonly arise in fairness applications, where one may constrain a classifier to have equitable performance across multiple subgroups  \citep{Hardt+16, Zafar+17}, as well as, in many practical tasks where one wishes to constrain a classifier's precision, coverage, or churn \citep{Eban+17, Goh+16,Cotter+19b}.  Such metrics and constraints can be expressed as general functions of the confusion matrix, and are categorised as complex owing to their non-decomposable structure. Standard algorithmic learning frameworks are not readily designed to handle such complexity in the objectives and constraints. Doing so requires rethinking the underlying optimization schemes, as well as conducting bespoke analysis to establish algorithmic and statistical soundness. 
Practical applications and the lack of general approaches to solve such problems, motivate us to address the following question:
\begin{center}
\emph{
How can we design consistent  algorithms for a general learning problem where the  objective and (optionally) constraints are defined by general functions of the confusion matrix?
}
\end{center}

While there has been much interest in designing consistent algorithms for various types of supervised learning problems, most of this work has focused on linear performance metrics. This includes work on the binary or multiclass {0-1} loss \citep{Bartlett+06,Zhang04a,Zhang04b,Lee+04,TewariBa07}, losses for specific problems such as {multilabel classification} \citep{GaoZh11}, ranking \citep{Duchi+10,Ravikumar+11,Calauzenes+12,yang2020consistency}, and classification with abstention \citep{YuanWeg10, Ramaswamy+18, Finocchiaro+20}, and some work on general multiclass loss matrices \citep{Steinwart07,RamaswamyAg12,Pires+13,Ramaswamy+13, Nowak+20}.  
The design of consistent algorithms for constrained classification problems has also received much attention recently, particularly in the context of fairness \citep{Agarwal+18,Kearns+18,Donini+18},
with the focus  largely being on linear metrics and constraints.

There has also been much interest in designing algorithms for more complex performance metrics. One of the seminal approaches in this area is the \svmp algorithm \citep{Joachims05}, which was developed primarily for the binary setting. 
Other examples include 
convex relaxation based approaches that seek to improve upon the performance of this method \citep{kar2014online,Kar+16,Narasimhan+19_generalized}, as well as,
algorithms for the binary $F_1$-measure and its multiclass and multilabel variants \citep{Dembczynski+11,Dembczynski+13,natarajan2016optimal,zhang2020convex}.
Parallelly, there has been increasing interest in designing \emph{consistent} algorithms for complex performance metrics.
Most of these methods are focused on the binary case \citep{Menon+13,Koyejo+14,Narasimhan+14,dembczynski2017consistency}, 
and typically require tuning a single threshold or cost parameter to optimize the metric at hand. However, this simple approach of performing a one-dimensional parameter search does not  extend easily to general $n$-class problems, where one may need to search over as many as $n^2$ parameters, requiring time exponential in $n^2$.

\begin{figure}
    \centering
    \includegraphics[scale=0.35]{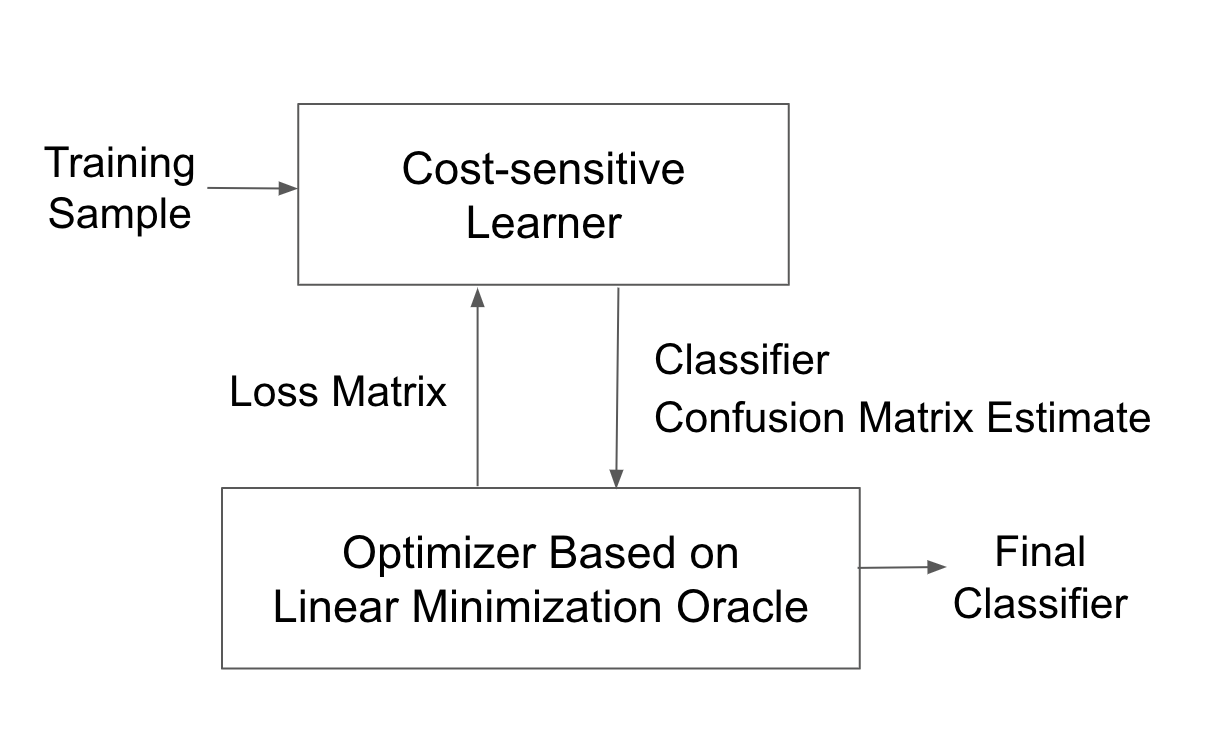}
    \vspace{-5pt}
    \caption{Simplified overview of the proposed framework.}
    \vspace{-10pt}
    \label{fig:overview}
\end{figure}

In this paper, we develop a general framework for designing statistically consistent and computationally efficient algorithms for complex multiclass performance metrics and constraints. 
Our key idea is to pose the learning problem as an optimization problem over the set of feasible and achievable confusion matrices, and to solve  this optimization problem using an optimization method that needs access to only a \textit{linear minimization} routine (see Figure \ref{fig:overview} for a simplified overview of the approach). Each of these linear minimization steps can be formulated as a \emph{cost-sensitive learning} task, a classical problem for which  numerous off-the-shelf solvers are available.

We provide instantiations of our framework under different assumptions on the performance metrics and constraints, and in each case establish rates of convergence to the optimal (feasible) classifier.
Our algorithms can be used to learn plug-in type classifiers that post-shift a pre-trained class probability model, and are shown to be effective in optimizing for the given metric and constraints on a variety of application tasks.


\subsection{Further Related Work}
The literature on complex performance metrics and constrained learning can be broadly divided into two categories:  algorithms that use surrogate relaxations \citep{Joachims05,kar2014online,Narasimhan+15b,Kar+16,Sanyal+18,Narasimhan+19_generalized}, and algorithms that use a plug-in classifier \citep{Ye+12,Menon+13,Koyejo+14,Narasimhan+14,Parambath+14,dembczynski2017consistency,yang2020fairness}. The former methods are sometimes dubbed as \textit{in-training} approaches, while the latter methods are referred to as \textit{post-hoc} approaches.
 
 A prominent example in the first category is the \svmp\ method of \cite{Joachims05},  which employs a  structural SVM formulation to construct convex surrogates for complex binary evaluation metrics. This approach does not however  extend to multiclass problems as it uses a cutting-plane finding routine whose running time grows exponentially with the number of classes. Moreover, follow-up work has shown that structural SVM style surrogates are not necessarily consistent for complex metrics \citep{Dembczynski+13}. More recent surrogate-based algorithms improve upon
this method, offering faster training procedures and better empirical performance \citep{Narasimhan+15b,Kar+16,Sanyal+18},
 but do not come with consistency guarantees.

 The second category of algorithms,
 which construct a classifier by tuning thresholds on a class-probability estimator,
 do enjoy consistency guarantees, but the bulk of the work here has focused on unconstrained binary metrics
\citep{Ye+12,Menon+13,Koyejo+14,Narasimhan+14,dembczynski2017consistency}, and for the reasons mentioned in the introduction, do not directly extend to multi-class problems.

The work that most closely relates to our paper is that of  \cite{Narasimhan+19_generalized}, where a family of algorithms is provided for optimizing complex metrics with and without constraints,
which includes as special cases some previous surrogate-based algorithms \citep{Narasimhan+15b,Kar+16}, as well as, the Frank-Wolfe based algorithm that appeared in a conference version of this paper \citep{Narasimhan+15}. 
Their key idea is to introduce auxiliary variables to re-formulate the learning task into a min-max problem, in which the 
minimization step entails solving a linear objective. They then propose  solving the  minimization step either approximately using surrogate losses, 
or exactly using a linear minimization oracle. They regard the use of surrogate relaxations to be more practical, and conduct all their empirical comparisons with this approach, although the guarantees they provide only show convergence to an optimal solution
for the surrogate-relaxed problem. We include their surrogate-based algorithms, available as a part of the TFCO library \citep{Cotter+19b}, as baselines in our experiments. 


In contrast to the methods of \cite{Narasimhan+19_generalized}, our focus is on designing algorithms that are statistically consistent, and do so using linear minimization oracles (such as plug-in classifiers) that are efficient to implement. We propose various algorithms for different problem settings, and in each case, provide consistency guarantees and rates of convergence to the optimal (feasible) classifier. For one particular problem setting (discussed in Sections \ref{sec:gda} and \ref{sec:con-gda}), both the metrics involved are  non-smooth convex functions of the confusion matrix. The algorithms we provide for this setting are a direct adaptation of the framework presented in \cite{Narasimhan+19_generalized}, but come with complete consistency analyses. 



Our paper is also closely related to the growing literature on machine learning fairness, where 
the use of constrained optimization has become one of the dominant approaches for enforcing fairness goals. The metrics handled here are typically \textit{linear} functions of (group-specific) confusion matrices \citep{Hardt+16}, with the approaches proposed using both surrogate relaxations \citep{Zafar+17,Zafar+17b, Goh+16, Cotter+19, Cotter+19b}
 and linear minimization oracles \citep{Agarwal+18,Kearns+18, yang2020fairness}.
Recently, \citet{Celis+19} extended the work of \citet{Agarwal+18} to handle more complex fairness constraints that can be written as a difference of linear-fractional metrics, but require solving a large  number of linearly-constrained sub-problems, with the number of sub-problems growing exponentially with the number of groups.

Other related work includes that of \cite{Eban+17} and \cite{Kumar+2021}, which use surrogate approximations to solve specialized non-decomposable constrained problems, such as maximizing precision subject to recall constraints. The work of \cite{chen2017robust} provides provable algorithms to minimize the maximum among multiple linear metrics 
using an oracle subroutine.

\subsection{Contributions}
The main contributions of this paper are summarized below.
\begin{itemize}[topsep=5pt,leftmargin=25pt,itemsep=2pt]
    \item We provide a characterization of the Bayes optimal classifier for unconstrained and constrained minimization of complex multiclass metrics (see Section \ref{sec:bayes}).
    \item We  propose a unified framework for designing  consistent algorithms for complex multiclass metrics and  constraints given access to a linear minimization oracle, i.e., a cost-sensitive learner (see Section \ref{sec:unconstrained}).
    \item 
    For  unconstrained metrics,  we identify four optimization algorithms that only require access to a linear minimization oracle. 
    These include (i) the Frank-Wolfe method for smooth convex metrics, 
    (ii) the gradient-descent ascent algorithm and
    (iii) the ellipsoid method for general convex metrics, and  (iv) the bisection method for ratio-of-linear metrics 
    (see Section \ref{sec:unconstrained}). 
    \item 
    For constrained learning problems, where the classifier is required to satisfy some constraints on the confusion matrix in addition to performing well on a complex metric, we provide four algorithms as counterparts to the ones mentioned above 
    (see Section \ref{sec:constrained}).
    \item We 
    show that the proposed algorithms are statistically consistent when used with a plug-in based linear minimization routine 
    (see Section \ref{sec:lmo}), and also show how they can be extended to handle fairness constraints over multiple subgroups (see Section \ref{sec:fairness}).
    \item We conduct an extensive evaluation of the proposed algorithms on benchmark multiclass, image classification,  and fair classification datasets, and show that they perform comparable to or better than the state-of-the-art approaches in each case. We also provide practical guidelines on choosing an appropriate algorithm for a given setting (see Section \ref{sec:experiments}).
\end{itemize}

 The following is a summary of the main differences from the conference
versions of this paper \citep{Narasimhan+15, Narasimhan18, Tavker+2020}. 
\begin{itemize}[itemsep=0pt,topsep=5pt,leftmargin=25pt]
    \item A definitive article on the broader topic of learning with complex metrics and constraints, with improved exposition and intuitive illustrations.
    \item New ellipsoid-based algorithms for convex performance metrics with a linear convergence rate (\emph{albiet} with a dependence on dimension).
    \item Improved  bisection-based algorithm for ratio-of-linear performance metrics with a better convergence rate for handling constraints.
    \item An adaptation of the gradient descent-ascent algorithm from \cite{Narasimhan+19_generalized} with a complete consistency analysis.
    \item Convergence results presented for a general linear minimization oracle, with the plug-in method as a special case.
    \item New set of experiments including benchmark image classification tasks. 
\end{itemize}

\noindent All  proofs not provided in the main text can be found in Appendix \ref{app:proofs}.

%% file: 2_prelim.tex
\section{Preliminaries and Examples}
\label{sec:prelim}

\textbf{Notations.}\ For $n\in\Z_+$, we denote $[n] = \{0,\ldots,n-1\}$. For matrices $\A,\B\in\R^{n\times n}$, we denote $\|\A\|_1 = \sum_{i,j} |A_{ij}|$ and $\langle \A,\B \rangle =\sum_{i,j} A_{i,j}B_{i,j}$. The notation $\argmin^*_{i\in[n]}$ will denote ties being broken in favor of the larger number. We use $\Delta_n$ to denote the $(n-1)$-dimensional probability simplex. See Table \ref{tab:symbols} in the appendix for a summary of other common symbols in the paper.~\\[-10pt]

We are interested in general multiclass learning problems with instance space $\X \subseteq \R^q$ and label space $\Y=[n]$. Given a finite training sample $S=((x_1,y_1),\ldots,(x_N,y_N)) \in(\X\times [n])^N$, the goal is to learn a multiclass classifier $h:\X\>[n]$, or more generally, a \emph{randomized} multiclass classifier $h:\X\>\Delta_n$, which given an instance $x$ predicts a class label in $[n]$ according to the probability distribution specified by $h(x)$. 
We assume examples are drawn iid from some distribution $D$ on $\X\times[n]$, and denote the marginal distribution over $\X$ by $\mu$, the class-conditional distribution by $\eta_i(x) = \P(Y=i \,|\, X=x)$, and the class prior probabilities by $\pi_i = \P(Y=i)$.

\subsection{Performance Metrics Based on the Confusion Matrix}
\label{subsec:performance-measures}
We will measure the performance of a classifier in terms of its confusion matrix.
\begin{defn}[Confusion matrix]
The \emph{confusion matrix},  $\C[h] \in [0,1]^{n\times n}$, of a randomised classifier $h$ w.r.t.\ a distribution $D$ has entries defined as 
\vspace{-2pt}
\[
C_{ij}[h] = \P_{(X,Y)\sim D,\,\widehat{Y} \sim h(X)}\big( Y=i,\, \widehat{Y}=j \big)
	\,,
\]
where $\widehat{Y} \sim h(X)$ denotes a random draw of label from $h(X)$. We can get the prior class probabilities, and fractions of instances predicted as a particular class from $\C[h]$ by marginalisation as follows :
$\sum_j C_{ij}[h] = \P(Y=i) := \pi_i $, and $\sum_i C_{ij}[h] = \P(h(X)=j)$.
\end{defn}

We will be interested in general, complex performance metrics that can be expressed as an arbitrary function of the entries of the confusion matrix $\C[h]$. 
For any function $\psi:[0,1]^{n\times n}\>\R_+$, we define the performance metric of $h$ 
follows:
\[
\performance[h] = \psi(\C[h]).
\]
We adopt the convention that \emph{lower} values of $\psi$ correspond to \emph{better} performance.

As the following examples show, this formulation captures both common cost-sensitive classification, which corresponds to  linear functions of the entries of the confusion matrix, and more complex performance metrics such as the G-mean, micro $F_1$-measure, and several others.

\begin{exmp}[Linear performance metrics]
Consider a multiclass loss matrix $\L\in\R^{n\times n}$, such that $L_{ij}$ represents the cost incurred on predicting class $j$ when the true class is $i$. In such ``cost-sensitive learning'' settings \citep{Elkan01}, the performance of a classifier $h$ is measured by the expected loss on a new example from $D$, which amounts to computing a linear function of the confusion matrix $\C[h]$:
\begin{eqnarray*}
\performance[h] 
	& = &
	\E\big[L_{Y,h(X)} \big] 
	~ = ~
	\sum_{i,j} L_{ij} \, C_{ij}[h]
	~ = ~
	\psi^\L(\C[h])
	\,,
\vspace{-8pt}
\end{eqnarray*}
where $\psi^\L(\C) = \langle \L,\C \rangle ~~\forall \C\in[0,1]^{n\times n}$. 
For example, for the 0-1 loss given by $L^\zo_{ij} = \1(i\neq j)$, we have $\psi^\zo(\C) = 1-\sum_i C_{ii}$; 
for the balanced 0-1 loss given by $L^{\bal}_{ij} = \frac{1}{n\pi_i}\1(i\neq j)$, we have $\psi^\bal(\C) = 1-\frac{1}{n}\sum_i \frac{1}{\pi_i}C_{ii}$;
for the absolute loss used in ordinal regression, $L^\ord_{ij} = |i-j|$, we have 
$\psi^\ord(\C) = \sum_{i,j} |i-j| C_{ij}$. 
\end{exmp}

\begin{exmp}[Binary performance metrics]
In the binary setting, the confusion matrix of a classifier contains the proportions of true negatives ($C_{00}=\TN$), false positives ($C_{01}=\FP$), false negatives ($C_{10}=\FN$), and true positives ($C_{11} =\TP$). Our framework therefore includes any binary performance metric that is expressed as a function of these quantities, including 
the balanced error rate metric \citep{Menon+13} given by
\(
\psi^{\BER}(\C) = \frac{1}{2} \big( \frac{\FP}{\TP+\FN} + \frac{\FN}{\TN+\FP} \big)
	\,,
\)
the $F_\beta$-measure  given by 
\(
\psi^{F_\beta}(\C) 
	 =  
	1 - \frac{(1+\beta^2) \, \TP}{(1+\beta^2)\,\TP + \beta^2\,\FN + \FP}
	\,
\)
for any $\beta > 0$, 
all ``ratio-of-linear'' binary performance metrics \citep{Koyejo+14}, and more generally, all ``non-decomposable'' binary performance metrics \citep{Narasimhan+14}.
\end{exmp}

\begin{exmp}[G-mean metric]\label{ex:g-mean}
The G-mean metric is used to evaluate both binary and multiclass classifiers in settings with class imbalance \citep{Sun+06,WangYao12}, and is given by
\[
\psi^\GM(\C) = 1 \,-\, \bigg( \prod_{i} \frac{C_{ii}}{\sum_{j} C_{ij}} \bigg)^{1/n}
	\,.
\]
\end{exmp}

\begin{exmp}[Micro $F_1$-measure]
\label{ex:micro-F1}
The micro $F_1$-measure is widely used to evaluate multiclass classifiers in information retrieval and information extraction applications \citep{Manning+08}. Many variants have been studied; we consider here the form used in the BioNLP challenge \citep{Kim+13}, which treats class 0 as a `default' class and is effectively given by the function\footnote{Another popular variant of the micro $F_1$ involves averaging the entries of the `one-versus-all' binary confusion matrices for all classes, and computing the $F_1$ for the averaged matrix; as pointed out by \citet{Manning+08}, this form of micro $F_1$ effectively  reduces to the 0-1 classification accuracy.
}
\[
\psi^{\micro F_1}(\C) = 1 \,-\, \frac{2\sum_{i\ne0} C_{ii}}{2 - \sum_{i} C_{0i} - \sum_{i} C_{i0}}
	\,.
\]
\end{exmp}

In \Tab{tab:perf-measures}, we provide other examples of performance metrics that are given by (complex) functions of the confusion matrix, which include the macro $F_1$-measure \citep{Lewis91}, the H-mean \citep{Kennedy+09}, the Q-mean \citep{Lawrence+98}, 
and the min-max metric in detection theory \citep{Vincent94} and for worst-case performance optimization \citep{chen2017robust}.

\begin{table}[t]
\begin{center}
\caption{Left: examples of complex multiclass performance metrics. Right:  examples of complex constraint functions. 
In the table on the right, $\pi_y = \P(Y=y)$, $\tau_i$ is the target value for class $i$, 
and $\epsilon>0$ is a small slack.
Rows 4--6 contain fairness metrics with $m$ protected groups, where $A(x) \in [m]$ is the protected group for instance $x$, $\mu_a = \P(A(X) =a)$, and $\mu_{a,i} = \P(A(X) =a, Y = i)$.
Row 5 is defined for binary labels $\Y = \{0,1\}$.
Rows 3--6 can be equivalently written as separate constraints on individual classes (and groups), but have been conveniently expressed in terms of the maximum constraint violation. 
}
\label{tab:perf-measures}
\vspace{2pt}
\begin{small}
\begin{tabular}{@{}cc@{}}
\hline
Metric \rule{0pt}{5pt} & $\psi(\C)$ \\[2pt]
\hline
\rule{0pt}{12pt} {G-mean} \rule{0pt}{12pt} & $1\,-\,\left( \prod_{i} \frac{C_{ii}}{\sum_{j} C_{ij}} \right)^{1/n}$ \\[6pt]
{H-mean} & $1 \,-\, n\left({\sum_{i} \frac{\sum_{j} C_{ij}}{C_{ii}}}\right)^{-1}$ \\[4pt]
Q-mean & $\sqrt{\frac{1}{n}\sum_{i}\left(1-\frac{C_{ii}}{\sum_{j} C_{ij}}\right)^2}$\\[6pt]
Micro $F_1$ & $1 \,-\, \frac{2\sum_{i>0} C_{ii}}{2 - \sum_{i}C_{1i} - \sum_{i} C_{i1}}$ \\[7pt]
Macro $F_1$ & $1 \,-\,\frac{1}{n}\sum_{i}\frac{2C_{ii}}{\sum_{j} C_{ij} \,+\, \sum_{j}C_{ji}}$ \\[6pt]
Min-max & $\max_{i}\left(1\,-\,\frac{C_{ii}}{\sum_{j} C_{ij}}\right)$\\[7pt]
\hline
\end{tabular}
\hspace{10pt}
\begin{tabular}{@{}cc@{}}
\hline
Constraint Function\rule{0pt}{5pt} & $\phi(\C)$ \\[2pt]
\hline
\rule{0pt}{12pt} 
{Class $i$ Precision} & $ 1\,-\,\frac{C_{ii}}{\sum_j C_{ji}} - \tau_i$ \\[6pt]
{ Quantification} & $
\sum_{i=1}^n\pi_i \log\left(
\frac{\pi_i}{\sum_{j=1}^n C_{ji}}
\right) - \epsilon
$
\\[8pt]
{Coverage}  & $\displaystyle\max_{i\in[n]}\textstyle|\sum_{j}C_{ji} \,-\, \tau_i| -\epsilon$
\\[8pt]
{Demographic Parity} & 
$\displaystyle\max_{a\in[m],\,i\in [n]}\textstyle\left|
	\frac{1}{\mu_a}\sum_{j}C^a_{ji} \,-\, \sum_{j}C_{ji}
\right| -\epsilon$
\\[8pt]
Equal Opportunity &
$
\displaystyle\max_{a\in[m]}
\textstyle
\left|
	\frac{1}{\mu_{a1}}C^a_{11} \,-\, \frac{1}{\pi_1}C_{11}
\right| - \epsilon$ 
\\[8pt]
Equalized Odds &
$
\displaystyle\max_{a\in[m],\, i,j \in [n]}
\textstyle
\left|
	\frac{1}{\mu_{ai}}C^a_{ij} \,-\, \frac{1}{\pi_i}C_{ij}
\right| - \epsilon$ 
\\[8pt]
\hline
\end{tabular}
\end{small}
\end{center}
\vspace{-12pt}
\end{table}

\subsection{Constraints Based on the Confusion Matrix}
\label{subsec:constraints}
We will also be interested in machine learning goals that can be expressed as constraints on a classifier's output. Specifically, we will consider constraints that can be expressed as a general function of the classifier's confusion matrix, i.e.\  constraints on $h$ of the form $\constraint_k[h] \leq 0, \,\forall k \in [K]$,  where
\[
\constraint_k[h] \,=\, \phi_k(\C[h])
\]
for some  $\phi_k: [0,1]^{n\times n}\>\R$. 
As shown in the following examples, this formulation includes constraints on precision, predictive coverage, fairness criteria and many others.
\begin{exmp}[Precision]
\label{ex:prec}
A common goal in real-world applications is to constrain the precision of a classifier for a particular class $i$ (i.e.\ the number of correct predictions for class $i$ divided by the total number of class  $i$ predictions) to be above a certain threshold $\tau_i$. Denoting
$\phi^{\prec\text{-}i}(\C) \,=\, 1 \,-\, \frac{C_{ii}}{\sum_j C_{ji}} - \tau_i$, this constraint can be written as $\phi^{\prec\text{-}i}(\C) \leq 0$.
\end{exmp}

\begin{exmp}[Coverage]
\label{ex:coverage}
A classifier's coverage for  class $i$ is the proportion of examples that are predicted as $i$. Prior work has looked at constraining the coverage for different classes to match a target distribution $\tau \in \Delta_n$ \citep{Goh+16, Cotter+19b}. This  can be formulated as a non-positivity constraint on the maximum coverage violation, given by $\phi^\cov(C) \,=\, \max_i|\sum_{j}C_{ji} \,-\, \tau_i| - \epsilon$, for a small slack $\epsilon > 0$. A variant of this constraint in the \textit{quantification} literature \citep{Fab1, Fab2} aims to match a classifier's coverage with the class prior distribution $\pi$, with the KL-divergence between the two distributions used as the measure of discrepancy: $\phi^\KLD(C) \,=\, \sum_{i=1}^n\pi_i \log\left(
\frac{\pi_i}{\sum_{j=1}^n C_{ji}}
\right) - \epsilon.$
\end{exmp}

We next provide examples of  fairness goals in machine learning that can be  expressed as constraints on (group-specific) confusion matrices.
In a typical fairness setup, each instance $x$ is associated with one of $m$ protected groups. For convenience, we will denote the protected group for a instance $x$ by $A(x) \in [m]$. 
\begin{defn}[Group-specific confusion matrix]
\label{defn:group-conf}
The confusion matrix of a classifier $h$  w.r.t.\ a distribution $D$  specific to group $a\in[m]$,  $\C^a[h] \in [0,1]^{n\times n}$, has entries defined as 
\[
C^a_{ij}[h] \,=\,\P_{(X,Y)\sim D,\,\widehat{Y} \sim h(X)}\big(Y=i,\, \widehat{Y}=j,\, A(X) =a\big),
\]
where 
$\widehat{Y} \sim h(X)$ denotes a random draw of label from $h(X)$. 
We denote the fraction of instances with protected attribute $a$ as $\mu_a$, i.e. $P(A(X)=a)=\mu_a = \sum_{i,j} C^a_{ij}$, and the fraction of instances with  protected attribute $a$ and label $i$ by $\mu_{a,i}$, i.e. $P(A(X)=a, Y=i)=\mu_{a,i} = \sum_{j} C^a_{ij}$. Clearly, the general confusion matrix can be expressed as $C_{ij}=\sum_{a\in[m]}C_{ij}^a$.
\end{defn}

The following fairness goals are given by general functions of the $m$ group-specific confusion matrices $\C^1, \ldots, \C^m$, and are also summarized in Table \ref{tab:perf-measures}.
\begin{exmp}[Demographic parity fairness]
A popular fairness criterion is demographic parity, which for a  problem with binary labels $\Y = \{0,1\}$, requires the proportion of class-1 predictions to be the same for each protected group \citep{Hardt+16}. This can be generalized to multiclass problems by requiring the proportion of prediction for each class $i$ to be the same for each protected group. We can enforce this criterion (approximately) by defining the demographic parity violation as $\phi^{\DP}(\C^{10}, \ldots, \C^{m-1}) \,=\, \max_{a\in[m],\,i\in [n]}\left|
	\frac{1}{\mu_a}\sum_{j}C^a_{ji} \,-\, \sum_{j}C_{ji}
\right| - \epsilon$, where $\epsilon > 0$ is a small slack that we allow,  and requiring that $\phi^{\DP}(\C^0, \ldots, \C^{m-1}) 
\leq 0$.
\end{exmp}

\begin{exmp}[Equal opportunity fairness]
Another popular fairness goal for problems with binary labels $\Y = \{0,1\}$ is the equal opportunity criterion \citep{Zafar+17, Hardt+16}, which requires that the  true positive rates be the same for  examples belonging to each group. One can approximately enforce this criterion by defining the equal opportunity violation  $\phi^{\EOpp}(\C^{0}, \ldots, \C^{m-1}) \,=\, \max_{a\in[m]}
\left|
	\frac{1}{\mu_{a1}}C^a_{11} \,-\, \frac{1}{\pi_1}C_{11}
\right| - \epsilon$ with a small slack $\epsilon > 0$, and imposing the constraint $\phi^{\EOpp}(\C^0, \ldots, \C^{m-1}) \leq 0$.
\end{exmp}


Other examples of constraints that can be defined by a general function of the confusion matrix or its generalizations include 
the equalized odds fairness constraint \citep{Hardt+16}, 
constraints on classifier churn \citep{Cormier+16, Goh+16, Cotter+19}, constraints on the performance of a classifier on multiple data distributions with varying quality \citep{Cotter+19}, and constraints that encode performance in select portions of the ROC or precision-recall curves \citep{Eban+17}. 

For ease of exposition, we will focus on metrics and constraints that are defined by a function of the overall confusion matrix $\C[h]$, and discuss in Section \ref{sec:fairness} how our approach can be extended to handle metrics defined by group-specific confusion matrices for fairness problems. 

\subsection{Learning Problems and Consistent Algorithms}
One of our goals in this paper is to design learning algorithms for optimizing a  performance metric of the form $\performance[h] = \psi(\C[h])$:
\begin{equation}
    \min_{h:\X\>\Delta_n}\,\performance[h].
    \tag{\optutag}
    \label{eq:opt-unconstrained}
\end{equation}
We will also be interested in designing consistent learning algorithms for optimizing a performance measure $\performance[h] = \psi(\C[h])$ subject to constraints on $\constraint_k[h] = \phi_k(\C[h]),\,\,\forall k \in [K]$:
\begin{equation}
    \min_{h:\X\>\Delta_n}\,\performance[h]~~~\text{s.t.}~~~\constraint_k[h] \leq 0, \, \forall k \in [K]
    \tag{\optctag}.
    \label{eq:opt-constrained}
\end{equation}

More specifically, we wish to design algorithms that are provably \emph{consistent} for \ref{eq:opt-unconstrained} and \ref{eq:opt-constrained}, in that they converge in probability to the optimal performance for these problems (and when there are constraints, to zero constraint violations)
as the training sample size increases.

\begin{defn}[Consistent algorithm for the unconstrained problem]
We define the optimal value w.r.t.\ $D$ for the unconstrained problem in 
\emph{\optutag} as the minimum value of the performance measure $\performance[h]$ over all randomized classifiers $h$:
\[
\performance_\usubscript^* = \inf_{h:\X\>\Delta_n} \performance[h].
\]
We say a multiclass algorithm that given a training sample $S$ returns a classifier $h_S:\X\>\Delta_n$ is \emph{consistent w.r.t.\ $D$} for \emph{\optutag} if $\forall \nu > 0$:
\[
\P_{S\sim D^N}\big( \performance[h_S] \,-\, \performance^*_\usubscript > \nu \big) \rightarrow 0 ~~~ \text{as $N\>\infty$}
	\,.
\]
\end{defn}

For the constrained problem, we require the algorithms to additionally converge to zero constraint violations in the large sample limit.

\begin{defn}[Consistent algorithm for the constrained problem]
We define the optimal value for the constrained problem in \emph{\optctag} as the minimum value of the performance measure $\performance[h]$ among all randomized classifiers $h$ that satisfy the $K$ constraints:
\[
\performance_\csubscript^* = \inf_{h:\X\>\Delta_n,\, \constraint_k[h] \leq0\,\forall k} \performance[h].
\]
Given a training sample $S$, we say a multiclass algorithm that, returns a classifier $h_S:\X\>\Delta_n$ is \emph{consistent w.r.t.\ $D$} for \emph{\optctag} if $\forall \nu > 0$:
\[
\P_{S\sim D^N}\big( \performance[h_S] \,-\, \performance^*_\csubscript > \nu \big) \rightarrow 0 ~~~ \text{and} ~~~
\P_{S\sim D^N}\big(\forall k,~ \constraint_k[h_S] > \nu\big) \rightarrow 0 ~~~ \text{as $N\>\infty$}
	\,.
\]
\end{defn}

In developing our algorithms, we will find it useful to also define the \emph{empirical} confusion matrix of a classifier $h$ w.r.t.\ sample $S$, denoted by $\widehat{\C}[h] \in [0,1]^{n\times n}$, as
\vspace{-2pt}
\[
\hat{C}_{ij}[h] = \frac{1}{N}\sum_{\ell=1}^N\1(y_\ell = i, h(x_\ell) = j)
	\,.
\vspace{-2pt}
\]
%

%% file: 3_bayes.tex
\section{Bayes Optimal Classifiers}
\label{sec:bayes}
As a first step towards designing consistent algorithms, we start by examining the form of Bayes optimal classifiers for \ref{eq:opt-unconstrained} and \ref{eq:opt-constrained}. 
It is well known that for the simpler linear performance measures (as is the case with cost-sensitive learning problems), any classifier that picks a class that minimizes the expected loss conditioned on the instance is optimal (see e.g.\ \citet{Lee+04}):
\begin{prop}
\label{prop:loss-opt}
Let $\L\in\R^{n\times n}$ be a loss matrix. Then any (deterministic) classifier $h^*$ satisfying
\vspace{-2pt}
\[
h^*(x) \in \argmin_{j\in[n]} \textstyle{\sum_{i=1}^n \eta_i(x) L_{ij}}
\vspace{-2pt}
\]
is optimal for $\psi^\L$, i.e.\ 
$\langle \L, \C[h^*]\rangle  = \displaystyle \min_{h:\X\>\Delta_n}\langle \L, \C[h] \rangle$. 
\end{prop}

\begin{figure}[t]
\centering
\begin{subfigure}[b]{0.31\linewidth}
\centering
\includegraphics[width=0.95\linewidth]{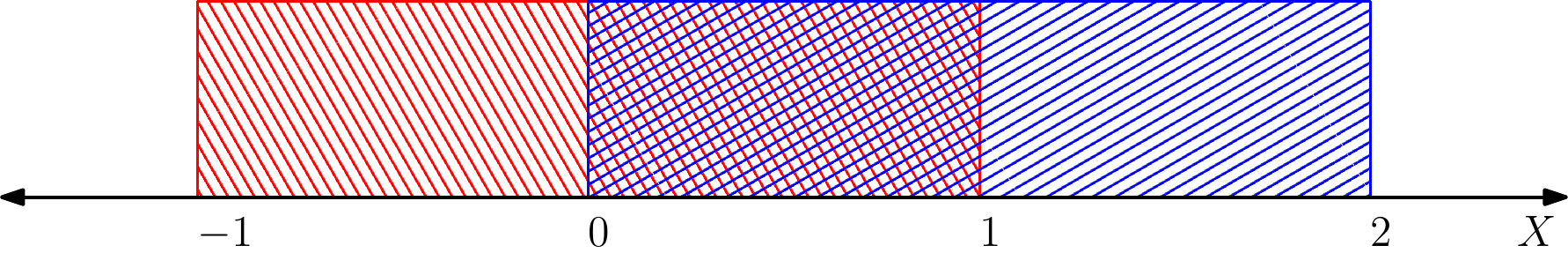}
\vspace{5pt}
\end{subfigure}
\hspace{5pt}
\begin{subfigure}[b]{0.31\linewidth}
\centering
\includegraphics[width=0.95\linewidth]{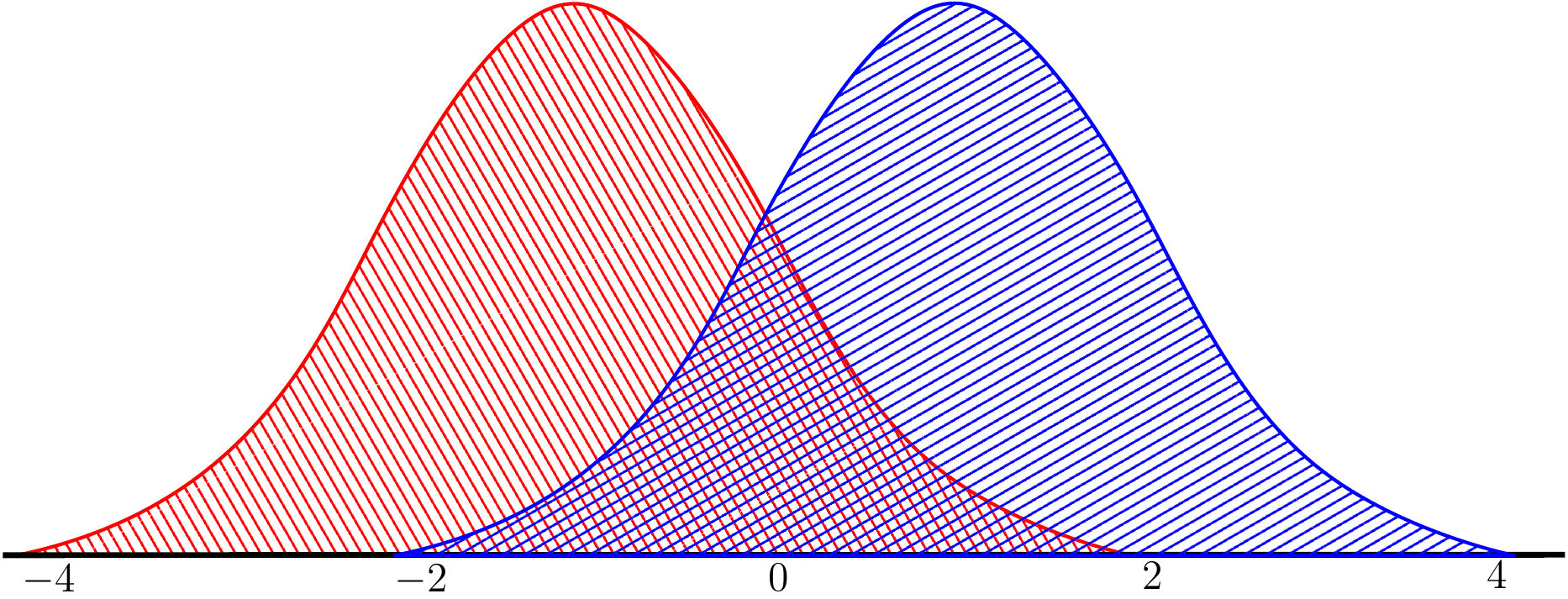}
\vspace{5pt}
\end{subfigure}
\hspace{5pt}
\begin{subfigure}[b]{0.31\linewidth}
\centering
\includegraphics[width=0.95\linewidth]{Figs/normal_dist.pdf}
\vspace{5pt}
\end{subfigure}
\begin{subfigure}[b]{0.31\linewidth}
\centering
\includegraphics[width=0.99\linewidth]{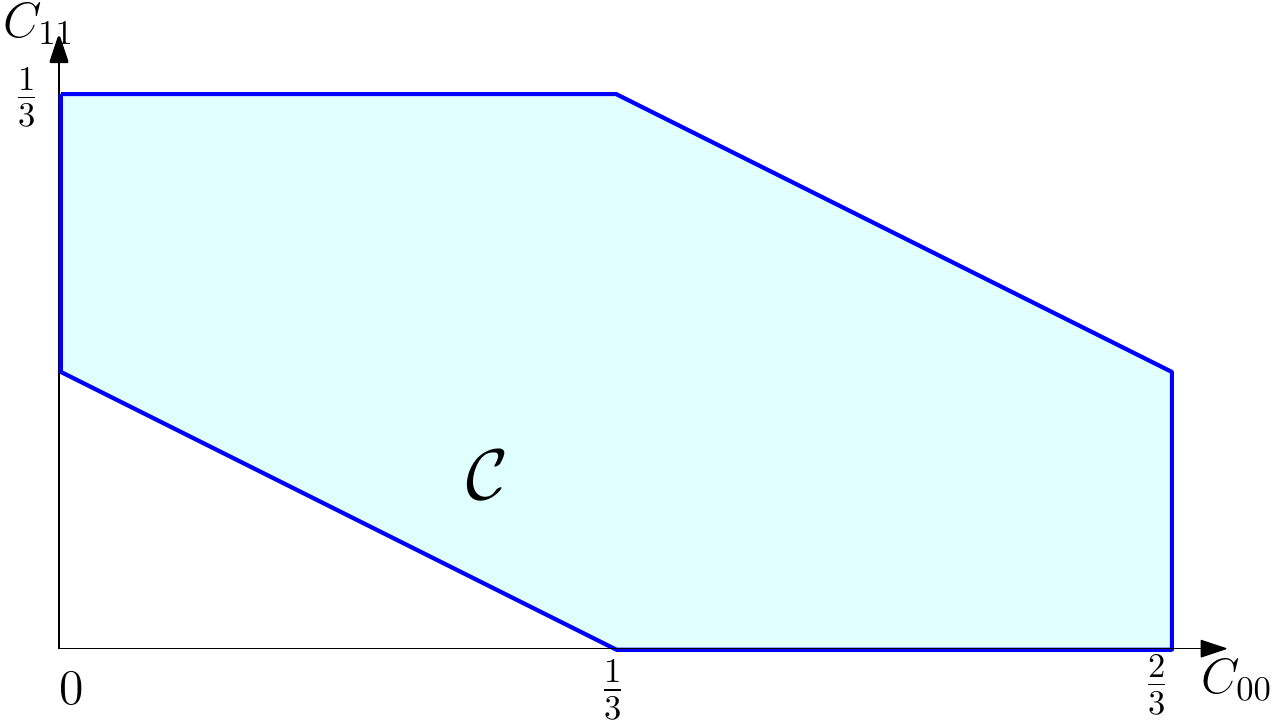}
\caption{\texttt{Unif}: $X|Y=0$ and $X|Y=1$ are uniform distributions over $[-1,1]$ and $[0,2]$ respectively; $P(Y=1)=\frac{1}{3}$.}
\label{fig:unif}
\end{subfigure}
\hspace{5pt}
\begin{subfigure}[b]{0.31\linewidth}
\centering
\includegraphics[width=0.8\linewidth]{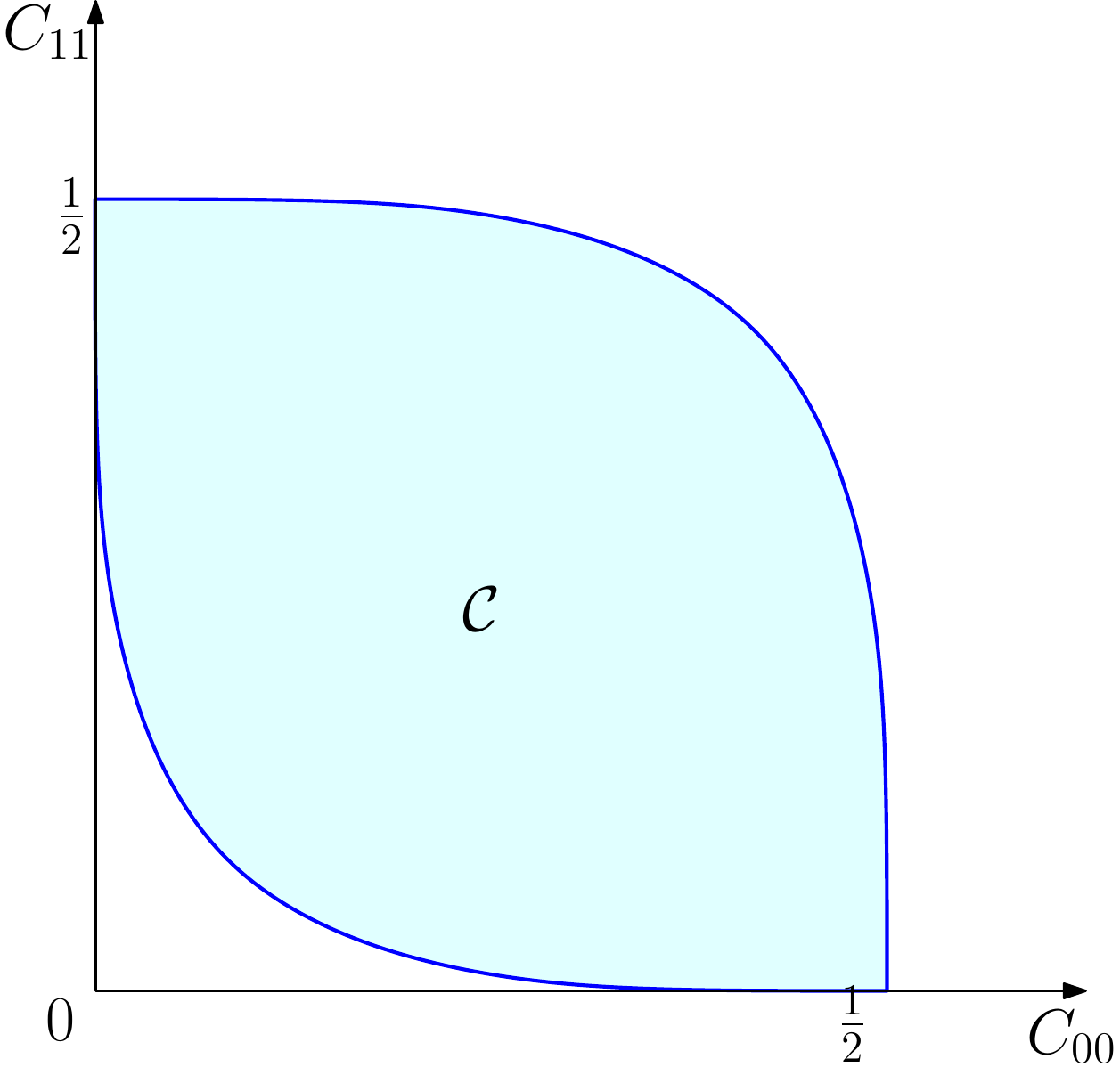}
\caption{\texttt{NormBal}: $X|Y=0$ and $X|Y=1$ are Gaussians with means $-\frac{1}{2}$ and $\frac{1}{2}$ respectively; $P(Y=1)=\frac{1}{2}$.} 
\label{fig:noram}
\end{subfigure}
\hspace{5pt}
\begin{subfigure}[b]{0.31\linewidth}
\centering
\includegraphics[width=0.8\linewidth]{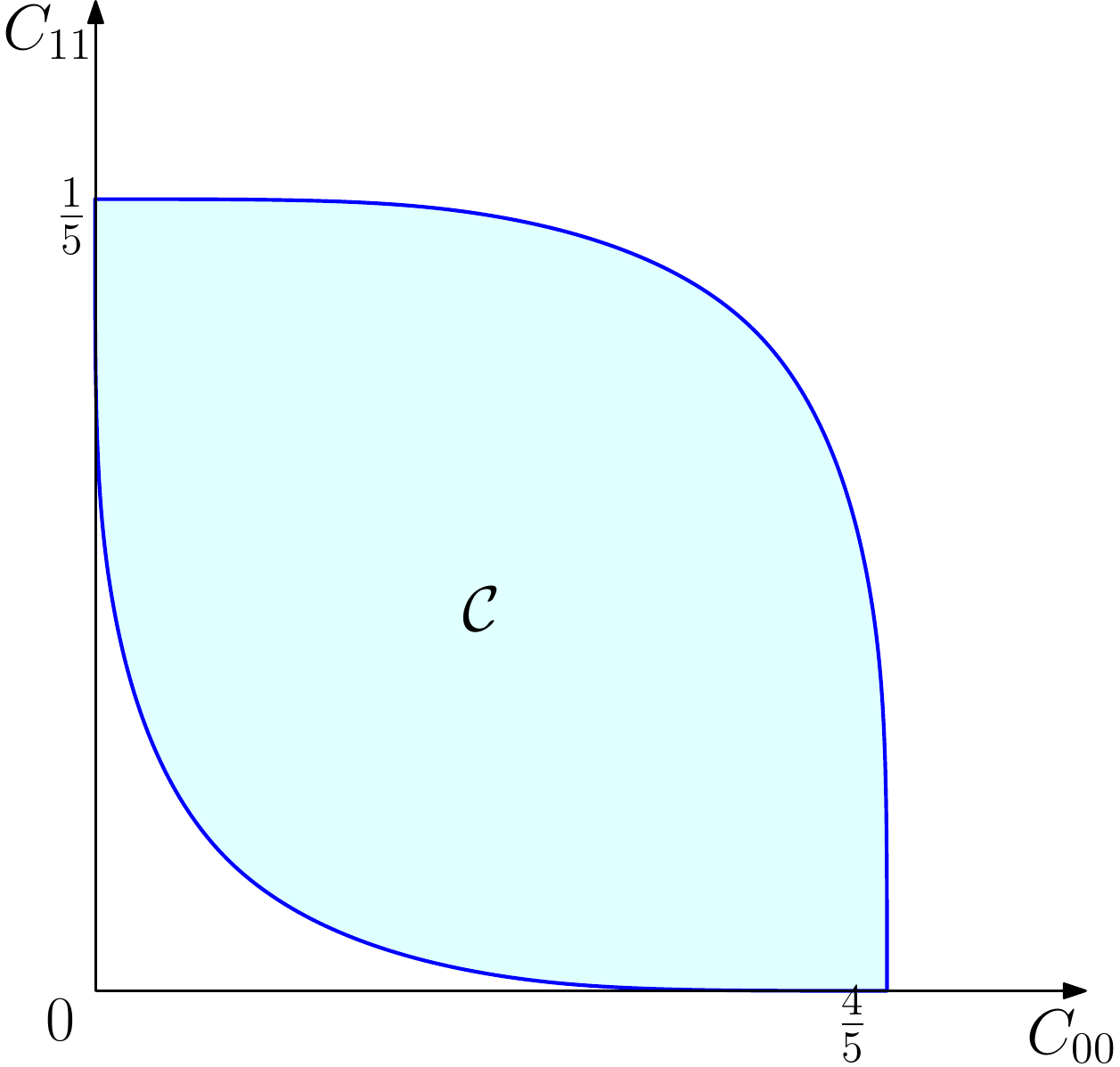}
\caption{\texttt{NormImbal}: $X|Y=0$ and $X|Y=1$ are Gaussians with means $-\frac{1}{2}$ and $\frac{1}{2}$ respectively; $P(Y=1)=\frac{1}{5}$.} 
\label{fig:normimbal}
\end{subfigure}

\caption{The set of achievable confusion  matrices $\cC$ for three example binary-labeled distributions: (a) \texttt{Unif}, (b) \texttt{NormBal}, and (c) \texttt{NormImbal}. The top row figures show the class-conditional distributions, and the bottom row figures represent the corresponding $\cC$.
While the confusion matrix  has four entries, there are only two degrees of freedom (the rows of the confusion matrix sum to the prior probabilities). We therefore only illustrate the projection of $\cC$ on to the diagonal entries $C_{00}$ and $C_{11}$. 
Note that the scales in the bottom row figures are different.}
\label{fig:example-feasible-conf}
\end{figure}

In order to understand optimal classifiers for the more complex learning problems in \ref{eq:opt-unconstrained} and  \ref{eq:opt-constrained} described in the previous section, we will find it useful to view these learning problems as optimization problems over all \emph{achievable confusion matrices}:
\begin{defn}[Achievable confusion matrices]
\label{defn:achievable-conf}
Define the set of \emph{achievable confusion matrices w.r.t.\ $D$} as the set of all confusion matrices achieved by some randomized classifier:
\[
\cC = \big\{ \textrm{\textup{vec}}(\C[h])\,|\, ~h:\X\>\Delta_n \big\} \subseteq \Delta_d
\]
where $\textrm{\textup{vec}}(\C[h]) \,=\, \left[C_{11}[h], \ldots, C_{1n}[h], \ldots, C_{n1}[h], \ldots, C_{nn}[h]\right]$ is  of dimension $d = n^2$.
\end{defn}

See Figure \ref{fig:example-feasible-conf} for an illustration of  the set of achievable confusion matrices for three simple synthetic distributions, which we will refer to as \texttt{Unif}, \texttt{NormBal} and \texttt{NormImBal}.  
For ease of exposition, in the above definition, we represent the  achievable confusion matrices  by a set of flattened vectors of dimension $d = n^2$. 
We will  also find it convenient from now on to overload notation and denote the performance measures by a function $\psi: [0,1]^d \> \R_+$ mapping a $d$-dimensional vector representation of the confusion matrix to a non-negative real number, and  the constraints by functions $\phi_1,\ldots,\phi_K: [0,1]^d \> \R_+$ defined on $d$-dimensional vectors. We will similarly represent an $n \times n$ loss matrix by a flattened $d$-dimensional vector $\L \in \R^d$.



\begin{prop}
\label{prop:CC-D-convex}
$\cC$ is a convex set.
\end{prop}
\begin{pf}
For any $\C_1,\C_2\in\cC$ and $\gamma\in[0,1]$, we will show $\gamma \C_1 + (1-\gamma)\C_2 \in \cC$. Clearly, there exists randomized classifiers $h_1,h_2:\X\>\Delta_n$ such that $\C_1 = \C[h_1]$ and $\C_2 = \C[h_2]$. Since
$h(x)= \gamma h_1(x) + (1-\gamma)h_2(x)$ is a valid randomized classifier,
$\C[h] = \gamma \C_1 + (1-\gamma)\C_2 \in \cC.$
\end{pf}
\vspace{2pt}



The set $\cC$ will play an important role in both our analysis of optimal classifiers and the subsequent development of consistent algorithms. Clearly, we can write \ref{eq:opt-unconstrained} as an unconstrained $d$-dimensional optimization problem over the convex set $\cC$:
 \begin{align}
 \min_{h:\X\>\Delta_n} \Psi[h] &= \min_{\C \in \cC} \psi(\C),
 \label{eq:unconstrained-reformulation}
 \tag{OP1*}
 \end{align}
 and  write \ref{eq:opt-constrained} as a constrained optimization problem over $\cC$:
 \begin{align}
 \min_{h:\X\>\Delta_n, \Phi_k [h]\leq 0, \forall k} \Psi[h] &= \min_{\C \in \cC, \bphi(\C)\leq \0} \psi(\C),
  \tag{OP2*}
 \label{eq:constrained-reformulation}
 \end{align}
where we denote $\boldsymbol{\phi}(\C) \,=\, [\phi_1(\C), \ldots, \phi_K(\C)]$.
 
\subsection{Bayes Optimal Classifier for the Unconstrained Problem}
While it is not clear if a classifier achieving the Bayes optimal performance
exists in general,
we show below that under mild assumptions, the optimal classifier for the unconstrained problem in \ref{eq:opt-unconstrained} can always be expressed as 
the optimal classifier for a certain  linear performance metric.
We show this for
 ``ratio-of-linear'' performance measures $\psi$, and for ``monotonic'' performance measures $\psi$ under a mild continuity assumption on $D$.

\begin{prop}[Bayes optimal classifier for ratio-of-linear $\psi$] 
 \label{prop:opt-classifier-ratio-linear}
Let the performance measure $\psi:[0,1]^{d}\>\R_+$ in \emph{\ref{eq:opt-unconstrained}} be of the form $\psi^\rl(\C) = \frac{\langle \A,\C \rangle}{\langle \B, \C \rangle}$ for some $\A,\B\in\R^{d}$ with $\langle \B,\C \rangle > 0 ~\forall \C\in\cC$. Then there exists loss matrix $\L^*$ (which depends on $\psi$ and $D$) such that any classifier that is optimal for the linear metric $\langle \L^*, \C \rangle$ is also optimal for \emph{\ref{eq:opt-unconstrained}}. 
\end{prop}

\begin{prop}[Bayes optimal classifier for monotonic $\psi$] 
\label{prop:opt-classifier-monotonic}
Let 
$\psi:[0,1]^{d}\>\R_+$ in \emph{\ref{eq:opt-unconstrained}} be differentiable and bounded, and be monotonically decreasing in $C_{ii}$ for each $i$ and non-decreasing in $C_{ij}$ for all $i,j$. Assume ${\boldeta}(X)$ is a continuous random vector. Then there exists a loss matrix $\L^*$ (which depends on $\psi$ and $D$) such that any classifier that is optimal for the linear metric $\langle \L^*, \C \rangle$ over $\cC$ is also optimal for \emph{\ref{eq:opt-unconstrained}}. 
\end{prop}

\begin{figure}[t]
\centering
\begin{subfigure}[b]{0.45\linewidth}
\includegraphics[width=0.9\linewidth]{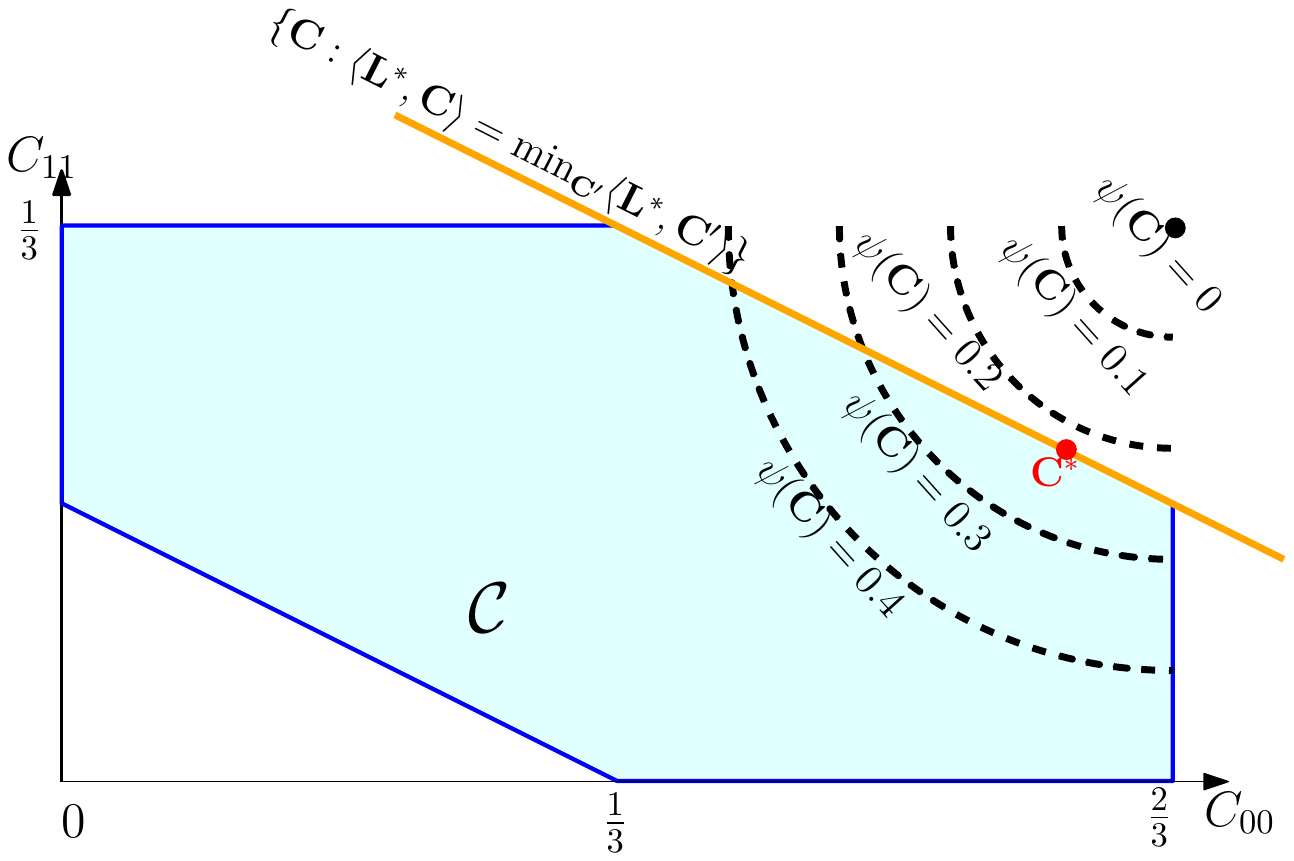}
\caption{}
\label{fig:contours-corner-point-unif}
\end{subfigure}
\hspace{20pt}
\begin{subfigure}[b]{0.4\linewidth}
\centering
\includegraphics[width=0.9\linewidth]{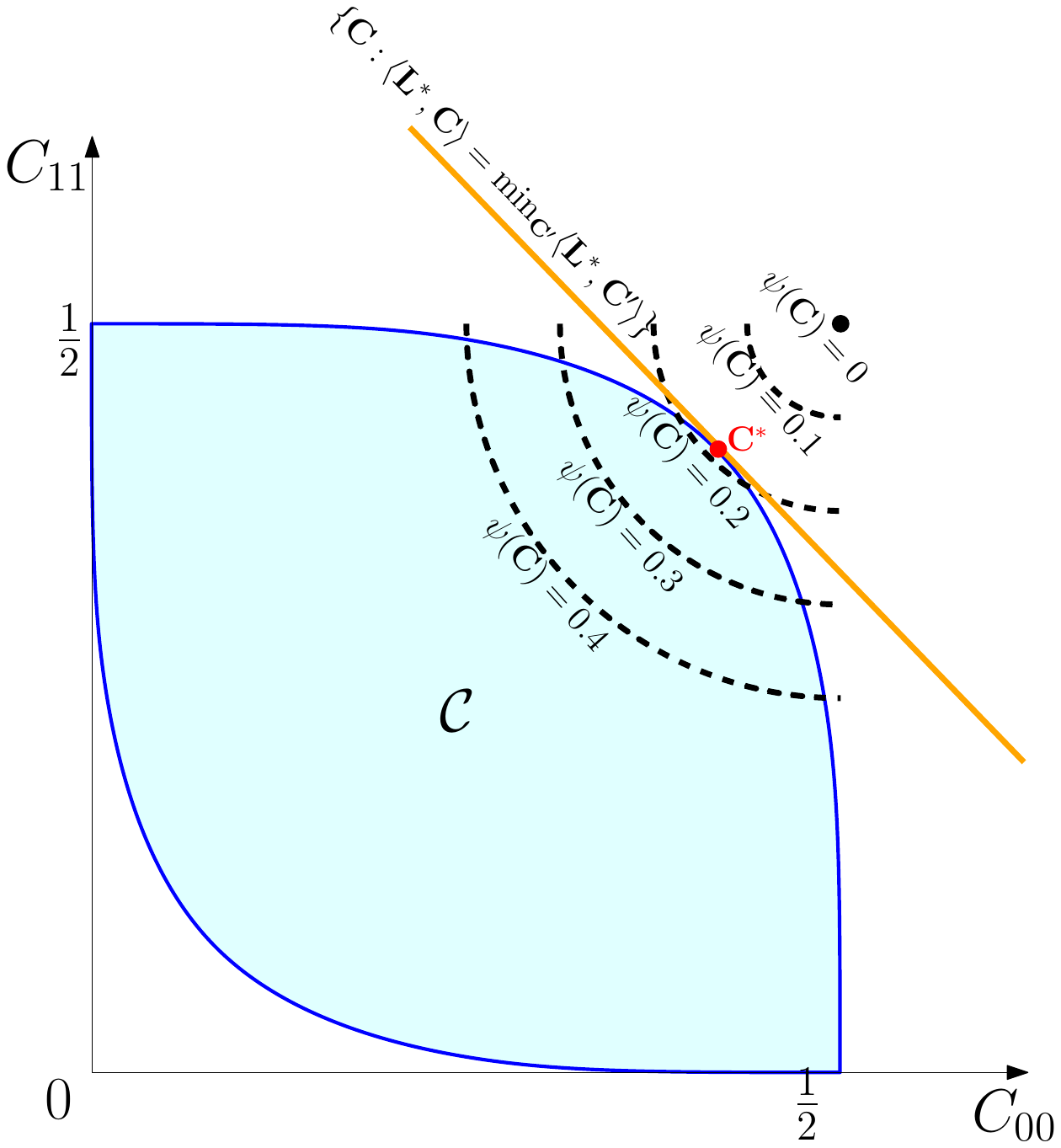}
\caption{} 
\label{fig:contours-corner-point-norm}
\end{subfigure}
\caption{Illustration of the Bayes-optimal classifier for 
the unconstrained problem in \ref{eq:opt-unconstrained} with a monotonic $\psi$. The figures show the set of confusion matrices $\cC$ for distributions \texttt{Unif} (left) and \texttt{NormBal} (right) in Figure \ref{fig:example-feasible-conf} (represented by the diagonal entries), the contours of the monotonic performance metric  $\psi$, and the corresponding solution $\C^*$ to $\min_{\C \in \cC} \psi(\C)$ (red dot). The black dot denotes the minimizer over 
all confusion matrices (even those that are not achievable). 
}
\label{fig:contours-corner-point}
\vspace{-5pt}
\end{figure}

See Appendices \ref{app:prop-bayes-ratio-of-linear} and \ref{app:proof-opt-classifier-monotonic} for the proofs.
 In Figure \ref{fig:contours-corner-point}, we provide an illustration for Proposition \ref{prop:opt-classifier-monotonic} using the 2-class example distributions \texttt{Unif} and \texttt{NormBal} from Figure \ref{fig:example-feasible-conf}.
We consider a monotonic performance metric $\psi$ whose contours are shown overlayed in the figure with the set of feasible confusion matrices $\cC$. It can be clearly seen that the minimal value of $\psi$ over $\cC$ is achieved by a point $\C^*$ on the boundary. Because $\cC$ is a convex set, it follows that all points on the boundary of $\cC$ are minimizers of some linear function $\langle \L, \C \rangle$ over $\C \in \cC$. Therefore, $\C^*$ is also a minimizer of $\langle \L^*, \C \rangle$ for some loss matrix $\L^*$. 

However, for $\C^*$ to be a unique minimizer of $\langle \L^*, \C \rangle$, 
we need the additional continuity assumption on $\boldeta(X)$ in Proposition \ref{prop:opt-classifier-monotonic} to hold. 
This does not hold for the \texttt{Unif} distribution in Figure \ref{fig:unif}, where the corresponding conditional-class probability vectors $\boldeta(X)$ take only 3 possible values in $\Delta_2$. In contrast, $\boldeta(X)$ is continuous for the \texttt{NormBal} distribution in Figure \ref{fig:contours-corner-point-norm}, and as result, 
the minimizer $\C^*$ of $\psi(\C)$, is also a unique minimizer for some linear function $\langle \L^*, \C \rangle$.


In Figure \ref{fig:3-class-best-classifiers}, we compare the forms of the Bayes-optimal classifier for the standard $\zo$ loss
and for the H-mean loss
in Table \ref{tab:perf-measures}. 
The latter seeks to explicitly balance the classifier's performance across all classes and is a monotonic function of (the diagonal elements of) $\C$. 
We provide plots of the optimal classifiers for a toy 3-class distribution, which contains equal class priors and has a
conditional-class probability distribution $\boldeta(X)$ which is continuous. 
We know that the optimal classifier for the $\zo$ loss simply outputs the label
with the maximum class probability
$h^*(x) = \argmax^*_{i} \eta_i(x)$. As seen in Figure \ref{fig:3-class-best-classifiers}(a), despite the class priors being equal,
this classifier predicts class 1 on only a small fraction of instances. 
On the other hand, for the H-mean loss,  Proposition \ref{prop:opt-classifier-monotonic} tells us that the  optimal classifier 
can be obtained by minimizing some linear function of $\C$, 
 the optimal classifier for which, in this particular case, is of the form
$h^*(x) = \argmax^*_{i} w^*_{i} \eta_i(x)$, 
for some distribution-dependent weights $w^*_i \in \R_+$. 
Note that
$w^*_i$ can be seen as the penalty associated with a wrong prediction 
on class $i$, which in this case is the highest for class 1. 
The resulting classifier, shown in Figure \ref{fig:3-class-best-classifiers}(b),
therefore yields equitable performance across the three classes.



\begin{figure}[t]
\centering
\begin{subfigure}[b]{0.45\linewidth}
\includegraphics[width=0.85\linewidth]{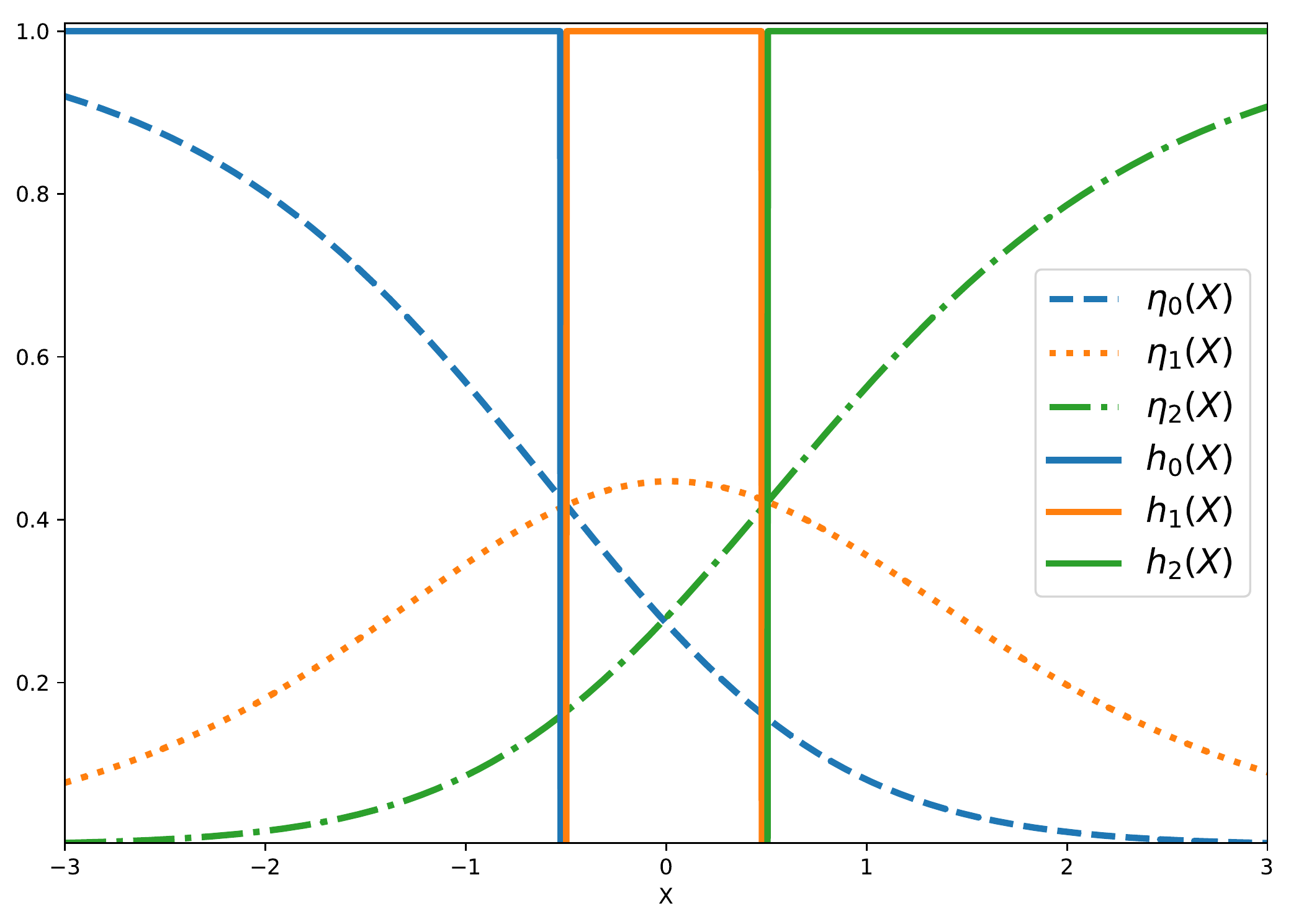}
\caption{Bayes-optimal classifier for 0-1 loss}
\label{fig:4a}
\end{subfigure}
\hspace{20pt}
\begin{subfigure}[b]{0.45\linewidth}
\centering
\includegraphics[width=0.85\linewidth]{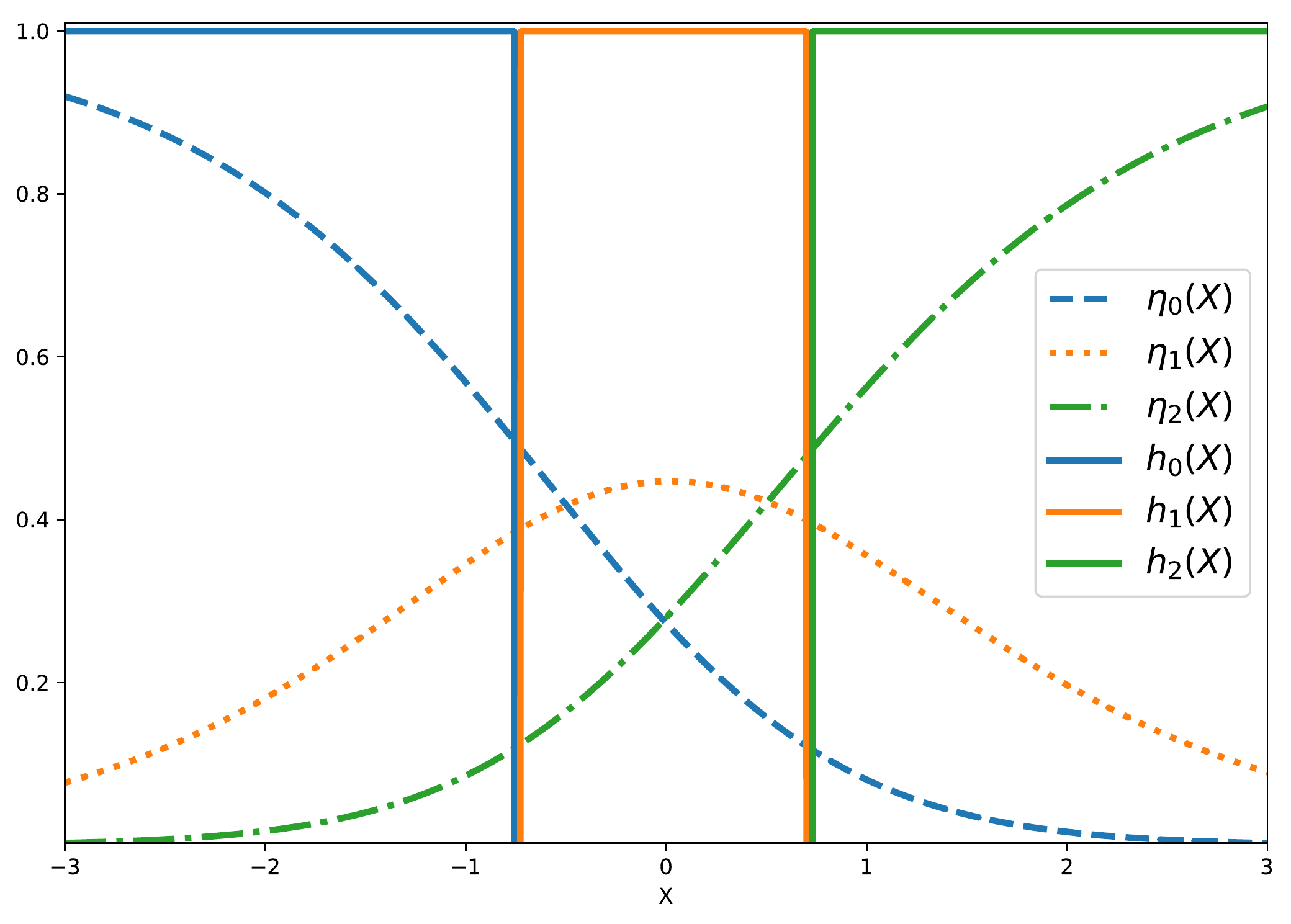}
\caption{Bayes-optimal classifier for H-mean loss} 
\label{fig:4b}
\end{subfigure}
\caption{
Comparison of Bayes-optimal classifiers for the $\zo$ loss (left)
and the H-mean loss (right).
We use a toy 3-class (denoted as class 0, 1 and 2) distribution over an one-dimensional instance space $\X = \R$, with equal priors, and with the class-conditional distribution for the three classes being a Gaussian distribution with means $-1, 0$ and $1$ respectively and variance 1.
We plot the conditional-class probability function $\eta_i(X)$, and the outputs of the optimal classifier
$h^*_i(X)$ for each class $i \in [3]$. For the $\zo$ loss, the optimal classifier predicts class 1 only on a small fraction of examples, whereas for the optimal classifier H-mean loss has greater coverage for class 1.
}
\label{fig:3-class-best-classifiers}
\vspace{-3pt}
\end{figure}

\subsection{Bayes Optimal Classifier for the Constrained Problem}
In both the characterizations in the previous section, the Bayes-optimal classifier for the unconstrained problem in \ref{eq:opt-unconstrained} is \textit{deterministic}.
An analogous statement  does not hold in general for the  constrained problem in \ref{eq:opt-constrained}. However, 
we can prove a weaker characterization for \ref{eq:opt-constrained} showing that the Bayes optimal classifier is a \textit{randomized} classifier that is supported by at most $d+1$ deterministic classifiers.
\begin{prop}[Bayes optimal classifier for continuous $\psi, \phi_1,\ldots, \phi_K$]
\label{prop:opt-classifier-constrained}
Let the performance measure $\psi: [0,1]^{d}\>\R_+$ and the constraint functions $\phi_1, \ldots, \phi_K: [-1,1]^{d}\>\R_+$ in \emph{\ref{eq:opt-constrained}} be continuous and bounded. Then there exists $d+1$ loss matrices $\L^*_1, \L^*_2, \ldots, \L^*_{d+1}$ (which can depend on $\psi, \phi_k$'s and $D$) 
such that an optimal classifier for \emph{\ref{eq:opt-constrained}} can be expressed as a randomized combination of the deterministic classifiers $h_1, h_2, \ldots, h_{d+1}$, where $h_i$ is optimal for the linear metric 
given by $\L^*_i$.
\end{prop}



See Appendix \ref{app:prop-bayes-constrained} for the proof.
 When the objective and constraints $\psi, \phi_1,\ldots,\phi_K$ together depend on fewer than $d=n^2$ entries of the confusion matrix, we can extend the above proposition to show that the number of deterministic classifiers needed to construct an  optimal classifier for {\ref{eq:opt-constrained}} is at most one plus the number of confusion matrix entries the metrics depend on. For example, if we wish to optimize the G-mean metric (Example \ref{ex:g-mean}) subject to a constraint on the class-1 precision (Example \ref{ex:prec}), the objective and constraints together depend only on $2n-1$ ``entries'' of the confusion matrix, and so an optimal classifier for this problem can be expressed as randomized combination of at most $2n$ deterministic classifiers. In Section \ref{sec:fairness}, we provide a more detailed discussion about succinct vector representations for confusion matrices that require fewer than $n^2$ entries.

Under continuity assumptions on $\boldeta(X)$ (which essentially translate to the space of achievable confusion matrices $\cC$ being \emph{strictly} convex),
one can further show that the Bayes-optimal classifier can be expressed as a randomized combination of \emph{two} deterministic classifiers $h_1$ and $h_2$, where $h_i$ is optimal for some linear metric $\L^*_i$
\citep{yang2020fairness}. The same characterization straight-forwardly holds for unconstrained minimization of a general performance metric $\psi$
\citep{wang2019consistent}.




\subsection{Na\"{i}ve Plug-in Approach}
\label{sec:naive}
The characterization results for the unconstrained problem in \ref{eq:opt-unconstrained} 
suggest a simple algorithmic approach to finding the optimal classifier: search over a large range of  loss matrices $\L$, estimate the optimal classifier for each such $\L$, and select among these a classifier that yields maximal $\psi$-performance (e.g.\ on a held-out validation data set). This is the analogue of ``plug-in'' type methods for binary 
performance metrics
(such as those considered by \citet{Koyejo+14} and \citet{Narasimhan+14}), 
where one searches over possible thresholds on the (estimated) class probability function.
However, while the binary case involves a 
search over values for a single threshold parameter, 
in the multiclass case,
one may need to perform a brute-force search over
as many as  $d$ 
parameters, requiring time exponential in $d$. For large $d$, such a na\"{i}ve plug-in approach is computationally intractable. In fact, this procedure becomes even more difficult to implement for the constrained problem in \ref{eq:opt-constrained}, where the optimal classifier is a randomized combination of multiple $\L$-optimal classifiers, requiring a brute-force search of over multiple loss matrices $\L$.

In what follows, we will  design efficient learning algorithms that instead search over the space of feasible confusion matrices $\cC$ using suitable optimization methods.

%% file: 4_unconstrained.tex
\begin{figure}[t]
\begin{center}
\begin{subfigure}{0.32\linewidth}
\centering
\includegraphics[width=0.95\linewidth]{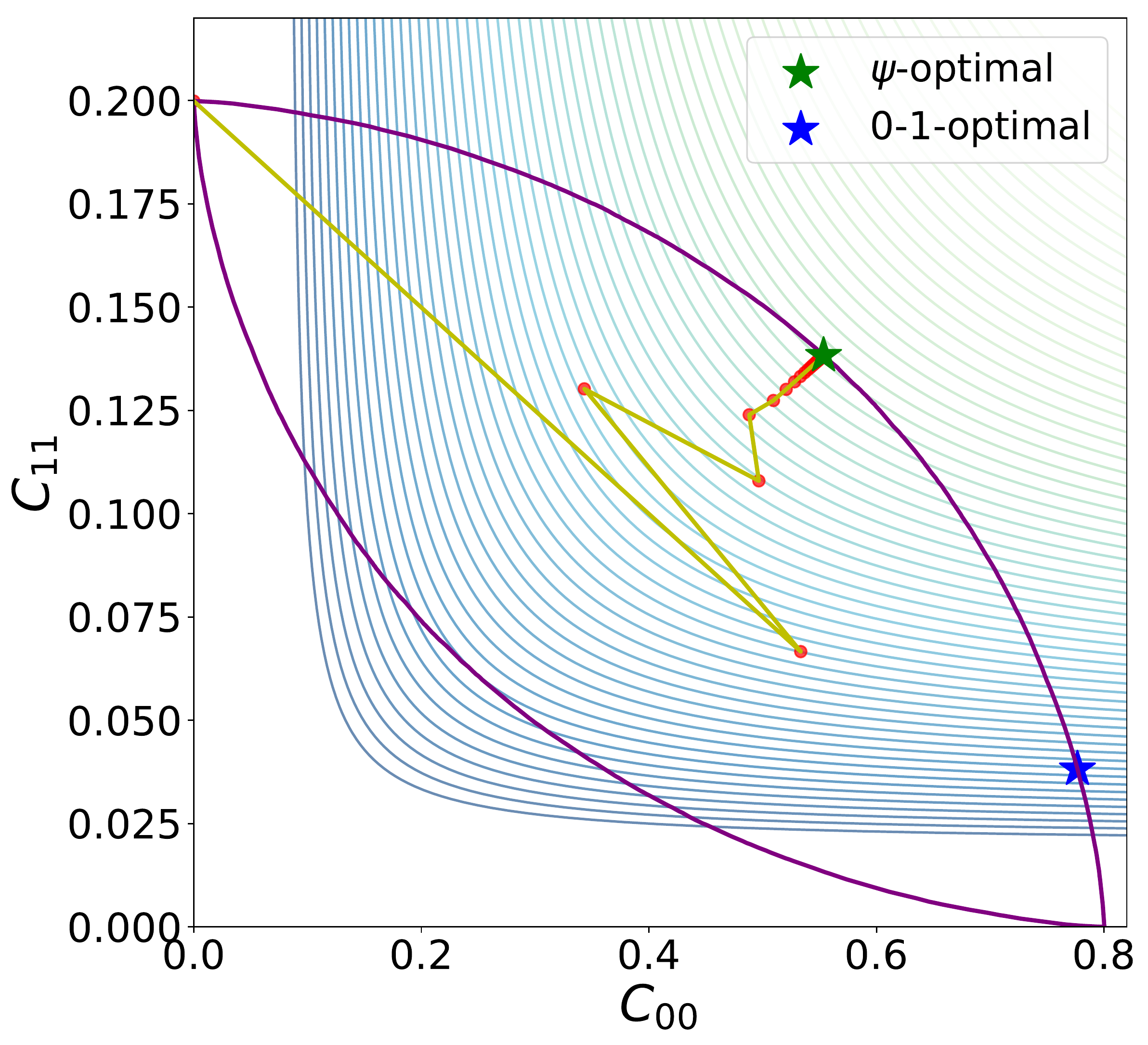}
\caption{Frank-Wolfe}
\label{fig:FW-trajectory}
\end{subfigure}
\begin{subfigure}{0.32\linewidth}
\centering
\includegraphics[width=0.95\linewidth]{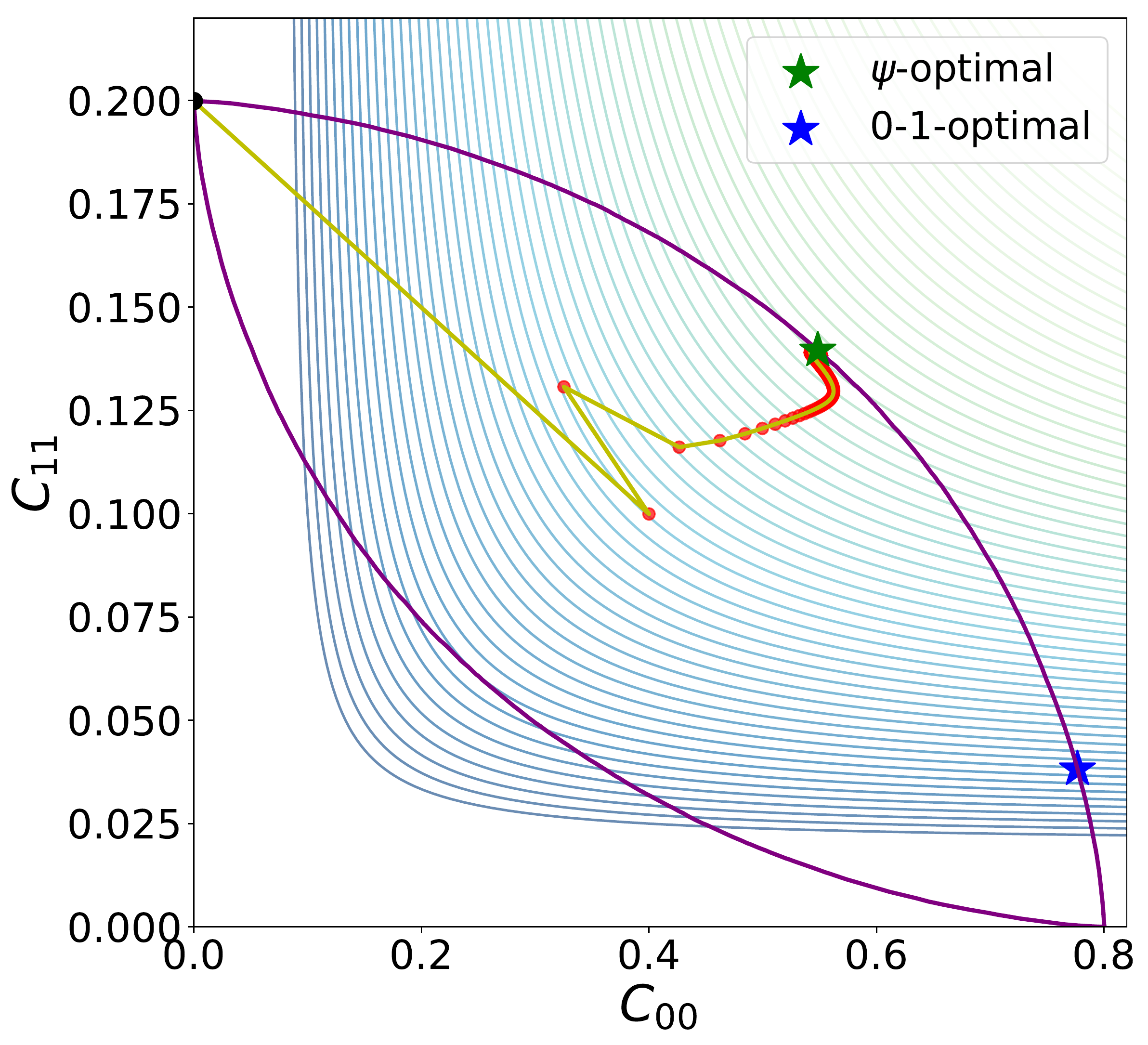}
\caption{GDA}
\label{fig:GDA-trajectory}
\end{subfigure}
\begin{subfigure}{0.32\linewidth}
\centering
\includegraphics[width=0.95\linewidth]{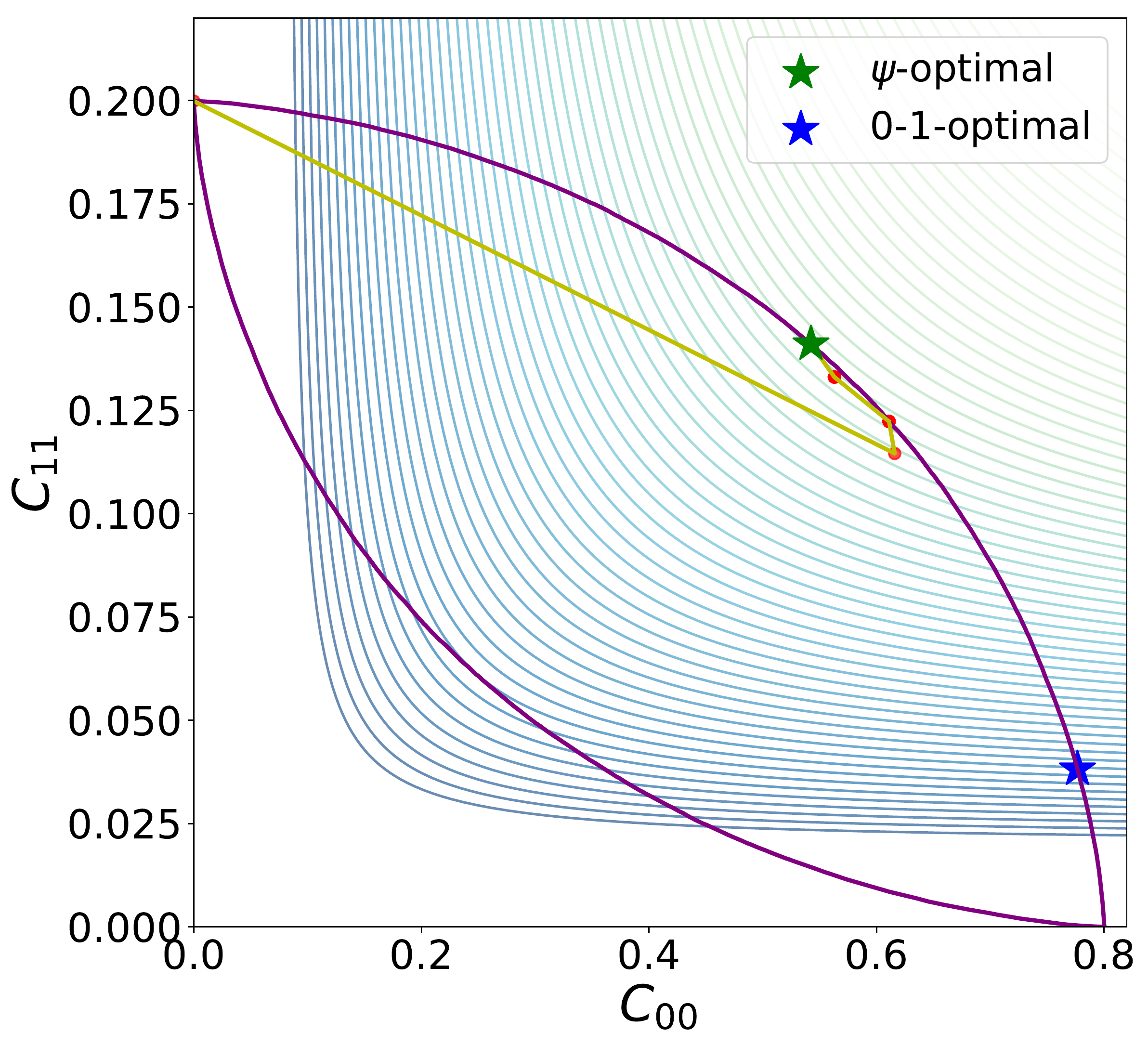}
\caption{Ellipsoid}
\label{fig:ellipsoid-trajectory}
\end{subfigure}
\caption{
Illustration of the Frank-Wolfe (Algorithm \ref{alg:FW}), Gradient Descent-Ascent (Algorithm \ref{alg:GDA}) and Ellipsoid (Algorithm \ref{alg:ellipsoid}) algorithms in minimizing the H-mean loss $\psi^{\HM}$ on the \texttt{NormImbal} distribution in Figure \ref{fig:normimbal}. The figures contain the space of achievable confusion matrices $\cC$ (with purple colored boundary), along with the contours of $\psi^{\HM}$. 
The trajectory of the confusion matrix $\C[\bar{h}^t]$ of the averaged classifier up until iteration $t$ is shown, where $\bar{h}^t = h^t$ for Frank-Wolfe, $\bar{h}^t = \frac{1}{t}\sum_{\tau=1}^\tau h^t$ for GDA, and  $\bar{h}^t = \frac{1}{t}\sum_{\tau=1}^t \alpha^*_\tau h^\tau$ for
 ellipsoid, with the optimal coefficients $\balpha^* \in \argmin_{\balpha \in \Delta_t} \psi\left(\sum_{\tau=1}^t \alpha_\tau \C^\tau\right)$ computed for iterates $1, \ldots t$. The averaged classifier is seen to converge to an optimal classifier for the H-mean loss and away from that for the 0-1 loss.
}
\label{fig:trajectories-unconstrained}
\end{center}
\vspace{-10pt}
\end{figure}

\begin{table}[t]
        \centering
        \caption{Algorithms for the unconstrained problem  in \ref{eq:opt-unconstrained}, with the number calls to the LMO and the optimality gap $\psi(\C[\bar{h}]) - \min_{\C\in \cC}\psi(\C)$ for the returned classifier $\bar{h}$.  Here $\rho^\eff = \rho+\sqrt{d}\rho'$.  
        } 
        \label{tab:algorithms-unc}
        \begin{small}
        \begin{tabular}{cccc}
        \hline
        \textbf{Algorithm} & Assumption on $\psi$
        & \textbf{\# LMO Calls} & \textbf{Optimality Gap}
        \\
        \hline
             Frank-Wolfe & Convex, smooth, Lipschitz & $\displaystyle\cO\left(1/\epsilon\right)$  &
             $\cO\big(\epsilon + \rho^\eff \big)$
             \\
            Gradient Descent-Ascent & 
            Convex, Lipschitz & 
        $\displaystyle\cO\left(1/\epsilon^2\right)$ &
        $\cO\big(\epsilon + \rho^\eff \big)$
        \\
        Ellipsoid & 
        Convex, Lipschitz &
        $\displaystyle\cO\left(d^2\log(d/\epsilon)\right)$ & 
        $\cO(\epsilon + \rho^\eff)$
        \\
        Bisection & 
        Ratio-of-linear &
        $\displaystyle\cO\left(\log(1/\epsilon)\right)$ 
        & $\cO\left(\epsilon + \rho^\eff\right)$\\
        \hline
        \end{tabular}
        \end{small}
    \end{table}

\section{Algorithms for Unconstrained Problems}
\label{sec:unconstrained}

We start with algorithms for solving the unconstrained learning problem in \ref{eq:opt-unconstrained}. As a running example to illustrate our algorithms, we will use the task of maximizing the H-mean loss  on the \texttt{NormImbal} distribution described in Figure \ref{fig:example-feasible-conf}(e).

%
%
%
%
 As noted in our discussion of \ref{eq:unconstrained-reformulation}, one can view \ref{eq:opt-unconstrained} as an optimization problem over $\cC$: $\min_{\C\in\cC} \psi(\C)$.
While $\cC$ is a convex set, 
it is not available directly to the learner
as the set of all confusion matrices is hard to characterize.
On the other hand, one operation that is easy to perform is to find an optimal classifier for a \emph{linear} loss $\langle \L, \C\rangle$ over $\cC$. Indeed this amounts to solving a  cost-sensitive learning problem \citep{Elkan01}, a task for which there are numerous classical methods available. So we assume access to an oracle for solving this linear minimization problem over $\cC$, which takes as input a loss matrix $\L$ and a sample $S$, and outputs a classifier $\widehat{g}$ and an estimate of the confusion matrix at $\widehat{g}$ with the following properties:
\begin{defn}[Linear minimization oracle]
\label{defn:lmo}
Let $\rho, \rho', \delta \in (0, 1)$. A linear minimization oracle, denoted by $\Omega$, takes a loss matrix $\L \in \R^d$ and a sample $S$ as input, and outputs a classifier $\widehat{g}$ and a confusion matrix $\hat{\bGamma} \in \R^d$. We say $\Omega$ is a $(\rho, \rho', \delta)$-approximate LMO for sample size $N$, if,  with probability $\geq 1 - \delta$ over draw of $S \sim D^N$, for any $\L \in \R_+^d$ with $\|\L\|_\infty \leq 1$, it outputs $(\widehat{g},\hat{\bGamma}) = \Omega(\L; S)$ such that:
\begin{equation*}
\langle \L, \C[\widehat{g}]\rangle \,\leq\, \min_{h:\X\>\Delta_n}\langle \L, \C[h] \rangle + \rho;\quad
\|\C[\widehat{g}] \,-\, \hat{\bGamma}\|_\infty \,\leq\, \rho'.
\label{eq:lmo2}
\end{equation*}
The approximation constants $\rho$ and $\rho'$ may in turn depend on  the sample size $N$, the dimension $d$ and the confidence level $\delta$. 
\end{defn}

In Section \ref{sec:lmo}, we discuss a practical plug-in based algorithm 
for implementing an LMO with these approximation properties.
Equipped with access to such an LMO, we develop algorithms based on iterative optimization methods for minimizing $\psi$ over $\cC$. Our algorithms do not require direct access to the set $\cC$, but only make use of calls to the LMO over $\cC$. 

We present four algorithms under different assumptions on the metric $\psi$ and show convergence guarantees in each case (see Table \ref{tab:algorithms-unc} for a summary of our results).  The proofs   build on existing techniques for showing convergence of the respective optimization solvers, and need to additionally  take into account the errors in the LMO calls. 


\begin{figure}
\begin{algorithm}[H]
\caption{Frank-Wolfe (FW) Algorithm for \ref{eq:opt-unconstrained} with Smooth Convex $\psi$}\label{alg:FW}
\begin{algorithmic}[1]
\STATE \textbf{Input:} $\psi:[0, 1]^{d} \to[0,1]$, an LMO $\Omega$, $S = \{(x_1,y_1), \ldots, (x_N,y_N)\}$, $T$
\STATE \textbf{Initialize:} 
$(h^0, \C^0) = \Omega(\L^0; S)$ for an arbitrary loss matrix $\L^0$
\STATE \textbf{For} $t =  1$ \textbf{to} $T$ \textbf{do}
\STATE ~~~~~$\L^t\,=\, \frac{\nabla\psi(\C^{t-1})}{ \|\nabla\psi(\C^{t-1})\|_\infty}$
\STATE ~~~~~$(\tilde{h}^t, \tilde{\C}^t) \,=\, \Omega(\L^t; S)$
\STATE ~~~~~${h}^{t} = \big(1-\frac{2}{t+1}\big) {h}^{t-1} + \frac{2}{t+1} \tilde{h}^t$
\STATE ~~~~~${\C}^t = \big(1-\frac{2}{t+1}\big) {\C}^{t-1} + \frac{2}{t+1}\tilde{\C}^t$
\STATE \textbf{End For}
\STATE \textbf{Output:} $\bar{h} = h^T$
\end{algorithmic}
\end{algorithm}
\vspace{-10pt}
\end{figure}



\subsection{Frank-Wolfe Algorithm for Smooth Convex Metrics}
The first algorithm that we describe uses the classical Frank-Wolfe method  \citep{FrankWolfe56} to minimize $\psi(\C)$ over $\C$ for 
performance measures $\psi$ that are convex and smooth over $\cC$. Examples of performance measures with these properties include the H-mean and Q-mean in \Tab{tab:perf-measures}.

The key idea behind this algorithm is to sequentially linearize the objective $\psi$ using its local gradients, and minimize the linear approximation over $\cC$ using the LMO. The
procedure, outlined in Algorithm \ref{alg:FW}, maintains iterates of confusion matrices $\C^t$, computes the gradient $\L^t = \grad\psi(\C^{t-1})$ for the current iterate, invokes the LMO to solve the resulting linear minimization problem $\min_{\C\in\cC_D} \langle \L^t,\C \rangle$,
and updates $\C^t$ based on the result of the linear minimization.
The minimizer $\C^*$ of $\psi(\C)$ can then be approximated by a combination of the iterates $\C^1, \ldots, \C^T$, with the final classifier that achieves this confusion matrix given by a randomized combination of classifiers learned across all the iterations. 

For metrics $\psi$ that are smooth, we  show that the algorithm takes $\cO(1/\epsilon)$ calls to the LMO to reach a classifier that is $\cO(\epsilon + c)$-optimal for a constant $c > 0$ that depends on the LMO error.
\begin{thm}[Convergence of FW algorithm]
\label{thm:FW-unc}
Fix $\epsilon \in (0,1)$. 
Let $\psi: [0,1]^d \> [0,1]$ be convex, $\beta$-smooth and $L$-Lipschitz w.r.t.\ the $\ell_2$-norm. Let $\Omega$ in Algorithm \ref{alg:FW} be a $(\rho , \rho', \delta)$-approximate LMO for sample size $m$.
Let $\bar{h}$ be a classifier returned by Algorithm \ref{alg:FW} when run for $T$ iterations. 
 Then with probability $\geq 1 - \delta$ over draw of $S \sim D^N$, after $T = \cO(1/\epsilon)$ iterations:
\[
\psi(\C[\bar{h}]) \,\leq\, \min_{\C \in \cC}\,\psi(\C) +8\beta\epsilon +2L\rho +4\beta \sqrt{d}\rho' \leq \min_{\C \in \cC}\,\psi(\C) + \cO(\epsilon + \rho^\eff),
\]
where $\rho^\eff=\rho+ \sqrt{d}\rho'$.
\end{thm}
\begin{proof}
See Appendix \ref{app:proof-fw}.
\end{proof}

The proof derives a version of the convergence guarantee for the  Frank-Wolfe method \citep{Jaggi13} which is robust to errors in the gradients and confusion matrix estimates. 

In Figure \ref{fig:FW-trajectory}, we illustrate the trajectory taken by the Frank-Wolfe algorithm in minimizing the H-mean loss $\psi^\HM$ in Table \ref{tab:perf-measures}. 
Notice that the linear minimization outputs $\tilde{\C}^t$ lie on the boundary of $\cC$, while the averaged confusion matrix iterates $\C^t$ lie in the interior. Also note that because \texttt{NormImbal} distribution we use for this illustration has significant class imbalance,  the minimizer for the 0-1 loss incurs a large H-mean loss. In contrast, Algorithm \ref{alg:FW} converges to a confusion matrix with substantially better H-mean loss. 

\subsection{Gradient Descent-Ascent Algorithm for Non-smooth Convex Metrics}
\label{sec:gda}
The next algorithm we propose is designed for performance measures $\psi$ that are convex, but \textit{not necessarily smooth}, such as the min-max metric in Table \ref{tab:perf-measures}. We make use of the ``three player'' framework proposed by \citet{Narasimhan+19_generalized} and  provide a slight variant of the ``oracle-based algorithm'' in their paper. 

As a first step, we decouple the confusion matrix $\C$ from the function $\psi$ in \ref{eq:unconstrained-reformulation} by introducing auxiliary slack variables $\bxi \in \Delta_{d}$, and arrive at the following equivalent problem:
 \begin{align}
\min_{\C \in \cC}\psi(\C) &= \min_{\C \in \cC,\, \bxi \in \Delta_d,\, \bxi = \C} \psi(\bxi),
 \label{eq:slack-reformulation}
 \end{align}
where we constraint the slack variables $\xi$ to be equal to the confusion matrix $\C$. We define the Lagrangian for the above problem introducing multipliers $\blambda \in \R^d$ for the $d$ equality constraints:
\begin{align}
\cL(\C, \bxi, \blambda) 
&= \psi(\bxi) + \langle \blambda, \C-\bxi \rangle,
\label{eq:lagrangian-unc}
\end{align}
and re-formulate \eqref{eq:slack-reformulation} 
as an equivalent min-max problem where we minimize the Lagrangian over $\bxi$ and $\C$, and maximize it over the Lagrange multipliers $\blambda$:
\begin{align}
\min_{\C \in \cC}\psi(\C) &= \min_{\C \in \cC,\, \bxi \in[0,1]^d}\,\max_{\blambda \in \R^d}\,\cL(\C, \bxi, \blambda).
 \label{eq:minmax-reformulation}
\end{align}

The minimizer of $\psi(\C)$ over $\C$ can be then obtained by finding a saddle point of the above min-max problem. To this end, we first notice that the Lagrangian $\cL$ is linear in $\C$, convex in $\bxi$ and linear in $\blambda$. 
Following \cite{Narasimhan+19_generalized}, 
we maintain iterates $\C^t, \bxi^t$ and $\blambda^t$ and at each iteration, perform a full minimization of $\cL$ using a call to the LMO, perform gradient descent updates on $\bxi$, and perform gradient ascent updates on $\blambda$. 
We constrain $\bxi$ to be within the probability simplex $\Delta_d$, and for technical reasons, also constrain $\blambda$ to be within a bounded set $\Lambda$, both of which
are accomplished using projection operations.

The resulting gradient descent-ascent procedure, outlined in Algorithm \ref{alg:GDA} 
can be shown to converge to  an approximate saddle point of \eqref{eq:minmax-reformulation}. In fact, one can further show that with $\cO(\log(d)/\epsilon^2)$ calls to the LMO, the algorithm finds a classifier that is $\cO(\epsilon+c)$-optimal for $\psi$, for some constant $c > 0$ that depends on the LMO errors:
\footnote{
\citet{Narasimhan+19_generalized} point out that the min-max formulation in \eqref{eq:minmax-reformulation} can be used to re-derive  the Frank-Wolfe based procedure in Algorithm \ref{alg:FW}. Specifically, by defining  $\omega(\C, \lambda) = \min_{\bxi \in[0,1]^d}\cL(\C, \bxi, \blambda)$, and reformulate \eqref{eq:constrained-reformulation} as the equivalent min-max problem $\min_{\C \in \cC}\,\max_{\blambda \in \R^d} \omega(\C, \lambda)$,  the Frank-Wolfe based algorithm can be shown to
minimize $\omega$ using a LMO over $\C \in \cC$ and maximize it over $\blambda \in \R^d$ by applying a Follow-The-Leader (FTL) update \citep{AbernethJ17}.}
\begin{thm}[Convergence of GDA algorithm]
\label{thm:gda-unc}
Fix $\epsilon \in (0,1)$. 
Let $\psi: [0,1]^d\>[0,1]$ be convex and $L$-Lipschitz w.r.t.\ the $\ell_2$-norm. Let $\Omega$ in Algorithm \ref{alg:GDA} be a $(\rho , \rho', \delta)$-approximate LMO for sample size $N$. Let the space of Lagrange multipliers $\Lambda = \{\blambda\in \R^d\,|\,\|\blambda\|_2\leq 2L\}$. 
Let $\bar{h}$ be a classifier returned by Algorithm \ref{alg:GDA} when run for $T$ iterations,
with step-sizes $\eta = \frac{1}{4L\sqrt{2T}}$
and $\eta' = \frac{4L}{\sqrt{2T}}$.  
 Then with probability $\geq 1 - \delta$ over draw of $S \sim D^N$, after $T = \cO(1/\epsilon^2)$ iterations:
\[
\psi(\C[\bar{h}]) \,\leq\, \min_{\C \in \cC}\,\psi(\C) \,+\,\cO\left(\epsilon + \rho^\eff\right),
\]
where $\rho^\eff = \rho+\sqrt{d}\rho'$ and $\cO$ hides constants independent of $\epsilon, \rho, \rho'$ and $d$.
\end{thm}
\begin{proof}
See Appendix \ref{app:proof-gda-unc}.
\end{proof}

Figure \ref{fig:GDA-trajectory} shows the trajectory of the iterates of the GDA algorithm on the same running example used to illustrate the Frank-Wolfe based algorithm. Notice that the GDA algorithm converges to an optimal confusion matrix (classifier) for the problem. 

\begin{figure}
\begin{algorithm}[H]
\caption{Gradient Descent-Ascent (GDA) Algorithm for \ref{eq:opt-unconstrained} with Non-smooth Convex $\psi$}\label{alg:GDA}
\begin{algorithmic}[1]
\STATE \textbf{Input:} $\psi:[0, 1]^{d} \to [0,1]$, an LMO $\Omega$ , $S = \{(x_1,y_1), \ldots, (x_N,y_N)\}$, $T$, space of Lagrange multipliers $\Lambda \subset \R^d$
\STATE \textbf{Parameters:} Step-sizes $\eta, \eta' > 0$
\STATE \textbf{Initialize:} 
$\blambda^0 \in \Lambda$
\STATE \textbf{For} $t =  0$ \textbf{to} $T-1$ \textbf{do}
\STATE ~~~~~$\L^t\,=\, \frac{\blambda^{t}}{ \|\blambda^{t}\|_\infty}$
\STATE ~~~~~$(h^{t}, \C^t) \,=\, \Omega(\L^t; S)$
\STATE ~~~~~$\tilde{\bxi}\,=\, \bxi^{t} \,-\, \eta\nabla_{\bxi}\cL(\C^{t}, \bxi^{t}, \blambda^{t});$~~~~~~$\bxi^{t+1}\,\in\,\argmin_{\bxi \in \Delta_d}\,\|\bxi - \tilde{\bxi}\|_2$
\STATE ~~~~~$\tilde{\blambda}\,=\, \blambda^{t} \,+\, \eta'\nabla_{\blambda}\cL(\C^{t}, \bxi^{t}, \blambda^{t});$~~~~~~$\blambda^{t+1}\,\in\,\argmin_{\blambda \in \Lambda}\|\blambda - \tilde{\blambda}\|_2$
\STATE \textbf{End For}
\STATE \textbf{Output:} $\bar{h} = \frac{1}{T}\sum_{t=1}^T h^t$
\end{algorithmic}
\end{algorithm}
\vspace{-15pt}
\end{figure}

\subsection{Ellipsoid Algorithm for Non-smooth Convex Metrics}

\begin{figure}[t]
\begin{algorithm}[H]
\caption{Ellipsoid Algorithm for \ref{eq:opt-unconstrained} with Non-smooth Convex $\psi$}\label{alg:ellipsoid}
\begin{algorithmic}[1]
\STATE \textbf{Input:} $\psi:[0, 1]^{d} \to [0,1]$, an LMO $\Omega$,   $S = \{(x_1,y_1), \ldots, (x_m,y_m)\}, T$
\STATE \textbf{Parameters:} Initial ellipsoid radius $a$ 
\STATE \textbf{Initialize:} $\tilde{\blambda}^0 = \0_d$, $\tilde{\A}^0 = a^2 \I_d$
\STATE \textbf{For} $t=0$ \textbf{to} $T-1$ \textbf{do}
\STATE ~~~~ \textbf{If} $\|\tilde{\blambda}^t\|_2>a:$
\STATE ~~~~~~~~ $\A^{t+1}, \blambda^{t+1} = \text{JLE}({\A}^t, {\blambda}^t, -{\blambda}^t)$
\STATE ~~~~~~~~ $h^t, \C^t = h^{0}, \C^{0}$; \textbf{continue}
\STATE  ~~~~ \textbf{Else:}
\STATE ~~~~~~~~ $\A^{t}, \blambda^{t} = \tilde{\A}^{t},\tilde{\blambda}^{t}$
\STATE ~~~~~~~~ $(h^t, \C^t)= \Omega(\blambda^t, S)$
\STATE ~~~~~~~~ $\bxi^t = \argmin_{\bxi \in \Delta_d} \psi(\bxi) - \langle \blambda^t, \bxi\rangle$
\STATE ~~~~~~~~ $\A^{t+1}, \blambda^{t+1} = \text{JLE}(\A^t, \blambda^t, \C^t-\bxi^t)$
\STATE \textbf{End For}
\STATE $\balpha^* \in \argmin_{\balpha \in \Delta_T} \psi\left(\sum_{t=0}^{T-1} \alpha_t \C^t\right)$ 
\STATE \textbf{Ouput:} $\bar{h} = \sum_{t=0}^{T-1} \alpha^*_t h^t$
\end{algorithmic}
\end{algorithm}
\vspace{-10pt}
\edef\currentthealgorithm{\thealgorithm}%
\renewcommand{\thealgorithm}{\currentthealgorithm(a)}
\addtocounter{algorithm}{-1}
\begin{algorithm}[H]
\caption{John-Lowner Ellipsoid (JLE) Construction
}\label{alg:jle}
\begin{algorithmic}[1]
\STATE \textbf{Input:} Positive-definite matrix $\A \in \R^{d\times d}$, $\blambda$, $\mathbf w$
\STATE \textbf{Output:}  $\A', \blambda'$ that parameterizes the smallest ellipsoid such that:
\[
E(\blambda',\A') \supseteq E(\blambda,\A) \cap \{\x: (\x - \blambda) ^\top \mathbf w \geq  0 \}
\]
where $E(\blambda, \A)= \{\x: (\x- \blambda)^\top (\A)^{-1} (\x-\blambda) \leq 1 \}$
\STATE $t = \frac{1}{d+1}, a = \frac{1}{(1-t)^2}, b = \frac{1-2t}{(1-t)^2}$
\STATE $\widetilde \w = \frac{A^{1/2} \w}{\|A^{1/2} \w\|_2}$
\STATE $\B^{-1} = a \widetilde\w  \widetilde\w^\top + b(I - \widetilde\w \widetilde\w^\top) $
\STATE $ \blambda' = \blambda + t \A^{\frac{1}{2}} \widetilde\w $
\STATE $(\A')^{-1} = \A^{-1/2} \B^{-1} \A^{-1/2}$  
\STATE \textbf{Return} $\A', \blambda'$
\end{algorithmic}
\end{algorithm}
\vspace{-10pt}
\end{figure}

Building on the Lagrangian dual formulation described above, we next design an approach based on the classical ellipsoid algorithm \citep{BoydVan04}, which for convex (non-smooth) performance measures $\psi$, requires only $\cO(d^2\log(d/\epsilon))$ calls to the LMO to reach an $\cO(\epsilon + c)$-optimal classifier. Note that unlike the two previous algorithms, the number of LMO calls in this case has a logarithmic dependence on $1/\epsilon$, but at the cost of a stronger dependence on dimension $d$. So for problems where $d$ is small, we expect this approach to enjoy faster convergence.

We begin by defining the Lagrange dual function for given  multipliers $\blambda$:
\begin{align*}
f(\blambda)           
&= \min_{\C\in \cC,\, \bxi \in\Delta_d } 
\cL(\C, \bxi, \blambda).
\end{align*}
Because $f$ is concave in $\blambda$,
we can  employ the ellipsoid algorithm to efficiently maximize $f$ over $\blambda$ and thus solve \ref{eq:opt-unconstrained}.
Each step of the algorithm requires computing a super-gradient for $f$ at the current iterate $\blambda^t$,
which serves as a hyper-plane separating $\blambda^t$ from the maximizer of $f$.
For this, we find the the minimizers
$\C^t\in\argmin_{\C \in \cC} \langle\blambda^t, \C\rangle$ and $\bxi^t \in \argmin_{\bxi \in \Delta_d} \psi(\bxi) - \langle\blambda^t, \bxi\rangle$; an application of Danskin's theorem \citep{danskin2012theory} then gives us that $\C^t - \bxi^t = \nabla_{\blambda} \cL(\C^t, \bxi^t, \blambda)$ is a super-gradient for $f$ at $\blambda^t$.
Note that the minimization over $\C$ can be performed (approximately) by calling the LMO $\Omega$, and the minimization over $\bxi$ is a simple convex program. 

The  algorithm uses the (approximate) super-gradient obtained above to maintain an ellipsoid containing a solution that approximately maximises $f(\cdot)$ (with the current iterate $\blambda^t$ serving as the center of the ellipsoid), and iteratively shrinks its volume until we reach a small-enough region enclosing the maximizer. In Algorithm \ref{alg:ellipsoid}, we outline the details of the procedure. Lines 5-7 of the algorithm are  added to ensure that the iterates $\blambda^t$ never leave the initial ball. 

The main loop of Algorithm \ref{alg:ellipsoid} gives us a solution $\blambda$ that is close to the optimal dual solution. All that remains is to convert this to a solution for the primal problem in \ref{eq:unconstrained-reformulation}. For this, we adopt an approach from \cite{lee2015faster}, which uses the fact that
the algorithm maintains a subset of solutions obtained from convex combinations of the confusion matrix iterates $\conv\left(\C^0, \ldots, \C^{T-1}\right)$, each of which is a primal-optimal solution. Furthermore because the ellipsoid algorithm returns a solution from this set which is (approximately) dual-optimal, we have that:
\begin{align*}
\max_{\blambda \in \R^d}\,\min_{\C\in \cC,\, \bxi \in\Delta_d }\cL(\C, \bxi, \blambda) \sim \max_{\blambda \in \R^d}\,\min_{
\substack{
\C\in \conv\left(\C^0, \ldots, \C^{T-1}\right)\\
\bxi \in\Delta_d }}
\cL(\C, \bxi, \blambda).
\end{align*}
An application of min-max theorem then gives us that an approximate primal-optimal solution can be found by solving:
\[
\min_{
\substack{
\C\in \conv\left(\C^0, \ldots, \C^{T-1}\right)\\
\bxi \in\Delta_d }}\, \max_{\blambda \in \R^d}
\cL(\C, \bxi, \blambda) 
~=~ \min_{
\C\in \conv\left(\C^0, \ldots, \C^{T-1}\right)}\,\psi(\C),
\]
which amounts to solving a convex program with no further calls to the LMO and does not require further access to the training data. Line 14 of Algorithm \ref{alg:ellipsoid} describes this post-processing step.

\begin{thm}[Convergence of Ellipsoid algorithm]
\label{thm:ellipsoid}
Fix $\epsilon \in (0,1)$. 
Let $\psi: [0,1]^d\>[0,1]$ be convex and $L$-Lipschitz w.r.t. the $\ell_2$ norm. Let $\Omega$ in Algorithm \ref{alg:ellipsoid} be a $(\rho , \rho', \delta)$-approximate LMO for sample size $N$. Let $\bar{h}$ be the classifier returned by Algorithm \ref{alg:ellipsoid} when run for $T$ iterations with  initial radius $a=2L$. Then with probability $\geq 1 - \delta$ over draw of $S \sim D^N$, after $T = \cO\left(d^2 \log\left(d/{\epsilon}\right)\right)$ iterations:
\begin{align*}
\psi(\C[\bar{h}]) \,&\leq\, \min_{\C\in\cC}\,\psi(\C) \,+ \cO\left(\epsilon + \rho^\eff \right),
\end{align*}
where $\rho^\eff = \rho+\sqrt{d}\rho'$ and the $\cO$ notation hides constant factors independent of $\rho, \rho',\epsilon,d$.
\end{thm}
\begin{proof}
See Appendix \ref{app:proof-ellipsoid-unc}.
\end{proof}

Figure \ref{fig:ellipsoid-trajectory} illustrates the trajectory taken by the LMO iterates and the final confusion matrix for the running example, and demonstrates the convergence of the algorithm to an optimal classifier.

\begin{figure}
\begin{algorithm}[H]
\caption{Bisection Algorithm for \ref{eq:opt-unconstrained} with Ratio-of-linear $\psi$}\label{alg:bisection}
\begin{algorithmic}[1]
\STATE \textbf{Input:} $\psi: [0,1]^d \> [0,1]$ s.t.\  $\psi(\C) = \frac{\langle \A, \C \rangle}{\langle \B, \C \rangle}$ with $\A,\B \in \R^{d}$
\STATE ~~~~~~~~~~~~an LMO $\Omega$, $S = \{(x_1,y_1), \ldots, (x_N,y_N)\}$, $T$
\STATE \textbf{Initialize:} 
$\alpha^0 = 0, \beta^0 = 1,$ arbitrary classifier $h^0$
\STATE \textbf{For} $t = 1~\text{to}~T$ \textbf{do}
\STATE ~~~~~$\gamma^t = (\alpha^{t-1} + \beta^{t-1})/{2}$
\STATE ~~~~~${\L}^t = \frac{\A \,-\, \gamma^t\B}{\|\A \,-\, \gamma^t\B\|_2}$
\STATE ~~~~~$(g^t, \C^t) \,=\, \Omega(\L^t; S)$
\STATE ~~~~~\textbf{If} $\psi(\C^t) \leq \gamma^t$ 
~\textbf{then}~ 
$\alpha^{t} = \alpha^{t-1}, ~~ \beta^{t} = \gamma^{t}, ~~h^t = g^t$
\STATE ~~~~~~~~~~~~~~~~~~~~~~~~~~~~~~\textbf{else}~ $\alpha^{t} = \gamma^{t}, ~~ \beta^t = \beta^{t-1}, ~~h^t = h^{t-1}$
\STATE \textbf{End For}
\STATE \textbf{Output:} $\bar{h} = h^T$
\end{algorithmic}
\end{algorithm}
\vspace{-10pt}
\end{figure}

\subsection{Bisection Algorithm for Ratio-of-linear Metrics}

The final algorithm we describe in this section uses the bisection method \citep{BoydVan04} and is designed for ratio-of-linear performance metrics that can be written in the form $\psi^{\rl}(\C) = \frac{\langle \A, \C \rangle}{\langle \B, \C \rangle}$ for some $\A,\B \in \R^{d}$, such as the micro $F_1$-measure in Example \ref{ex:micro-F1}. 

For these performance measures, it is easy to see that:
\[
\min_{\C\in\cC} \psi(\C) \geq \gamma \Longleftrightarrow \min_{\C\in\cC} \langle \A-\gamma\B, \C \rangle \geq 0.
\]
Thus, to test whether the optimal value of $\psi$ is greater than $\gamma$, one can simply solve the linear minimization problem $\min_{\C\in\cC} \langle \A-\gamma\B, \C \rangle$ and test 
the value of $\psi$ at the resulting minimizer.
Based on this observation, one can employ the bisection method to conduct a binary search for the minimal value (and the minimizer) of $\psi(\C)$ using only a linear minimization subroutine. 

As outlined in Algorithm \ref{alg:bisection}, our proposed approach maintains a confusion matrix $\C^t$ implicitly via classifier $h^t$, together with lower and upper bounds 
$\alpha^t$ and $\beta^t$ on the minimal value of $\psi$. At each iteration, it determines whether this minimal value is greater than the midpoint $\gamma^t$ of these bounds using a call to the LMO, and then update $\C^t$ and $\alpha^t, \beta^t$ accordingly. 
Since for ratio-of-linear performance measures there is always a deterministic classifier achieving the optimal performance (see \Prop{prop:opt-classifier-ratio-linear}), 
here it suffices to maintain deterministic classifiers $h^t$.

Like the previous ellipsoid-based algorithm, the bisection algorithm also enjoys a logarithmic convergence rate:\footnote{In fact, the bisection algorithm can be viewed as a special case of the ellipsoid algorithm in one dimension \citep{BoydVan04}.}
\begin{thm}[Convergence of Bisection algorithm]
\label{thm:bisection-unc}
Fix $\epsilon \in (0,1)$. 
Let $\psi: [0,1]^d \> [0,1]$ be such that $\psi(\C) \,=\, \frac{\langle \A, \C \rangle}{\langle \B, \C \rangle}$, where $\A, \B \in \R^{n \times n}$, 
and $\min_{\C \in \cC} {\langle \B, \C \rangle} \,=\, b$ for some $b > 0$. Let $\Omega$ in Algorithm \ref{alg:bisection} be a $(\rho , \rho', \delta)$-approximate LMO for sample size $N$.
Let $\bar{h}$ be a classifier returned by Algorithm \ref{alg:bisection} when run for $T$ iterations. 
 Then with probability $\geq 1 - \delta$ over draw of $S \sim D^N$, after $T = \log(1/\epsilon)$ iterations:
\[
\psi(\C[\bar{h}]) \,\leq\, \min_{\C \in \cC}\,\psi(\C) \,+\,\cO\left(\epsilon + \rho^\eff\right),
\]
where 
$\rho^\eff = \rho + \sqrt{d}\rho'$ and 
the $\cO$ notation hides constant factors independent of $\rho,\rho', \epsilon$ and $d$.
\end{thm}
\begin{proof}
See Appendix \ref{app:proof-bisection}.
\end{proof}


%% file: 5_constrained.tex
\begin{figure}[t]
    \begin{center}
        \begin{subfigure}{0.32\linewidth}
        \centering
        \includegraphics[width=0.95\linewidth]{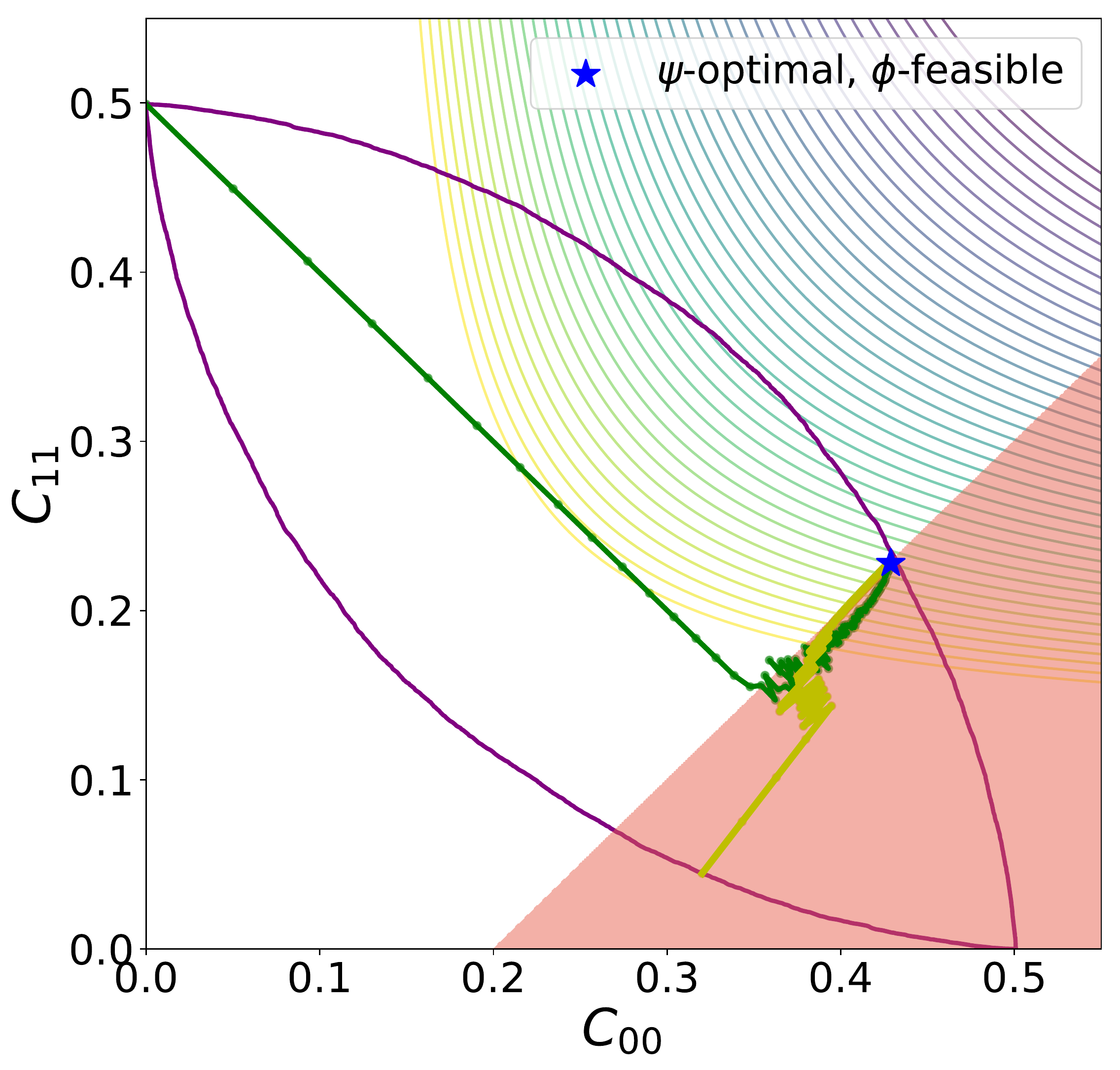}
        \caption{SplitFW}
        \label{fig:trajectory-splitfw}
        \end{subfigure}
        \begin{subfigure}{0.32\linewidth}
        \centering
        \includegraphics[width=0.95\linewidth]{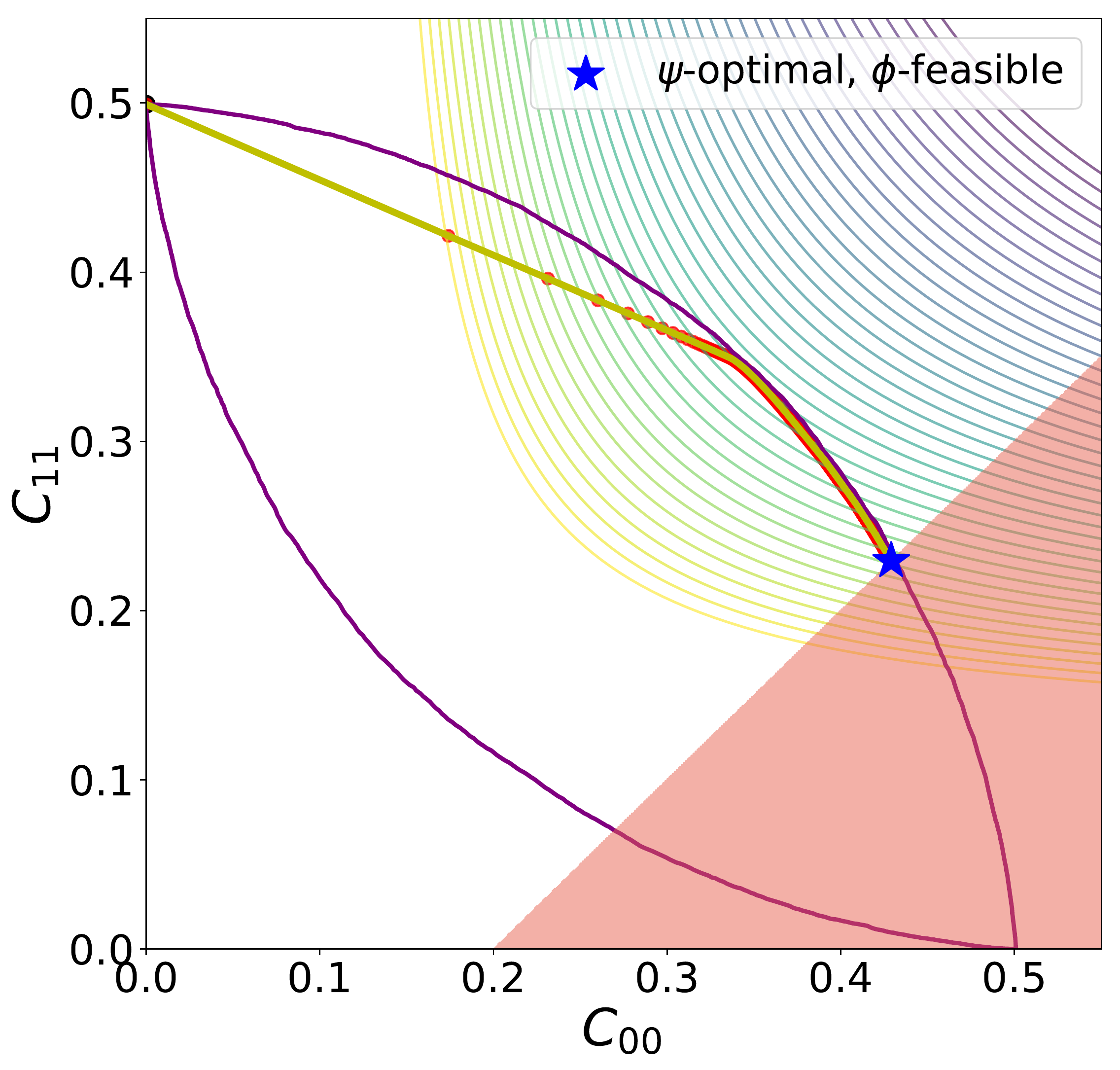}
        \caption{ConGDA}
        \label{fig:trajectory-congda}
        \end{subfigure}
        \begin{subfigure}{0.32\linewidth}
        \centering
        \includegraphics[width=0.95\linewidth]{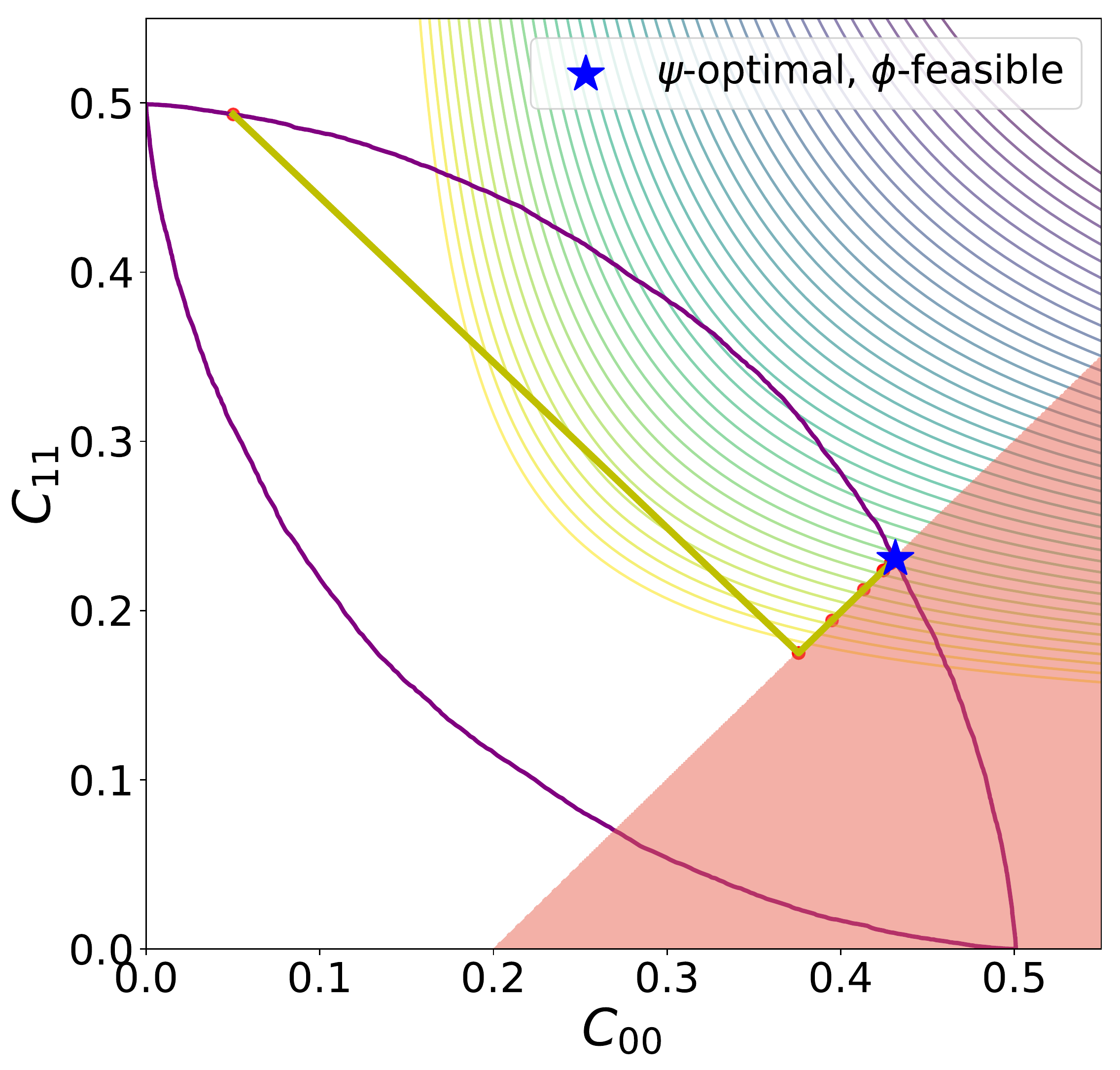}
        
        \caption{ConEllipsoid}
        \label{fig:trajectory-conellipsoid}
        \end{subfigure}
    \end{center}
    \vspace{-10pt}
    \caption{
    Illustration of the Split Frank-Wolfe (Algorithm \ref{alg:FW}), Constrained GDA (Algorithm \ref{alg:GDA-con}) and Constrained Ellipsoid (Algorithm \ref{alg:ellipsoid-con}) algorithms in minimizing the H-mean loss $\psi^{\HM}$
    subject to the constraint $C_{00} - C_{11} \geq 0.2$
    on the \texttt{NormBal} distribution in Figure \ref{fig:noram}. The figures contain the space of achievable confusion matrices $\cC$ (with purple colored boundary), along with the contours of $\psi^{\HM}$. 
    The trajectory of the averaged confusion matrix $\C[\bar{h}^t]$ for the averaged classifier is shown in green, 
    where $\bar{h}^t = h^t$ for Frank-Wolfe, $\bar{h}^t = \frac{1}{t}\sum_{\tau=1}^t h^t$ for GDA, and  $\bar{h}^t = \frac{1}{t}\sum_{\tau=1}^t \alpha^*_\tau h^\tau$ for
 ellipsoid, with the optimal coefficients $\balpha^* \in \argmin_{\balpha \in \Delta_t:\, \bphi(\sum_{\tau=1}^t \alpha_\tau \C^\tau) \leq \0} \psi\left(\sum_{\tau=1}^t \alpha_\tau \C^\tau\right)$  computed for iterates $1, \ldots t$. For SplitFW, we additionally plot the set of feasible confusion matrices $\cF$ that satisfy the constraint (shaded red region), along with the trajectory of the
    the averaged auxiliary variables $\F^t$ (gold). The algorithms can be seen to converge to an optimal feasible solution.
    }
    \label{fig:trajectories-constrained}
    \vspace{-5pt}
\end{figure}

\section{Algorithms for Constrained Problems}
\label{sec:constrained}
We next present iterative algorithms for
solving the constrained learning problem in \ref{eq:opt-constrained},
which as noted earlier, can be viewed as a minimization problem over $\cC$:
\begin{align}
\min_{\C \in \cC, \bphi(\C)\leq \0} \psi(\C).
  \tag{OP2*}
 \end{align}
As in the previous section, we will assume access to an LMO with the properties in
Definition \ref{defn:lmo}. 
 
 A simple approach to solving \ref{eq:opt-constrained} 
 for convex $\psi$'s and $\phi$'s is
 to formulate an equivalent convex-concave saddle point problem in terms of its Lagrangian:
 \[
\min_{\C \in \cC} \max_{\blambda \in \R_+^K}\,\psi(\C) + \sum_{k=1}^K\lambda_k\phi_k(\C)~=~
\max_{\blambda \in \R_+^K}
\underbrace{
\min_{\C \in \cC} \,\psi(\C) + \sum_{k=1}^K\lambda_k\phi_k(\C)}_{\nu(\lambda)},
 \]
 where $\lambda_k$ is the Lagrange multiplier for constraint $\phi_k$, and we use strong duality to exchange the `min' and `max'. 
For a fixed $\blambda$, the minimization over $\C$ is an unconstrained convex problem in $\C$. This resembles \ref{eq:opt-unconstrained} and can be solved with any of Algorithm \ref{alg:FW}--\ref{alg:ellipsoid} proposed in the previous section. 
One can therefore apply a standard gradient ascent procedure to maximize the dual function $\nu(\blambda)$, where the gradients w.r.t.\ $\blambda$ can be computed by solving the minimization of $\C$. However, this vanilla dual-ascent approach does not enjoy strong convergence guarantees because of the multiple levels of nesting. For example, with the Frank-Wolfe based algorithm (Algorithm \ref{alg:FW}) for the inner minimization, this procedure would take $\cO(1/\epsilon^3)$ calls to the LMO to reach an $\cO(\epsilon)$-optimal, $\cO(\epsilon)$-feasible solution \citep{Narasimhan18}. 

 In what follows, we describe four algorithms for solving \ref{eq:opt-constrained} which require fewer calls to the LMO than the vanilla approach described above (see Table \ref{tab:algorithms-con} for a summary of our results). The proofs  build on standard techniques for showing convergence of the respective optimization solvers, but need to additionally take into account the errors in the LMO calls and need to translate the dual-optimal solution guarantees to optimality and feasibility guarantees for the primal solution.
 
 The proposed algorithms can be seen as ``constrained'' counterparts to the four unconstrained algorithms described in the previous section. All our algorithms will assume that the constraints $\phi_k(\C)$ are convex in $\C$. 
 %
 %
 %
 %
 As a running example to illustrate our algorithms, we will use the task of maximizing the H-mean loss on the \texttt{NormBal} distribution described in Figure \ref{fig:noram}, subject to the constraint that coverage on class 1 be no more than 0.3. This constraint is linear in $\C$ and can be written as $C_{01} + C_{11} \leq 0.3$, or equivalently re-written as $C_{00} - C_{11} \geq 0.2$.

\begin{table}[]
        \centering
        \caption{Algorithms for the constrained problem  in \ref{eq:opt-constrained}, with the number calls to the LMO, and the optimality gap $\psi(\C[\bar{h}]) - \min_{\C\in \cC}\psi(\C)$ and feasibility gap $\max_k \phi_k(\C[\bar{h}])$
        for the returned classifier $\bar{h}$. In rows 1--3, we assume $\psi$ is Lipschitz w.r.t.\ the $\ell_2$-norm, and in all rows, we assume that $\phi_1, \ldots, \phi_K$ are convex and Lipschitz, and satisfy the strict feasibilty condition in Assumption \ref{assp:strict-feasibility}. 
        In row 4, $\psi(\C) \,=\, \frac{\langle \A, \C \rangle}{\langle \B, \C \rangle}$ with $\min_{\C \in \cC}\langle \B, \C \rangle > 0$.
        We denote $\bar{d} = d + K$, and $\rho^\eff = \rho+\sqrt{d}\rho'$.
         }
        \label{tab:algorithms-con}
        \begin{small}
        \begin{tabular}{ccccc}
        \hline
        \textbf{Algorithm} & \textbf{Assumption on $\psi$}
        & \textbf{\# LMO Calls} & \textbf{Opt.\ Gap} &
         \textbf{Feasibility Gap}
        \\
        \hline
            Split Frank-Wolfe & Convex, Smooth & $\displaystyle\cO\left(1/\epsilon^2\right)$  &
             $\cO\left(\epsilon + \sqrt{\rho^\eff}\right)$ & $\cO\left(\epsilon + \sqrt{\rho^\eff}\right)$ 
             \\
            Con.\ GDA & 
            Convex & 
        $\displaystyle\cO\left(K/\epsilon^2\right)$ &
        $\cO\left(\epsilon + \rho^\eff\right)$ & $\cO\left(\epsilon + \rho^\eff\right)$
        \\
         Con.\ Ellipsoid & 
        Convex &
        $\displaystyle\cO\left(\bar{d}^2\log(\bar{d}/\epsilon)\right)$ 
        &
        $\cO\left(\epsilon + \rho^\eff\right)$
        & $\cO\left(\rho^\eff)\right)$
        \\
        Con.\ Bisection & 
        Ratio-of-linear &
        $\displaystyle\cO\left(K\log(1/\epsilon)/\epsilon^2\right)$ &
        $\cO\left(\epsilon + \rho^\eff\right)$
        &
         $\cO\left(\epsilon + \rho^\eff\right)$\\
        \hline
        \end{tabular}
        \end{small}
    \end{table}



\subsection{(Split) Frank-Wolfe Algorithm for Smooth Convex Metrics}
In this section, we adapt the Frank-Wolfe approach in Algorithm \ref{alg:FW} to constrained learning problems \ref{eq:opt-constrained} for smooth convex metrics $\psi$. The key idea is to pose \ref{eq:constrained-reformulation} as an optimization problem over the intersection of two sets:
\begin{align}
    \min_{\C \in \cC:\, \bphi(\C)\leq \0} \psi(\C) &=
    \min_{\C \in \cC \cap \cF} \psi(\C),
    \label{eq:op-con-interesection}
\end{align}
where $\cF = \{\F \in \Delta_d\,|\, \bphi(\F)\leq \0\}$ is the set of all points in $\Delta_d$ that satisfy the $K$ inequality constraints. While the set $\cF$ is convex (and so is the intersection $\cC \cap \cF$), we will not be able to  apply the classical Frank-Wolfe method to this problem as we cannot directly solve a linear minimization over the intersection $\cC \cap \cF$. However, we already have access to an LMO for the set $\cC$ alone, and performing a linear minimization over the set $\cF$ amounts to solving a straight-forward convex program. We therefore adopt the Frank-Wolfe based variant proposed by \cite{Gidel2018} for optimizing a (smooth) convex function over the intersection of two convex sets with access to linear minimization oracles for the individual sets. 

To this end, we introduce auxiliary variables $\F \in \Delta_d$ in  \eqref{eq:op-con-interesection} and decouple the two constraint sets, giving us the following equivalent optimization problem:
\begin{equation} \label{eq:op-con-equality-constraint} 
    \min_{\C \in \cC, \F \in \cF} \psi(\C) + \psi(\F) \enskip \text{s.t.} \enskip \C - \F =0.
\end{equation}
We then define the augmented Lagrangian of the above problem as:
\begin{equation}\label{eqn:augmented-lagrangian}
\cL^{\aug}(\C,\F,\blambda) = \psi(\C) + \psi(\F) + \langle\blambda, \C - \F\rangle + \dfrac{\zeta}{2} || \C - \F||_2^2,
\end{equation}
where $\blambda$ is a vector of Lagrange multipliers for the equality constraints and $\zeta > 0$ is a constant weight on the quadratic penalty term. We apply the approach of \cite{Gidel2018} to solve  \eqref{eq:op-con-equality-constraint} by using a  gradient ascent step to maximize $\cL^{\aug}$ over $\blambda$, a 
linear minimization step for $\C$ over $\cC$, and a linear minimization step for $\F$ over $\cF$.

\begin{figure}[t]
\begin{algorithm}[H]
\caption{Split Frank-Wolfe (SplitFW) Algorithm for \ref{eq:opt-constrained} with Smooth Convex $\psi$}\label{alg:FW-con}
\begin{algorithmic}[1]
\STATE \textbf{Input:} $\psi, \phi_1,\ldots,\phi_k:[0, 1]^{d} \to[0,1]$, an LMO $\Omega$, $S = \{(x_1,y_1), \ldots, (x_N,y_N)\}$, $T \in \N$, $\zeta, \eta > 0$. 
\STATE \textbf{Initialize:} 
$(h^0, \C^0) = \Omega(\L^0; S)$ for an arbitrary loss matrix $\L^0$
\STATE \textbf{For} $t =  1$ \textbf{to} $T$ \textbf{do}
\STATE ~~~~~$\L^t\,=\, \frac{\ba^t}{\|\ba\|_2},$ where $\ba^t = \nabla_\C\cL^\aug(\C^{t-1}, \F^{t-1}, \blambda^{t-1})$
\STATE ~~~~~$(\tilde{h}^t, \tilde{\C}^t) \,=\, \Omega(\L^t; S)$
\STATE ~~~~ $\tilde\F^t = \argmin_{\F \in \cF} \left\langle \bb^{t}, \F \right\rangle $, \text{where }$\bb^{t}=\nabla_\F \cL^\aug(\C^{t-1}, \F^{t-1}, \blambda^{t-1})$
\STATE ~~~~ $\gamma^t = \argmin_{\gamma\in[0,1]} \cL^\aug\big((1-\gamma) \C^{t-1} + \gamma \tilde\C^t, \ (1-\gamma) \F^{t-1} + \gamma \tilde\F^t, \ \blambda^{t-1}\big)$ 
\STATE ~~~~~${h}^{t} = \big(1-\gamma^t\big) {h}^{t-1} + \gamma^t \tilde{h}^t$

\STATE ~~~~ $\C^t = (1-\gamma^t) \C^{t-1}  + \gamma^t \tilde{\C}^t $
\STATE ~~~~ $\F^t = (1-\gamma^t) \F^{t-1}  + \gamma^t \tilde{\F}^t $
\STATE ~~~~ $\blambda^t = \blambda^{t-1} + \frac{\eta}{t} (\C^t -\F^t)$ 
\STATE \textbf{End For}
\STATE \textbf{Output:} $\bar{h} = h^{t_*}$ and $\bar{\C} = \C^{t_*}$, where  $t_* = \argmin_{t>T/2} ||\C^t - \F^t||_2^2$ 
\end{algorithmic}
\end{algorithm}
\vspace{-10pt}
\end{figure}

This procedure, outlined in Algorithm \ref{alg:FW-con}, is guaranteed to converge to an optimal feasible classifier under the assumption that there exists a confusion matrix which is strictly feasible.
\begin{assump}[Strict feasibility]
\label{assp:strict-feasibility}
For some $r > 0$,  there exists a confusion matrix $\C' \in \cC$ such that
$\max_{k\in [K]} \phi_k(\C') \leq -r$. 
\end{assump}
\begin{thm}[Convergence of SplitFW algorithm]
\label{thm:FW-con}
Fix $\epsilon>0$. Let $\psi: [0,1]^d \> [0,1]$ be convex, $\beta$-smooth and $L$-Lipschitz w.r.t.\ the $\ell_2$-norm, and let $\phi_1,\ldots, \phi_K: [0,1]^d \> [-1,1]$ be convex and $L$-Lipschitz w.r.t.\ the $\ell_2$-norm. Let $\Omega$ in Algorithm \ref{alg:FW-con} be a $(\rho , \rho', \delta)$-approximate LMO for sample size $N$.
Let $\bar{h}$ be a classifier returned by Algorithm \ref{alg:FW-con} when run for $T$ iterations with some $\zeta>0$. 
Let the strict feasibility condition in Assumption \ref{assp:strict-feasibility} hold for radius $r>0$. Then,  with probability $\geq 1 - \delta$ over draw of $S \sim D^N$, after $T = \cO(1/\epsilon^2)$ iterations:
\[
\textbf{Optimality:}~~\psi(\C[\bar{h}]) \,\leq\, \min_{\C \in \cC, \phi_k(\C)\leq 0, \forall k}\,\psi(\C) \,+\,\cO\left(\epsilon + \sqrt{\rho^\eff}\right);
\vspace{-5pt}
\]
\[
\textbf{Feasibility:}~~\phi_k(\C[\bar{h}]) \,\leq\,  \cO\left( \epsilon + \sqrt{\rho^\eff}  \right) ,~\forall k \in [K].
\]
where $\rho^\eff = \rho+ \sqrt{d}\rho'$ and the $\cO$ notation hides constant factors independent of $\rho, \rho', T, d$ and $K$ for small enough $\rho, \rho'$ and large $T$.
\end{thm}
\begin{proof}
See Appendix \ref{app:proof-split-fw}
\end{proof}

Unlike the Frank-Wolfe based algorithm for the unconstrained problem (see Theorem \ref{thm:FW-unc}) which  needed only $\cO(1/\epsilon)$ calls to the LMO to reach an $\cO(\epsilon+c)$-optimal solution, the proposed algorithm for handling constraints requires $\cO(1/\epsilon^2)$ calls to reach an $\cO(\epsilon+c)$-optimal, feasible solution. 


Figure \ref{fig:trajectory-splitfw} illustrates the trajectories
of the algorithm applied to the previously described running example. 
As seen, both the iterates $\C^t$ and  $\F^t$, representing the achievable and feasible confusion matrices respectively, are seen to 
converge to a solution that is optimal and feasible for the problem. 


\subsection{Gradient Descent-Ascent Algorithm for Non-smooth Convex Metrics}
\label{sec:con-gda}
Next, we modify the gradient descent-ascent approach in Algorithm \ref{alg:GDA} to handle constraints. Our proposal is a slight variant of the oracle-based algorithm in \citet{Narasimhan+19_generalized} for optimizing with constraints. As before, we introduce slack variables $\bxi \in \Delta_{d}$ to decouple the functions $\psi, \phi_1,\ldots,\phi_K$ from the confusion matrix $\C$, and re-write \ref{eq:constrained-reformulation} as:
 \begin{align}
\min_{\C \in \cC:\, \bphi(\C)\leq \0}\psi(\C) &= \min_{
\substack{
\C \in \cC,\, \bxi \in \Delta_d\\\bxi = \C,\, \phi_k(\bxi)\leq 0,\forall k}.
}
\psi(\bxi)
 \label{eq:slack-reformulation-constrained}
 \end{align}
We then define the Lagrangian for the above problem with multipliers $\blambda \in \R^d$ for the $d$ equality constraints and $\bmu \in \R_+^K$ for the $K$ inequality constraints:
\begin{align}
\cL^\con(\C, \bxi, \blambda, \bmu) 
&= \psi(\bxi) + \langle \blambda, \C-\bxi \rangle + 
\langle\bmu, \boldsymbol{\phi}(\bxi)\rangle
\label{eq:lagrangian-con},
\end{align}
and re-formulate \eqref{eq:slack-reformulation-constrained} 
as 
the following min-max problem:
\begin{align}
\min_{\C \in \cC, \phi_k(\C)\leq 0,\forall k} &= \min_{\C \in \cC,\, \bxi \in\Delta_d}\,\max_{\blambda \in \R^d, \bmu \in \R_+^K}\,\cL^\con(\C, \bxi, \blambda, \bmu).
 \label{eq:minmax-reformulation-constrained}
\end{align}
The gradient descent-ascent procedure for solving an approximate saddle point of \eqref{eq:minmax-reformulation-constrained} is shown in Algorithm \ref{alg:GDA} and enjoys the following convergence guarantee for a convex, non-smooth metric $\psi$:
\begin{thm}[Convergence of ConGDA algorithm]
\label{thm:gda-con}
Fix $\epsilon \in (0,1)$. 
Let $\psi: [0,1]^d\>[-1,1]$ and $\phi_1,\ldots,\phi_K: [0,1]^d\>[-1,1]$ be convex and $L$-Lipschitz  w.r.t.\ the $\ell_2$-norm. Let $\Omega$ in Algorithm \ref{alg:GDA-con} be a $(\rho , \rho', \delta)$-approximate LMO for sample size $N$. 
Suppose the strict feasibility condition in Assumption \ref{assp:strict-feasibility} holds for radius $r>0$.
Let the space of Lagrange multipliers $\Lambda = \{\blambda\in \R^d\,|\,\|\blambda\|_2\leq 2L(1 + 1/r)\}$,
and $\Xi = \{\bmu\in \R_+^K\,|\,\|\bmu\|_1\leq 2/r\}$. 
Let $\bar{h}$ be a classifier returned by Algorithm \ref{alg:GDA-con} when run for $T$ iterations,
with step-sizes $\eta = \frac{1}{\bar{L}\sqrt{2T}}$
and $\eta' = \frac{\bar{L}}{(1 + 2\sqrt{K})\sqrt{2T}}$, where $\bar{L} = 4(1 + 1/r)L + 2/r$.
 Then with probability $\geq 1 - \delta$ over draw of $S \sim D^N$, after $T = \cO(K/\epsilon^2)$ iterations:
\[
\textbf{Optimality:}~~\psi(\C[\bar{h}]) \,\leq\, \min_{\C \in \cC:\, \bphi(\C) \leq \0}\,\psi(\C) \,+\,\cO\left(\epsilon + \rho^\eff \right);
\]
\vspace{-5pt}
\[
\textbf{Feasibility:}~~
\phi_k(\C[\bar{h}]) \,\leq\, \cO\left(\epsilon+\rho^\eff \right), \forall k \in [K].
\]
where $\rho^\eff = \rho+\sqrt{d}\rho'$ and the $\cO$ notation hides constant factors independent of $\rho,\rho', T, d$ and $K$.
\end{thm}
\begin{proof}
See Appendix \ref{app:proof-gda-con}.
\end{proof}

\begin{figure}
\begin{algorithm}[H]
\caption{Constrained GDA (ConGDA) Algorithm for \ref{eq:opt-constrained} with Non-smooth Convex $\psi$}\label{alg:GDA-con}
\begin{algorithmic}[1]
\STATE \textbf{Input:} $\psi, \phi_1, \ldots, \phi_K:[0, 1]^{d} \to [0,1]$, an LMO $\Omega$, $S = \{(x_1,y_1), \ldots, (x_N,y_N)\}$, $T$, space of Lagrange multipliers $\Lambda \subset \R^d, \Xi \subset \R_+^K$
\STATE \textbf{Parameters:} Step-sizes $\eta_{\bxi}, \eta_{\blambda}, \eta_{\bmu} > 0$
\STATE \textbf{Initialize:} 
$(h^0, \C^0) = \Omega(\L^0; S)$ for an arbitrary loss matrix $\L^0$
\STATE \textbf{For} $t =  1$ \textbf{to} $T$ \textbf{do}
\STATE ~~~~~$\L^t\,=\, \frac{\blambda^{t-1}}{ \|\blambda^{t-1}\|_2}$
\STATE ~~~~~$(h^t, \C^t) \,=\, \Omega(\L^t; S)$
\STATE ~~~~~$\tilde{\bxi}\,=\, \bxi^{t-1} \,-\, \eta_\bxi\nabla_{\bxi}\cL^\con(\C^{t}, \bxi^{t-1}, \blambda^{t-1}, \bmu^{t-1})$;~~~~~~$\bxi^{t+1}\,\in\,\argmin_{\bxi \in [0,1]^d}\,\|\bxi - \tilde{\bxi}\|_2$
\STATE ~~~~~$\tilde{\blambda}\,=\, \blambda^{t-1} \,+\, \eta_{\blambda} \nabla_{\blambda}\cL^\con(\C^{t}, \bxi^{t-1}, \blambda^{t-1}, \bmu^{t-1})$;~~~~~~$\blambda^{t+1}\,\in\,\argmin_{\blambda \in \Lambda}\|\blambda - \tilde{\blambda}\|_2$
\STATE ~~~~~$\tilde{\bmu}^t\,=\, \bmu^{t-1} \,+\, \eta_\bmu\nabla_{\bmu}\cL^\con(\C^{t}, \bxi^{t-1}, \blambda^{t-1}, \bmu^{t-1})$;~~~~~~$\bmu^{t+1}\,\in\,\argmin_{\bmu \in \Xi}\|\bmu - \tilde{\bmu}\|_2$
\STATE \textbf{End For}
\STATE \textbf{Output:} $\bar{h} = \frac{1}{T}\sum_{t=1}^T h^t$ 
\end{algorithmic}
\end{algorithm}
\vspace{-13pt}
\end{figure}


Figure \ref{fig:trajectory-congda}
shows the trajectory of the iterates of the algorithm on the same running example used 
for the SplitFW algorithm.
The algorithm is seen to converge to an optimal-feasible classifier.

\subsection{Ellipsoid Algorithm for Non-smooth Convex Metrics}
Our next algorithm extends the  ellipsoid method in Algorithm \ref{alg:ellipsoid} to handle constraints $\bphi(\C)\leq \0$.  We use the Lagrangian $\cL^\con(\C, \bxi, \blambda, \bmu)$ for the constrained problem defined in the previous section in \eqref{eq:lagrangian-con},
and work with its dual function $f$:
\begin{align*}
f^\con(\blambda, \bmu)           
&= \begin{cases}
\min_{\C\in \cC,\, \bxi \in\Delta_{d}} 
\cL^\con(\C, \bxi, \blambda, \bmu) & \text{ if }\bmu \geq 0 \\
-\infty & \text{ otherwise}
\end{cases},
\end{align*}
where we note that the Lagrange multipliers $\bmu$ for the 
$K$ inequality constraints are not allowed to be negative.

Following the unconstrained case, we 
seek to  maximize the dual function 
over $\blambda \in \R^d$ and over $\bmu \in \R_+^K$. 
Because $f^\con$ is concave in $\blambda$ and $\bmu$, we can employ the ellipsoid method with the JLE subroutine in Algorithm \ref{alg:jle} to maximize $f^\con(\blambda, \bmu)$, and use a post-processing step to convert the dual solution to a near-optimal and near-feasible solution for the primal problem. As shown in Algorithm \ref{alg:ellipsoid-con},
at each iteration, the procedure maintains an ellipsoid containing the maximizer of $f^\con$, with the current iterate $[\blambda^t, \bmu^t]$ serving as the center of the ellipsoid

Lines 5 to 10 of the algorithm simply ensure the iterate $[\blambda^t, \bmu^t]$ stays within the initial ellipsoid, and $\bmu^t$ remains non-negative. As before, to compute a super-gradient for $f$ at a given $[\blambda^t, \bmu^t]$, we compute 
$\C^t \in\argmin_{\C \in \cC} \langle\blambda^t, \C\rangle$ and $\bxi^t \in \argmin_{\bxi \in \Delta_d} \psi(\bxi) - \langle\blambda^t, \bxi\rangle + \langle \bmu^t, \bphi(\bxi) \rangle$, and  evaluate $[\C^t - \bxi^t, \bphi(\bxi^t)]$. 
Note that $\C^t$ can be obtained via a linear minimization oracle over   $\cC$ and $\bxi^t$ is the solution of a convex program that has no dependence on the data distribution. The approximate nature of the LMO (and in turn the supergradient of $f^\con$) require a modified proof from the standard ellipsoid to argue that the errors at each iteration do not add up catastrophically. The dual solution is converted to a primal-feasible solution in line 16 of the algorithm by solving a convex optimization problem that requires 
no access to the training data.



Figure \ref{alg:ellipsoid-con} illustrates the working of the algorithm. Note that the initial classifier $h^0$ can be any classifier, as it is the result of the LMO where the loss is the zero matrix. We assume here that the initial classifier $h^0$ is strictly feasible for convenience.  

\begin{figure}
\begin{algorithm}[H]
\caption{Constrained Ellipsoid (ConEllipsoid) Algorithm for \ref{eq:opt-constrained} with Non-smooth Convex $\psi$}\label{alg:ellipsoid-con}
\begin{algorithmic}[1]
\STATE \textbf{Input:} $\psi:[0, 1]^{d} \to [0,1]$, an LMO $\Omega$,   $S = \{(x_1,y_1), \ldots, (x_N,y_N)\}, T$
\STATE \textbf{Parameters:} Initial ellipsoid radius $a$, a strictly feasible classifier $h^0$ 
\STATE \textbf{Initialize:} ${\blambda}^{0} = \0_d, {\bmu}^{0}=0$, ${\A}^{0} = a^2 \I_{d+K}$, $\C^0 = {\C}[h^0]$
\STATE \textbf{For} $t=0$ \textbf{to} $T-1$:
\STATE  ~~~~ \textbf{If} $\|[{\blambda}^{t}, {\bmu}^{t}]\|_2 > a$: 
\STATE  ~~~~~~~~ ${\A}^{t+1},[{\blambda}^{t+1}, {\bmu}^{t+1}] 
= \text{JLE}({\A}^t, [{\blambda}^t, {\bmu}^t], [-{\blambda}^t, -{\bmu}^t])$
\STATE ~~~~~~~~ $h^t, \C^t = h^0, \C^0$ ; \textbf{continue}
\STATE  ~~~~ \textbf{Else If $ {\bmu} ^t \ngeq \0$ }: 
\STATE  ~~~~~~~~ ${\A}^{t+1},[{\blambda}^{t+1}, {\bmu}^{t+1}] = \text{JLE}({\A}^t, [{\blambda}^t, {\bmu}^t], [\0_{d}, \text{pos}(- \bmu^t)])$, where $\text{pos}(u)=\max(u,0)$.
\STATE ~~~~~~~~ $h^t, \C^t = h^0, \C^0$ ; \textbf{continue}
\STATE ~~~~\textbf{Else:}
\STATE   ~~~~~~~~ $(h^t, \C^t)= \Omega(\blambda^t, S)$
\STATE   ~~~~~~~~ $\bxi^t = \argmin_{\bxi \in \Delta_d} \psi(\bxi) - \langle \blambda^t, \bxi\rangle + \langle\bmu^t, \boldsymbol{\phi}(\bxi)\rangle$
\STATE   ~~~~~~~~ ${\A}^{t+1},[{\blambda}^{t+1}, {\bmu}^{t+1}] = \text{JLE}(\A^t, [\blambda^t, \bmu^t], [\C^t-\bxi^t, \boldsymbol{\phi}(\bxi^t)])$
\STATE \textbf{End For}
\STATE $\balpha^* \in \underset{\balpha \in \Delta_T:\, \bphi(\sum_t \alpha_t \C^t) \leq \0}{\argmin} \psi\left(\sum_{t=0}^{T-1} \alpha_t \C^t\right)$ 
\STATE \textbf{Ouput:} $\bar{h} = \sum_{t=0}^{T-1} \alpha^*_t h^t$
\end{algorithmic}
\end{algorithm}
\vspace{-15pt}
\end{figure}

\begin{thm}[Convergence of ConEllipsoid]
\label{thm:ellipsoid-con}
Fix $\epsilon \in (0,1)$. 
Let $\psi: [0,1]^d\>[0,1], \phi_1,\ldots,\phi_K: [0,1]^d\>[-1,1]$ be convex and $L$-Lipschitz w.r.t. the $\ell_2$ norm. Let $\Omega$ in Algorithm \ref{alg:GDA-con} be a $(\rho , \rho', \delta)$-approximate LMO for sample size $N$. Suppose the strict feasibility condition in Assumption \ref{assp:strict-feasibility} holds for some $r>0$. Let the initial classifier $h^0$ satisfy this condition, i.e. $\bphi(\C[h^0]) \leq -r$ and $\C[h^0]=\C^0$. Let $\bar{d}=d+K$. Let $\bar{h}$ be the classifier returned by Algorithm \ref{alg:ellipsoid-con} when run for $T>2\bar{d}^2 \log(\frac{\bar{d}}{\epsilon})$ iterations with  initial radius $a>2(L+\frac{L+1}{r})$.
 Then with probability $\geq 1 - \delta$ over draw of $S \sim D^N$, we have
\begin{align*}
\textbf{Optimality:}~~\psi(\C[\bar{h}]) \,&\leq\, \min_{\C\in\cC:\, \phi_k(\C) \leq 0,\forall k}\,\psi(\C) + \cO(\epsilon+\rho^\eff) ; \\
\textbf{Feasibility:}~~\phi_k(\C[\bar{h}]) \,&\leq\, \cO(\rho^\eff) ,~\forall k \in [K],
\end{align*}
where $\rho^\eff = \rho+\sqrt{d}\rho'$ and the $\cO$ notation hides constant factors independent of $\rho, \rho',T,d$ and $K$.
\end{thm}

\begin{proof}
See Appendix \ref{app:proof-ellipsoid}.
\end{proof}

The theorem above gives guarantees on the convergence of the constrained ellipsoid algorithm to the optimal feasible solution. Notice the exponential convergence rate in $1/\epsilon$ at the cost of a quadratic dependence on dimension $d$ and number of constraints $K$. 
Figure \ref{fig:trajectory-conellipsoid} shows the trajectory of the iterates of the algorithm on the same running example used previously. The algorithm is clearly seen to converge to an optimal-feasible classifier.

\subsection{Bisection Algorithm for Fractional-linear Metrics}
The final constrained algorithm we describe is a straightforward extension of the bisection method in Algorithm \ref{alg:bisection} for ratio-of-linear performance measures that can be written in the form $\psi^{\rl}(\C) = \frac{\langle \A, \C \rangle}{\langle \B, \C \rangle}$ for some $\A,\B \in \R^{d}$. The key observation here is that testing whether the optimal solution to the constrained problem \ref{eq:constrained-reformulation} with a ratio-of-linear $\psi$  is greater than a threshold $\gamma$ is equivalent to minimizing a linear metric with constraints:
\[
\min_{\C\in\cC:\, \bphi(\C) \leq \0} \psi(\C) \geq \gamma \Longleftrightarrow \min_{\C\in\cC:\, \bphi(\C) \leq \0} \langle \A-\gamma\B, \C \rangle \geq 0.
\]
The latter can be solved using any of constrained learning methods outlined Algorithms \ref{alg:FW-con}--\ref{alg:ellipsoid-con}. Therefore one can employ the bisection method as before to conduct a binary search for the minimal value (and minimizer) of $\psi(\C)$ by calling one of these algorithms at each step. We outline this procedure in Algorithm \ref{alg:bisection-con}, with the ConGDA method (Algorithm \ref{alg:GDA-con}) used for the inner minimization.

\noindent We then have the following convergence guarantee:\footnote{Because the inner subroutine uses the ConGDA algorithm, the rate of convergence has a dependence of $\tilde{O}\left(1/\epsilon^2\right)$ on $\epsilon$, which is an improvement over the $\tilde{O}\left(1/\epsilon^3\right)$ dependence in the previous conference paper \citep{Narasimhan18}.}
\begin{thm}[Convergence of ConBisection algorithm]
\label{thm:bisection-con}
Fix $\epsilon \in (0,1)$. 
Let $\psi: [0,1]^d \> [0,1]$ be such that $\psi(\C) \,=\, \frac{\langle \A, \C \rangle}{\langle \B, \C \rangle}$, where $\A, \B \in [0,1]^{d}$, 
and $\min_{\C \in \cC} {\langle \B, \C \rangle} \,=\, b$ for some $b > 0$. Let $\phi_1,\ldots,\phi_K: [0,1]^d\>[-1,1]$ be convex and $L$-Lipschitz w.r.t.\ the $\ell_2$-norm. Let $\Omega$ in Algorithm \ref{alg:bisection-con} be a $(\rho , \rho', \delta)$-approximate LMO for sample size $N$. 
Suppose the strict feasibility condition in Assumption \ref{assp:strict-feasibility} holds for some $r>0$.
Let $\Lambda$, $\Xi$, $\eta$ and $\eta'$ in the call to Algorithm \ref{alg:GDA-con} be set as in Theorem \ref{thm:gda-con} with Lipschitz constant $L' = \max\{L, \|\A\|_2 + \|\B\|_2\}$. 
Let $\bar{h}$ be a classifier returned by Algorithm \ref{alg:bisection-con} when run for $T$ outer iterations and $T'$ inner iterations. 
 Then with probability $\geq 1 - \delta$ over draw of $S \sim D^N$, after $T = \log(1/\epsilon)$ outer iterations and $T'=\cO(K/\epsilon^2)$ inner iterations:
\[
\textbf{Optimality}:~~\psi(\C[\bar{h}]) \,\leq\, \min_{\C\in\cC:\, \bphi(\C) \leq \0}\,\psi(\C) \,+\,\cO\left(\epsilon + \rho^\eff\right);
\vspace{-5pt}
\]
\[
\textbf{Feasibility}:~~\phi_k(\C[\bar{h}]) \,\leq\, \cO\left(\epsilon +\rho^\eff \right),~\forall k \in [K],
\]
where $\rho^\eff = \rho+\sqrt{d}\rho'$ and the $\cO$ notation hides constant factors independent of $\rho, \rho',T,d$ and $K$.
\end{thm}
\begin{proof}
See Appendix \ref{app:proof-bisection-con}.
\end{proof}

\begin{figure}
\begin{algorithm}[H]
\caption{Constrained Bisection (ConBisection) Algorithm for \ref{eq:opt-constrained} with Ratio-of-linear $\psi$}\label{alg:bisection-con}
\begin{algorithmic}[1]
\STATE \textbf{Input:} $\psi: [0,1]^d \> [0,1]$ s.t.\  $\psi(\C) = \frac{\langle \A, \C \rangle}{\langle \B, \C \rangle}$ with $\A,\B \in \R^{d}$ and $\phi_1,\ldots, \phi_K: [0,1]^d\>[0,1]$
\STATE ~~~~~~~~~~~~an LMO $\Omega$, $S = \{(x_1,y_1), \ldots, (x_N,y_N)\}$, $T$, $T'$, ConGDA parameters: $\Lambda$, $\Xi$, $\eta$ and $\eta'$ 
\STATE \textbf{Initialize:} 
$\alpha^0 = 0, \beta^0 = 1,$ a classifier $h^0$ that satisfies the  constraints, i.e.\ $\bphi(\C[h^0]) \leq \0$
\STATE \textbf{For} $t = 1~\text{to}~T$ \textbf{do}
\STATE ~~~~~$\gamma^t = (\alpha^{t-1} + \beta^{t-1})/{2}$
\STATE ~~~~~$({g}^t, \C^t) \,=\, \text{ConGDA}(\psi', \bphi, S, \Omega, T', \Lambda, \Xi, \eta, \eta'),$ where $\psi'(\C) = 
\langle\A \,-\, \gamma^t\B, \C\rangle$
\STATE ~~~~~\textbf{If} $\psi^\rl(\C^t) \geq \gamma^t$ 
~\textbf{then}~ 
$\alpha^{t} = \gamma^t, ~~ \beta^{t} = \beta^{t-1}, ~~h^t = {g}^t$
\STATE ~~~~~~~~~~~~~~~~~~~~~~~~~~~~~~\textbf{else}~ $\alpha^{t} = \alpha^{t-1}, ~~ \beta^t = \gamma^t, ~~h^t = h^{t-1}$
\STATE \textbf{End For}
\STATE \textbf{Output:} $\bar{h} = h^T$
\end{algorithmic}
\end{algorithm}
\vspace{-12pt}
\end{figure}


%% file: 6_lmo.tex
\section{Plug-in Based Linear Minimization Oracle}
\label{sec:lmo}
All the learning algorithms we have presented have assumed access to an approximate linear minimization oracle (LMO)  (see Definition \ref{defn:lmo}). In this section, we describe a practical  plug-in based LMO with the desired approximation properties. This method seeks to approximate the Bayes-optimal classifier for the given linear metric using an estimate $\widehat{\boldeta}: \X \> \Delta_n$ of the conditional-class probability distribution $\eta_i(X) = \P(Y=i|X)$. 
%

Specifically, for a flattened loss matrix $\L \in \R_+^{d}$, where $L_{n(i-1)+j}$ is the cost of predicting class $j$ when the true class is $i$, 
we have from Proposition \ref{prop:loss-opt} that the Bayes-optimal classifier is given by
 $h^*(x) = \argmin^*_{j\in[n]} \sum_{i=1}^n \eta_i(x) L_{n(i-1)+j}$. 
The plug-in based LMO  outlined in Algorithm \ref{alg:plug-in}
approximates this classifier
with the class probability model $\widehat{\boldeta}$. 
The classifier and confusion matrix returned by the algorithm satisfy the LMO approximation properties laid out in Definition \ref{defn:lmo}:
\begin{thm}[Regret bound for plug-in  LMO]
\label{thm:plug-in}
Fix $\delta \in (0, 1)$. Then with probability $\geq 1 - \delta$ over draw of sample $S \sim D^N$, for any  loss matrix $\L \in \R^d$, the classifier and confusion matrix $(\widehat{g}, \widehat{\boldsymbol{\Gamma}})$ returned by Algorithm \ref{alg:plug-in} satisfies:
\begin{equation*}
\langle \L, \C[\widehat{g}]\rangle \,\leq\, \min_{h:\X\>\Delta_n}\langle \L, \C[h] \rangle + \|\L\|_\infty \E_X\big[\big\|\widehat{\boldeta}(X) \,-\, \boldeta(X)\big\|_1\big];
\vspace{-5pt}
\label{eq:lmo1}
\end{equation*}
\begin{equation*}
\|\C[\widehat{g}] \,-\, \hat{\bGamma}\|_\infty \,\leq\, \cO\bigg(\sqrt{\displaystyle\frac{d\log(n)\log(N) + \log(d/\delta)}{N}}\bigg).
\label{eq:lmo2}
\end{equation*}
\end{thm}
\begin{proof}
See Appendix \ref{app:proof-plugin}
\end{proof}

\begin{figure}[t]
\begin{algorithm}[H]
\caption{Plug-in Based LMO}\label{alg:plug-in}
\begin{algorithmic}[1]
\STATE \textbf{Input:} Loss matrix $\L \in \R_+^d$, Class prob.\ model $\widehat{\boldeta}: \X \> \Delta_n$, 
$S = \{(x_1,y_1), \ldots, (x_N,y_N)\}$
\STATE Construct classifier $\widehat{g}(x) \,=\, \argmin^*_{j\in[n]} \sum_{i=1}^n \widehat{\eta}_i(x) L_{n(i-1) + j}$
\STATE $\widehat{\bGamma} =\vec\big(\widehat{\C}^S[\widehat{g}]\big)$
\STATE \textbf{Output:} $\widehat{g}$, $\widehat{\bGamma}$
\end{algorithmic}
\end{algorithm}
\vspace{-12pt}
\end{figure}

\subsection{Consistency of Proposed Algorithms with Plug-in LMO}
Theorem \ref{thm:plug-in} tells us that the quality of the classifier $\widehat{g}$ returned by the plug-in based LMO depends on the  estimation error  $\E_X\big[\big\|\widehat{\boldeta}(X) \,-\, \boldeta(X)\big\|_1\big]$, which measures the gap between the class probability model $\widehat{\boldeta}$ and the true conditional class probabilities $\boldeta$.
By combining this result with  Theorem \ref{thm:FW-unc}--\ref{thm:ellipsoid-con}, we can show that the algorithms described in Sections \ref{sec:unconstrained} and \ref{sec:constrained}, when used with the plug-in based LMO, are statistically consistent. 
For the sake of brevity, we present the consistency analysis for the GDA algorithm and its constrained counter-part alone. The analysis for the other algorithms follow identical steps.

Let $S_1$ and $S_2$ be equal splits of the training sample $S$, and suppose we provide $S_1$ to the outer optimization methods in 
Algorithms \ref{alg:GDA} and \ref{alg:GDA-con}
and $S_2$ to the inner LMO implemented using the plug-in method in Algorithm \ref{alg:plug-in}. We then have:
\begin{cor}[Regret bound for GDA algorithm]
Let $\psi: [0,1]^d\>[0,1]$ be convex and $L$-Lipschitz w.r.t.\ the $\ell_2$-norm. Let the LMO $\Omega$ in Algorithm \ref{alg:GDA} be a plug-in based LMO (as in Algorithm \ref{alg:plug-in}) with a CPE argument $\widehat\boldeta$. 
Let $\bar{h}$ be a classifier returned by Algorithm \ref{alg:GDA} when run for $T$ iterations with
the parameter settings in Theorem \ref{thm:gda-unc}. Then with probability $\geq 1 - \delta$ over draw of $S \sim D^N$, after $T = \cO(N)$ iterations: 
\[
\psi(\C[\bar{h}]) \,\leq\, \min_{\C \in \cC}\,\psi(\C) \,+\,
{\cO}\left(\E_X[\|\widehat\boldeta(X) - \boldeta(X)\|_1] + \sqrt{d}\sqrt{\frac{d\log(n)\log(N) + \log(d/\delta)}{N}}\right).
\]
\end{cor}

\begin{cor}[Regret bound for ConGDA algorithm]
Let $\psi: [0,1]^d\>[0,1]$ and $\phi_1,\ldots,\phi_K: [0,1]^d\>[-1,1]$ be convex and $L$-Lipschitz. Let the LMO $\Omega$ in Algorithm \ref{alg:GDA} be a plug-in based LMO (as in Algorithm \ref{alg:plug-in}) with a CPE argument $\widehat\boldeta$. 
Let $\bar{h}$ be a classifier returned by Algorithm \ref{alg:GDA-con} when run for $T$ iterations with the parameter settings in Theorem \ref{thm:gda-con}. 
 Then with probability $\geq 1 - \delta$ over draw of $S \sim D^N$, after $T = \cO(KN)$ iterations:
\[
\psi(\C[\bar{h}]) \,\leq\, \min_{\C \in \cC}\,\psi(\C) \,+\,\cO\left(\E_X[\|\widehat\boldeta(X) - \boldeta(X)\|_1] + \sqrt{d}\sqrt{\frac{d\log(n)\log(N) + \log(d/\delta)}{N}}\right);
\vspace{-5pt}
\]
\[
\phi_k(\C[\bar{h}]) \,\leq\, \cO\left(\E_X[\|\widehat\boldeta(X) - \boldeta(X)\|_1] + \sqrt{d}\sqrt{\frac{d\log(n)\log(N) + \log(d/\delta)}{N}}\right),~\forall k \in [K].
\]
\end{cor}

When the class probability model
 $\widehat{\boldeta}$ used by the LMO is learned by an algorithm that guarantees $\E_X[\|\widehat{\boldeta}(X) - \boldeta(X)\|_1] \> 0$ as $N\>\infty$, then Algorithm \ref{alg:GDA} is statistically consistent for the unconstrained problem in \eqref{eq:opt-unconstrained}, and Algorithm \ref{alg:GDA-con} is statistically consistent for the constrained problem in \eqref{eq:opt-constrained}. 
 The property that the learned class probability estimation error goes to zero in the large sample limit is  true for
 any algorithm that minimizes a strictly proper composite multiclass loss (e.g.\ the standard cross-entropy loss) over a suitably large function class \citep{Vernet+11}.
 
 While our consistency results require that the samples used by the outer optimization method and the inner LMO to be drawn independently, this may be inconvenient in real-world applications where data is scarce and limited. In practice, we find that using the same sample for both the outer and inner routines does not hurt performance, and  this is the approach we adopt in our experiments.
 
A practical advantage of the plug-in based LMO is that one can pre-train the class probability model $\widehat{\boldeta}$ and re-use the same model each time the LMO is invoked.
In practice, there are other off-the-shelf algorithms that one can use to implement the LMO, such as cost-weighted decision trees  \citep{ting2002instance} and those based on optimizing a cost-weighted surrogate loss (e.g. \citet{Lee+04}), which require training a new classifier for each given loss vector $\L$. While a  majority of our experiments will use a plug-in based LMO, we also explore the use of cost-weighted surrogate losses for implementing the LMO.



%% file: 7_fairness.tex
\section{Extension to Fairness Metrics and Other Refinements}
\label{sec:fairness}
To keep the exposition concise, we have so far focused on metrics defined by a function of the overall confusion matrix $\C[h]$. 
We now discuss how the algorithms in Sections \ref{sec:unconstrained} and \ref{sec:constrained} can be extended to handle 
the group-based fairness metrics described in  Section \ref{subsec:constraints}, which are defined in terms of group-specific confusion matrices (see Definition \ref{defn:group-conf}). 

\subsection{Group-based Fairness Metrics}
In the fairness setup we consider, each instance $x \in \X$ is associated with a group $A(x) \in [m]$, and the objective and constraints are defined by functions of $m$ group-specific confusion matrices $\C^1[h], \ldots, \C^m[h]$. Note that even for binary problems where $n = 2$, the presence of multiple groups 
poses challenges in solving the resulting learning problems in \eqref{eq:opt-unconstrained} and \eqref{eq:opt-constrained}. For example, a na\"{i}ve approach one could take for binary labels is to construct a simple plug-in classifier for these problems that assigns a separate threshold for each group, but tuning $m$ thresholds via a brute-force search can quickly become infeasible when $m$ is large.

Our approach to solving the learning problems in \eqref{eq:opt-unconstrained} and \eqref{eq:opt-constrained} with group fairness metrics is to once again reformulate as an optimization problem over the  set of achievable group-specific confusion matrices, in this case, represented by vectors of dimension $d = mn^2$.
\begin{defn}[Achievable group-specific confusion matrices]
Define the set of \emph{achievable group-specific confusion matrices w.r.t.\ $D$} as:
\[
\cC^{[m]} = \big\{ \big[\vec(\C^0[h]), \ldots, \vec(\C^{m-1}[h])\big]\big|\, ~h:\X\>\Delta_n \big\}.
\]
\end{defn}
%
%

Algorithms \ref{alg:FW}--\ref{alg:bisection-con} can now be directly applied to solve the resulting optimization over $\cC^{[m]}$, at each iteration, assuming access to an oracle for approximately solving a linear minimization problem over $\cC^{[m]}$. 
This linear minimization sub-problem can again be solved using a  plug-in based LMO similar Algorithm \ref{alg:plug-in}. The details of the  plug-in variant  for the fairness setup are provided in Algorithm \ref{alg:plug-in-group}, where we denote the empirical group-specific confusion matrix for group $a$ from sample $S = \{(x_1,y_1),\ldots,(x_N,y_N)\}$ by:
\[
\vspace{-2pt}
\hat{C}^{a}_{ij}[h] = \frac{1}{N}\sum_{\ell=1}^N\1(y_\ell = i, h(x_\ell) = j, A(x_\ell) = a)
	\,.
\vspace{-2pt}
\]

\begin{figure}
\begin{algorithm}[H]
\caption{Plug-in Based LMO for Fairness Problems}\label{alg:plug-in-group}
\begin{algorithmic}[1]
\STATE \textbf{Input:} Loss matrix $\L \in \R_+^d$, Class prob.\ model $\widehat{\boldeta}: \X \> \Delta_n$, 
$S = \{(x_1,y_1),  \ldots, (x_N,y_N)\}$
\STATE ~~~~~~~~~~~~Group assignment $A: \X \> [m]$
\STATE Define $\sigma(x,i, j) = mn(A(x)-1) + n(i-1) + j$
\STATE Construct $\widehat{g}(x) \,=\, \argmin_{j\in[n]} \sum_{i=1}^n \widehat{\eta}_i(x) L_{\sigma(x,i, j)}$
 \STATE
 $\widehat{\bGamma} =\big[\vec\big(\widehat{\C}^{0}[\widehat{g}]\big),\ldots,\vec\big(\widehat{\C}^{m-1}[\widehat{g}]\big))\big]$
\STATE \textbf{Output:} $\widehat{g}$, $\widehat{{\bGamma}}$
\end{algorithmic}
\end{algorithm}
\vspace{-15pt}
\end{figure}

\subsection{Succinct Confusion Matrix Representations}
Before closing, we note that 
 for simplicity, we have allowed the $d$-dimensional vector representation of the confusion matrix to contain all $n^2$ entries (or all $mn^2$ for fairness metrics). In practice, we only need to take into account those entries of the confusion matrix performance measures and constraints we seek to optimize depend upon. For example, the G-mean metric in Example \ref{ex:g-mean} is defined on only the diagonal entries of the confusion matrix, and so the vector representation in this case needs to only contain the $n$ diagonal entries. 
%
%
In fact, for some metrics, it suffices to represent the confusion matrix using a small number of linear transformations. 
For example, the coverage metric described in Example \ref{ex:coverage} is defined on only the column sums of the confusion matrix, and hence the $d$-dimensional vector representation in this case only needs to contain the $n$ column sums.

More generally, we can work with  succinct vector representations given by linear transformations of the confusion matrices:
\begin{defn}[Generalized confusion vectors]
Define the set of \emph{(achievable) generalized confusion vectors w.r.t.\ $D$} as:
\[
\cC^{\gen} = \big\{ \big[\varphi_1(\C^0[h],\ldots,\C^{m-1}[h]), \ldots, \varphi_d(\C^{0}[h],\ldots,\C^{m-1}[h])\big]\big|\, ~h:\X\>\Delta_n \big\},
\]
where each $\varphi_k: [0,1]^{mn^2} \> \R_+$ is a \emph{linear map}.
\end{defn}
The set $\cC^{\gen}$ is convex. In the simplest case, we can have a linear map $\varphi_k$ of dimension $d=mn^2$, where each coordinate picks one entry from the $m$ confusion matrices. However, for most of the performance metrics described in Section \ref{subsec:performance-measures} and \ref{subsec:constraints}, it suffices to use a a small number of  $d<<mn^2$ linear transformations and we can translate the corresponding learning problems in \ref{eq:opt-unconstrained} and \ref{eq:opt-constrained} into equivalent optimization problems over  $\cC^{\gen}$. The iterative algorithms discussed in Sections \ref{sec:unconstrained} and \ref{sec:constrained} can then be applied to solve the resulting lower-dimensional optimization problem over $\C$, with the plug-in procedure in Algorithm \ref{alg:plug-in} straightforwardly adapted to solve the linear minimization over $\cC^{\gen}$ at each step. 

%% file: 8_experiments.tex
\section{Experiments}
\label{sec:experiments}
We present an experimental evaluation of the algorithms presented in Sections \ref{sec:unconstrained} and \ref{sec:constrained} on a variety of multi-class datasets and multi-group fair classification tasks. Broadly, we cover the following:
\begin{enumerate}[topsep=5pt,leftmargin=25pt,itemsep=2pt]
    \item We showcase on a synthetic dataset that our algorithms converge in the large sample limit to optimal (feasible) classifier (Section \ref{sec:expt-consistency}).
    \item We demonstrate that the proposed algorithms are competitive or better than the state-of-the-art algorithms for the real-world tasks we consider (Sections \ref{sec:expts-unc}--\ref{sec:experiments-cons}).
    \item We provide practical guidance on which algorithm is better suited for a given application, and investigate two different choices for the LMO (Sections \ref{sec:expts-practical}--\ref{sec:expts-LMO}).
    \item We illustrate with image classification case-studies how  our algorithms can be applied to tackle class-imbalance and label noise (Section \ref{sec:expts-cifar}).
\end{enumerate}

A summary of the datasets we use is provided in Tables \ref{tab:multi-datasets} and \ref{tab:fair-datasets}, along with the model architecture we use in each case. 
The  details of the data pre-processing are provided in Appendix \ref{app:expts}.
With the exception of the CIFAR datasets, which comes with standard train-test splits, we split all other datasets into 2/3-rd for training and 1/3-rd for testing, and repeat our experiments over multiple such random splits. All our methods were implemented in Python using PyTorch and Scikit-learn.\footnote{Code available at: \url{https://github.com/shivtavker/constrained-classification}}

\begin{table}[t]
    \centering
    \caption{Multi-class datasets used in our experiments}
    \label{tab:multi-datasets}
    \begin{footnotesize}
    \begin{tabular}{ccccccc}
        \hline
        Dataset   & \#Classes & \#Train & \#Test & \#Features & $\frac{\min_y \pi_y}{\max_y \pi_y}$ & Model\\
        \hline
             Abalone   & 12 & 2923 & 1254 & 8 & 0.149  & Linear\\
             PageBlock & 5  & 3831 & 1642 & 10 & 0.0057 & Linear\\
             MACHO    & 8  & 4241 & 1818 & 64 & 0.0148 & Linear \\
             Sat-Image & 6  & 4504 & 1931 & 36 & 0.408  & Linear \\
             CovType & 7  & 406708 & 174304 & 14 & 0.0097 & Linear \\ 
             \hline
             CIFAR-10-Flip & 10 & 27500 & 5500 & 32 $\times$ 32 & 0.1 & ResNet-50\\
             CIFAR-55 & 55 & 50000 & 10000 & 32 $\times$ 32 & 0.1 & ResNet-50\\
        \hline
    \end{tabular}
    \end{footnotesize}
\end{table}

\begin{table}[t]
    \centering
    \caption{Multi-group fairness datasets with binary labels used in our experiments.}
    \label{tab:fair-datasets}
    \begin{footnotesize}
    \begin{tabular}{ccccccc}
        \hline
        Dataset & \#Train & \#Test & \#Features & Protected Attr.\ & Prot. Group Frac.
        & Model
        \\
        \hline
        Communities \& Crime & 1395 & 599 & 132 & Race (binary) & 0.49 & Linear 
        \\
        COMPAS & 4320 & 1852 & 32 & Gender & 0.19 & Linear 
        \\
        Law School & 14558 & 6240 & 16 & Race (binary) & 0.06 & Linear 
        \\
        Default & 21000 & 9000 & 23 & Gender & 0.40 & Linear
        \\
        Adult & 34189 & 14653 & 123 & Gender & 0.10 & Linear
        \\
        \hline
    \end{tabular}
    \end{footnotesize}
\end{table}

\subsection{Baselines}
In a majority of the experiments, our algorithms will use the plug-in method in Algorithm \ref{alg:plug-in} for the inner linear minimization oracle, with a  logistic regression model $\hat{\seta}:\X\>\Delta_n$ used to estimate the conditional-class probabilities. As baselines, we compare with methods for minimizing the standard 0-1 loss and the balanced 0-1 loss, both of which are simpler alternatives to the metrics we consider, and the state-of-the-art approach for directly optimizing with complex metrics and constraints. 
\begin{enumerate}[label=(\roman*),itemsep=0pt,topsep=5pt,leftmargin=16pt]    
\item A plug-in classifier that  predicts the class with the maximum class probability, i.e. $\argmax_{i} \hat{\eta}_i(x)$; this method is consistent for the 0-1 loss.
    \item A plug-in classifier that weighs the class probabilities by the inverse class priors, and predicts the class with the highest weighted probability $\argmax_{i} \frac{1}{\hat{\pi}_i}\hat{\eta}_i(x)$, where $\hat{\pi}_i$ is an estimate of the  prior for class $i$; this method is consistent for the balanced 0-1 loss.
    \item The approach of \citet{Narasimhan+19_generalized} for optimizing with complex performance metrics and constraints, available as a part of the TensorFlow Constrained Optimization (TFCO) library.\footnote{\url{https://github.com/google-research/tensorflow_constrained_optimization}}
\end{enumerate}

TFCO  uses an optimization procedure similar to the GDA method in Algorithm \ref{alg:GDA}, but instead of fitting a plug-in classifier to a pre-trained class probability model, performs online updates on surrogate approximations. Therefore one key difference between our use of plug-in classifiers and the approach taken by  TFCO is that the latter is an in-training method which trains a classifier from scratch. Unlike our proposal, it does not come with  consistency guarantees. 
It is worth noting that TFCO can be seen as a strict generalization to previous surrogate-based methods for complex evaluation metrics \citep{Narasimhan+15b, Kar+16}. 

All the plug-in based methods use the same class probability estimator $\hat{\seta}$. We employ the same architecture as $\hat{\seta}$ for the model trained by TFCO.

We do not include the previous $\text{SVM}^{\text{perf}}$ method \citep{Joachims05} as a baseline because it has a running time that is exponential in the number of classes, and as shown in the previous conference version of this paper, can be prohibitively expensive to run even for a moderate number of classes \citep{Narasimhan+15}. Moreover, this method was proposed for unconstrained problems, and does not explicitly allow for imposing constraints on metrics.

\subsection{Post-processing}
Recall that the Frank-Wolfe,  GDA and ellipsoid algorithms that we propose for convex metrics return classifiers that \emph{randomize} over $T$ plug-in classifiers.  
When implementing their constrained counterparts, we additional apply ``pruning'' step 
to the returned randomized classifier,  which re-computes the convex combination of the $T$ iterates $\C^1, \ldots, \C^T$ so that the constraints are exactly satisfied (if such a solution exists). 
Specifically, the final classifier is given by  $\frac{1}{T}\sum_{t=1}^T \alpha_*^t\,g^t$, where $\balpha_* \in \underset{\balpha \in \Delta_T:\, \sum_t \alpha^t\bphi( \C[g^t]) \leq \0}{\argmin} \sum_{t=1}^T\alpha^t\psi\left(\C[g^t]\right)$. Note that  the objective here is an approximation to the true objective 
$\psi\left(\sum_{t=1}^T\alpha^t\C[g^t]\right)$, with the former upper bounding the latter when $\psi$ is convex. This approximation to the objective allows us to compute the optimal coefficients $\balpha_*$ by solving a simple linear program. 
The use of a post-processing pruning step is prescribed by the TFCO library \citep{Cotter+19b, Narasimhan+19_generalized}, and is also applied to the classifier returned by the TFCO baseline. 
 In Appendix \ref{app:hparam}, we provide other details such as how we choose the hyper-parameters for our algorithms and the baselines. 

 We additionally note that the H-mean, Q-mean and G-mean metrics we consider in our experiments can be written as functions of normalized diagonal entries of the confusion matrix: $\frac{C_{ii}}{\pi_i}, \forall i \in [n]$ (see Table \ref{tab:perf-measures}). For these metrics, we formulate \ref{eq:opt-unconstrained} and \ref{eq:opt-constrained} as optimization problems over normalized confusion diagonal entries $ \left[\frac{C_{11}}{\pi_1}, \ldots, \frac{C_{nn}}{\pi_n}\right]^\top \in [0,1]^n$, which is of 
 lower-dimensional than the space of full confusion matrices. 
 This requires a small modification to the GDA and ellipsoid algorithms, where the slack variables $\bxi$ will have to be constrained to be in $[0,1]^n$ instead of in the simplex $\Delta_{n^2}$.
Similarly, when the fairness constraints in Table \ref{tab:perf-measures} are enforced on binary-labeled problems, we can write the objective and constraints as functions of normalized diagonal confusion entries of group-specific confusion matrices, resulting in an optimization over vectors in $[0,1]^{2m}$.




\begin{figure}[t]
\centering
\begin{subfigure}[b]{0.32\linewidth}
\includegraphics[width=0.95\linewidth]{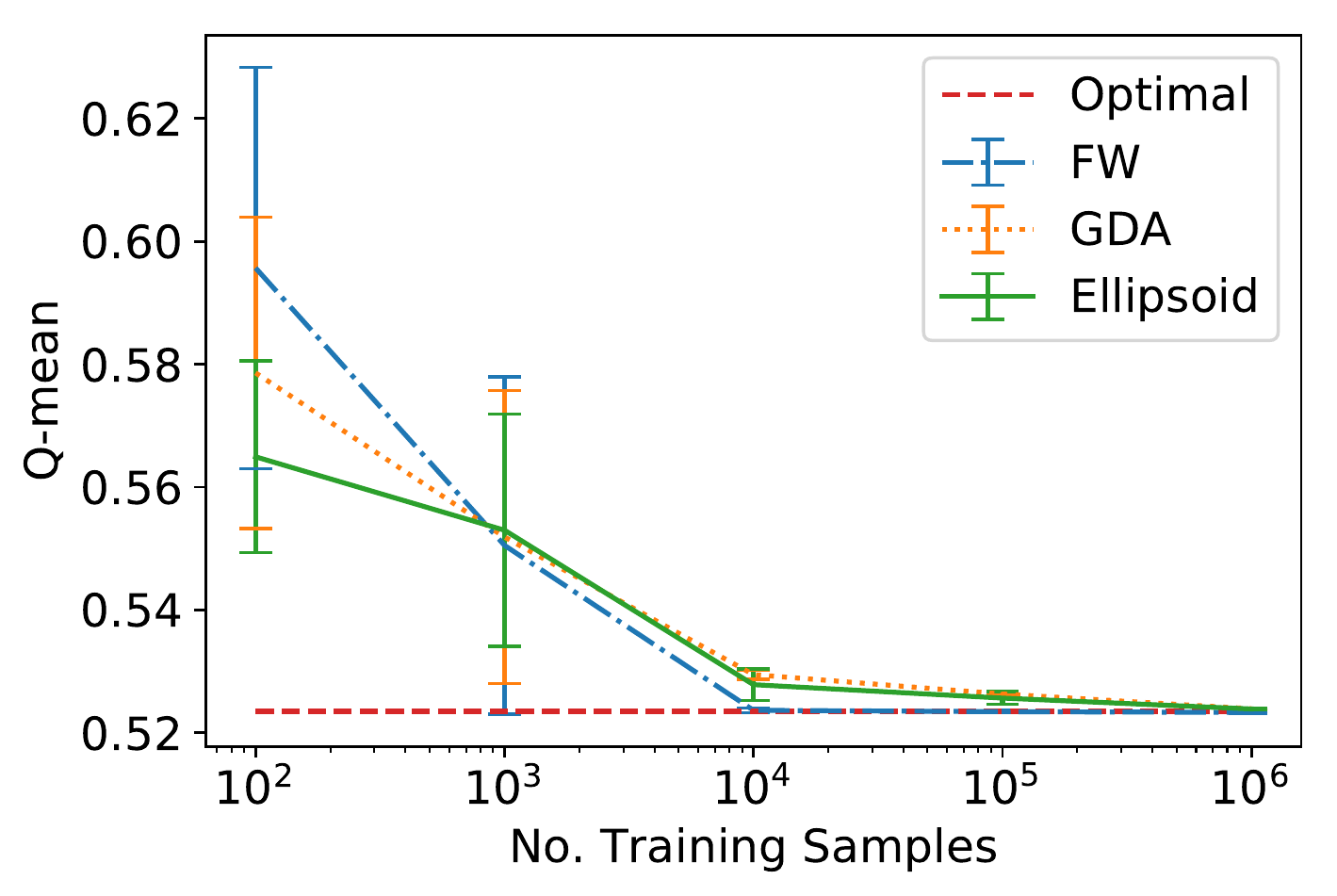}
\caption{Unconstrained}
\label{fig:q-mean-unconstrained}
\end{subfigure}
\begin{subfigure}[b]{0.64\linewidth}
\centering
\includegraphics[width=0.49\linewidth]{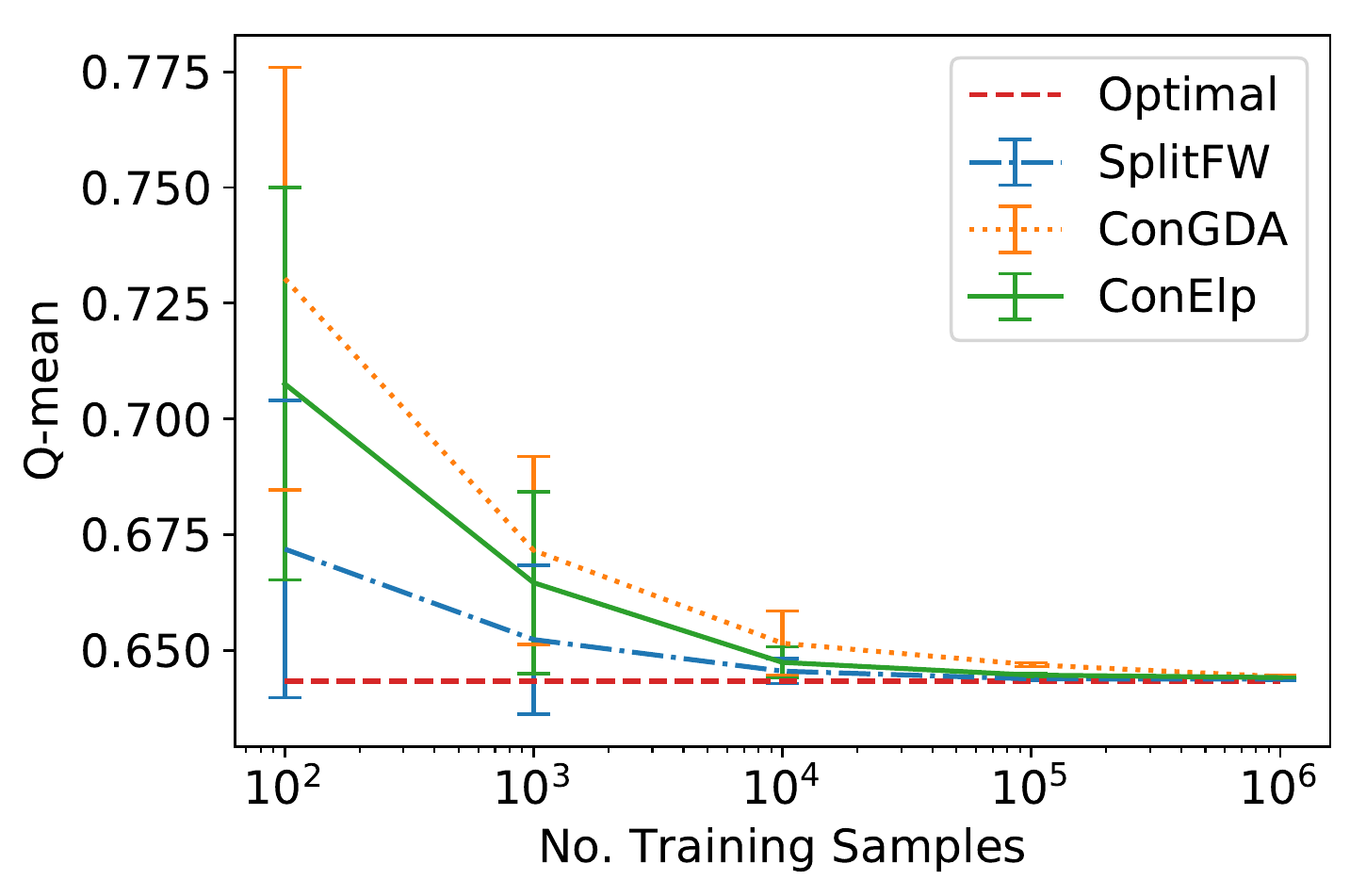}
\includegraphics[width=0.48\linewidth]{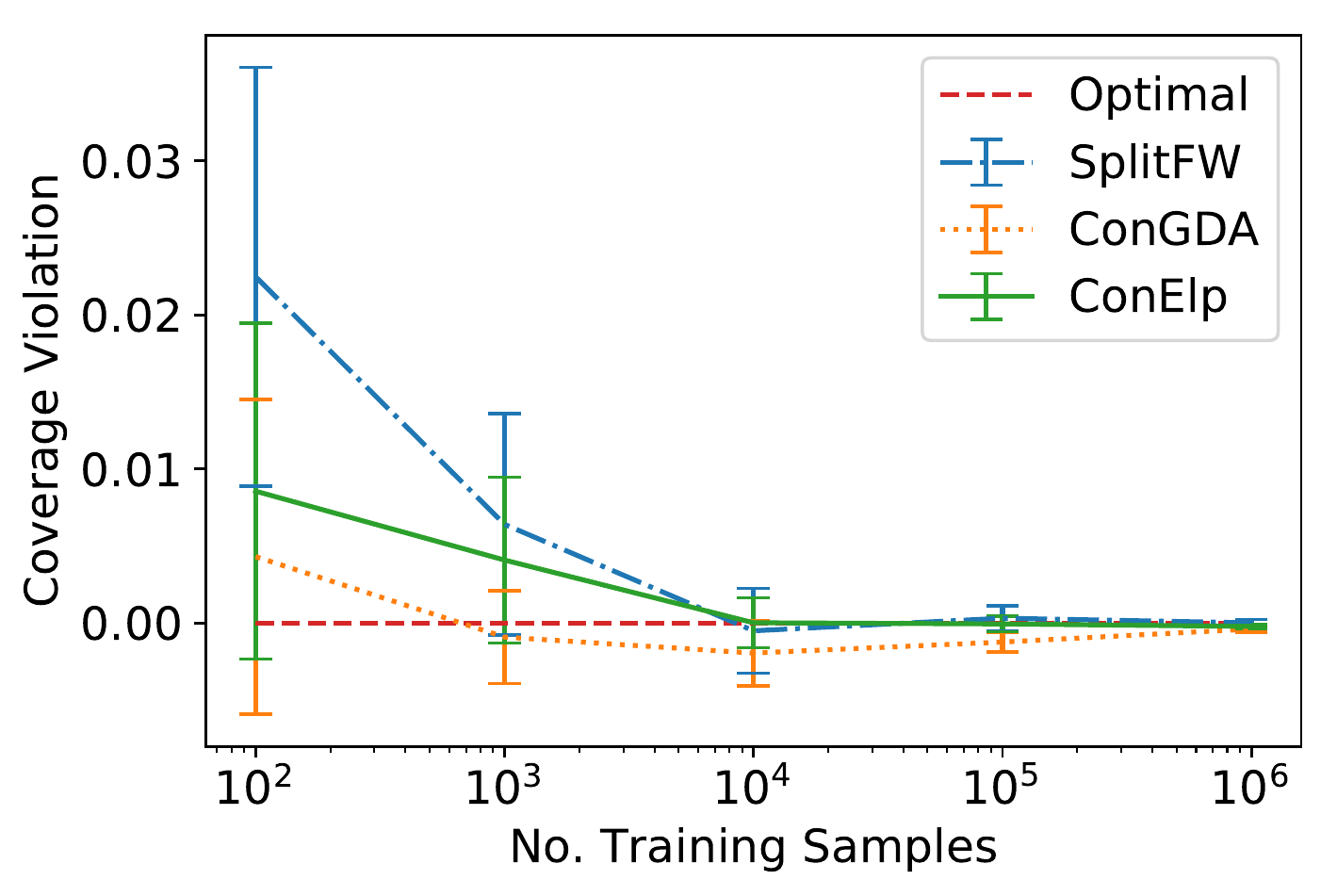}
\caption{Constrained}
\label{fig:q-mean-cov-constraint}
\end{subfigure}
\caption{Convergence of the proposed algorithms on synthetic data to (a) the Bayes optimal classifier  for of the Q-mean loss, and (b) the optimal-feasible classifier for the task of minimizing Q-mean loss subject to a coverage constraint, with the Q-mean loss shown on the left and the coverage constraint violation  $\max_{i\in[3]}\big|\sum_j C_{ji} - \pi_i\big| - 0.01$ shown on the right. The results are reported on the test set, and averaged  over training with 5 random draws of the dataset.}
\vspace{-5pt}
\label{fig:q-mean-unconstraint}
\end{figure}


\subsection{Convergence to the Optimal Classifier}
\label{sec:expt-consistency}
In our first set of experiments, we test the consistency behavior of the algorithms on a synthetic data set for which the Bayes optimal performance could be calculated. We use a 3-class synthetic data set with instances in $\X=\R^2$ generated as follows: examples are chosen from class 1 with probability 0.85, from class 2 with probability 0.1 and from class 3 with probability 0.05; instances in the three classes are then drawn from multivariate Gaussian distributions with means $(1,1)^\top$, $(0,0)^\top$, and $(-1,-1)^\top$ respectively, and with the same covariance matrix
$\big[ \begin{smallmatrix} 5 & 1 \\ 1 & 5 \end{smallmatrix} \big]$. 
The conditional-class probability function $\seta:\R^2\>\Delta_3$ for this distribution is a softmax of linear functions, and can be computed in closed-form. 

We first consider the unconstrained task of optimizing the Q-mean loss in Table \ref{tab:perf-measures}, given by $\psi^{\text{QM}}(\C) = \Big(\frac{1}{n}\sum_{i}\left(1-\frac{C_{ii}}{\sum_{j} C_{ij}}\right)^2\Big)^{1/2}$. 
Note that this performance metric is a smooth convex function of $\C$, and can be optimized with any one of the proposed Frank-Wolfe, GDA or ellipsoid methods (Algorithms \ref{alg:FW}--\ref{alg:ellipsoid}). Because the metric and the distribution satisfy the conditions of Proposition \ref{prop:opt-classifier-monotonic}, and  the Bayes-optimal classifier is of the form $h^*(x) = \argmax^*_{i\in[3]} w^*_{i}\, \eta_i(x)$, for some distribution-dependent coefficients $w^*_i \in \R$,. To compute the Bayes-optimal classifier, we run a brute-force grid search for $w^*_i$. 

Our algorithms use the plug-in method in Algorithm \ref{alg:plug-in} for the LMO subroutine. Specifically, they fit a linear logistic regression model $\hat{\seta}:\R^2\>\Delta_3$ to the training set, and iteratively learn a randomized combination of classifiers of the form $h(x) = \argmax^*_{i\in[3]} w_{i} \,\hat{\eta}_i(x)$.
In Figure~\ref{fig:q-mean-unconstrained}, we plot the Q-mean loss  for the classifier learned by the proposed algorithms, evaluated on a test set of $10^6$ examples, for different sizes of the training sample. 
In each case, we average the results over 5 random draws of the training sample.
As seen, all three methods converge to the performance of the Bayes-optimal classifier.

We next consider the task of optimizing the Q-mean loss subject to a coverage constraint, requiring the proportion of predictions made for class $i$ to be (approximately) equal to the class prior $\pi_i$. Specifically, we constraint the max coverage deviation, $\max_{i\in[3]}\big|\sum_j C_{ji} - \pi_i\big|$ to be at most 0.01. This is a constrained problem with a convex smooth objective  and a convex constraint in $\C$, and can be solved using the constrained counter-parts to the Frank-Wolfe, GDA and ellipsoid methods (Algorithm \ref{alg:FW-con}--\ref{alg:ellipsoid-con}). 
Following \cite{yang2020fairness},
we have that the optimal-feasible classifier for this problem is a randomized classifier of two classifiers $h^{1,*}(x) = \argmax^*_{i\in[3]} w^{1,*}_{i} \eta_i(x)$ and $h^{2,*}(x) = \argmax^*_{i\in[3]} w^{2,*}_{i} \eta_i(x)$, for distribution-dependent coefficients $w^{1,*}_{i}$ and $w^{2,*}_{i}$.\footnote{Proposition \ref{prop:opt-classifier-constrained} tells us that the support of the Bayes-optimal classifier randomizes over as many as $d+1$ deterministic classifiers. For the 3-class distribution we consider, $\boldeta(X)$ satisfies additional continuity conditions,
under which the optimal classifier can be shown to be a randomized combination of at most \emph{two} deterministic classifier \citep{wang2019consistent, yang2020fairness}.}
We compute these coefficients and the optimal randomized combination via a brute-force grid search.
Figure \ref{fig:q-mean-cov-constraint} plots the Q-mean loss and the constraint violation for the three algorithms. All of them can be seen to converge to the Q-mean of the optimal-feasible classifier and to zero constraint violation.

\subsection{Performance on Unconstrained Problems}
\label{sec:expts-unc}
We next compare the proposed algorithms for unconstrained problems on five benchmark multiclass datasets: (i) Abalone, (ii) PageBlock, (iii) CovType, (iv) SatImage and (v) MACHO. The first four were obtained from the UCI Machine Learning repository \citep{uci}. The fifth dataset pertains to the task of classifying celestial
objects from the Massive Compact Halo Object (MACHO)
catalog using photometric time series data \citep{alcock2000macho, kim2011quasi}. Each celestial object is described by measurements from 
6059 light curves, and is categorized either as one of seven celestial objects or as a miscellaneous category. 

We consider two performance metrics from Table \ref{tab:perf-measures}: (i) the H-mean metric 
$\psi^{\HM}(\C) = 1 \,-\, n\left({\sum_{i} \frac{\sum_{j} C_{ij}}{C_{ii}}}\right)^{-1}$
and (ii) the micro F-measure $\psi^{{\micro F_1}}(\C) = 1 \,-\, \frac{2\sum_{i\ne k} C_{ii}}{2 - \sum_{i}C_{ki} - \sum_{i} C_{ik}}$, where $k \in [n]$ is a designated default  class. The first metric is convex in $\C$, for which we compare the performances of the Frank-Wolfe, GDA, and ellipsoid algorithms (Algorithms \ref{alg:FW}--\ref{alg:ellipsoid}); the second metric is ratio-of-linear in $\C$, and for this, we apply the bisection algorithm (Algorithm \ref{alg:bisection}).
Our algorithms use a plug-in based LMO  with a linear logistic regression model used to estimate the conditional-class probabilities. We compare our methods with the 0-1 plug-in, balanced plug-in and TFCO baselines.



\begin{table}[t]
    \centering
    \caption{Unconstrained optimization of the (convex) H-mean loss. \textit{Lower} values are \textit{better}. The results are averaged over 10 random train-test splits.}
    \label{tab:hmean-uncon}
    \begin{footnotesize}
    \begin{tabular}{c|ccc|ccc}
        \hline
        \textbf{Dataset} & \textbf{Plugin [$\zo$]} & \textbf{Plugin (bal)} & \textbf{TFCO} & \textbf{FW} & \textbf{GDA} & \textbf{Ellipsoid} \
        \\
        \hline
        Abalone   & $1.0 \pm 0.0$ & $0.890 \pm 0.038$ & $0.824 \pm 0.018$ & $\textbf{0.816} \pm \textbf{0.020}$ & $0.818 \pm 0.017$ & $0.817 \pm 0.019$ \\
        Pgblk  & $0.416 \pm 0.128$ & $0.130 \pm 0.034$ & $0.200 \pm 0.023$ & $0.120 \pm 0.028$ & $0.130 \pm 0.04$ & $\textbf{0.110} \pm \textbf{0.025}$ \\
        MACHO     & $0.210 \pm 0.043$ & $0.130 \pm 0.015$ & $0.143 \pm 0.019$ & $0.124 \pm 0.017$ & $\textbf{0.124} \pm \textbf{0.015}$ & $0.125 \pm 0.017$ \\
        SatImage     & $0.279 \pm 0.01$ & $0.173 \pm 0.008$ & $\textbf{0.170} \pm \textbf{0.006}$ & $0.171 \pm 0.007$ & $0.173 \pm 0.008$ & $\textbf{0.170} \pm \textbf{0.006}$ \\
        CovType     & $1.0 \pm 0.0$ & $0.507 \pm 0.001$ & $0.469 \pm 0.001$ & $0.463 \pm 0.001$ & $0.463 \pm 0.001$ & $ \textbf{0.461} \pm \textbf{0.001}$ \\
        \hline
    \end{tabular}
    \end{footnotesize}
\vspace{-8pt}
\end{table}

\begin{table}[t]
    \centering
    \caption{Unconstrained optimization of the (ratio-of-linear) micro $F_1$ loss. \textit{Lower} values are \textit{better}. The results are averaged over 10 random train-test splits.}
    \label{tab:fmeasure-uncon}
    \begin{footnotesize}
    \begin{tabular}{c|ccc|c}
        \hline
        \textbf{Datasets} & \textbf{Plugin [$\zo$]} & \textbf{Plugin (bal)} & \textbf{TFCO} & \textbf{Bisection} \\
        \hline
        Abalone      & $0.713 \pm 0.006$  & $0.760 \pm 0.004$ 
        & $0.728 \pm 0.012$ & $\mathbf{0.693 \pm 0.006}$ \\
        Pgblk  & $0.218 \pm 0.012$ & $0.441 \pm 0.033$ 
        & $0.216 \pm 0.018$ & $\mathbf{0.211 \pm 0.016}$ \\
        MACHO    & $\textbf{0.089} \pm \textbf{0.005}$ & $0.106 \pm 0.007$ 
        & $0.110 \pm 0.005$
        & $\mathbf{0.089 \pm 0.005}$ \\
        SatImage    & $\textbf{0.180} \pm \textbf{0.005}$ & $0.185 \pm 0.007$ 
        & $0.234 \pm 0.003$
        & $\mathbf{0.180 \pm 0.005}$ \\
        CovType    & $0.548 \pm 0.001$ & $0.625 \pm 0.003$ 
        & $0.486 \pm 0.001$
        & $\mathbf{0.403 \pm 0.001}$ \\
        \hline
    \end{tabular}
    \end{footnotesize}
\vspace{-8pt}
\end{table}

The results of optimizing the two metrics are shown in Tables \ref{tab:hmean-uncon} and \ref{tab:fmeasure-uncon} respectively. As expected both the 0-1 and balanced plug-in classifiers are often seen to perform poorly on the H-mean and micro $F_1$ metrics. For example, on the Abalone and CovType dataset, the plug-in (0-1) yields a H-mean loss of 1 as it achieves high accuracies on the higher-frequency classes at the cost of yielding zero accuracy on one or more minority classes. In contrast, the proposed algorithms provide equitable performance across all classes, and are able to yield a much lower H-mean score. This demonstrates the advantage of using algorithms that directly optimize for the metric of interest. In most experiments, TFCO is seen to be a competitive baseline: with the H-mean metric, the proposed algorithms yields significantly better performance over this method on two of the five datasets, and 
with the micro $F_1$ metric it yields significantly better performance than TFCO on four of the five datasets . We stress that our algorithms are able to provide these gains despite TFCO using a more flexible class of randomized classifiers. In fact, with the MACHO dataset, TFCO can be seen to perform worse than our method as a result of over-fitting to the training set.

We also note that all the algorithms compared beat  a trivial classifier that predicts all classes with equal probability (see Appendix \ref{app:trivial-classifier} for the performance of the trivial classifier on the different datasets with different metrics).



\subsection{Performance on Constrained Problems}
\label{sec:experiments-cons}
\begin{figure}[t]
\centering
\begin{subfigure}[b]{0.48\linewidth}
\centering
\includegraphics[width=0.48\linewidth]{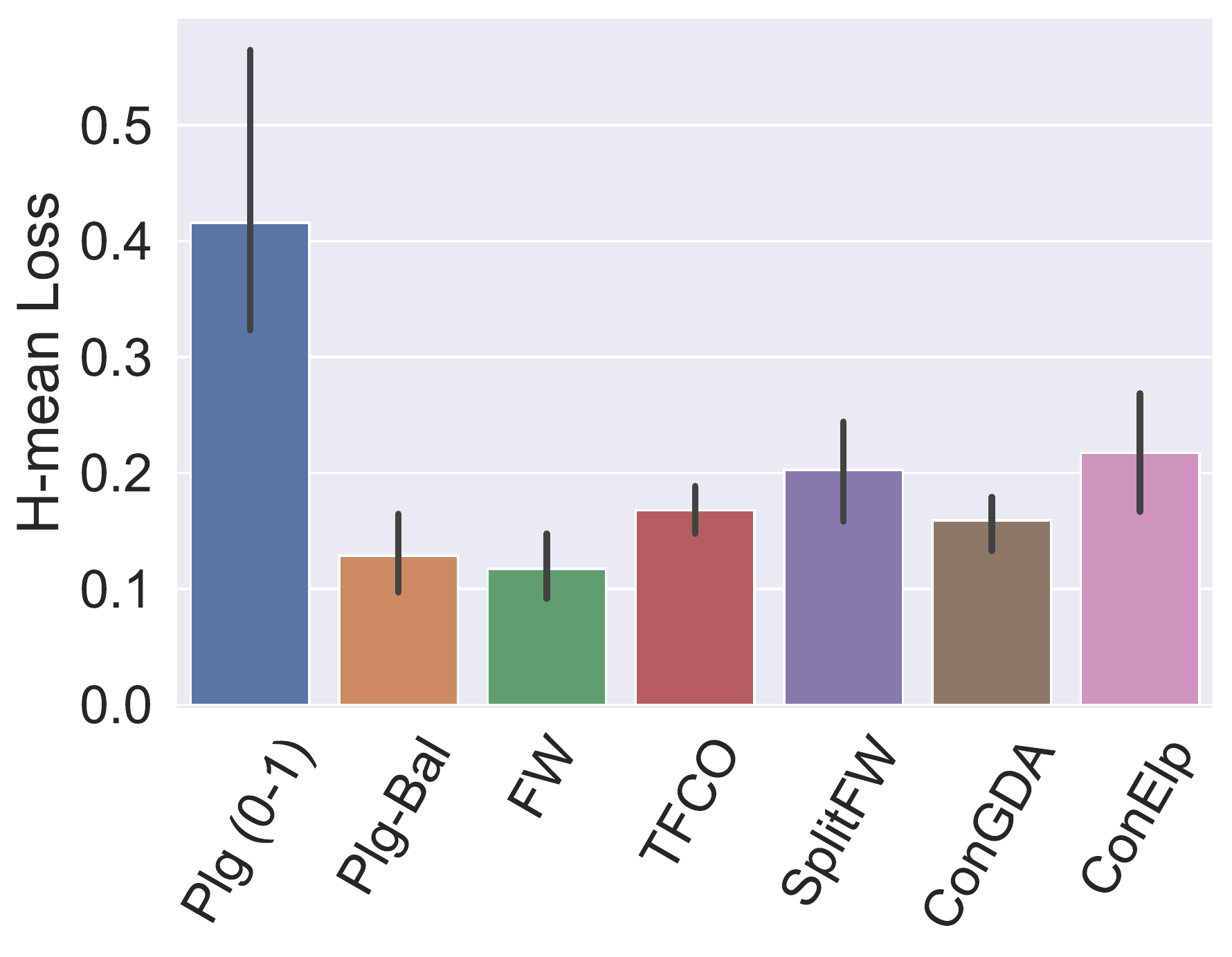}
\includegraphics[width=0.48\linewidth]{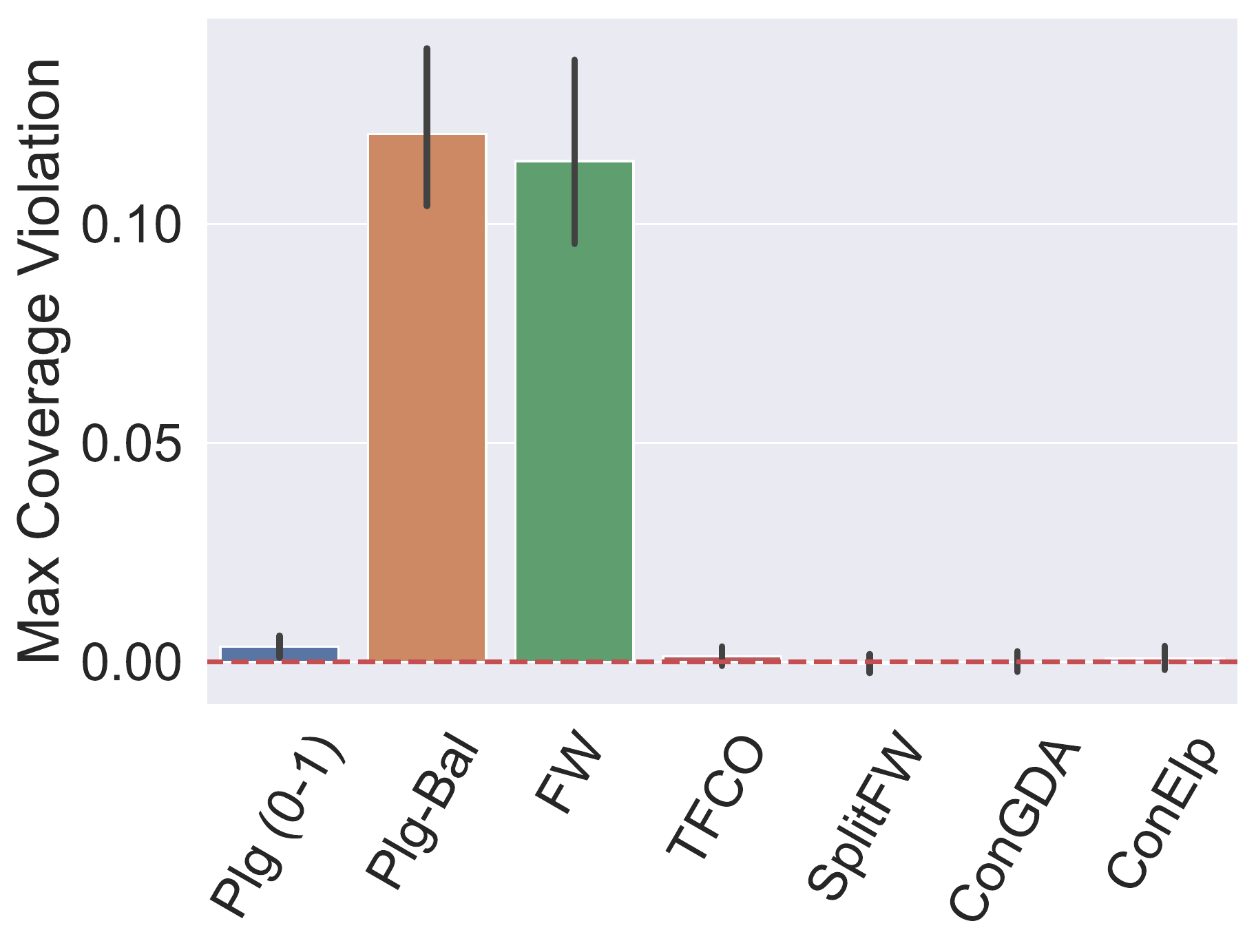}
\label{fig:pageblocks}
\caption{PageBlock}
\end{subfigure}
\begin{subfigure}[b]{0.48\linewidth}
\centering
\includegraphics[width=0.48\linewidth]{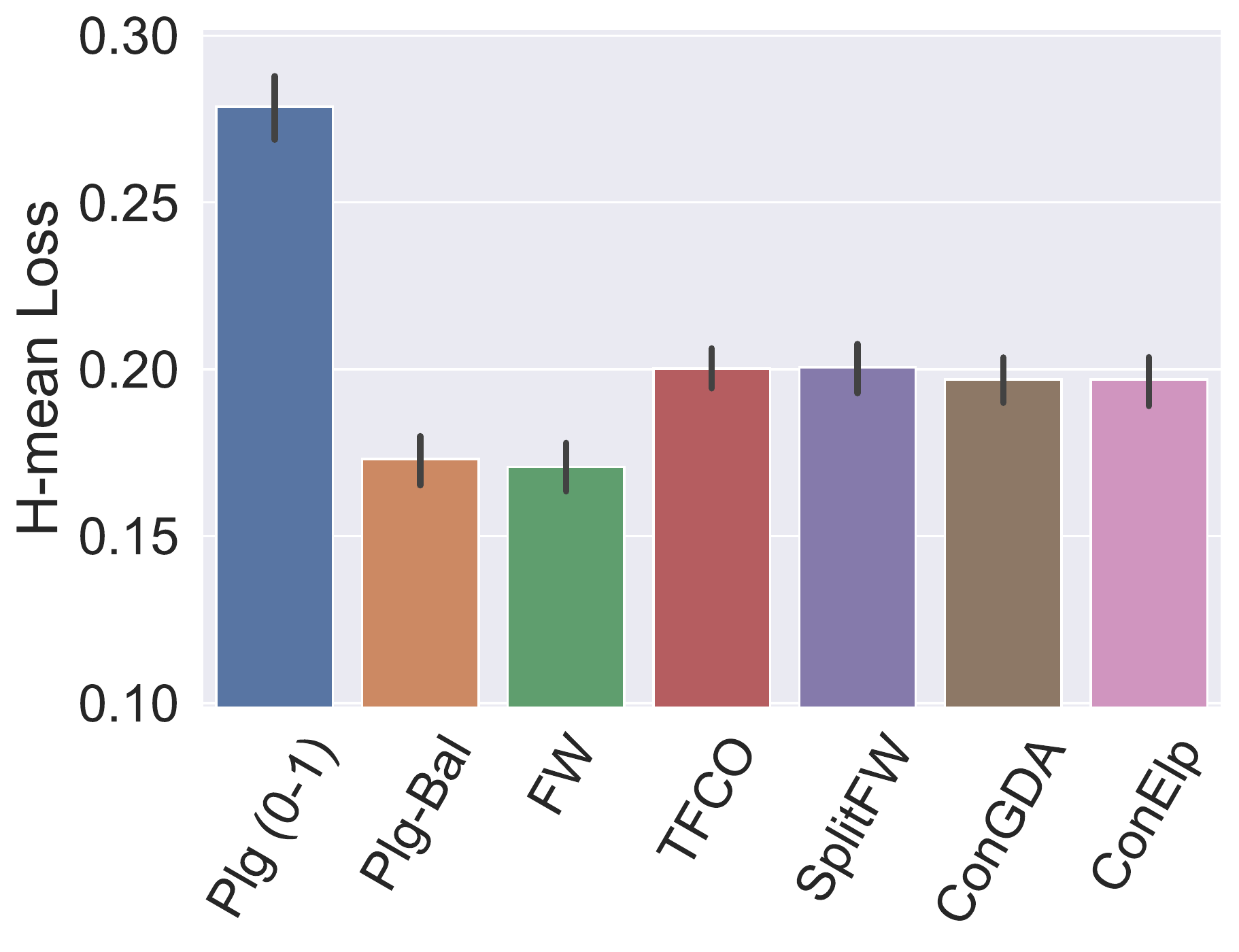}
\includegraphics[width=0.48\linewidth]{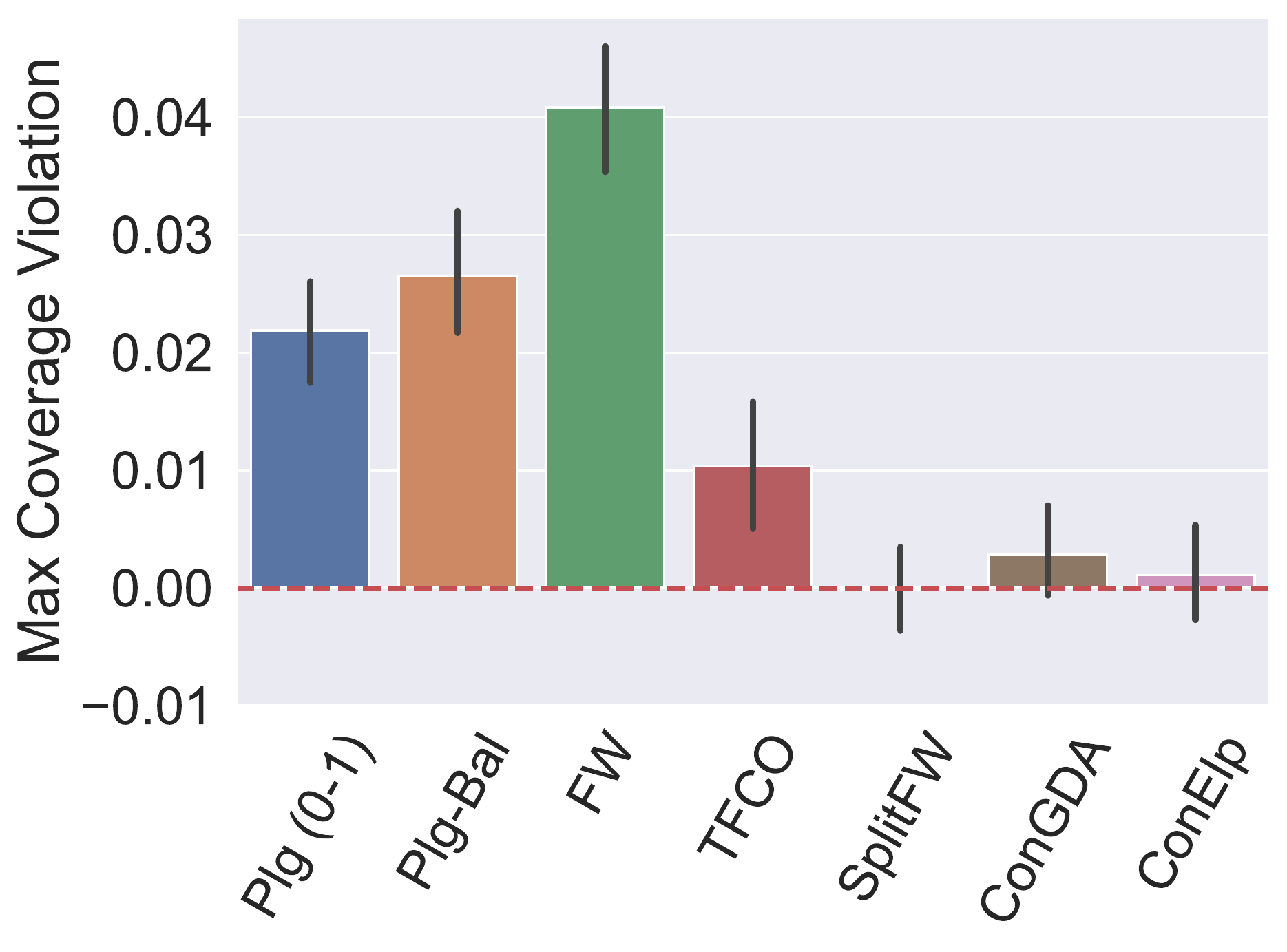}
\label{fig:satimage}
\caption{SatImage}
\end{subfigure}
\caption{Optimizing the H-mean loss subject to the coverage constraint $\max_i|\sum_{j}C_{ji} \,-\, \pi_i| \leq 0.01$. The plots on the left show the H-mean loss on the test set and those on the right show the coverage violation $\max_i|\sum_{j}C_{ji} \,-\, \pi_i| - 0.01$ on the test set. 
\textit{Lower} H-mean value are \textit{better}, and the constraint values need to be $\leq 0$. 
The results are averaged over 10 random train-test splits. The error bars indicate 95\% confidence intervals.}
\label{fig:hmean-cov}
\end{figure}
 
\begin{figure}[t]
\centering
\begin{subfigure}[b]{0.48\linewidth}
\centering
\includegraphics[width=0.48\linewidth]{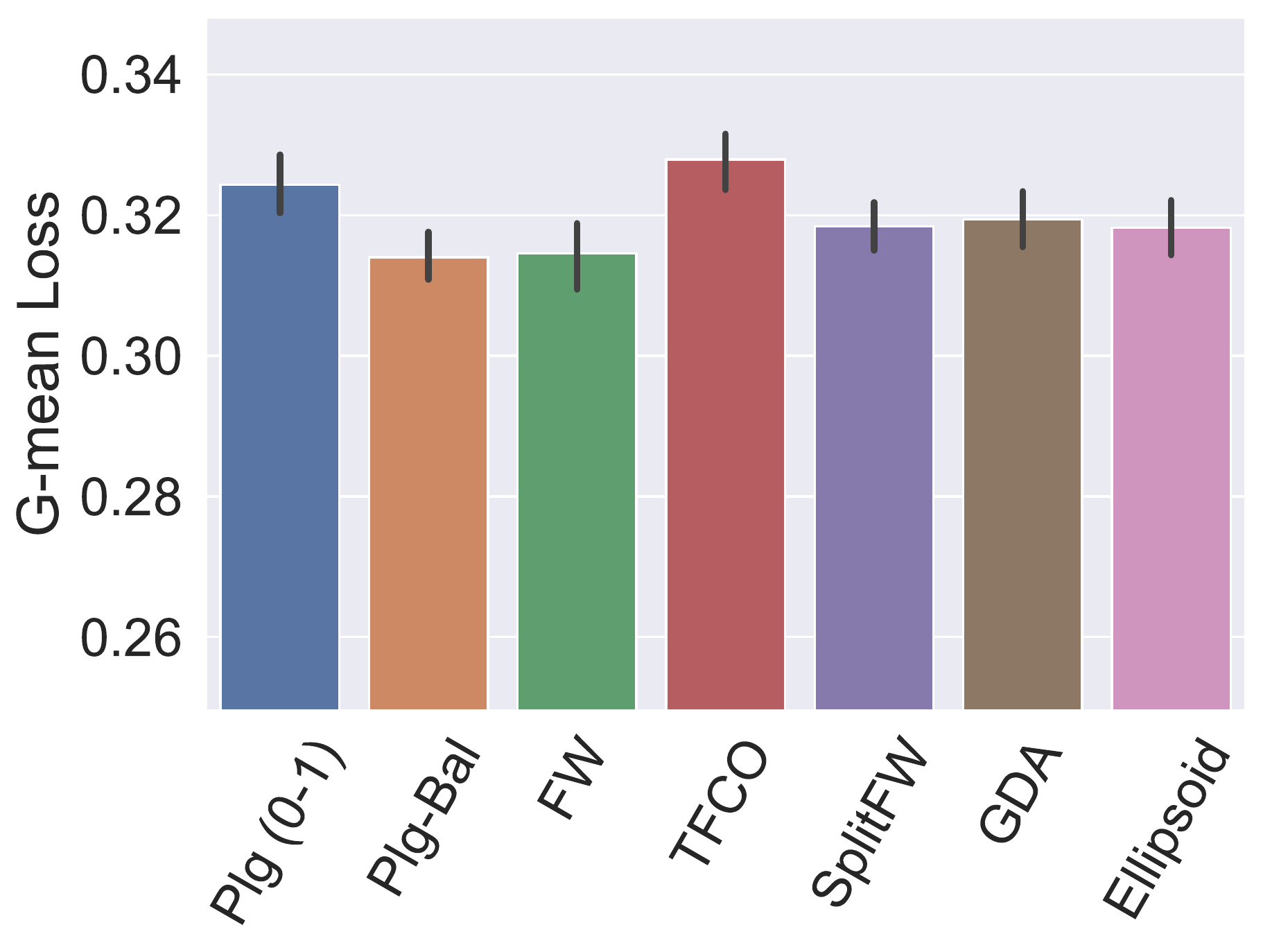}
\includegraphics[width=0.48\linewidth]{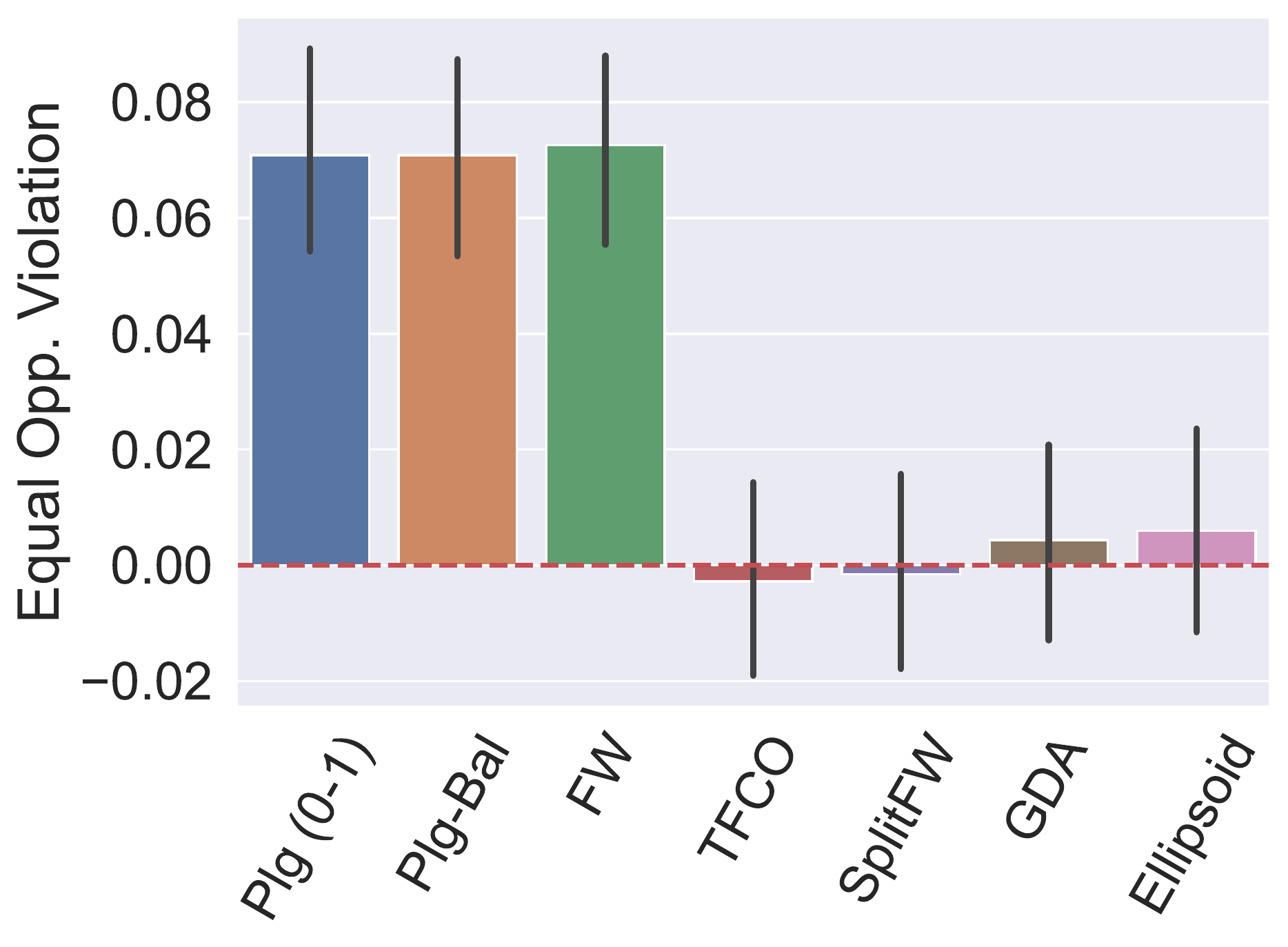}
\caption{COMPAS}
\end{subfigure}
\begin{subfigure}[b]{0.48\linewidth}
\centering
\includegraphics[width=0.48\linewidth]{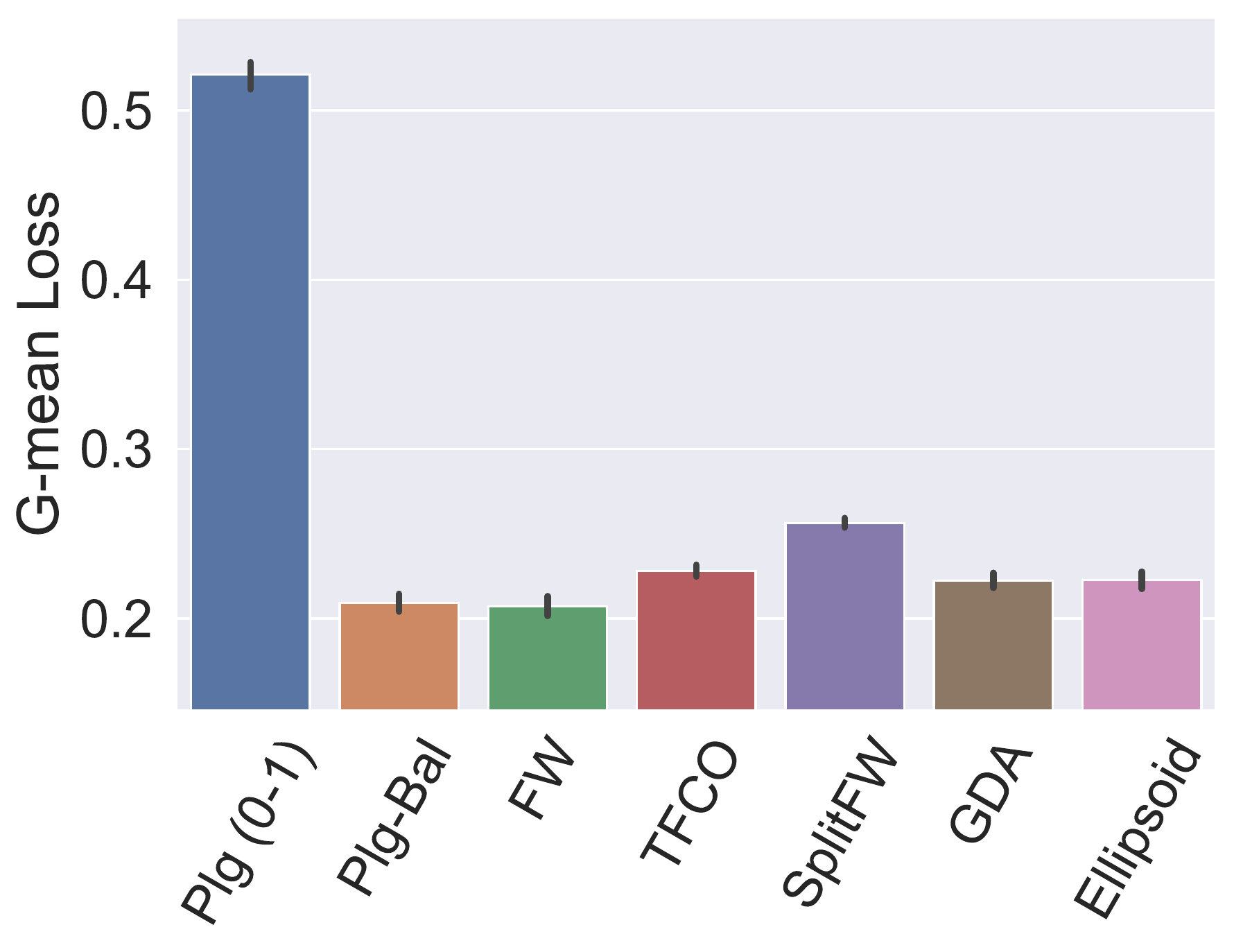}
\includegraphics[width=0.48\linewidth]{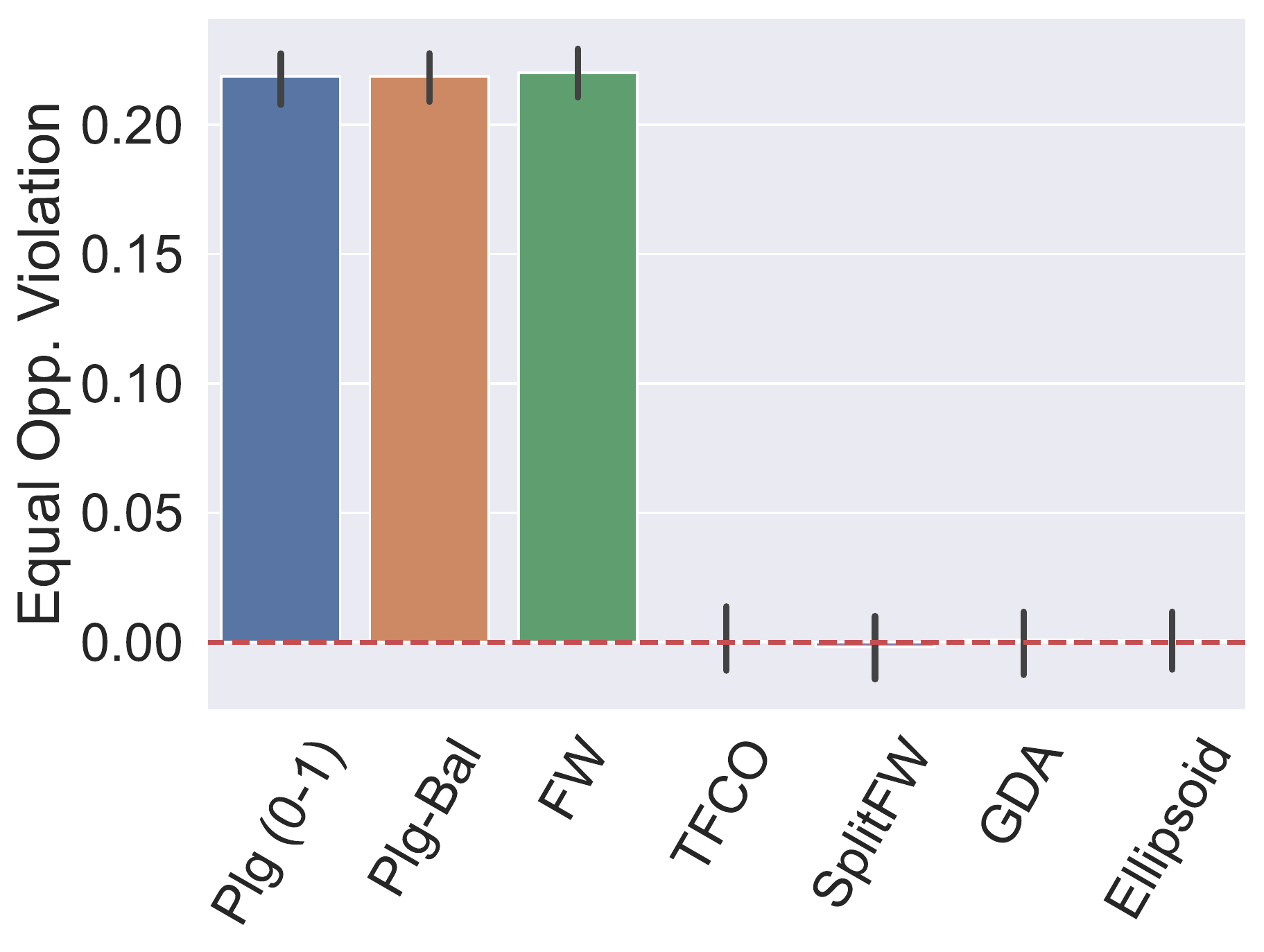}
\caption{Law School}
\end{subfigure}
\caption{Optimizing the G-mean loss subject to the equal-opportunity fairness constraint $\max_{a\in[m]}
\big|
	\frac{1}{\mu_{a1}}C^a_{11} \,-\, \frac{1}{\pi_1}C_{11}
\big| \leq 0.05$. The plots on the left show the G-mean loss on the test set and those on the right show the equal opportunity violation $\max_{a\in[m]}
\big|
	\frac{1}{\mu_{a1}}C^a_{11} \,-\, \frac{1}{\pi_1}C_{11}
\big|$ on the test set. 
\textit{Lower} G-mean value are \textit{better}, and the constraint violations need to be $\leq 0$. 
The results are averaged over 10 random train-test splits. The error bars indicate 95\% confidence intervals.}
\label{fig:gmean-eqopp}
\vspace{-10pt}
\end{figure}

Having showed the efficacy of our algorithms on unconstrained problems, we move to constrained problems. The first task we consider is to minimize the 
 H-mean loss subject to  coverage constraint requiring the proportion of predictions for each class $i$ to match the class prior $\pi$. Specifically, we require the maximum coverage violation over the $n$ classes $\max_{i\in[n]}|\sum_{j}C_{ji} \,-\, \pi_i|$ to be at most 0.01. 
 In Figure \ref{fig:hmean-cov}, we report both the H-mean and the maximum coverage violation for the three proposed constrained learning algorithms (Algorithms \ref{alg:FW-con}--\ref{alg:ellipsoid-con}) for this problem (see Appendix \ref{app:expts-additional} for additional results). 
 For comparison, we also report the performance of the 0-1 plug-in, balanced plug-in, and TFCO baselines, as well as the unconstrained Frank-Wolfe (FW) method, which seeks to optimize only the H-mean ignoring the constraint.
 We find that all three algorithms satisfy the constraint on the training set, but  occasionally incur some violations on the test set. In contrast, all baselines expect TFCO fail to satisfy the constraint.  On  SatImage, TFCO satisfies the constraint on the training set, but fails to satisfy it on the test set, while the proposed methods incur much lower test violations. This is also the case with MACHO, where TFCO incurs lower constraint violation and loss value on the training set, but compared to our methods is worse of on both metrics on the test set. The reason our methods are less prone to over-fitting is because they use a plug-in based LMO that post-shifts a pre-trained class-probability estimator, and therefore have fewer parameters to optimize when compared to TFCO.

 \begin{figure}[t]
\centering
\begin{subfigure}[b]{0.24\linewidth}
\centering
\includegraphics[width=0.99\linewidth]{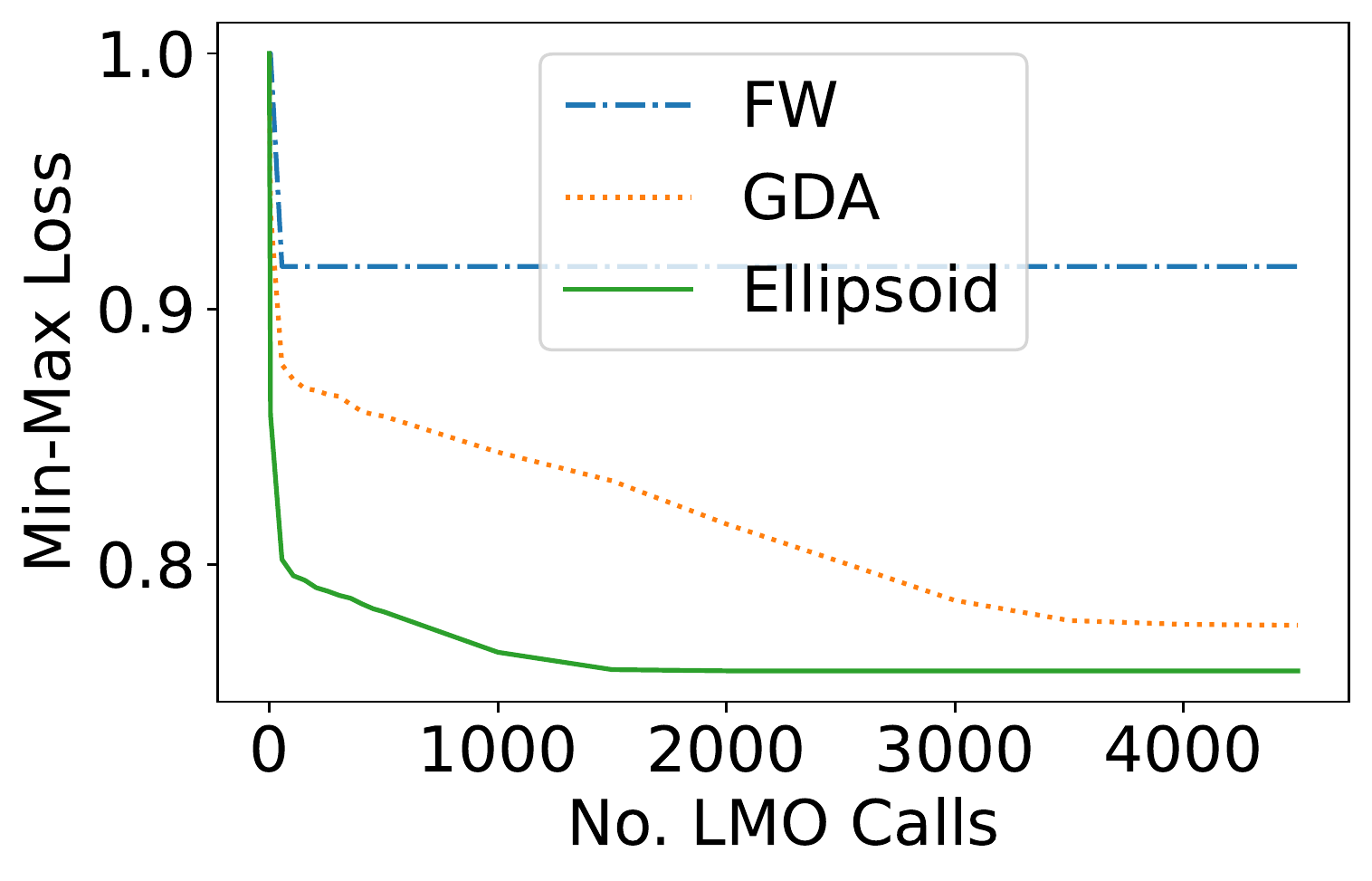}
\caption{Abalone (Train)}
\label{fig:lmo-abalone-minmax-train}
\end{subfigure}
\begin{subfigure}[b]{0.24\linewidth}
\centering
\includegraphics[width=1.02\linewidth]{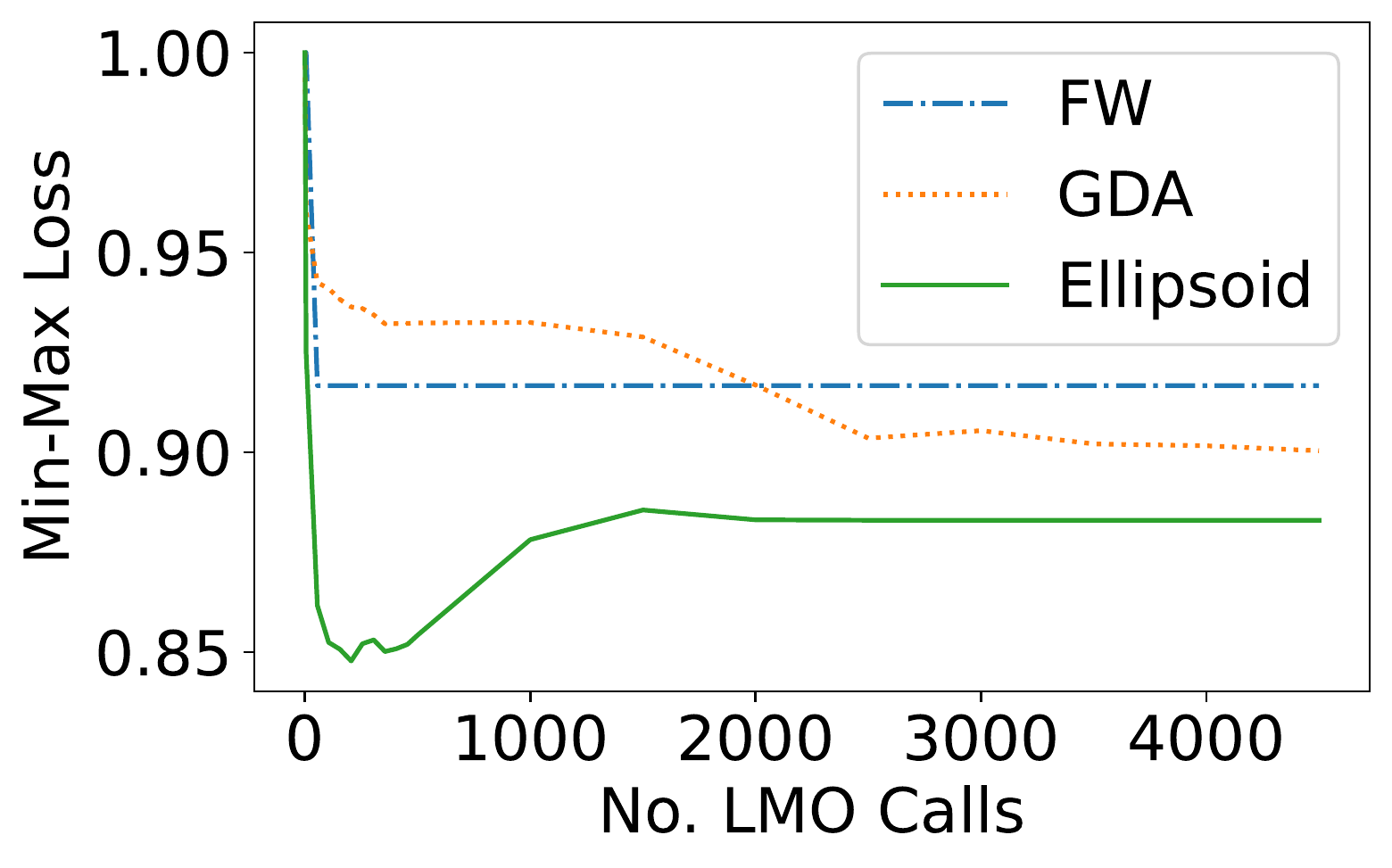}
\caption{Abalone (Test)}
\label{fig:lmo-abalone-minmax-test}
\end{subfigure}
\begin{subfigure}[b]{0.24\linewidth}
\centering
\includegraphics[width=0.99\linewidth]{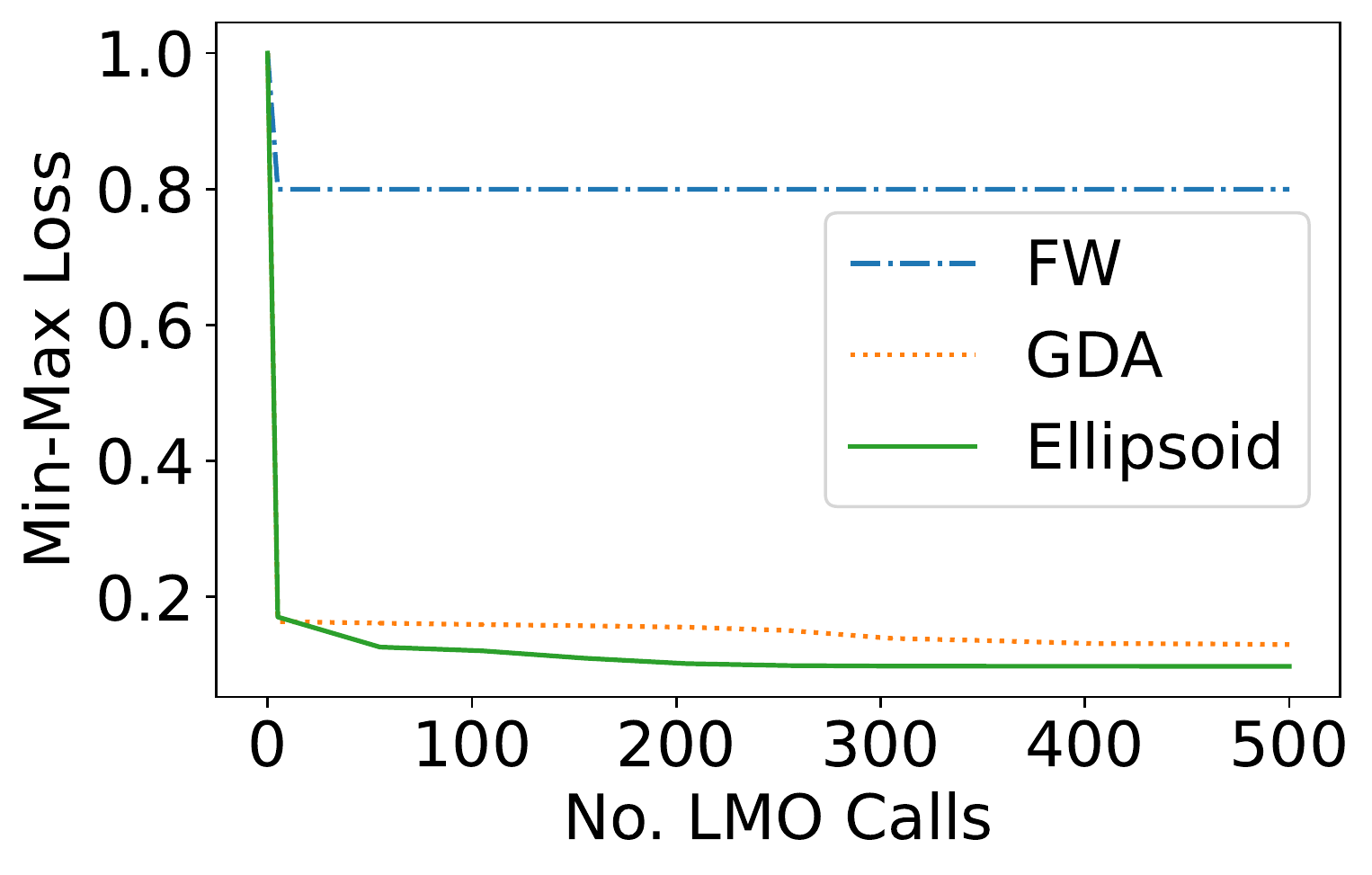}
\caption{PageBlock (Train)}
\label{fig:lmo-pageblocks-minmax-train}
\end{subfigure}
\begin{subfigure}[b]{0.24\linewidth}
\centering
\includegraphics[width=0.99\linewidth]{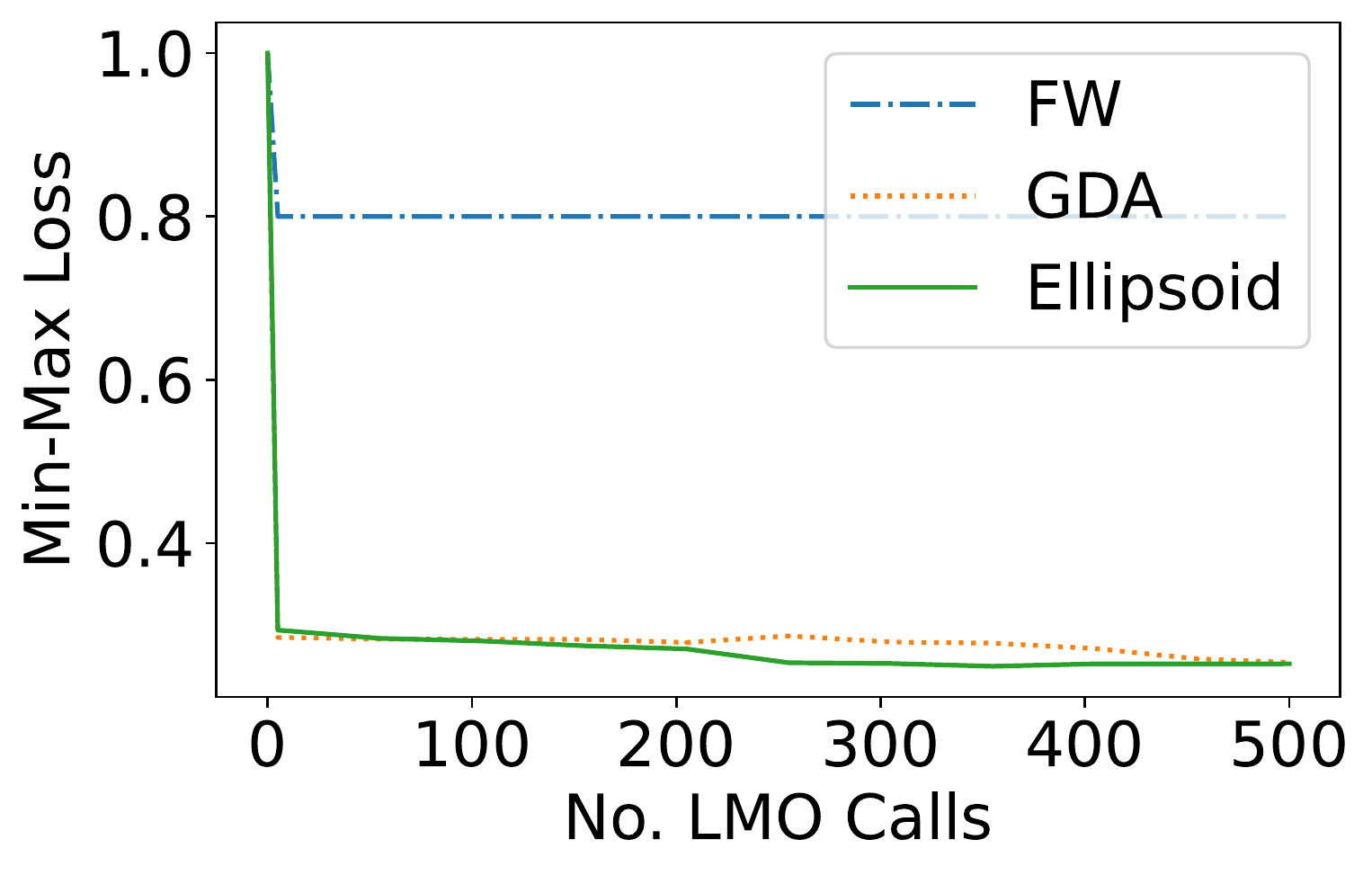}
\caption{PageBlock (Test)}
\label{fig:lmo-pageblocks-minmax-test}
\end{subfigure}
%
\caption{Optimizing the \emph{Min-max} loss: Comparison of performance of the Frank-Wolfe, GDA and ellipsoid methods as a function of the number of LMO calls. \textit{Lower} values are \textit{better}. Because the min-max loss is non-smooth, Frank-Wolfe is seen to converge to a sub-optimal classifier.}
\label{fig:min-max}
\vspace{-10pt}
\end{figure}


\begin{figure}[t]
\centering
\begin{subfigure}[b]{0.24\linewidth}
\centering
\includegraphics[width=0.99\linewidth]{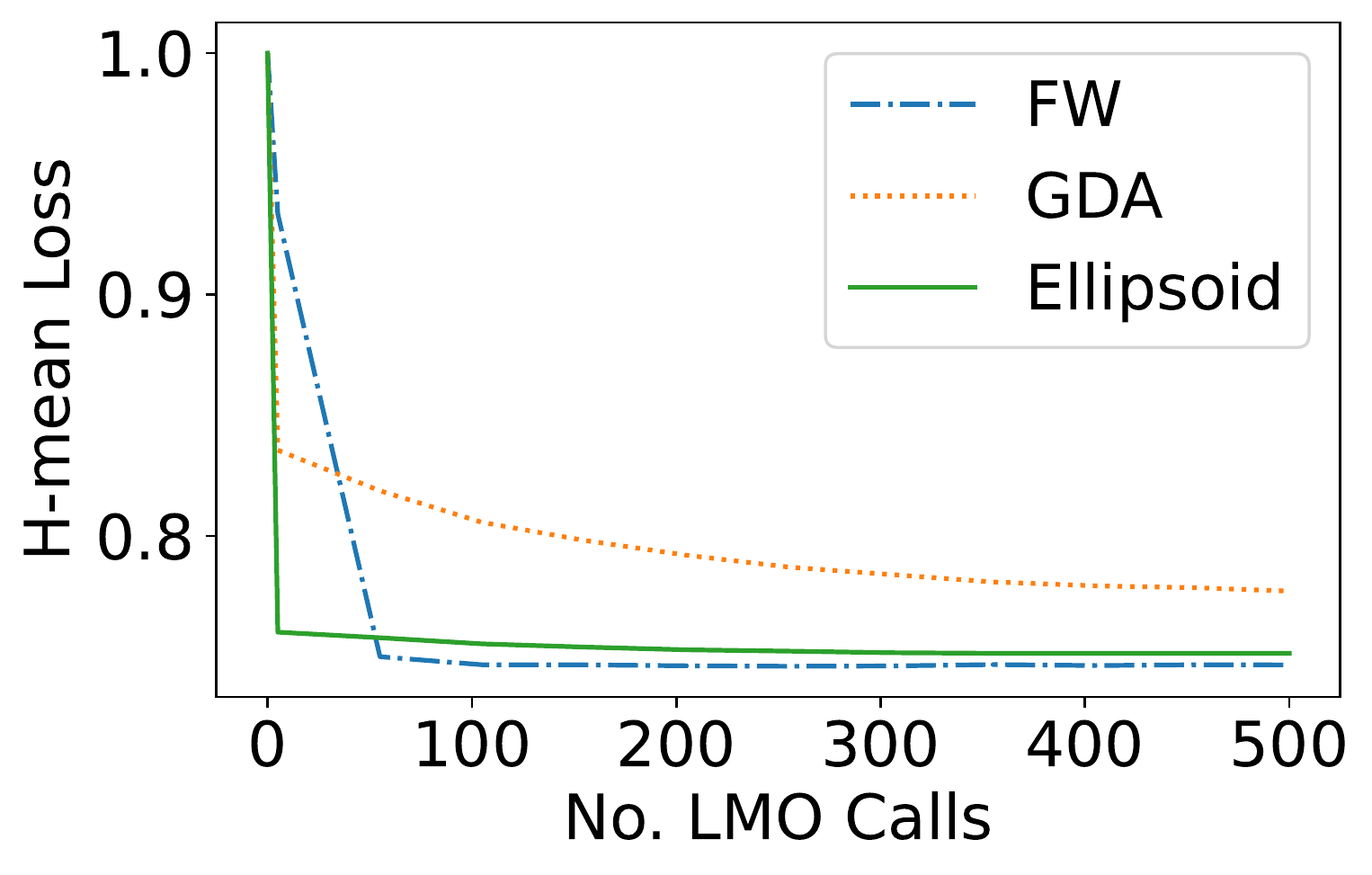}
\caption{Abalone (Train)}
\label{fig:lmo-abalone-hmean-uncon-train}
\end{subfigure}
\begin{subfigure}[b]{0.24\linewidth}
\centering
\includegraphics[width=0.99\linewidth]{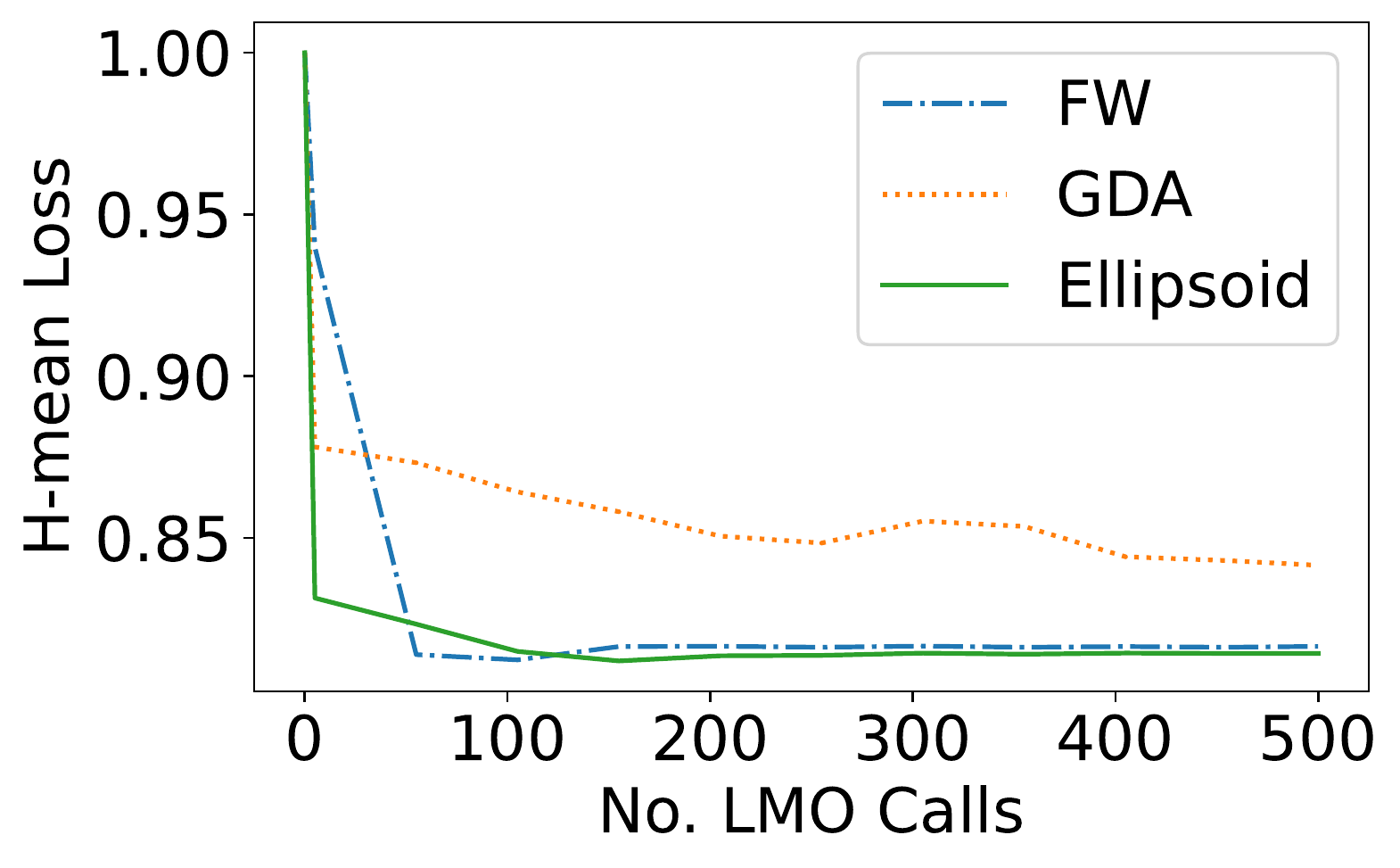}
\caption{Abalone (Test)}
\label{fig:lmo-abalone-hmean-uncon-test}
\end{subfigure}
\begin{subfigure}[b]{0.24\linewidth}
\centering
\includegraphics[width=0.99\linewidth]{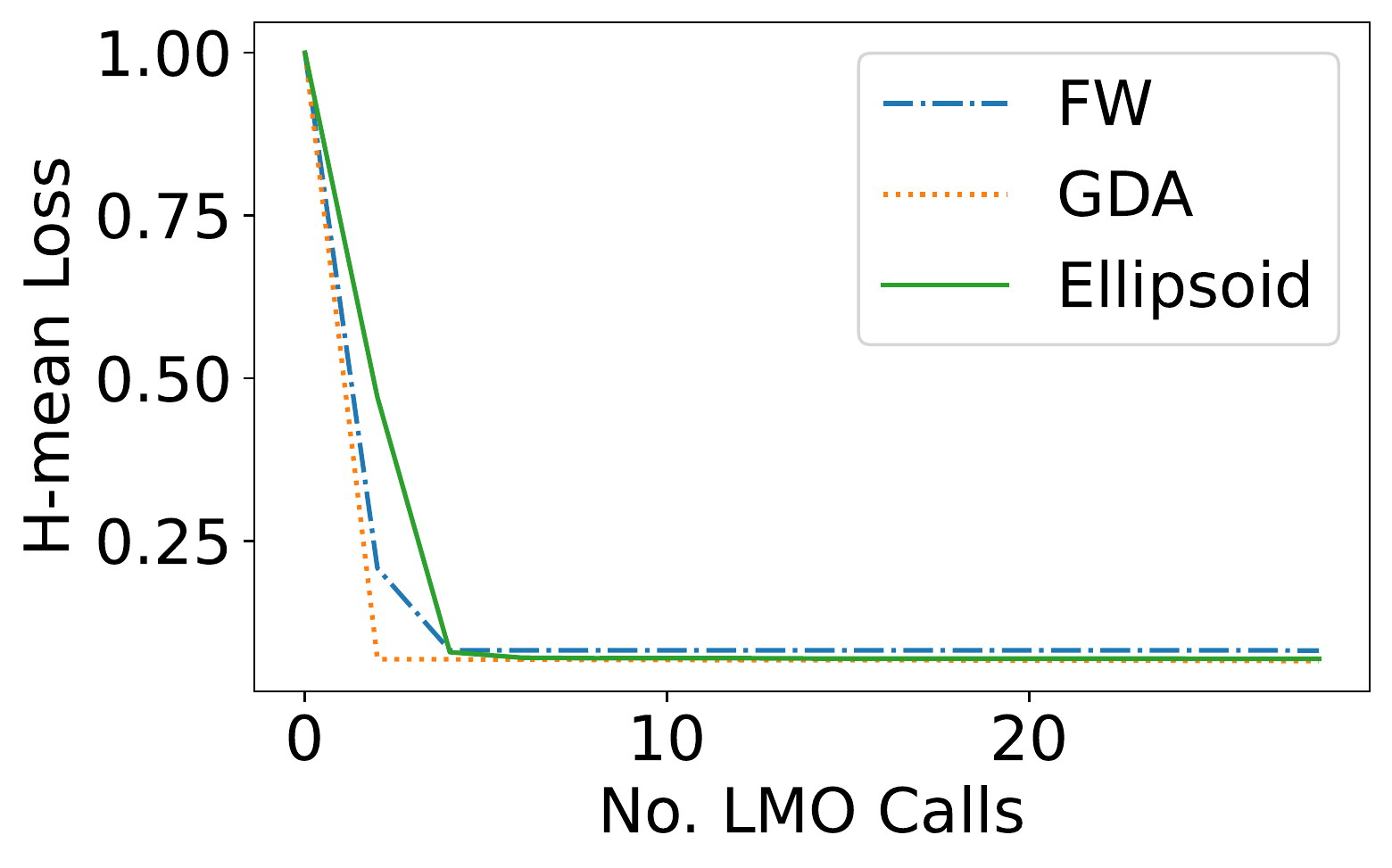}
\caption{PageBlock (Train)}
\label{fig:lmo-pageblocks-hmean-uncon-train}
\end{subfigure}
\begin{subfigure}[b]{0.24\linewidth}
\centering
\includegraphics[width=0.99\linewidth]{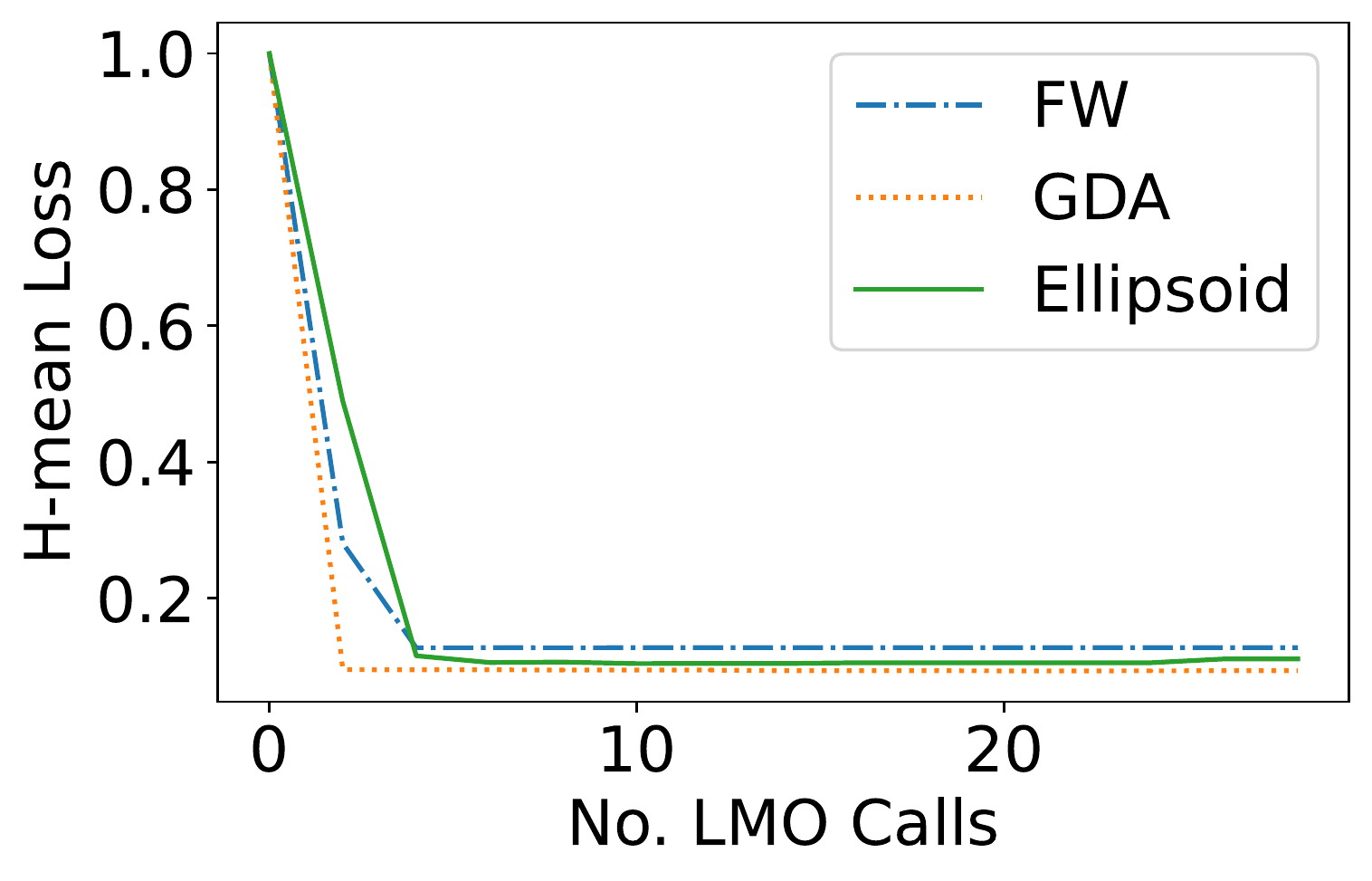}
\caption{PageBlock (Test)}
\label{fig:lmo-pageblocks-hmean-uncon-test}
\end{subfigure}
\caption{Optimizing the \emph{H-mean} loss: Comparison of performance of the Frank-Wolfe, GDA and ellipsoid methods as a function of the number of LMO calls. \textit{Lower} values are \textit{better}.
}
\label{fig:hmean-uncon}
\vspace{-12pt}
\end{figure}

 Our second task seeks to impose fairness constraints on benchmark fair classification datasets containing protected group information. These include:
 (1) \textit{COMPAS}, where the goal is to predict recidivism with \textit{gender} as the protected attribute \citep{Angwin+16}; (2) \textit{Communities \& Crime}, where the goal is to predict if a community in the US has a crime rate above the $70$th percentile \citep{uci}, and we consider communities having a black population above the $50$th percentile as protected \citep{Kearns+18}; 
(3) \textit{Law School}, where the task is to predict whether a law school student will pass the bar exam, with \textit{race} (black or other) as the protected attribute \citep{Wightman:1998}; (4) \textit{Adult}, where the task is to predict if a person's income exceeds 50K/year, with \textit{gender} as the protected attribute \citep{uci}; (5) \textit{Default}, where the task is to predict if a credit card user defaulted on a payment, with gender as the protected attribute \citep{uci}. While these are all binary-labelled datasets, because  we wish to evaluate performance separately on the individual protected groups, the number of threshold parameters  needed to learn a na\"{i}ve plug-in classifier like the one described in Section \ref{sec:naive} would grow exponentially with the number of groups, making the  algorithms proposed in this paper desirable even in these multi-group settings.
 
The specific optimization goal is to minimize the G-mean loss $\psi^\GM(\C) = 1\,-\,\big( \prod_{i} \frac{C_{ii}}{\sum_{j} C_{ij}} \big)^{1/n}$ subject to an equal opportunity constraint $\max_{a\in[m]}
\big|
	\frac{1}{\mu_{a1}}C^a_{11} \,-\, \frac{1}{\pi_1}C_{11}
\big| \leq 0.05$, requiring the true positive rates for different protected groups to be similar. The plots in Figure \ref{fig:gmean-eqopp} presents the results for the three proposed algorithms relevant to this problem, and show both the G-mean loss and the equal opportunity violation (more results in Appendix \ref{app:expts-additional}). 
In addition to the 0-1 plug-in, balanced plug-in and TFCO baselines, we include an unconstrained Frank-Wolfe (FW) method which seeks to minimize only the G-mean ignoring the constraint. All these methods incur large constrained violations.  The objectives are largely comparable for the three proposed methods, except on LawSchool, where SplitFW yields a higher loss. The constraint violations for our methods are comparable to or lower than TFCO, with TFCO failing to satisfy the constraint on Crimes as a result of over-fitting to the training set.



\begin{figure}
\centering
\begin{subfigure}[b]{0.48\linewidth}
\centering
\includegraphics[width=0.48\linewidth]{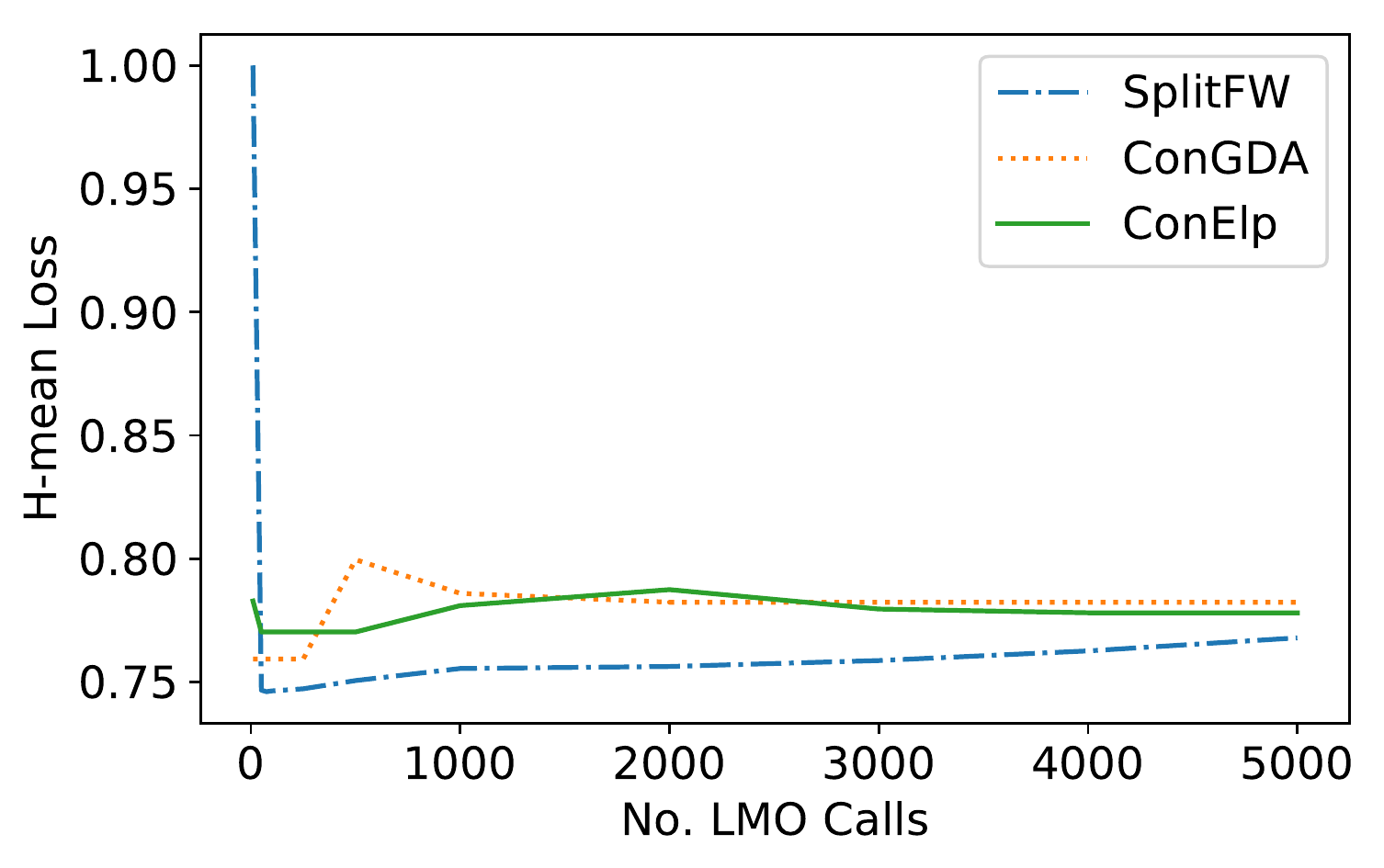}
\includegraphics[width=0.48\linewidth]{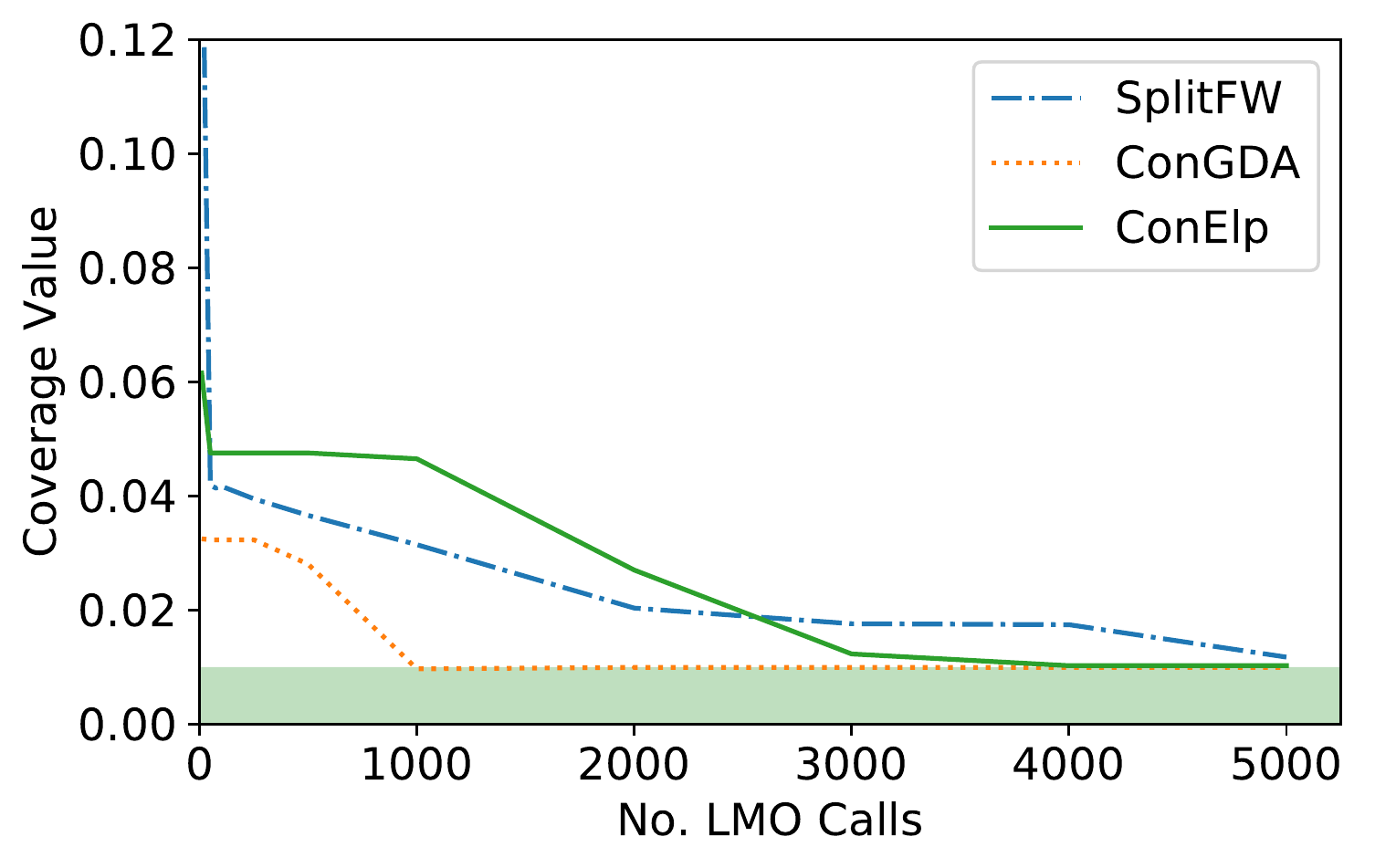}
\caption{Abalone (Train)}
\end{subfigure}
\begin{subfigure}[b]{0.48\linewidth}
\centering
\includegraphics[width=0.48\linewidth]{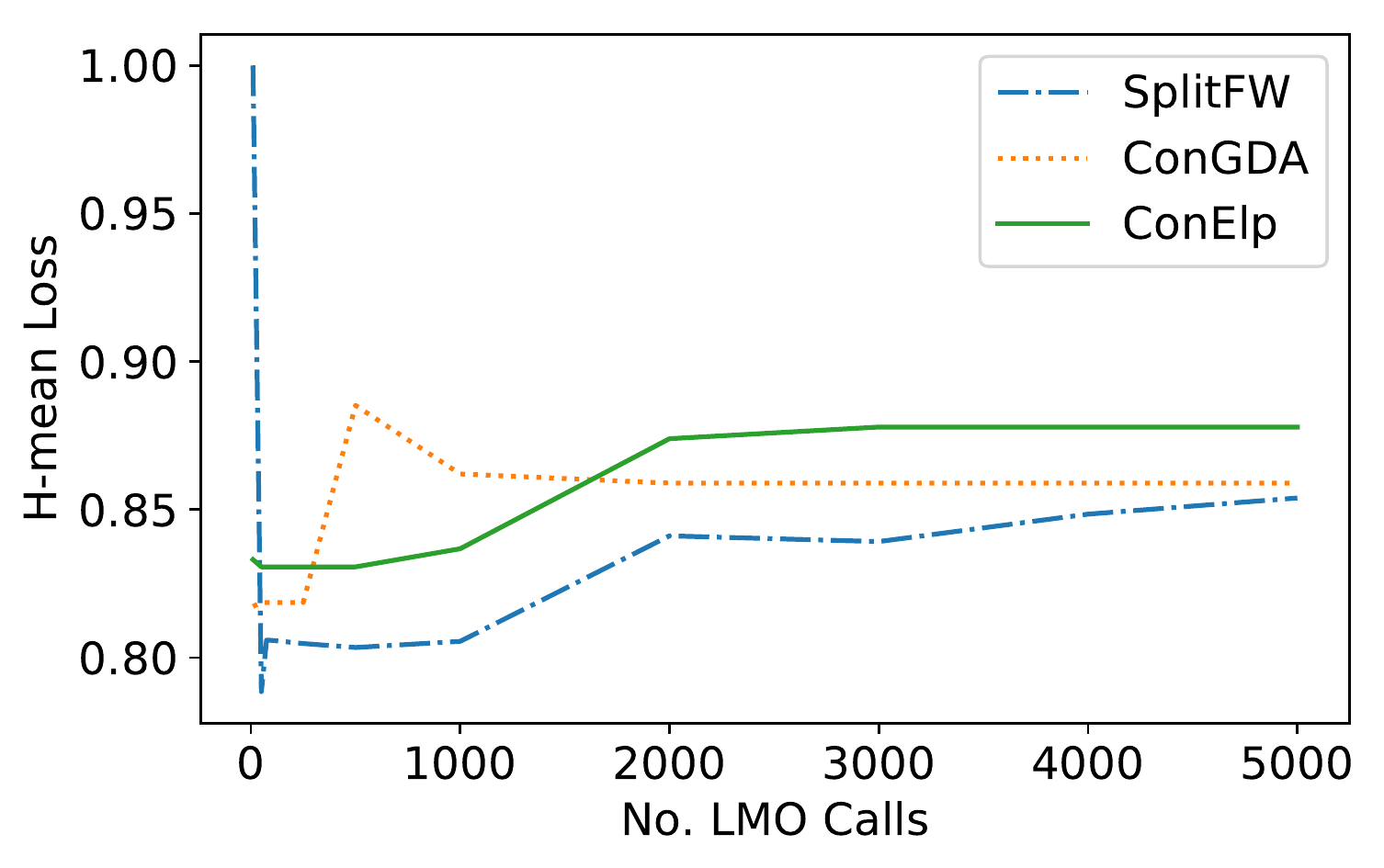}
\includegraphics[width=0.48\linewidth]{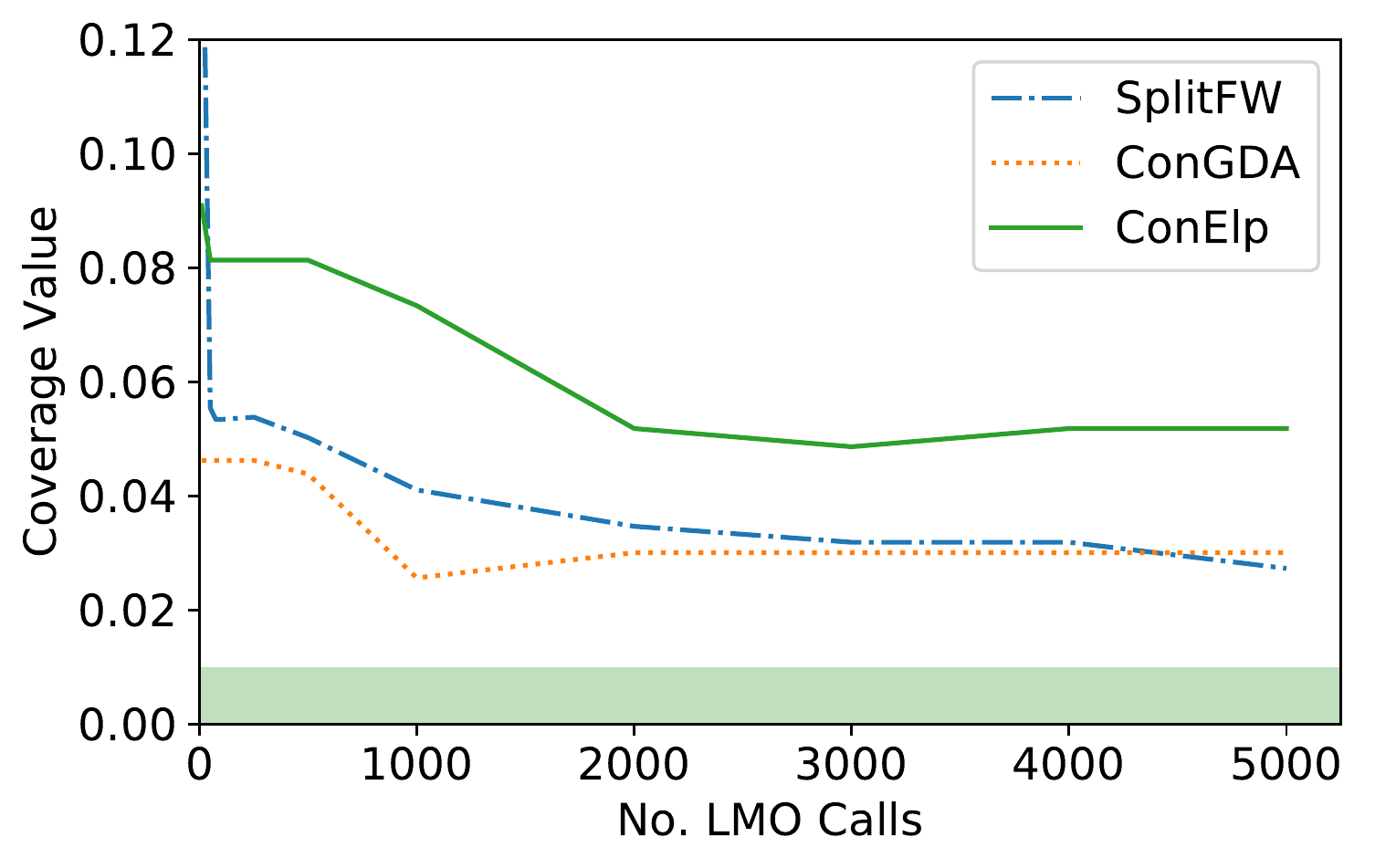}
\caption{Abalone (Test)}
\end{subfigure}
\begin{subfigure}[b]{0.48\linewidth}
\centering
\includegraphics[width=0.48\linewidth]{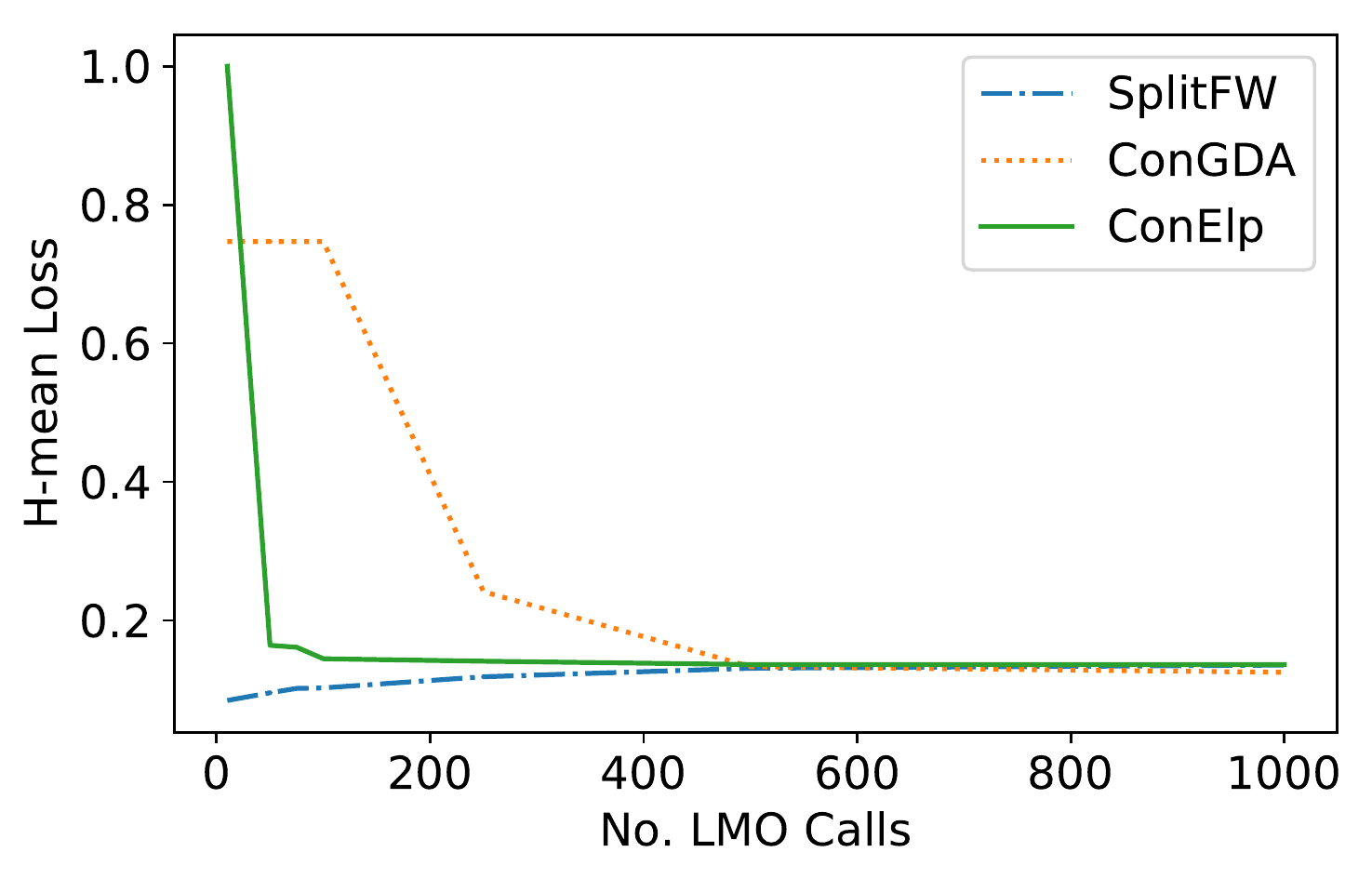}
\includegraphics[width=0.48\linewidth]{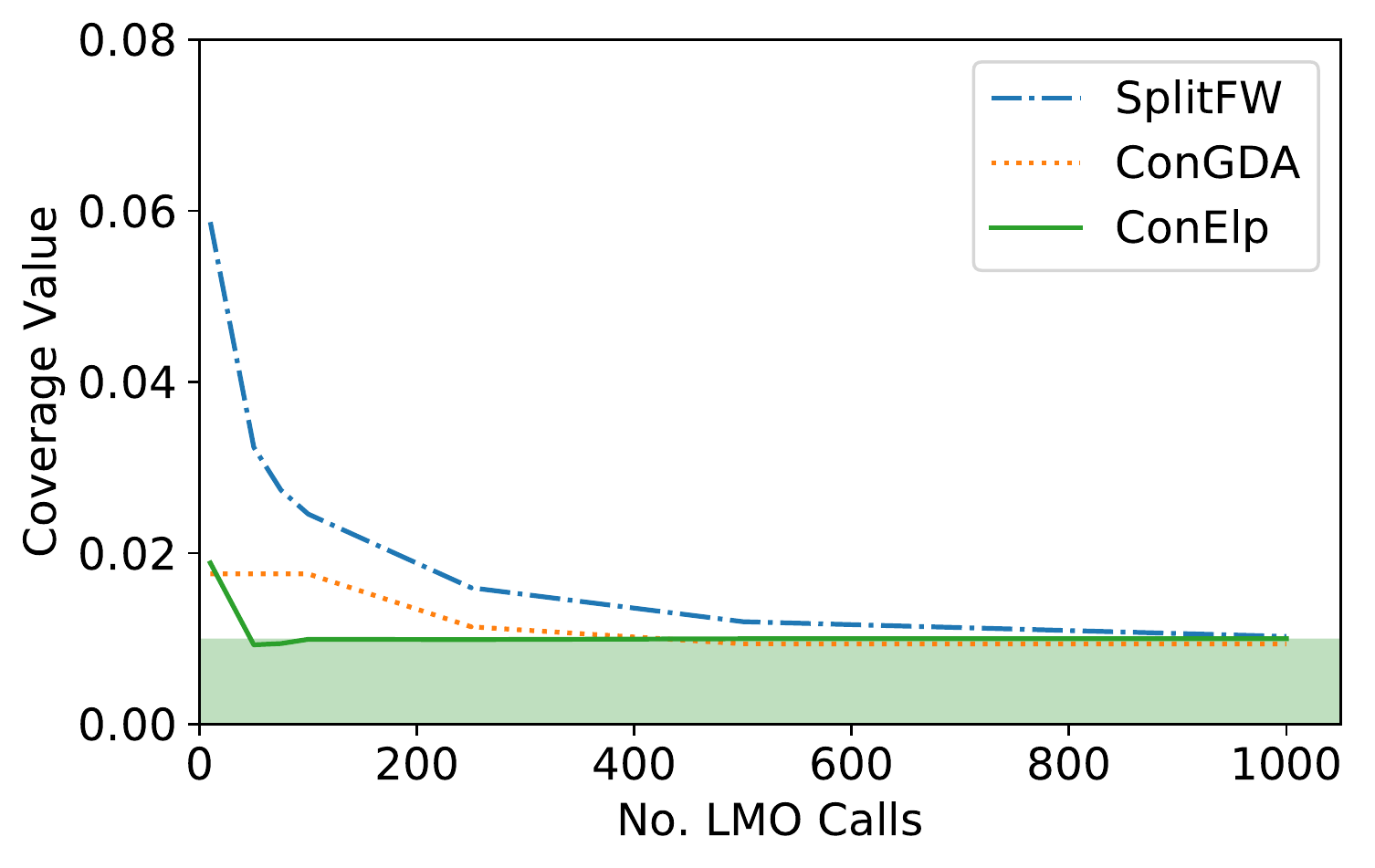}
\caption{PageBlock (Train)}
\end{subfigure}
\begin{subfigure}[b]{0.48\linewidth}
\centering
\includegraphics[width=0.48\linewidth]{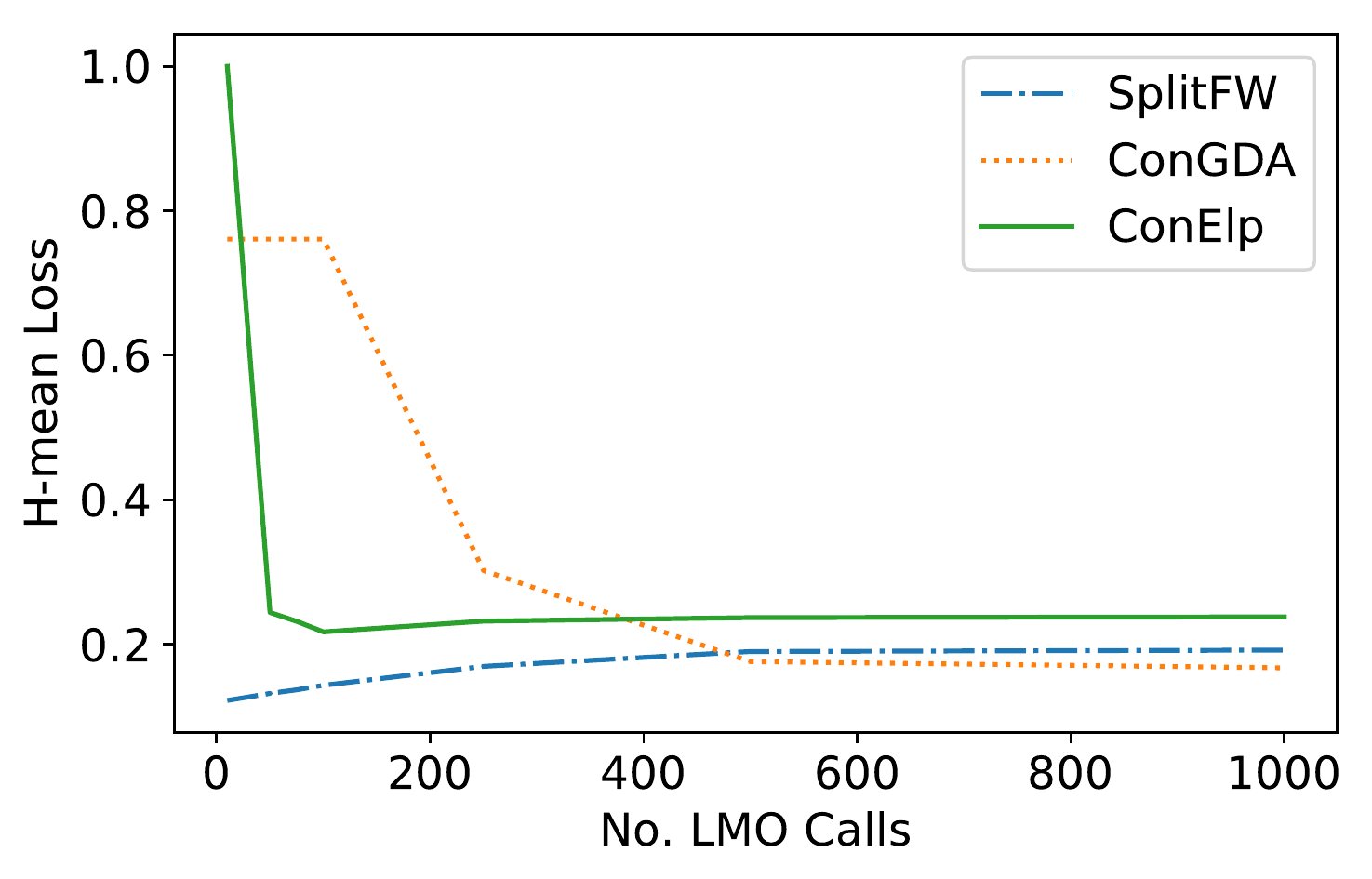}
\includegraphics[width=0.48\linewidth]{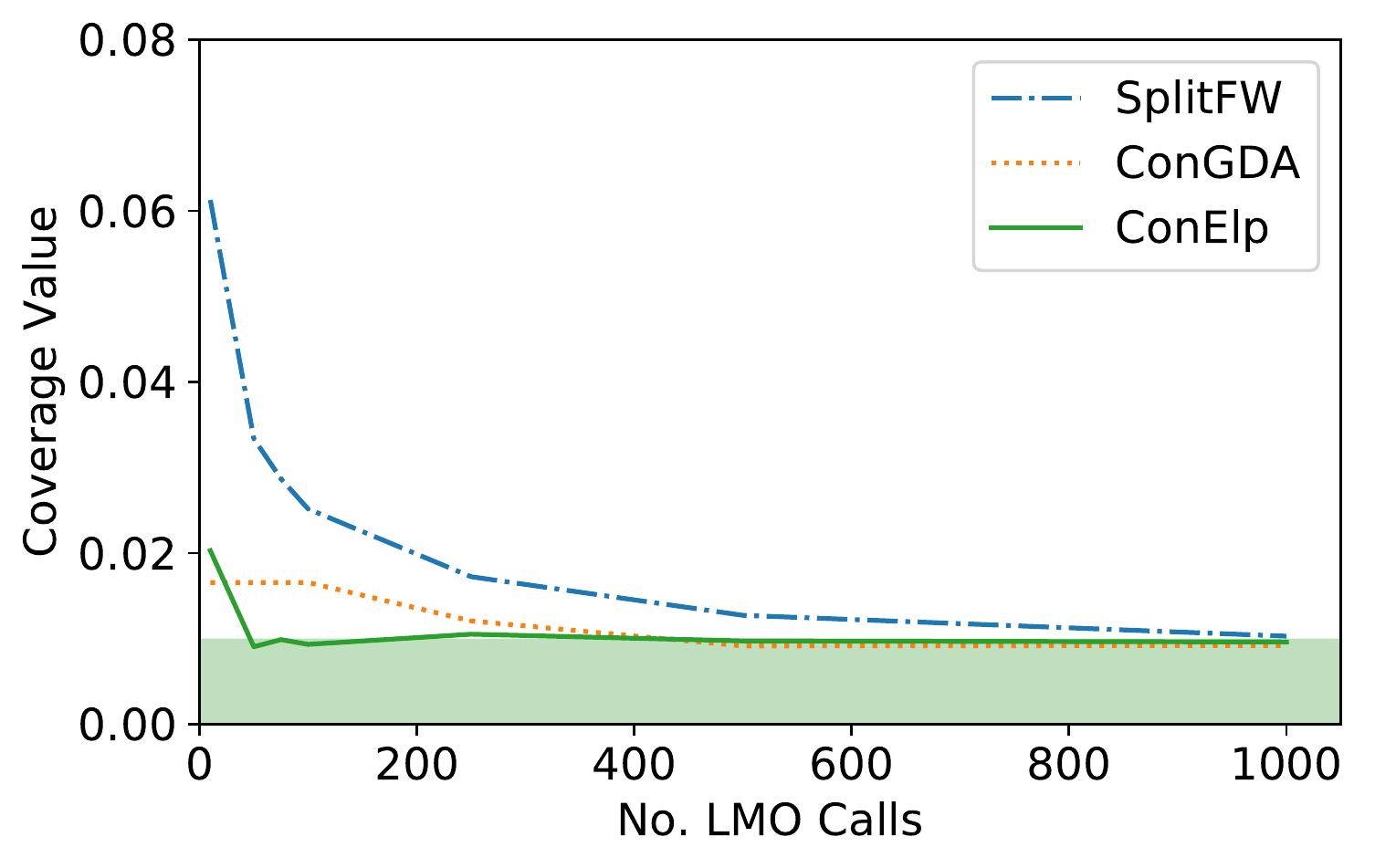}
\caption{PageBlock (Test)}
\end{subfigure}
\caption{Optimizing H-mean subject to coverage constraint: Comparison of performance of SplitFW, ConGDA and ConEllipsoid algorithms as a function of the number of LMO calls. \textit{Lower} H-mean values are \textit{better}. Green shaded region denotes coverage values that satisfy the constraints.
}
\label{fig:lmo-cons-abalone}
\vspace{-12pt}
\end{figure}

\subsection{Practical Guidance on Algorithm Choice}
\label{sec:expts-practical}
Of the three types of algorithms we have proposed for convex metrics $\psi$, the choice of the algorithm to use in an application  would depend on three factors: the smoothness of the metric, the presence of constraints, and the dimension of the problem. In Figure \ref{fig:min-max}, we consider the task of optimizing the min-max metric $\psi^{\text{MM}}(\C) = \max_{i} \left(1 - \frac{C_{ii}}{\sum_j C_{ij}}\right)$, a \emph{non-smooth} function of $\C$, and plot the performance of the three algorithms (with a plug-in based LMO) on the training and test sets as a function of the number of calls to the LMO. Since the objective for this unconstrained problem does not satisfy the smoothness property required by the Frank-Wolfe algorithm, as expected, it fairs poorly even with a large number of LMO calls. The ellipsoid algorithm is often seen to exhibit faster convergence than GDA on the training set, but there isn't a clear winner on the test set. In Figure \ref{fig:min-max}, we repeat the experiment with the smooth H-mean metric, and find that Frank-Wolfe algorithm does converge to a similar performance as the other methods, and is in fact the fastest to do so on the 12-class Abalone dataset. Moreover unlike the GDA, the Frank-Wolfe algorithm has no additional hyper-parameters to tune and is therefore an attractive option for smooth convex metrics.

On the other hand, when it comes to constrained problems, we find the (constrained) GDA algorithm to exhibit the fastest convergence. 
In this case, the (constrained) ellipsoid algorithm may take longer to converge to the optimal-feasible solution, particularly when the number of classes is high (as seen from the strong dependence on dimension its convergence rate has in Theorem \ref{thm:ellipsoid-con}). For example, this evident with the 12-class Abalone dataset in Figure \ref{fig:lmo-cons-abalone}(a)--(b), where we seek to maximize the H-mean loss subject  to the coverage constraint described in Section \ref{sec:experiments-cons}, and find the GDA algorithm to converge the fastest to a feasible classifier. In contrast, the (constrained) ellipsoid algorithm exhibits the fastest convergence on the smaller 5-class PageBlock dataset (although it yields slightly worse H-mean values than the other methods on the test set). See Appendix \ref{app:expts-additional} for additional experimental results.

Overall, we prescribe using 
the ellipsoid algorithm (or its constrained counterpart) for problems with a small number of classes, the FrankWolfe algorithm if the metric is smooth and there are no constraints, and the GDA algorithm (or its constrained counterpart) for all other scenarios.

\begin{table}[t]
    \centering
    \caption{
    Comparison of the plug-in and weighted logistic regression (WLR) based LMOs on the task of optimizing the (convex) H-mean loss. The number of iterations, i.e.\ calls to the LMO, is fixed at \emph{20}. \textit{Lower} values are \textit{better}. The results are averaged over 10 random train-test splits.}
    \label{tab:cost-sen}
    \begin{footnotesize}
    \begin{tabular}{ccccccc}
        \hline
        \multirow{2}{*}{\textbf{Data}} &
        \multicolumn{2}{c}{\textbf{FW}} &
        \multicolumn{2}{c}{\textbf{GDA}} &
        \multicolumn{2}{c}{\textbf{Ellipsoid}} 
        \\
        &
        \textbf{Plugin} & 
        \textbf{WLR} & \textbf{Plugin} & \textbf{WLR}  & \textbf{Plugin} & \textbf{WLR} \\
        \hline
        Aba  & $0.797 \pm 0.008$ & $\textbf{0.791} \pm \textbf{0.004}$ & $0.892 \pm 0.038$ & $\textbf{0.838} \pm \textbf{0.017}$ & $0.833 \pm 0.038$  & $0.833 \pm 0.038$   \\
        PgB & $0.13 \pm 0.038$  & $\textbf{0.084} \pm \textbf{0.015}$ & $0.129 \pm 0.034$ & $\textbf{0.083} \pm \textbf{0.018}$ & $0.105 \pm 0.019$ & $\textbf{0.080} \pm \textbf{0.017}$ \\
        MAC    & $\textbf{0.125} \pm \textbf{0.01}7$ & $0.245 \pm 0.027$ & $\textbf{0.124} \pm \textbf{0.015}$  & $0.206 \pm 0.028$ & $\textbf{0.122} \pm \textbf{0.015}$ & $0.247 \pm 0.027$  \\
        Sat    & $0.174 \pm 0.007$  & $\textbf{0.171} \pm \textbf{0.007}$ & $\textbf{0.173} \pm \textbf{0.008}$ & $0.176 \pm 0.006$ & $0.168 \pm 0.006$ & $\textbf{0.167} \pm \textbf{0.006}$ \\
        Cov & $0.468 \pm 0.001$ & $\textbf{0.453} \pm \textbf{0.001}$ & $0.488 \pm 0.001$ & $\textbf{0.453} \pm \textbf{0.001}$ & $0.463 \pm 0.001$ & $\textbf{0.447} \pm \textbf{0.001}$\\
        \hline
    \end{tabular}
        \end{footnotesize}

\end{table}

\subsection{Choice of LMO: Plug-in vs. Weighted Logistic Regression}
\label{sec:expts-LMO}
In previous experiments, we have seen that the proposed algorithms were less prone to over-fitting because of the use of a plug-in based LMO that post-fit a small number of parameters to a pre-trained model. We now compare the performance of these algorithms with an LMO that re-trains a classifier from scratch each time it is called. Specifically, we repeat the H-mean optimization task from Section \ref{sec:expts-unc}, with weighted (linear) logistic regression as the LMO. For a given (diagonal) loss matrix $\L$, this LMO learns a classifier by optimizing a weighted logistic loss, where the per-class weights are set to be the diagonal entries of $\L$. Note that such a weighted surrogate loss is calibrated for $\L$ \citep{TewariBa07}. Unlike the simple plug-in LMO, each call to weighted logistic regression can be expensive; hence it is important that we are able to limit the number of calls to it. 

In Table \ref{tab:cost-sen}, we present results comparing performance of the Frank-Wolfe, GDA and ellipsoid algorithms with the plug-in and weighted logistic regression LMOs when run for 20 iterations. Appendix \ref{app:expts-additional} contains results of these experiments when the algorithms are allowed 100 iterations. The performance with the two LMOs are comparable on Abalone and SatImage. On PageBlocks and CovType, weighted logistic regression has a moderate to significant advantage. Interestingly, on MACHO, the plug-in based LMO, despite learning from a less flexible hypothesis class (post-hoc adjustments to a fixed model), is substantially better. This is because weighted logistic regression over-fits to the training set in this case. 

Overall, we find that an LMO such as weighted logistic regression,
while being computationally expensive, does sometimes provide metric gains over a less-flexible plug-in type approach. However, this method can be prone to over-fitting because of its added flexibility.

\subsection{Case Study: Image Classification with Imbalance and Label Noise}
\label{sec:expts-cifar}
As case studies, we demonstrate two natural workflows our algorithms in  (i) tackling label imbalance in CIFAR-55 and (ii) mitigating label noise in a noisy version of CIFAR-10.

\subsubsection{Class Imbalance with Large Number of Classes} 
One of the undesirable effects of learning with a class-imbalanced dataset is that the 
learned classifier tends to over-predict classes that are more prevalent and under-predict classes that are rare. We consider two approaches to avoid this problem:  minimizing a loss such as the H-mean that emphasizes equal performance across all classes, and  constraining the proportions of predictions the classifier makes for each class to match the true prevalence of the class.

For this experiment, we use the  CIFAR-100 dataset \citep{Krizhevsky09learningmultiple}, which contains images labelled with one of 100 classes. 
We create an imbalanced 55-class dataset by merging 50 classes in CIFAR-100 into 5 ``super-classes'' (see Appendix \ref{app:case-studies} for details), and leaving the rest of the classes untouched.
In the resulting class distribution, 5 of the classes are 10 times more prevalent than the remaining 50. 
All our methods use a plug-in based LMO which uses a pre-trained class probability estimator.
We train a ResNet-50 model for the class probability estimator, using SGD to minimize the standard cross-entropy loss. We use a batch size of 64, a base learning rate of 0.01 
(with a warm-up cosine schedule), and a momentum of 0.9. We apply a weight decay of 0.01 and train for 39 epochs.


In Table \ref{tab:cifar55}, we analyze the  performance of a ResNet-50 model trained with the standard cross-entropy loss (Plugin [0-1]), where the class that receives the highest estimated probability is predicted as the output label, and report its 0-1 loss, its H-mean loss, and the deviation of its class prediction rates from the prior probabilities, i.e.\ its maximum coverage violation: $\max_{i\in[n]}|\sum_{j}C_{ji} \,-\, \pi_i|$. 
We find that na\"{i}vely optimizing for the 0-1 loss yields a high coverage violation. Moreover, it yields high accuracies on the 5 super-classes, at the cost of a much lower accuracy on the 50 minority classes, resulting in a high H-mean loss. To emphasize better performance on the minority classes, we train classifiers to minimize the H-mean loss (FW [H-mean]), and minimize the 0-1 loss subject to the maximum coverage violation being within a tolerance of 0.01 (SplitFW [0-1]). We also consider a combination of both, i.e.\ minimizing the H-mean loss subject to the maximum coverage violation being within 0.01 (SplitFW [H-mean]). It can be seen that all three algorithms do only slightly worse than the Plugin [$\zo$] baseline in terms of 0-1 loss, but do significantly better in terms of both the H-mean loss and the coverage violation.

\begin{table}[t]
    \centering
    \caption{Results on CIFAR-55 imbalanced dataset. The train and test sets are imbalanced, with 5 classes being 10 times larger in size than the remaining 50 classes. We report the 0-1 loss, the H-mean loss, and the coverage violation $\max_{i\in[n]}|\sum_{j}C_{ji} \,-\, \pi_i| - 0.01$. \emph{Lower} values are \emph{better}.}
    \label{tab:cifar55}
    \begin{footnotesize}
    \begin{tabular}{lcccccc}
     \hline
          \multirow{2}{*}{\textbf{Method}} &
          \multicolumn{3}{c}{\textbf{Train (Imbalanced)}} &
          \multicolumn{3}{c}{\textbf{Test (Imbalanced)}} \\
                  &
          \textbf{0-1} &
          \textbf{H-mean} &
          \textbf{Violation} &
          \textbf{0-1} &
          \textbf{H-mean} &
          \textbf{Violation} 
          \\
        \hline
        Plugin [$\zo$]     & \textbf{0.278} & 0.457 & 0.030 & 0.437 & 0.709 & 0.045 \\
        \hline
        FW [H-mean]        & 0.307 & \textbf{0.323} & 0.026 & 0.481 & \textbf{0.564} & 0.029 \\
        SplitFW [$\zo$]    & 0.279 & 0.391 & \textbf{0.000} & \textbf{0.436} & 0.636 & 0.007 \\
        SplitFW [H-mean]   & 0.279 & 0.342 & \textbf{0.000} & 0.448 & 0.595 & \textbf{0.000} \\
        \hline
    \end{tabular}
    \end{footnotesize}
    \vspace{-5pt}
\end{table}

\subsubsection{Class Imbalance with Label Noise} Our next experiment demonstrates how label noise can be mitigating by imposing coverage constraints on the classifier. We use a class-imbalanced version of the CIFAR-10 dataset \citep{Krizhevsky09learningmultiple}, where we sub-sample images from classes 1 to 5 by a factor of 10, with the resulting class priors are given by $\pi_y = \frac{2}{110}$ when $y\in\{1,2,3,4,5\}$ and $\pi_y=\frac{2}{11}$ otherwise. Our algorithms assume the knowledge of $\pi$. 
In addition to class imbalance, 
very often one has to work with noisy training labels to building a classifier that performs well on uncorrupted test data. We simulate this scenario by adding a label noise corruption to the training data, which is chosen such that classes 1 to 9 are left undisturbed, and the labels of images from class 10 are randomly chosen from 1 to 10. This effectively mimics a crowd-sourced label collection with 9 easy labels, and one difficult or incomprehensible label. Our algorithms do \emph{not} have knowledge of this corruption.

Equipped with the knowledge of the class priors $\pi$, we propose constraining the proportion of predictions made for each class to match the priors $\pi$. While this is not necessarily equivalent to training the classifier with uncorrupted labels, we expect that these additional coverage constraints will dampen the effect of the noisy labels. As with the previous experiment, we evaluate the classifier on two criteria: (i) how well it performs on the (balanced) H-mean metric on the test data, and (ii) how well the class prediction rates match the priors $\pi$ on the test data. Our framework can be applied to this problem by minimizing the H-mean error on the corrupted training dataset, subject to a coverage constraint on the classifier forcing it to predict classes at a rate that matches $\pi$.

In addition to a  ResNet model baseline that optimizes the cross-entropy loss (Plugin [$\zo$]), we include the state-of-the-art method of \citet{patrini2017making}, which uses the predictions from Plugin [$\zo$] on the training data to compute an estimate of the label noise transition matrix, and re-trains the classifier with a (forward) correction computation applied to the loss (Noise Correction [Estimate]). For completeness, we also include an \emph{idealized} version of this method, where the ``exact'' noise transition matrix is used for the forward correction (Noise Correction [Exact]). While this baseline is unrealistic, it provides us with an estimate best possible 0-1 loss achievable for this problem.

We provide the result of this experiment in Table~\ref{tab:cifar10}, 
where FW [H-mean] corresponds to a classifier that minimizes the H-mean loss on the corrupted training dataset, and
SplitFW [0-1] (resp.\ SplitFW [H-mean]), correspond to a classifier that minimizes the 0-1 loss (resp.\ H-mean loss) on the corrupted training dataset, while enforcing the coverage constraint to a tolerance of $0.01$. All three methods use the same underlying class probability model as Plugin [$\zo$]. 
%
%
It is seen that only SplitFW [0-1] and SplitFW [H-mean] achieve low coverage violations on the test set, and are still only moderately worse than the idealized Noise Correction [Exact] method in terms of their respective objective metrics. The FW [H-mean] algorithm achieves the best H-mean on clean test data despite being trained on the corrupted training labels.

\begin{table}[t]
    \centering
    \caption{Results on imbalanced CIFAR-10 dataset with label noise. The train set is imbalanced and has label noise, while the test set is imbalanced but clean. We report the 0-1 loss, the H-mean loss, and the  coverage violation $\max_{i\in[n]}|\sum_{j}C_{ji} \,-\, \pi_i| - 0.01$. \emph{Lower} values are \emph{better}.}
    \label{tab:cifar10}
    \begin{footnotesize}
    \begin{tabular}{lcccccc}
     \hline
          \multirow{2}{*}{\textbf{Method}} &
          \multicolumn{3}{c}{\textbf{Train (Flipped)}} &
          \multicolumn{3}{c}{\textbf{Test (Imbalanced)}} \\
                  &
          \textbf{0-1} &
          \textbf{H-mean} &
          \textbf{Violation} &
          \textbf{0-1} &
          \textbf{H-mean} &
          \textbf{Violation} 
          \\
        \hline
        Plugin [$\zo$]    
        & 0.266 & 0.896 & 0.170 & 0.332 & 0.899  & 0.169 \\
        Noise Correction [Estimate]
        & \textbf{0.174} & \textbf{0.370}  & 0.054 & \textbf{0.179}  & {0.430} & 0.055 \\
        \hline
        FW [H-mean]      
        & 0.348 & 0.481 & 0.072 & 0.329 & \textbf{0.396} & 0.076 \\
        SplitFW [$\zo$]   
        & 0.272 & 0.609 & \textbf{0.001} & 0.196 & 0.610  & 0.004 \\
        SplitFW [H-mean] 
        & 0.292 & 0.523 & 0.003 & 0.221 & 0.471 & \textbf{0.003} \\
        \hline
        Noise Correction [Exact]
        & 0.196 & 0.394 & 0.022 & {0.151}  & 0.358 & 0.018 \\
        \hline
    \end{tabular}
    \end{footnotesize}
\vspace{-7pt}
\end{table}


%% file: 9_conclusions.tex
\section{Conclusions}
\label{sec:conclusions}
We have developed a  framework for designing  consistent and efficient algorithms for multiclass performance metrics and constraints that are general functions of the confusion matrix. As instantiations of this framework, we provided four algorithms for optimizing unconstrained metrics, and four analogous counterparts for solving constrained learning problems. In each case, we have shown convergence guarantees for the algorithms under different assumptions on the performance metrics and constraints. 

Our key idea was to reduce the complex learning problem into a sequence of linear minimization problems, for which we recommended an efficient plug-in based approach that applies a post-hoc transformations to a pre-trained class probability model. The results of these linear minimization problems are then combined to return a final classifier. One of the main challenges in instantiating this  idea was to identify optimization algorithms for different problem settings that only required access to a linear minimization oracle (LMO).

We also presented extensive experiments on a variety of multiclass and fairness datasets and demonstrated that the proposed algorithms (despite being limited to performing adjustments to a fixed model) are competitive or better than the state-of-the-art TFCO approach \citep{Cotter+19b} which works with a more flexible hypothesis class. We additionally provided precise guidance for which of the proposed algorithms are best suited for a given multiclass problem, and highlighted scenarios where one might want to use a more expensive LMO that trains a new classifier from scratch at each iteration.

Over the years, the conference versions of this paper have attracted several follow-up works, which have adapted our ideas to 
optimizing multiclass extensions of the F-measure \citep{Weiwei+16},
to balancing accuracy with fairness objectives \citep{Alabi+18}, 
to eliciting multi-class performance metrics \citep{hiranandani2019multiclass},
to training classifiers to optimize more general multi-output classification metrics \citep{wang2019consistent},
to imposing fairness constraints with overlapping protected groups \citep{yang2020fairness}, 
and to optimizing black-box evaluation metrics \citet{hiranandani2021optimizing}.

A number of follow-up directions arise from the proposed framework. First, it would interesting to derive lower bounds on the number of calls to the LMO needed under different assumptions on the performance metrics and constraints. 

Second, while the optimality (and feasibility) gap for most of our proposed algorithms depend linearly on the LMO approximation errors $\rho$ and $\rho'$, the split Frank-Wolfe method (Algorithm \ref{alg:FW-con} alone has a square-root dependence on these parameters. Are these dependencies on the LMO errors optimal or simply artifacts of the analysis?

Third, for algorithms where the  convergence rates have a linear (or quadratic) dependence on the dimension of the problem $d$ (which is typically the same order as the number of classes), how does one extend our framework to handle problems with an extremely large number classes (perhaps under additional structural assumptions on the classes, akin to \citet{ramaswamy2015convex}) and problems with an extremely large number of constraints \citep{narasimhan2020approximate}?

Fourth, our experiments in Section \ref{sec:expts-LMO} show that in some applications, using a flexible LMO that trains a classifier from scratch can yield significant gains over a plug-in based LMO, but this however comes at the cost of added computational time. Can one devise an intermediate approach, where each call to the LMO only needs to run a constant number of optimization steps on a surrogate loss (akin to the TFCO baseline of \cite{Cotter+19b}), while still guaranteeing that the outer algorithm provably convergences to the optimal (feasible) classifier? 

Finally,  except for the bisection method, all the  algorithms we propose rely on the use of a randomized classifier. In some applications, deploying a randomized classifier can be undesirable for  ethical reasons or because of the engineering difficulties it poses. In these scenarios, one could approximate the learned randomized classifier with a deterministic classifier using, for example, the approach of \citet{Cotter+19_stochastic}. Understanding the loss in performance and constraint satisfaction as a result of such de-randomization procedures is an interesting direction for future work.

\section*{Acknowledgements}
The authors thank Aadirupa Saha for providing helpful inputs and for running experiments for a conference version of this paper \citep{Narasimhan+15}. HG thanks the Robert Bosch Centre for Data Science and Artificial Intelligence for their support. HN thanks Pavlos Protopapas, IACS, Harvard University, for providing us access to the MACHO celestial object detection dataset.

%% file: appendix.tex
\appendix

\begin{table}[]
    \caption{Table of notations.}
    \label{tab:symbols}
    \centering
\begin{small}
\begin{tabular}{cl}
\hline
\textbf{Notation} & \textbf{Description}\\
\hline
    $n$ & Number of classes \\
    $N$ & Number of training examples\\
    $m$ & Number of protected groups\\
    $K$ & Number of constraints \\
    $d$ & Dimension of the vector representation for the confusion matrix \\
    $T$ & Number of iterations for the proposed iterative algorithms \\
    $i,j$ & Indices over $n$ classes \\
    $a$ & Index over $m$ protected groups \\
    $k$ & Index over $K$ constraints \\
    $t$ & Index over $T$ iterations\\
    $\ell$ & Index over $N$ training instances \\$\C$ & $n \times n$ Confusion matrix, or an equivalent vector representation of dimension $d=n^2$ \\
    $\L$ & $n \times n$ Loss matrix, or an equivalent vector representation of dimension $d=n^2$ \\
    $\cC$ & Set of achievable confusion matrices, represented by vectors of dimension $d=n^2$\\
    $S$ & Training sample with $N$ instances\\
\hline    
\end{tabular}
\end{small}
\end{table}

\section{Proofs}
\label{app:proofs}

\subsection{Proof of Proposition \ref{prop:opt-classifier-ratio-linear} (Bayes-optimal Classifier for Ratio-of-linear $\psi$)}
\label{app:prop-bayes-ratio-of-linear}
\begin{prop*}[(Restated) Bayes optimal classifier for ratio-of-linear $\psi$] 
Let the performance measure $\psi:[0,1]^{d}\>\R_+$ in \emph{\ref{eq:opt-unconstrained}} be of the form $\psi^\rl(\C) = \frac{\langle \A,\C \rangle}{\langle \B, \C \rangle}$ for some $\A,\B\in\R^{d}$ with $\langle \B,\C \rangle > 0 ~\forall \C\in\cC$. 
Let $t^*=\inf_{\C\in\cC} \psi(\C)$ and 
$\L^* = \A - t^*\B$. 
Then any classifier that is optimal for the linear metric $\langle \L^*, \C \rangle$ is also optimal for \emph{\ref{eq:opt-unconstrained}}. 
\end{prop*}

We will find the following lemma useful in the proof of the proposition.
\begin{lem}
\label{lem:ratio-lin-lemma}
 Let $\psi:[0,1]^d\>\R_+$ be such that $\psi(\C) = \frac{\langle \A,\C \rangle}{\langle \B, \C \rangle},$
 for some matrices $\A,\B\in\R^{d}$ with $\langle \B,\C \rangle > 0$ for all $\C\in\cC$. Let $t^*=\inf_{\C\in\cC} \psi(\C)$. Then $\inf_{\C\in\cC} \langle \A - t^* \B, \C \rangle = 0.$
 \end{lem}
\begin{proof}  
Define $\varphi:\R\>\R$ as $\varphi(t) =\inf_{\C\in\cC} \langle \A - t\B, \C \rangle.$ It is easy to see that $\varphi$ (being a point-wise supremum of linear functions) is convex, and hence a continuous function over $\R$. 
Let $t^*=\inf_{\C\in\cC} \psi(\C)$. 
We then have for all $\C\in\cC$,
 \begin{eqnarray*}
  \frac{\langle \A,\C \rangle}{\langle \B, \C \rangle} \,\geq\, t^* ~~~\text{or equivalently}~~~
  \varphi(t^*) \,=\, \langle \A - t^* \B , \C \rangle \,\geq \, 0.
 \end{eqnarray*}
Thus 
\begin{eqnarray}
\varphi(t^*)= \sup_{\C\in\cC} \langle \A - t^*\B, \C \rangle \geq 0 \;. \label{eqn:lem-ratio-lin-1} 
\end{eqnarray}
Also, by continuity of $\frac{\langle \A,\C \rangle}{\langle \B, \C \rangle}$ in $\C$, for any $t<t^*$, there exists $\C\in\cC$ such that 
\begin{eqnarray*}
 \frac{\langle \A,\C \rangle}{\langle \B, \C \rangle} \,>\, t ~~~\text{or equivalently}~~~
 \varphi(t) \,=\, \langle \A - t \B , \C \rangle \,<\, 0.
\end{eqnarray*}
Thus for all $t<t^*$,
$$\varphi(t)= \inf_{\C\in\cC} \langle \A - t\B, \C \rangle < 0 \;.$$

Next, by continuity of $\varphi$, for any monotonically increasing sequence of real numbers $\{t_i\}_{i=1}^\infty$ converging to $t^*$, we have that $\varphi(t_i)$ converges to $\varphi(t^*)$; since for each $t_i$ in this sequence $\varphi(t_i) < 0$, at the $t^*$, we have that $\varphi(t^*)\geq 0$. Along with \eqref{eqn:lem-ratio-lin-1}, this gives us
$$\inf_{\C\in\cC} \langle \A - t^*\B, \C \rangle = \varphi(t^*) = 0,$$
as desired.
\end{proof}

We are now ready to prove Proposition \ref{prop:opt-classifier-ratio-linear}.
\begin{proof}[Proof of Proposition \ref{prop:opt-classifier-ratio-linear}]
Let $\h^*:\X\>\Delta_n$ be a classifier that is optimal for $\L^* = \A - t^*\B$, i.e., 
\[
  \langle \A - t^* \B, \C[\h^*] \rangle 
  \,=\,
  \inf_{\C\in\cC} \langle \A - t^* \B, \C \rangle.
\]
 Note from Lemma \ref{lem:ratio-lin-lemma} that $\langle \A - t^* \B, \C[\h^*] \rangle  = 0$. Hence, 
\begin{eqnarray*} 
\frac{\langle \A, \C[\h^*]\rangle}{\langle \B, \C[\h^*] \rangle} \,=\, t^*,
\text{or equivalently},~~~~
\psi(\C[\h^*]) 
 \,=\, 
 \inf_{\C\in\cC} \psi(\C),
\end{eqnarray*}
which shows that $h^*$ is also $\psi$-optimal.
\end{proof}

\subsection{Proof of Proposition \ref{prop:opt-classifier-monotonic} (Bayes-optimal Classifier for Monotonic $\psi$)}
\label{app:proof-opt-classifier-monotonic}
\begin{prop*}[(Restated) Bayes-optimal classifier for monotonic $\psi$]
Let the performance measure $\psi:[0,1]^{d}\>\R_+$ in \emph{\ref{eq:opt-unconstrained}} be differentiable and bounded, and be monotonically decreasing in $C_{ii}$ for each $i$ and non-decreasing in $C_{ij}$ for all $i,j$. Let   $\boldeta(X)$ 
be a continuous random vector. 
Then there exists a loss matrix $\L^*$ (which depends on $\psi$ and $D$) such that any classifier that is optimal for the linear metric given by $\L^*$ is also optimal for \emph{\ref{eq:opt-unconstrained}}.
\end{prop*}
Let $\overline{\cC}$ denote the closure of $\cC$.
 We will find the following lemma crucial to our proof.
\begin{lem}
\label{lem:opt-in-closure-same-as-opt-general}
Let 
$\eta(X)$ 
be a continuous random vector. 
Let $\L\in\R^d$ be such that no two columns are identical. Then,
\[
\argmin_{\C\in\overline{\cC}}\, \langle \L, \C \rangle \,=\, \argmin_{\C\in\cC} \langle \L, \C \rangle.
\]
Moreover, the above set is a singleton. 
\end{lem}

The proof for the lemma is highly technical and can be found in the conference version of the paper, specifically Lemma 12 in \citet{Narasimhan+15}.

\begin{proof}[Proof of Proposition \ref{prop:opt-classifier-monotonic}]
Let $\C^*=\argmin_{\C\in\overline{\cC}} \psi(\C)$. Such a $\C^*$ always exists by compactness of $\overline{\cC}$ and continuity of $\psi$. 
 By first order optimality, and convexity of $\overline{\cC}$, we have that for all $\C\in\overline{\cC}$
 $$ \langle \nabla\psi(\C^*), \C^* \rangle \leq \langle \nabla\psi(\C^*), \C \rangle \;.$$
 For $\L^* = \nabla\psi(\C^*)$,  we have that
 $\C^* \in \argmin_{\C\in\overline{\cC}} \langle \L^*, \C \rangle.$
 
 Due to the monotonicity condition on $\psi$ the diagonal elements of its gradient $\nabla\psi(\C^*)$ are positive, and the off-diagonal elements are non-positive, and hence no two columns of  $\L^*$ are identical. Thus by a direct application of Lemma \ref{lem:opt-in-closure-same-as-opt-general}, we have that $\C^*\in\cC$, and moreover $\C^*$ is the unique minimizer of $\langle \L^*, \C \rangle$ over all $\C \in \cC$.
\end{proof}

\subsection{Proof of Proposition \ref{prop:opt-classifier-constrained} (Bayes-optimal Classifier for Continuous $\psi, \phi_1,\ldots, \phi_K$)}
\label{app:prop-bayes-constrained}
\begin{prop*}[Bayes optimal classifier for continuous $\psi, \phi_1,\ldots, \phi_K$]
Let the performance measure $\psi: [0,1]^{d}\>\R_+$ and the constraint functions $\phi_1, \ldots, \phi_K: [0,1]^{d}\>\R$ in \emph{\ref{eq:opt-constrained}} be continuous and bounded. Then there exists $d+1$ loss matrices $\L^*_1, \L^*_2, \ldots, \L^*_{d+1}$ (which can depend on $\psi, \phi_k$'s and $D$) 
such that an optimal classifier for \emph{\ref{eq:opt-constrained}} can be expressed as a randomized combination of the deterministic classifiers $h_1, h_2, \ldots, h_{d+1}$, where $h_i$ is optimal for the linear metric 
given by $\L^*_i$.
\end{prop*}
\begin{proof}
We first note that $\cC$ is a compact set, and as a result there exists an optimal-feasible solution $\C^*$ for \ref{eq:opt-constrained} in $\cC$. It is straightforward to see that $\cC$ is bounded. To see that $\cC$ is closed, 
note that from Proposition \ref{prop:loss-opt},  any linear function over $\cC$ is minimized by some confusion matrix in $\cC$. Since every extreme point is a unique minimizer of a linear function in $\cC$, every extreme point of  $\cC$ is also in $\cC$. It follows that all convex combinations of the extreme points of $\cC$ are also in $\cC$ and as a result, so are all its limit points.


Now to prove the characterization in the proposition, let us use $\cE$ to denote the set of extreme points of $\cC$. By Krein–Milman theorem, we have that because $\cC$ is a compact convex set, it is equal to the convex hull of $\cE$ \citep{narici2010topological}. We further have from Carath\'{e}odory's theorem that any $\C \in \cC$ can be the be expressed as a convex combination of $d+1$ points in $\cE$ \citep{steinitz1913bedingt}. Since $\C^*$ is in $\cC$, we have that there exists $d+1$ confusion matrices $\C^*_1, \ldots, \C^*_{d+1} \in \cE$ and coefficients $\balpha \in \Delta_{d+1}$ such that $\C^* = \sum_{i=1}^{d+1}\alpha_i\,\C^*_i$. 

Next, because each $\C^*_i$ is an extreme point of $\cC$, there exists a supporting hyperplane $\L^*_i$ for $\cC$ at $\C^*_i$ such that $\langle \L^*_i, \C^*_i \rangle \leq \langle \L^*_i, \C \rangle,\, \forall \C \in \cC$. Moreover, owing to $\C^*$ being an extreme point, it is also the unique minimizer of the linear function $\langle \L^*_i, \C \rangle$. Furthermore, because Proposition \ref{prop:loss-opt} tells us that there exists a deterministic classifier $h_i$ that is optimal for the linear metric $\langle \L^*_i, \C \rangle$, it follows that  $\C[h^*_i] = \C^*_i$. Consequently, we have that the randomized classifier $h^* = \sum_{i=1}^m \alpha_i h_i$ is an optimal feasible solution for \ref{eq:opt-constrained}, i.e. $\C[h^*] = \sum_{i=1}^m\alpha_i \C[h_i] = \sum_{i=1}^m\alpha_i \C^*_i = \C^*$.
\end{proof}

\subsection{Proof of Theorem \ref{thm:FW-unc} (Frank-Wolfe for Unconstrained Problems)}
\label{app:proof-fw}
\begin{thm*}[(Restated) Convergence of FW algorithm]
Fix $\epsilon \in (0,1)$. 
Let $\psi: [0,1]^d \> [0,1]$ be convex, and $\beta$-smooth and $L$-Lipschitz w.r.t.\ the $\ell_2$-norm. Let $\Omega$ in Algorithm \ref{alg:FW} be a $(\rho , \rho', \delta)$-approximate LMO for sample size $m$.
Let $\bar{h}$ be a classifier returned by Algorithm \ref{alg:FW} when run for $T$ iterations. 
 Then with probability $\geq 1 - \delta$ over draw of $S \sim D^N$, after $T = \cO(1/\epsilon)$ iterations:
\[
\psi(\C[\bar{h}]) \,\leq\, \min_{\C \in \cC}\,\psi(\C) \,+\,\cO\left(\beta\epsilon \,+\, L\rho +\beta\sqrt{d}\rho'\right).
\]
\end{thm*}

We first prove an important lemma where we bound the approximation error of the linear minimization oracle used in the algorithm. This result coupled with the standard convergence analysis for the Frank-Wolfe method \citep{Jaggi13} will then allow us to prove the above theorem.

\begin{lem}
 \label{lem:approximation-factor-FW}
 Let $\psi: [0, 1]^d \> \R_+$ be convex over $\cC$, and $L$-Lipschitz and $\beta$-smooth w.r.t. the $\ell_2$ norm. Let classifiers $\tilde h^1,\ldots,\tilde h^T$, and $h^0,h^1,\ldots,h^T$ be as defined in Algorithm \ref{alg:FW}. Then for any $\delta \in (0, 1]$, with probability at least $1-\delta$ (over draw of $S$ from $D^N$), we have for all $1\leq t \leq T$
 \begin{eqnarray*}
  \langle \nabla\psi(\C[h^{t-1}]), \C[\tilde h^t] \rangle \leq \min_{\g:\X:\>\Delta_n} \langle \nabla\psi(\C[h^{t-1}]), \C[\g] \rangle + L\rho + 2\beta\sqrt{d}\rho'.
 \end{eqnarray*}
\end{lem}
\begin{proof}  
For any $1\leq t \leq T$, let 
 $\g^{t,*}\in \argmin_{\g:\X\>\Delta_n} \langle \nabla\psi(\C[h^{t-1}]), \C[\g] \rangle.$ We then have
\begin{eqnarray*}
\lefteqn{
\min_{\g:\X:\>\Delta_n} \langle \nabla\psi(\C[h^{t-1}]), \C[\g] \rangle
\,-\,
\langle \nabla\psi(\C[h^{t-1}]), \C[\tilde h^t] \rangle}
\hspace{1cm}
\\
&=&
\langle \nabla\psi(\C[h^{t-1}]), \C[\g^{t,*}] \rangle
\,-\,
\langle \nabla\psi(\C[h^{t-1}]), \C[\tilde h^t] \rangle
\\
&=&
\underbrace{
\langle \nabla\psi(\C[h^{t-1}]), \C[\g^{t,*}] \rangle
\,-\,
\langle \nabla\psi(\C^{t-1}), \C[ g^{t,*}] \rangle
}_{\text{term}_1}\\
&& \hspace{2cm}
\,+\,
\underbrace{
\langle \nabla\psi(\C^{t-1}), \C[ g^{t,*}] \rangle
\,-\,
\langle \nabla\psi(\C^{t-1}), \C[\tilde h^t] \rangle
}_{\text{term}_2}\\
&& \hspace{2cm}
\,+\,
\underbrace{
\langle \nabla\psi(\C^{t-1}), \C[\tilde h^t] \rangle
\,-\,
\langle \nabla\psi(\C[h^{t-1}]), \C[\tilde h^t] \rangle}
_{\text{term}_3}.
\end{eqnarray*}

We next bound each of these terms. We start with $\text{term}_2$. For any $1\leq t \leq T$, let $\L^t$ be as defined in Algorithm \ref{alg:FW}. 
For all $1\leq t \leq T$, 
 \begin{eqnarray*}
{
\langle \nabla\psi(\C^{t-1}), \C[\g^{t,*}] \rangle - \langle \nabla\psi(\C^{t-1}), \C[\tilde h^t] \rangle
}
&=&
\|\nabla\psi(\C^{t-1})\|_\infty ( \langle \L^t, \C[\tilde h^t] \rangle - \langle \L^t, \C[\g^{t,*}] \rangle ) \nonumber \\
&\leq&
\|\nabla\psi(\C^{t-1})\|_2 ( \langle \L^t, \C[\tilde h^t] \rangle - \langle \L^t, \C[\g^{t,*}] \rangle ) \nonumber \\
&\leq&
L\rho, \label{eqn:lem-apx-1}
\end{eqnarray*}
which follows from the property of the LMO (in Definition \ref{defn:lmo}) and from $L$-Lipchitzness of $\psi$, and holds with probability at least $1-\delta$ (over draw of $S$).

Next, for $\text{term}_1$, we have by an application of Holder's inequality
\begin{eqnarray*}
\lefteqn{\langle \nabla\psi(\C[h^{t-1}]), \C[\g^{t,*}] \rangle
\,-\,
\langle \nabla\psi(\C^{t-1}), \C[ g^{t,*}] \rangle}
\hspace{3cm}
\\
&\leq&
\big\|\nabla\psi(\C^{t-1}) - \nabla\psi(\C[h^{t-1}]) \big\|_\infty
\|\C[\g^{t,*}]\|_1\\
&=& \big\|\nabla\psi(\C^{t-1}) - \nabla\psi(\C[h^{t-1}]) \big\|_\infty (1)\\
&=& \big\|\nabla\psi(\C^{t-1}) - \nabla\psi(\C[h^{t-1}]) \big\|_2\\
&\leq&
 \beta \big\| \C^{t-1} - \C[h^{t-1}] \big\|_2 \nonumber \\
&\leq&
 \beta \sqrt{d}\big\| \C^{t-1} - \C[h^{t-1}] \big\|_\infty \nonumber \\
 &\leq&
 \beta \sqrt{d} \rho', \label{eqn:lem-apx-2}
\end{eqnarray*}
where the third step follows from $\beta$-smoothness of $\psi$; the last step uses the property of the LMO and holds with probability at least $1-\delta$ (over draw of $S$).
One can similarly bound $\text{term}_3$. We thus have for all $1\leq t \leq T$, with probability at least $1-\delta$ (over draw of $S$),
\begin{eqnarray*}
{\max_{\g:\X\>\Delta_n} \langle \nabla\psi(\C[h^{t-1}]), \C[\g] \rangle  - \langle \nabla\psi(\C[h^{t-1}]), \C[\tilde h^t] \rangle }
 &\leq&
L \rho + 2\beta\sqrt{d}\rho',
\end{eqnarray*}
as desired.
\end{proof}

We are now ready to prove Theorem \ref{thm:FW-unc}.
\begin{proof}[Proof of Theorem \ref{thm:FW-unc}]
Our proof shall make use of 
the standard convergence result for the Frank-Wolfe algorithm for minimizing a convex function over a convex set \citep{Jaggi13}. We will find it useful to first define the following quantity, referred to as the curvature constant in \cite{Jaggi13}. 
\begin{eqnarray*}
 C_\psi 
 &=&
 \sup_{\C_1,\C_2\in\cC, \gamma\in[0,1]} \frac{2}{\gamma^2} \Big( \psi\big(\C_1+\gamma(\C_2-\C_1)\big) -\psi\big(\C_1\big) - \gamma \big\langle \C_2-\C_1,\nabla\psi(\C_1) \big\rangle \Big).
 \end{eqnarray*}
Also, define two positive scalars $\epsilon_S$ and $\delta_\text{apx}$ required in the analysis of \cite{Jaggi13}:
\begin{eqnarray*}
\epsilon_S &=& L\rho + 2\beta\sqrt{d}\rho'\\
\delta_\text{apx} &=& \frac{(T+1)\epsilon_S}{C_\psi},
\end{eqnarray*}
where $\delta \in (0, 1]$ is as in the theorem statement. Further, let the classifiers $\tilde h^1,\ldots,\tilde h^T$, and $h^0,\ldots,h^T$ be as defined in Algorithm \ref{alg:FW}.
We then have from Lemma \ref{lem:approximation-factor-FW} that the following holds with probability at least $1-\delta$, for all $1\leq t \leq T$,
\begin{eqnarray}
 \left\langle \nabla\psi\left(\C[h^{t-1}]\right), \C[\tilde h^t] \right\rangle 
 &\leq& \min_{\g:\X:\>\Delta_n} \left\langle \nabla\psi\left(\C\left[h^{t-1}\right]\right), \C\left[\g\right] \right\rangle + \epsilon_S \nonumber \\
 &=& \min_{\C\in\cC} \left\langle \nabla\psi\left(\C\left[h^{t-1}\right]\right), \C \right\rangle + \epsilon_S \nonumber \\
 &=& \min_{\C\in\ocC} \left \langle \nabla\psi\left(\C\left[h^{t-1}\right]\right), \C \right \rangle + \frac{1}{2}\delta_\text{apx} \frac{2}{T+1} C_\psi \nonumber \\
 &\leq& \min_{\C\in\ocC} \langle \nabla\psi\left(\C\left[h^{t-1}\right]\right), \C \rangle + \frac{1}{2}\delta_\text{apx} \frac{2}{t+1} C_\psi \;. \label{eqn:FW-iterates-prop-2}
\end{eqnarray}
Also observe that for the two sequences of iterates given by the confusion matrices of the above classifiers,
\begin{eqnarray}
 \C[h^t] &=& \left(1-\frac{2}{t+1}\right)\C[h^{t-1}] + \frac{2}{t+1} \C[\tilde h^t], \label{eqn:FW-iterates-prop-1} \;.
\end{eqnarray}
for all $1\leq t \leq T$. Based on \eqref{eqn:FW-iterates-prop-2} and \eqref{eqn:FW-iterates-prop-1}, one can now apply the result of \cite{Jaggi13}.

In particular, the sequence of iterates $\C[h^0],\C[h^1],\ldots,\C[h^T]$ can be considered as the sequence of iterates arising from running the Frank-Wolfe optimization method to minimize $\psi$ over $\overline\cC$ with a linear optimization oracle that is $\frac{1}{2}\delta_\text{apx} \frac{2}{t+1} C_\psi$ accurate at iteration $t$. Since $\psi$ is a convex function over the convex constraint set $\ocC$, one has from  Theorem 1 in \cite{Jaggi13} that the following convergence guarantee holds with probability at least $1-\delta$:
\begin{eqnarray}
\psi(\C[\bar{h}]) ~=~ \psi(\C[h^T])
 &\leq& \min_{\C\in\ocC} \psi(\C) + \frac{2C_\psi}{T+2} (1+\delta_\text{apx})  \nonumber\\
 &=& \min_{\C\in\ocC} \psi(\C) + \frac{2C_\psi}{T+2} + \frac{2\epsilon_S (T+1)}{T+2}  \nonumber \\
 &\leq& \min_{\C\in\ocC} \psi(\C) + \frac{2C_\psi}{T+2} +2\epsilon_S  \label{eqn:jaggi}
 \end{eqnarray}

We can further upper bound $C_\psi$ 
in terms of the the smoothness parameter of $\psi$:
\begin{eqnarray*}
 C_\psi 
 &=&
 \sup_{\C_1,\C_2\in\cC, \gamma\in[0,1]} \frac{2}{\gamma^2} \Big( \psi\big(\C_1+\gamma(\C_2-\C_1)\big) -\psi\big(\C_1\big) - \gamma \big\langle \C_2-\C_1,\nabla\psi(\C_1) \big\rangle \Big)\\
 &\leq&
 \sup_{\C_1,\C_2\in\cC, \gamma\in[0,1]} \frac{2}{\gamma^2} \Big( \frac{\beta}{2} \gamma^2||\C_1-\C_2||^2_2 \Big)
~=\,
 4\beta\;,
\end{eqnarray*}
where the second step follows from the $\beta$-smoothness of $\psi$. Substituting back in \eqref{eqn:jaggi}, we finally have with probability at least $1-\delta$,
\begin{eqnarray*}
\psi(\C[\bar{h}])
  &\leq& \min_{\C\in\ocC} \psi(\C) + \frac{8\beta}{T+2} + 2\epsilon_S
  ~=~ \min_{\C\in\ocC} \psi(\C)  + \frac{8\beta}{T+2} + 2L\rho + 4\beta\sqrt{d}\rho'.
\end{eqnarray*}
Setting $T = 1/\epsilon$ completes the proof.
\end{proof}

\subsection{Proof of Theorem \ref{thm:gda-unc} (GDA for Unconstrained Problems)}
\label{app:proof-gda-unc}
\input{gda-proof}

\subsection{Proof of Theorem \ref{thm:ellipsoid} (Ellipsoid For Unconstrained Problems)}
\label{app:proof-ellipsoid-unc}
Theorem \ref{thm:ellipsoid} follows from Theorem \ref{thm:ellipsoid-con} under the special case of $K=0$. The algorithms for the constrained and unconstrained case become identical and the same bounds apply with $r=\infty$. Please see Appendix \ref{app:proof-ellipsoid} for proof of Theorem \ref{thm:ellipsoid-con}.

\subsection{Proof of Theorem \ref{thm:bisection-unc} (Bisection For Unconstrained Problems)}
\label{app:proof-bisection}
Theorem \ref{thm:bisection-unc} follows from Theorem \ref{thm:bisection-con} under the special case of $K=0$ and $T' = 1$.  Please see Appendix \ref{app:proof-bisection-con} for proof of Theorem \ref{thm:bisection-con}.
\subsection{Proof of Theorem \ref{thm:FW-con} (SplitFW for Constrained Problems)}
\label{app:proof-split-fw}
\input{10_SBFW_Proof}
\subsection{Proof of Theorem \ref{thm:gda-con} (GDA for Constrained Problems)}
\label{app:proof-gda-con}
\input{con-gda-proof}

\subsection{Proof of Theorem \ref{thm:ellipsoid-con} (Ellipsoid for Constrained Problems)}
\label{app:proof-ellipsoid}
\input{con-ellipsoid-proof}

\subsection{Proof of Theorem \ref{thm:bisection-con} (Bisection for Constrained Problems)}
\label{app:proof-bisection-con}
\begin{thm*}[(Restated) Convergence of ConBisection algorithm]
Fix $\epsilon \in (0,1)$. 
Let $\psi: [0,1]^d \> [0,1]$ be such that $\psi(\C) \,=\, \frac{\langle \A, \C \rangle}{\langle \B, \C \rangle}$, where $\A, \B \in [0,1]^{d}$, 
and $\min_{\C \in \cC} {\langle \B, \C \rangle} \,=\, b$ for some $b > 0$. Let $\phi_1,\ldots,\phi_K: [0,1]^d\>[-1,1]$ be convex and $L$-Lipschitz w.r.t.\ the $\ell_2$-norm. Let $\Omega$ in Algorithm \ref{alg:bisection-con} be a $(\rho , \rho', \delta)$-approximate LMO for sample size $N$. 
Suppose the strict feasibility condition in Assumption \ref{assp:strict-feasibility} holds for radius $r>0$.
Let $\Lambda$, $\Xi$, $\eta$ and $\eta'$ in the call to Algorithm \ref{alg:GDA-con} be set as in Theorem \ref{thm:gda-con} with Lipschitz constant $L' = \max\{L, \|\A\|_2 + \|\B\|_2\}$. 
Let $\bar{h}$ be a classifier returned by Algorithm \ref{alg:bisection-con} when run for $T$ outer iterations and $T'$ inner iterations. 
 Then with probability $\geq 1 - \delta$ over draw of $S \sim D^N$, after $T = \log(1/\epsilon)$ outer iterations and $T'=\cO(K/\epsilon^2)$ inner iterations:
\[
\textbf{Optimality}:~~\psi(\C[\bar{h}]) \,\leq\, \min_{\C\in\cC:\, \bphi(\C) \leq \0}\,\psi(\C) \,+\,\cO\left(\kappa(\epsilon + \rho^\eff)\right);
\]
\[
\textbf{Feasibility}:~~\phi_k(\C[\bar{h}]) \,\leq\, \cO\left(L'(\epsilon + \rho^\eff)\right),~\forall k \in [K],
\]
where $\kappa = L'/b$ and  $\rho^\eff =  {\rho} + \sqrt{d}{\rho'}$.
\end{thm*}

The proof follows similar steps as that for Theorem \ref{thm:bisection-unc}. 
 We will first state a couple of lemmas:
\begin{lem}[{Invariant in Algorithm \ref{alg:bisection-con}}]
\label{lem:lower-upper-bounds-con}
Under the assumptions made in Theorem \ref{thm:bisection-con}, the following invariant is true at the end of each iteration $0 \leq t \leq T$ of Algorithm \ref{alg:bisection-con}:
\[
\alpha^t  - \cO\left(\kappa (\epsilon + \rho^\eff)\right) \,\leq\, \min_{\C\in\cC:\, \phi_k(\C) \leq 0,\forall k} \psi(\C) \,\leq\, \psi(\C[h^t])  \,<\, \beta^t + \cO\left(\kappa (\epsilon + \rho^\eff)\right);
\]
\[
\phi_k(\C[\bar{h}]) \,\leq\, \cO\left(L'(\epsilon + \rho^\eff)\right),~\forall k \in [K].
\]
where $L'$, $\kappa$, and $\rho^\eff$ are defined as in Theorem \ref{thm:bisection-con}.
\end{lem}
\begin{proof}
We shall prove this lemma by mathematical induction on the iteration number $t$. For $t=0$, the invariant holds trivially as $0 \leq \psi(\C[h^0]) \leq 1$ and $h^0$ satisfies the constraints. Assume the invariant holds at the end of iteration $t-1 \in \{0, \ldots, T-1\}$; we shall prove that the invariant holds at the end of iteration $t$. 

First note that the 
linear function $\psi'(\C) =  \langle \A - \gamma^t\B, \,\C \rangle$
in step 6 of the algorithm is Lipschitz w.r.t.\ the $\ell_2$-norm with Lipschitz parameter of at most $\|\A - \gamma^t\B\|_2 \leq \|\A\|_2 + \|\B\|_2 \leq  L'$. We then have from Theorem \ref{thm:gda-con} that the
classifier $g^t$ returned by the ConGDA algorithm (Algorithm \ref{alg:GDA-con}) after $T' = \cO(K/\epsilon^2)$ runs enjoys the following guarantee:
\begin{eqnarray}
\langle \A - \gamma^t\B,\,\C[{g}^t]) \rangle
&\leq& \min_{\C\in\cC:\, \phi_k(\C) \leq 0,\forall k}\, \langle \A - \gamma^t\B, \,\C \rangle \,+\, \cO\left(L'(\epsilon + \rho^\eff)\right);
\label{eq:bs-splitfw-obj}
\\
\phi_k(\C[g^t]) &\leq& \cO\left(L'(\epsilon + \rho^\eff)\right),~\forall k \in [K],
\label{eq:bs-splitfw-con}
\end{eqnarray}
where we have used the fact that both $\psi'(\C) = \langle \A - \gamma^t\B, \,\C \rangle$ and $\phi_k(\C)$s are $L'$-Lipschitz w.r.t.\ the $\ell_2$-norm. 
We further have that from the property of the LMO (Definition \ref{defn:lmo}) used in turn by Algorithm \ref{alg:FW-con} that:
\begin{equation}
    \|\C^t - \C[g^t]\|_\infty \leq \rho'
    \label{eq:bs-lmo}
\end{equation}

We now consider two cases at iteration $t$. In the first case, $\psi(\C^t) \leq \gamma^t$, leading to the assignments $\alpha^t = \alpha^{t-1}$, $\beta^{t} = \gamma^t$, and $h^{t} = g^t$. We have from \eqref{eq:bs-lmo} that:
\begin{eqnarray*}
\langle \A - \gamma^t\B,\,\C[{g}^t]) \rangle
&\leq& \langle \A - \gamma^t\B, \,\C^t \rangle
  \,+\, \|\A - \gamma^t\B\|_1\rho'\\
&\leq& \langle \B,\,\C^t \rangle \big(\psi^\rl(\C^t) \,-\, \gamma^t\big)
  \,+\, \|\A - \gamma^t\B\|_1\rho'\\
&\leq& 0 \,+\, \|\A - \gamma^t\B\|_1\rho'\\
&\leq& (\|\A\|_2 + \|\B\|_2)\sqrt{d}\rho' 
\,\leq\, L'\sqrt{d}\rho'
\,<\, L'(\epsilon + \rho^\eff),
\end{eqnarray*}
where the third step follows from our case assumption that $\psi^\rl(\C^t) \leq \gamma^t$ and $\langle \B,\,\C^t \rangle > 0$, and the last step follows from triangle inequality and $0 \leq \gamma^t \leq 1$. 
The above inequality further gives us
\begin{eqnarray*}
\frac{\langle  \A,\, \C[{g}^t] \rangle}{\langle \B,\, \C[{g}^t] \rangle} 
\,<\, \gamma^t \,+\, \cO\left(\frac{L'}{b}(\epsilon + \rho^\eff)\right)
\,<\, \beta^t \,+\, \cO(\kappa (\epsilon + \rho^\eff)).
\end{eqnarray*}
 In other words,
\[
\psi^\rl\big(\C[h^{t}]\big) \,=\, \psi^\rl(\C[{g}^t]) \,=\, \frac{\langle \A,\, \C^D[{g}^t] \rangle}{\langle \B,\, \C^D[{g}^t] \rangle}  \,<\, \beta^t \,+\, \cO(\kappa (\epsilon + \rho^\eff)).
\]
Moreover, by our assumption that the invariant holds at the end of iteration $t-1$, we have
\[
\alpha^{t} \,-\, \cO\left(\kappa(\epsilon + \rho^\eff)\right) \,=\, \alpha^{t-1} \,-\, \cO\left(\kappa(\epsilon + \rho^\eff)\right)  \,\leq\, \min_{\C \in \cC}\psi^\rl(\C) \,\leq\, \psi^\rl(\C[h^{t}]) \,<\, \beta^t \,+\, \cO(\kappa(\epsilon + \rho^\eff)).
\]
Further, from \eqref{eq:bs-splitfw-con}, $\phi_k(\C[h^t]) = \phi_k(\C[g^t]) \leq \cO(L'\bar{\rho}), \forall k$. 
Thus under the first case, the invariant holds at the end of iteration $t$.

In the second case, $\psi^\rl(\C^t) > \gamma^t$ at iteration $t$, which would lead to the assignments $\alpha^t = \gamma^{t}$, $\beta^{t} = \beta^{t-1}$, and $h^{t} = h^{t-1}$. Since the invariant is assumed to hold at the end of iteration $t-1$, we have
\begin{equation}
\label{eqn:case-2-partial-con}
\beta^t + \cO\left(\kappa(\epsilon + \rho^\eff)\right) \,=\, \beta^{t-1} + \cO(\kappa(\epsilon + \rho^\eff))  \,>\, \psi^\rl(\C[h^{t-1}]) \,=\, \psi^\rl(\C[h^{t}]).
\end{equation}
Next for $\C^* \,\in\, \underset{\C\in\cC:\, \bphi(\C) \leq \0}{\argmin}\,\, \langle \A - \gamma^t\B, \,\C \rangle$, we have from \eqref{eq:bs-splitfw-obj},
\begin{eqnarray*}
\langle \A - \gamma^t\B, \,\C^* \rangle 
&=& \langle \A - \gamma^t\B, \,\C[h^t] \rangle  \,-\, \cO(L'(\epsilon + \rho^\eff))\\
&\geq& \langle \A - \gamma^t\B, \,\C^t \rangle \,-\, \|\A - \gamma^t\B\|_1\|\C^t - \C[h^t]\|_\infty \,-\, \cO(L'(\epsilon + \rho^\eff))\\
&\geq& \langle \A - \gamma^t\B, \,\C^t \rangle \,-\, \|\A - \gamma^t\B\|_1\rho' 
\,-\, \cO(L'(\epsilon + \rho^\eff))\\
&=& \langle\B, \, \C^t\rangle(\psi(\C^t) - \gamma^t)
  \,-\, \|\A - \gamma^t\B\|_1\rho' \,-\, \cO(L'(\epsilon + \rho^\eff))\\
&\geq& \langle\B, \, \C^t\rangle(0)
  \,-\, \|\A - \gamma^t\B\|_1\rho' \,-\, \cO(L'(\epsilon + \rho^\eff))\\
&\geq& -(\|\A\|_2 + \|\B\|_2)\sqrt{d}\rho' 
\,-\, \cO(L'(\epsilon + \rho^\eff)) \,=\, -\cO(L'(\epsilon + \rho^\eff)),
\end{eqnarray*}
where the second step follows from the property of the LMO, the second step follows from Holder's inequality, the third step using \eqref{eq:bs-lmo}, the forth step follows from our case assumption that $\psi^\rl(\C^t) \geq \gamma^t$ and $\langle\B, \, \C^t\rangle > 0$, the last step follows 
from the fact that $\|\mathbf{z}\|_1 \geq \|\mathbf{z}\|_2$ and from the triangle inequality and that $0 \leq \gamma^t \leq  1$. 
%
In particular, we have for all $\C \in \cC$ such that $\bphi(\C) \leq \0$,
\[
\langle \A - \gamma^t\B, \,\C \rangle  \,\geq\, -\cO(L'(\epsilon + \rho^\eff)),
\]
or
\[
\frac{\langle \A, \,\C \rangle}{\langle \B, \,\C \rangle} \,\geq\, \gamma^t 
-\cO\left(\frac{L'}{\langle \B, \,\C \rangle}(\epsilon + \rho^\eff)\right)
\,\geq\,
\gamma^t 
-\cO\left(\frac{L'}{b}(\epsilon + \rho^\eff)\right)
\,=\,
\gamma^t 
-\cO\left(\kappa(\epsilon + \rho^\eff)\right).
\]
or
\[
\psi^\rl(\C) \,\geq\, \gamma^t \,-\, \cO\left(\kappa(\epsilon + \rho^\eff)\right) \,=\, \alpha^t -\cO\left(\kappa(\epsilon + \rho^\eff)\right).
\]	
By combining the above with \eqref{eqn:case-2-partial-con}, we can see that the invariant holds in iteration $t$ under this case as well. 
We also have $\phi_k(\C[h^t]) = \phi_k(\C[h^{t-1}]) \leq \cO\left(L'(\epsilon + \rho^\eff)\right), \forall k$.
This completes the proof of the lemma.
\end{proof}

\begin{lem}[{Multiplicative Progress in Each Iteration of Algorithm \ref{alg:bisection-con}}]
\label{lem:multiplicative-progress-con}
Let $\psi$ be as defined in Theorem \ref{thm:bisection-con}. Then the following is true in each iteration $1 \leq t \leq T$ of Algorithm \ref{alg:bisection-con}: 
\[
\beta^t - \alpha^t ~=~ \frac{1}{2}\big(\beta^{t-1} - \alpha^{t-1}\big)	.
\]
\end{lem}
\begin{proof}
We consider two cases in each iteration of Algorithm \ref{alg:bisection-con}. If in an iteration $t \in \{1, \ldots, T\}$, 
$\psi^\rl(\C^t) \leq \gamma^t$, leading to the assignment $\beta^t = \gamma^t$, then
\begin{eqnarray*}
\beta^t - \alpha^t &=& \gamma^t - \alpha^{t-1}
~=~
\frac{\alpha^{t-1} + \beta^{t-1}}{2} - \alpha^{t-1}
~=~
\frac{1}{2}(\beta^{t-1} - \alpha^{t-1}).
\end{eqnarray*}
On the other hand, if
$\psi^\rl(\C^t) > \gamma^t$, leading to the assignment $\alpha^t = \gamma^t$, then
\begin{eqnarray*}
\beta^t - \alpha^t &=& \beta^{t-1} - \gamma^t
~=~
\beta^{t-1} - \frac{\alpha^{t-1} + \beta^{t-1}}{2}
~=~
\frac{1}{2}\big(\beta^{t-1} - \alpha^{t-1}\big).
\end{eqnarray*}
Thus in both cases, the statement of the lemma is seen to hold.
\end{proof}


We are now ready to prove Theorem \ref{thm:bisection-con}.
\begin{proof}[Proof of Theorem \ref{thm:bisection-con}]
For the classifier $\bar{h} = h^T$ output by Algorithm \ref{alg:bisection-con} after $T$ iterations, we have from \Lem{lem:lower-upper-bounds-con},
\begin{eqnarray*}
\psi^\rl\big(\C[h^T]\big) \,-\,
\min_{\C \in \cC} \psi^\rl(\C)
&< & 
\beta^t - \alpha^t \,+\, \cO(\kappa(\epsilon + \rho^\eff))
\\
&\leq& 2^{-T}\big(\beta^0 - \alpha^0\big) \,+\, \cO(\kappa(\epsilon + \rho^\eff))
\\
&=& 2^{-T}\big(1 - 0\big) \,+\, \cO(\kappa(\epsilon + \rho^\eff))\\
&=& 2^{-T} \,+\, \cO(\kappa(\epsilon + \rho^\eff)),
\end{eqnarray*}
where 
the second step  follows from Lemma \ref{lem:multiplicative-progress-con}. We additionally have from Lemma \ref{lem:lower-upper-bounds-con}, 
$\phi_k(\C[\bar{h}]) \,\leq\, \cO\left(L'(\epsilon + \rho^\eff)\right),~\forall k \in [K].$ Setting $T = \log(1/\epsilon)$ completes the proof.
\end{proof}

\subsection{Proof of Theorem \ref{thm:plug-in}}
\label{app:proof-plugin}
\begin{thm*}[(Restated) Regret bound for plug-in LMO]
Fix $\delta \in (0, 1)$. Then with probability $\geq 1 - \delta$ over draw of sample $S \sim D^N$, for any  loss matrix $\L \in \R_+^d$ with $\|\L\|_\infty = 1$, the classifier and confusion matrix $(\widehat{g}, \widehat{\boldsymbol{\Gamma}})$ returned by Algorithm \ref{alg:plug-in} satisfies:
\begin{equation*}
\langle \L, \C[\widehat{g}]\rangle \,\leq\, \min_{h:\X\>\Delta_n}\langle \L, \C[h] \rangle + \E_X\big[\big\|\widehat{\boldeta}(X) \,-\, \boldeta(X)\big\|_1\big];
\vspace{-5pt}
\label{eq:lmo1}
\end{equation*}
\begin{equation*}
\|\C[\widehat{g}] \,-\, \hat{\bGamma}\|_\infty \,\leq\, \cO\bigg(\sqrt{\displaystyle\frac{d\log(n)\log(N) + \log(d/\delta)}{N}}\bigg).
\label{eq:lmo2}
\end{equation*}
\end{thm*}

\begin{proof}
For simplicity, we will represent both $\L$ and $\C$ as $n \times n$ matrices instead of flattened $n^2$-dimensional vectors.
Let us denote the columns of $\L$ by  $\bell_1, \ldots, \bell_n$, where $\ell_j = [L_{1, j}, L_{2, j}, \ldots, L_{n, j}]^\top$.

We can then re-write:
\begin{align*}
\langle \L , \C[{h}] \rangle 
    &=
    \sum_{i,j}  L_{ij} \,C_{ij}[h] 
    ~=\,
    \sum_{i,j} \E_{X}\left[
        \eta_i(X)\, L_{ij}\, \1(h(X) = j)
        \right]\\
    &=\, \sum_{j=1}^n \E_{X}\left[\1(h(X) = j)\,\boldeta(X)^\top \bell_j\right]
    ~=\,  \E_{X}\left[\boldeta(X)^\top \bell_{h(X)}\right].
\end{align*}

Let $h^*$ be the Bayes-optimal classifier for the linear metric $\langle \L , \C[\widehat{h}] \rangle$. 
For the first part, we bound the $\L$-regret as follows:
\begin{eqnarray*}
\lefteqn{\langle \L , \C[\widehat{g}] \rangle - \langle \L , \C[h^*] \rangle }\\
&=&
\E_X\big[
\boldeta(X)^\top \bell_ {\widehat g(X)}\big]  - \E_X  \big[\boldeta(X)^\top \bell_ {h^*(X)}\big] \\[0.1em]
&=&
\E_X  \big[\hat\boldeta(X)^\top \bell_ {\widehat g(X)}\big] + \E_X  \big[(\boldeta(X)-\hat\boldeta(X)^\top \bell_ {\widehat g(X)}\big] - \E_X  \big[\boldeta(X)^\top \bell_ {h^*(X)}\big] \\[0.1em]
&\leq&
\E_X  \big[\hat\boldeta(X)^\top \bell_ {h^*(X)}\big] + \E_X  \big[(\boldeta(X)-\hat\boldeta(X)^\top \bell_ {\widehat g(X)}\big] - \E_X  \big[\boldeta(X)^\top \bell_ {h^*(X)}\big] \\[0.1em]
&=&
\E_X  \big[(\boldeta(X)-\hat\boldeta(X))^\top (\bell_ {\widehat g(X)}-\bell_ {h^*(X)})\big]  \\[0.1em]
&\leq&
\E_X \big[\big\|\boldeta(X)-\hat\boldeta(X)\big\|_1 \cdot\big\|\bell_ {\widehat g(X)}-\bell_ {h^*(X)}\big\|_\infty\big] \\[0.1em]
&\leq&
\E_X \big[\big\|\boldeta(X)-\hat\boldeta(X)\big\|_1\big],
\end{eqnarray*}
where in the
third step, we use the fact that 
$\widehat{g}(x) \,=\, \argmin^*_{j\in[n]}  \widehat{\boldeta}(x)^\top \bell_j$; in the
last step, we have use the fact that $\|\L\|_\infty =  1$.

For the second part, we denote the class of all plug-in classifiers constructed from a fixed class-probability estimator $\widehat{\eta}$ by:
\[\H = \left\{h:\X\>[n],\, h(x)\,=\, \argmin^*_{y\in[n]} \bell_y^\top \widehat{\boldeta}(x) \,|\,\L \in [0,1]^{n\times n}\right\},\]
and provide a uniform convergence bound over all classifiers in $\H$, and in turn applies to the classifier $\widehat{g}$ output by Algorithm \ref{alg:plug-in}.

For any  $a,b\in[n]$, we have
\begin{eqnarray*}
\sup_{h\in\H_\q}\left| \hat{C}^S_{a,b}[h] - C_{a,b}[h] \right| 
&=&
\sup_{h\in\H}\left| \frac{1}{m}\sum_{i=1}^{m} \left(\1(y_i=a, h(x_i)=b) - \E [\1(Y=a,h(X)=b)] \right) \right| \\
&=&
\sup_{h\in\H^b}\left| \frac{1}{m}\sum_{i=1}^{ m} \left(\1(y_i=a, h(x_i)=1) - \E [\1(Y=a,h(X)=1)] \right) \right|\;, \\
\end{eqnarray*}
where for a fixed $b \in [n]$, 
$$\H^b \,=\, \left\{h:\X\>\{0,1\}: \exists \L\in [0,1]^{n\times n}, \forall x\in\X,\, h(x)=\1\left(b= \argmin^*_{y\in[n]}\bell_y^\top \widehat{\boldeta}(x)\right) \right\}.$$ 
The set $\H^b$ can be seen as hypothesis class whose concepts are the intersection of $n$ halfspaces in $\R^n$ (corresponding to $\widehat{\boldeta}(x)$) through the origin.  Hence we have from Lemma 3.2.3 of  \citet{Blumer+89} that the VC-dimension of $\H^b$ is at most $2n^2\log (3n)$. 

From standard uniform convergence arguments we have that for each $a, b \in [n]$, the following holds with at least probability $1-\delta$ (over draw of $S \sim D^N$),
\begin{eqnarray*}
\sup_{h\in\H}\left| \hat{C}^S_{a,b}[h] - C_{a,b}[h] \right| 
\,\leq\, 
\cO\left(\sqrt{\frac{n^2\log(n)\log( N ) + \log(\frac{1}{\delta})}{ N} }\right).
\end{eqnarray*}
Applying union bound over all $a,b \in [n]$, we have that the following holds with probability $\geq 1-\delta$:
\begin{eqnarray*}
\left|\left| \hat\C^S[\hat{g}]  -  \C[\hat{g}]  \right|\right|_\infty \,\leq\, \sup_{h\in\H}\left|\left| \hat\C^S[h]  -  \C[h]  \right|\right|_\infty
\,\leq\,
\cO\left(\sqrt{\frac{n^2\log(n)\log( N ) + \log(\frac{n^2}{\delta})}{ N } }\right).
\end{eqnarray*}
Plugging $d= n^2$ completes the proof.
\end{proof}

\section{Additional Experimental Details}
\label{app:expts}

\subsection{Hyper-parameter Selection}
\label{app:hparam}
We run the Frank-Wolfe and GDA algorithms  for 5000 LMO calls, and the ellipsoid algorithm for 1000 LMO calls. We run the constrained algorithms for 10000, 10000 and 1000 LMO calls respectively.  The unconstrained Frank-Wolfe algorithm has no other hyper-parameters to tune. 
The GDA algorithm has two step-size parameters $\eta$ and $\eta'$, which we tune using a two-dimensional grid-search over  $\{0.001, 0.01, 0.1\}^2$, picking the parameters that yield the lowest objective on the training set. For the ellipsoid algorithm, we fix 
the initial ellipsoid radius $a$ to 1000.


The constrained counterpart to the Frank-Wolfe algorithm (SplitFW) in Algorithm \ref{alg:FW-con} has two additional hyper-parameters: the weight on the quadratic penalty $\zeta$, which we set to 10, and the step-size $\eta$, for which, we adopt the same schedule used by \cite{Gidel2018}, and set it to 0.5 for first $T/3$ iterations, 0.1 for the next $T/3$ iterations, and 0.001 for the final $T/3$ iterations. Additionally, we find it sufficient to avoid the explicit line search for $\gamma^t$ in line 7 and instead set to $\frac{2}{t+2}$, akin to the standard Frank-Wolfe setup. For the constrained version of GDA algorithm, we set the step-sizes $\eta_{\blambda} =\eta_{\bmu} = \eta'$, and tune $\eta_{\bxi}$ and $\eta'$ using the same the two-dimensional grid search used for unconstrained GDA, picking among those that satisfy the constraints on the training set, the ones with the least training objective (When none of the parameters satisfy the constraints, we pick the one with the minimum constraint violation). The hyper-parameters for the constrained ellipsoid algorithm were chosen in the same way as the unconstrained version. For TFCO, we tuned the learning rates for the model and constraint from $\{0.001, 0.01, 0.1\}$ and ran it for $5000$ iterations.



\subsection{Additional Details for CIFAR Case Studies}
\label{app:case-studies}
Below, we list the five super-classes in the CIFAR-55 dataset described in Section \ref{sec:expts-cifar}, and the 10 classes that each of them comprise of:
    (i) \textbf{Flowers and Fruits:} Orchid, Poppy, Rose, Sunflower, Tulip, Mushroom, Orange, Pear, Apples, and Sweet Pepper.
    (ii) \textbf{Aquatic Animals:} Beaver, Dolphin, Otter, Aquarium Fish, Ray, Flat Fish, Shark, Trout, Whale, and Seal.
    (iii) \textbf{Household Items:} Clock, Bed, Chair, Couch, Keyboard, Telephone, Television, Wardrobe, Table, and Lamp.
    (iv) \textbf{Large Outdoor Scenes:} Bridge, Castle, House, Road, Mountain, Skyscraper, Cloud, Forest, Plain, and Sea.
    (v)
    \textbf{Mammals:} Camel, Cattle, Chimpanzee, Elephant, Kangaroo, Porcupine, Possum, Raccoon, Fox, and Skunk.
%
%
We employ standard data augmentation techniques on the CIFAR datasets by applying  random crops and horizontal flips.\footnote{The learning rate schedules 
were adopted from: \url{https://github.com/huyvnphan/PyTorch_CIFAR10}.}

\begin{figure}[t]
\centering
\begin{subfigure}[b]{0.48\linewidth}
\centering
\includegraphics[width=0.48\linewidth]{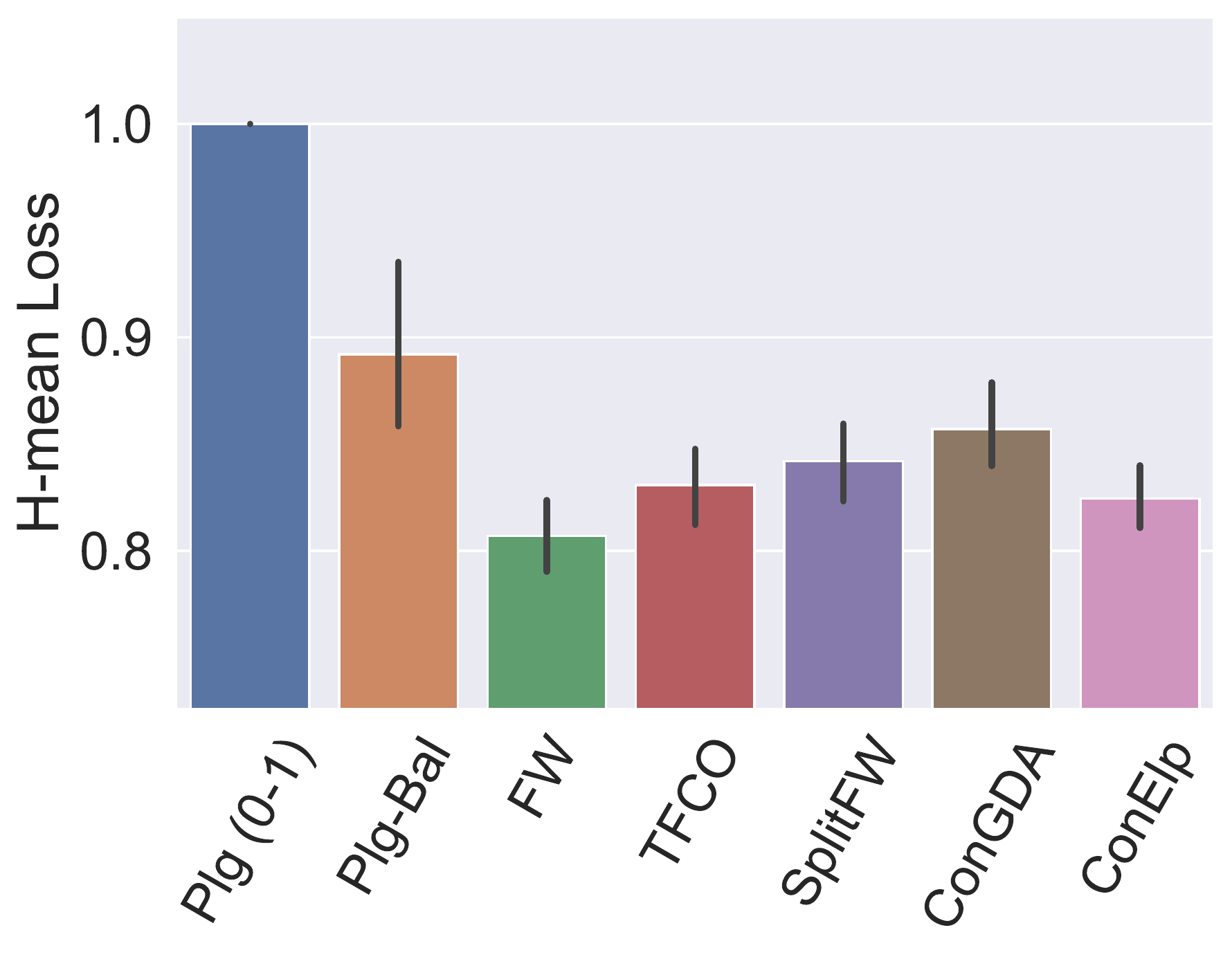}
\includegraphics[width=0.48\linewidth]{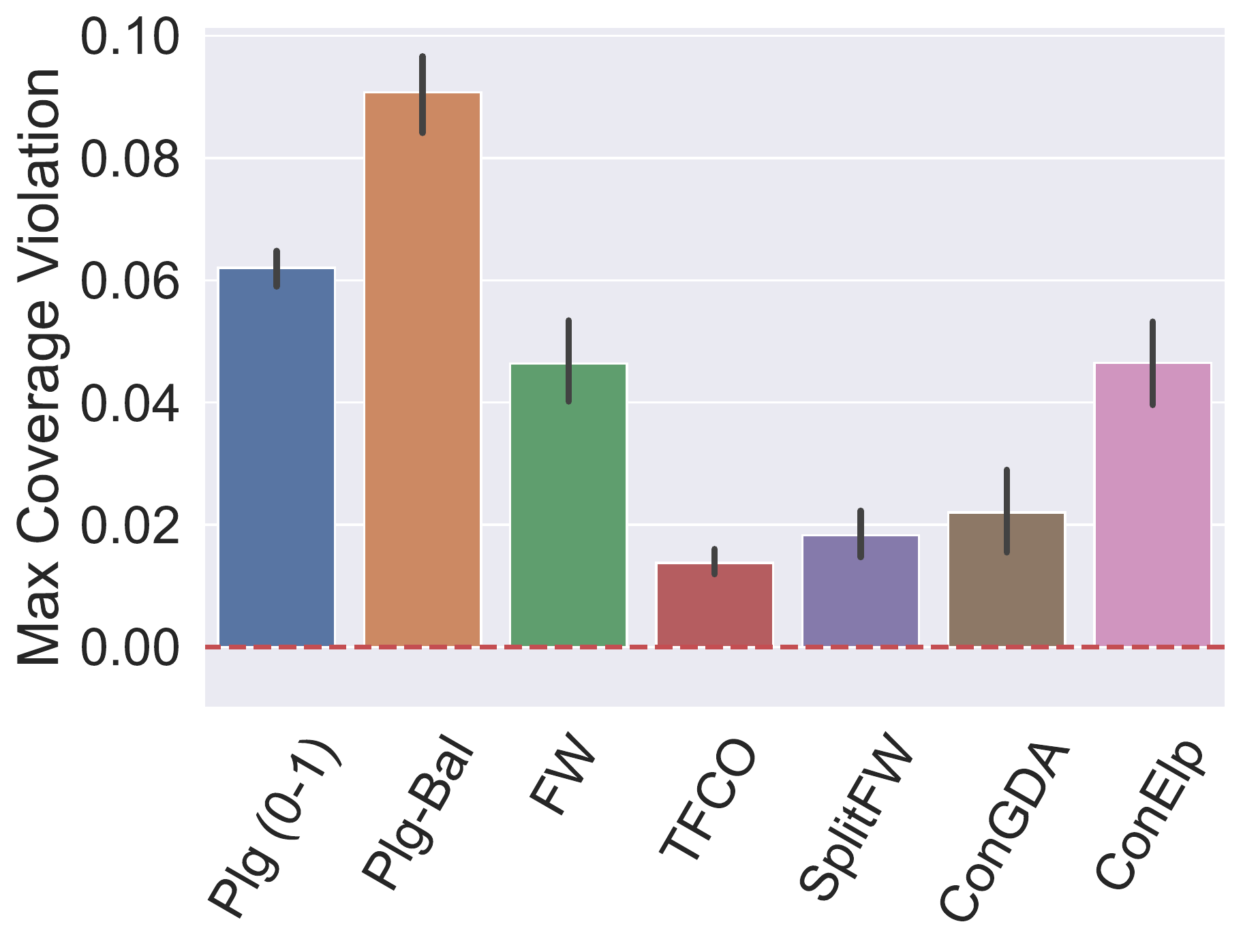}
\label{fig:app-abalone}
\caption{Abalone}
\end{subfigure}
\begin{subfigure}[b]{0.48\linewidth}
\centering
\includegraphics[width=0.48\linewidth]{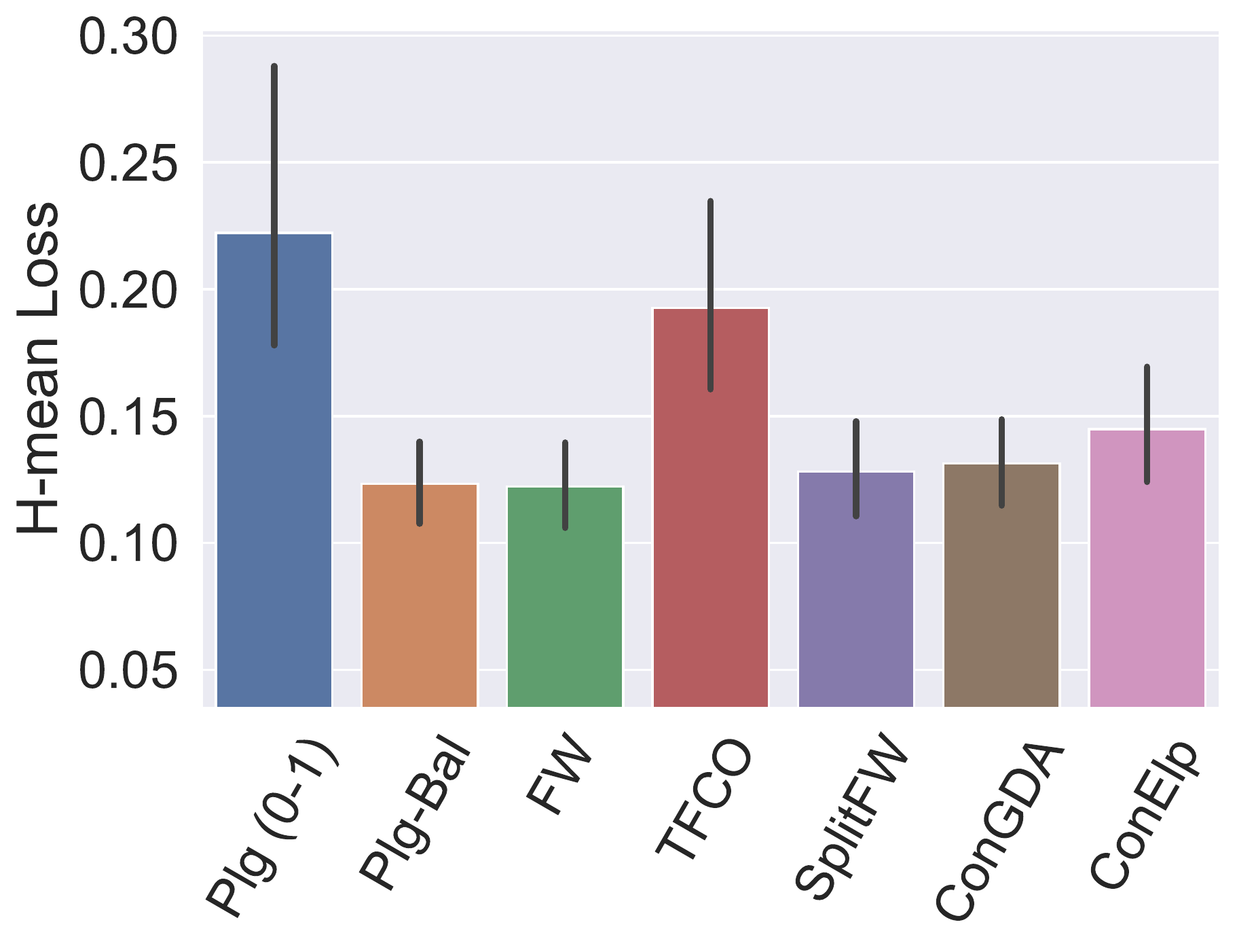}
\includegraphics[width=0.48\linewidth]{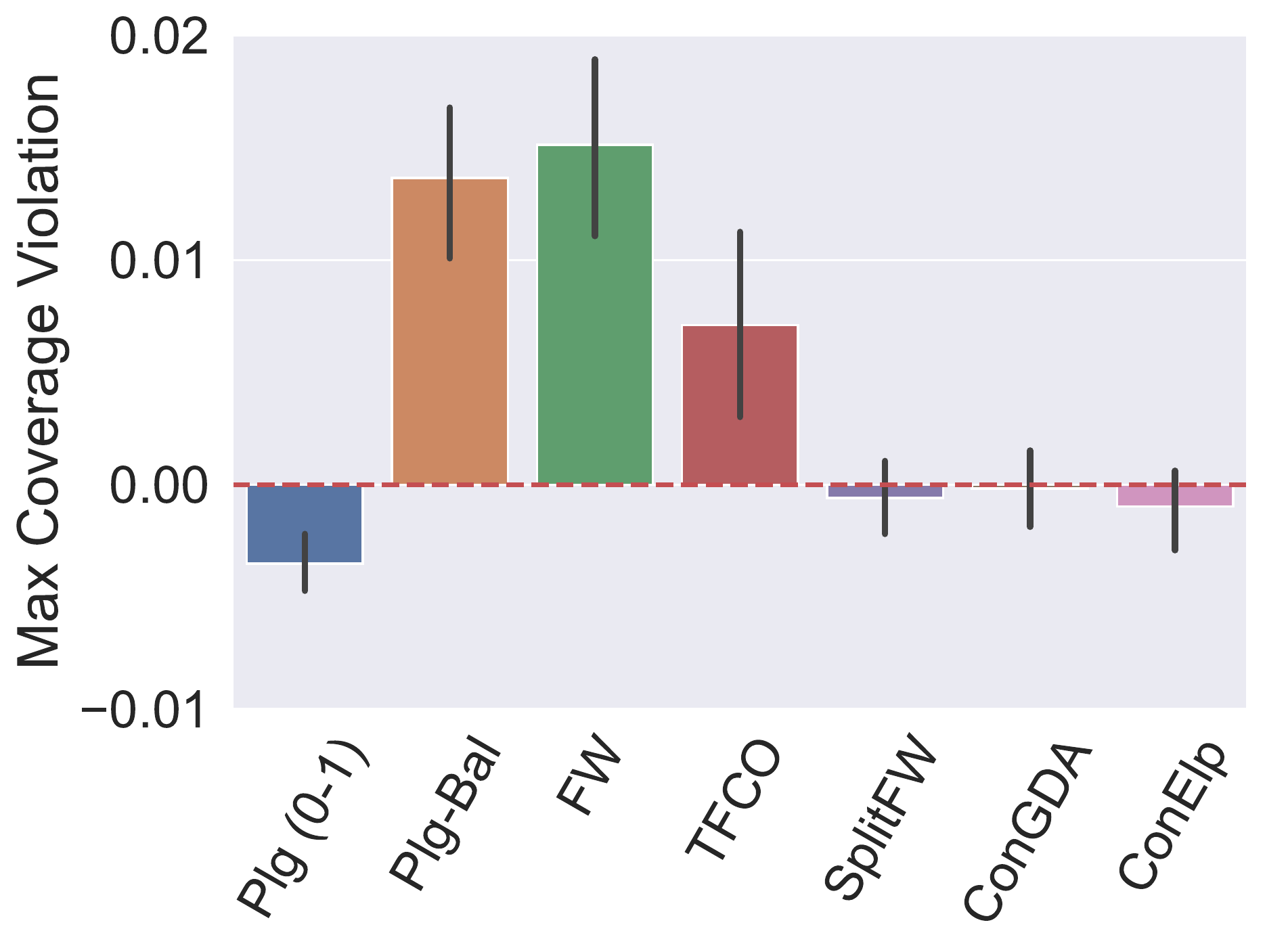}
\label{fig:app-macho}
\caption{MACHO}
\end{subfigure}
\begin{subfigure}[b]{0.48\linewidth}
\centering
\vspace{3pt}
\includegraphics[width=0.48\linewidth]{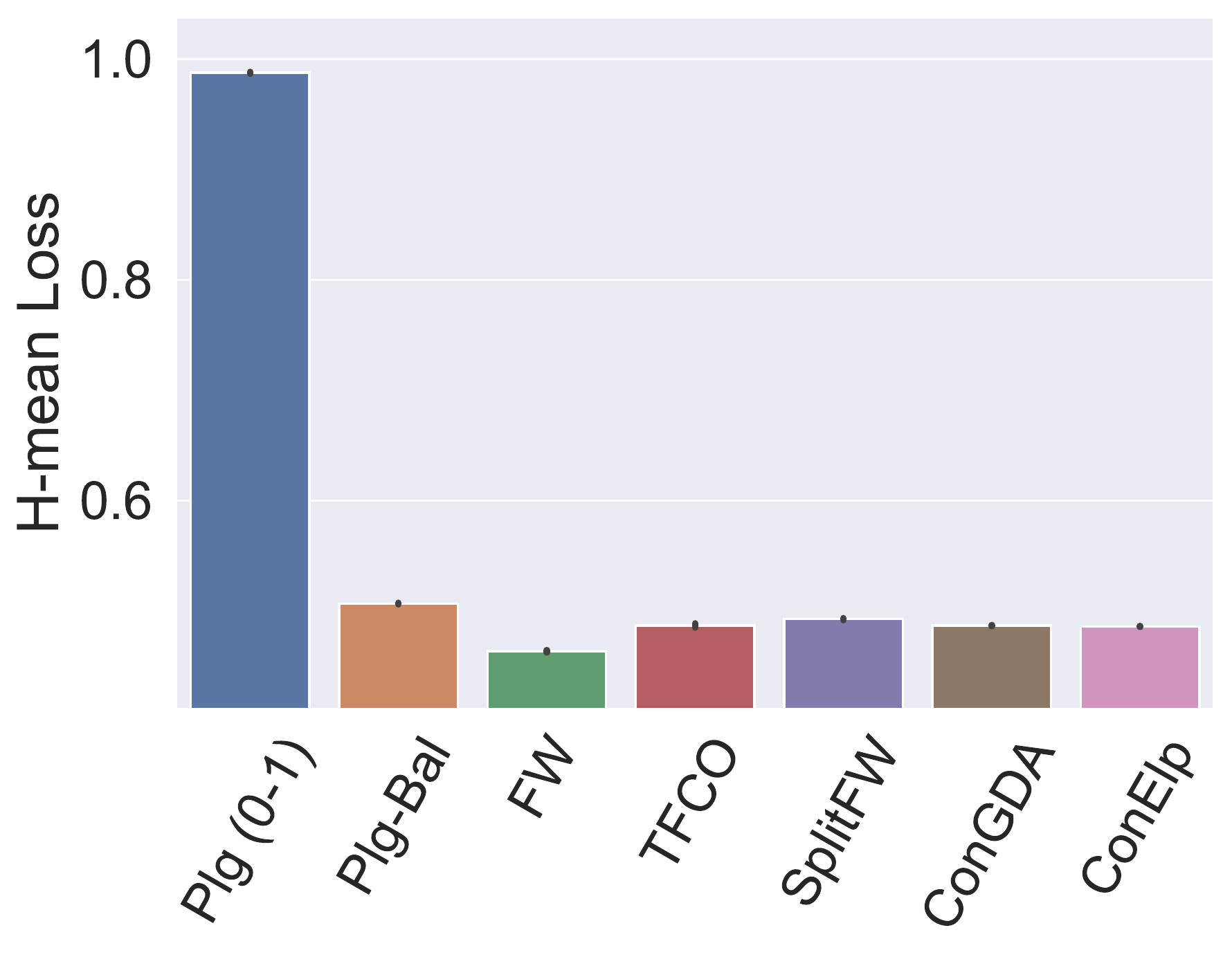}
\includegraphics[width=0.48\linewidth]{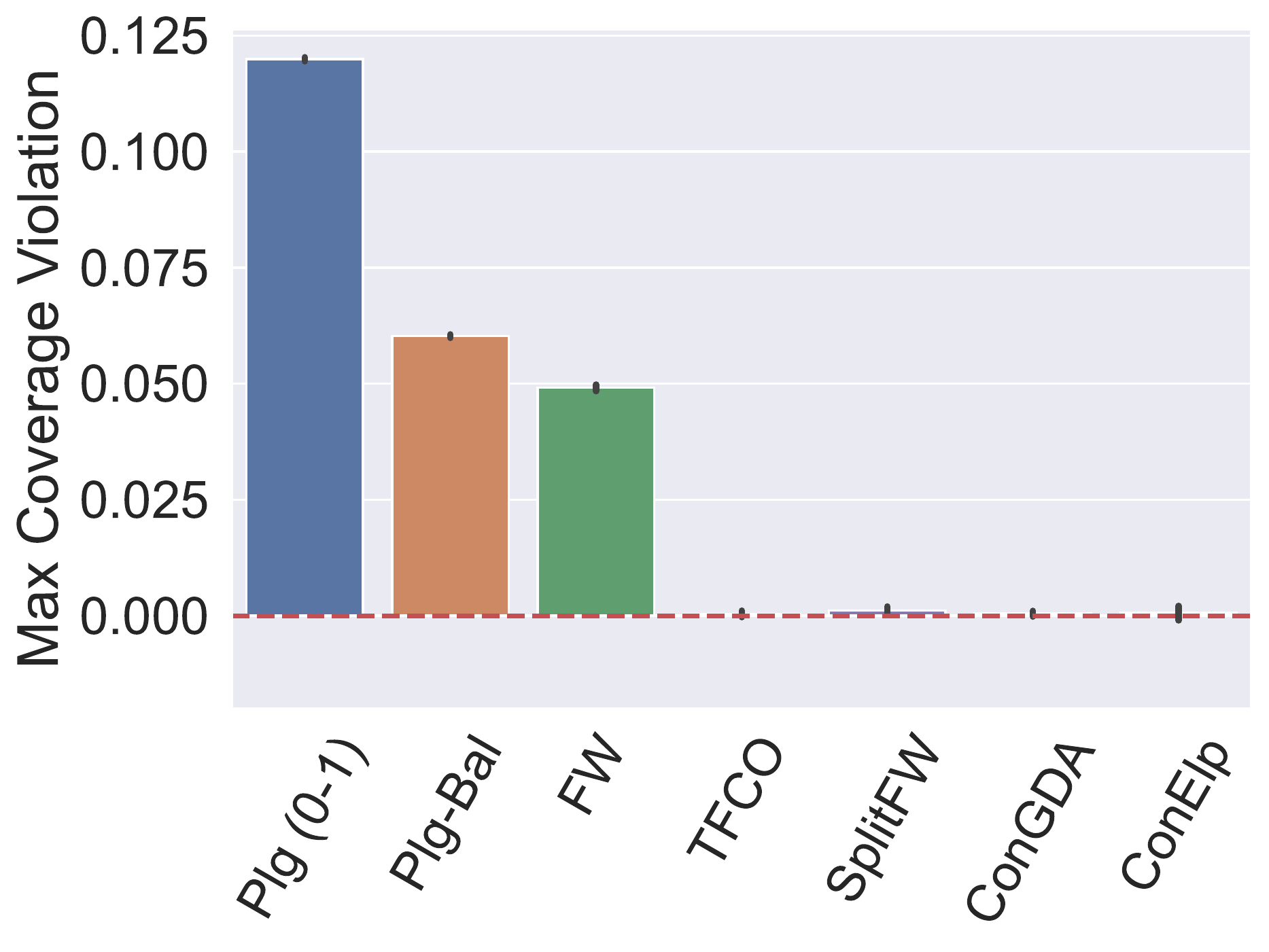}
\label{fig:app-covtype}
\caption{CovType}
\end{subfigure}
\caption{Optimizing the H-mean loss subject to the coverage constraint $\max_i|\sum_{j}C_{ji} \,-\, \pi_i| \leq 0.01$. The plots on the left show the H-mean loss on the test set and those on the right show the coverage violation $\max_i|\sum_{j}C_{ji} \,-\, \pi_i| - 0.01$ on the test set. 
\textit{Lower} H-mean value are \textit{better}, and the constraint values need to be $\leq 0$. 
}
\label{fig:app-hmean-cov}
\vspace{-3pt}
\end{figure}

\begin{figure}[t]
\centering
\begin{subfigure}[b]{0.48\linewidth}
\centering
\includegraphics[width=0.48\linewidth]{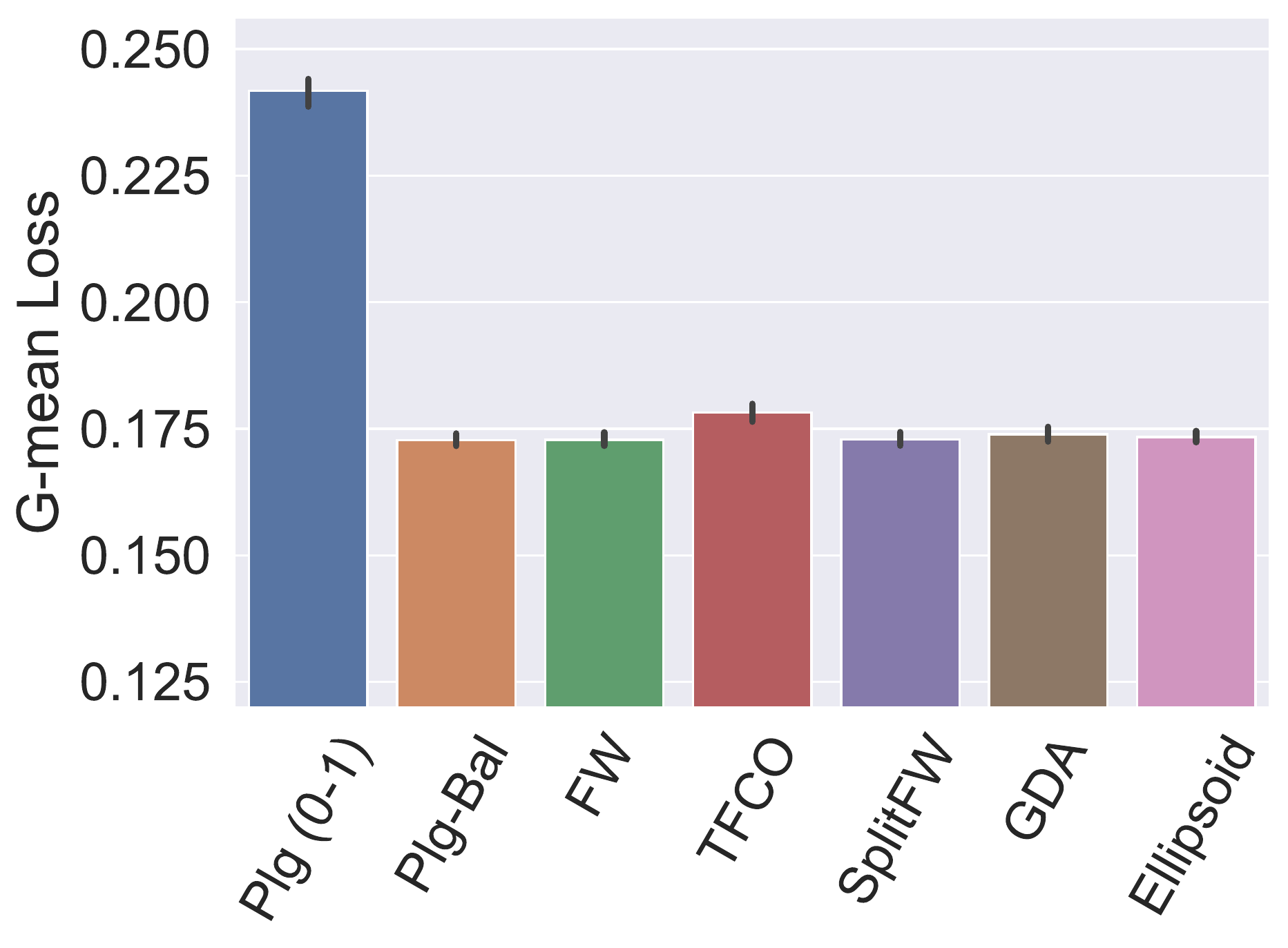}
\includegraphics[width=0.48\linewidth]{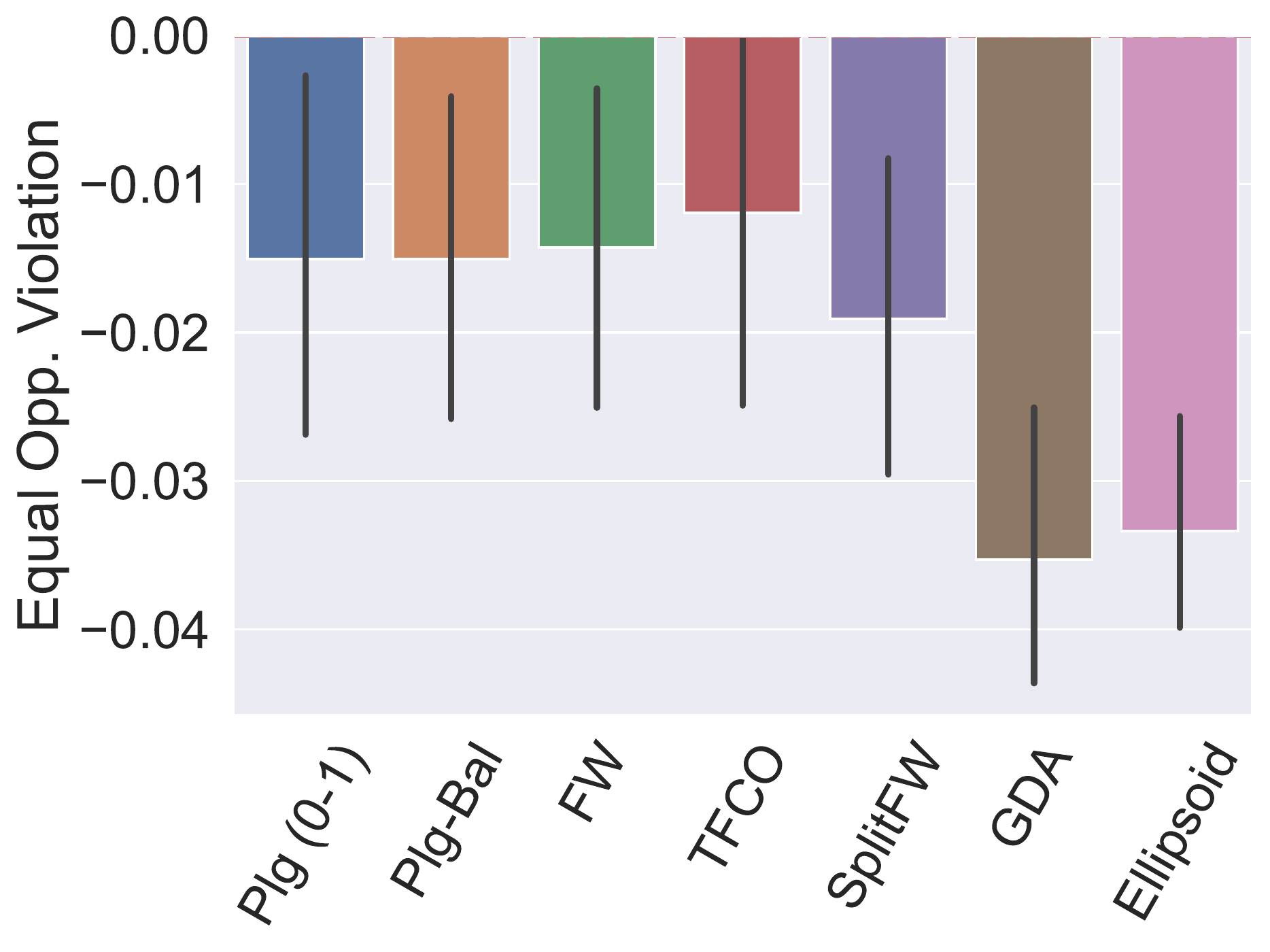}
\label{fig:app-adult}
\caption{Adult}
\end{subfigure}
\begin{subfigure}[b]{0.48\linewidth}
\centering
\includegraphics[width=0.48\linewidth]{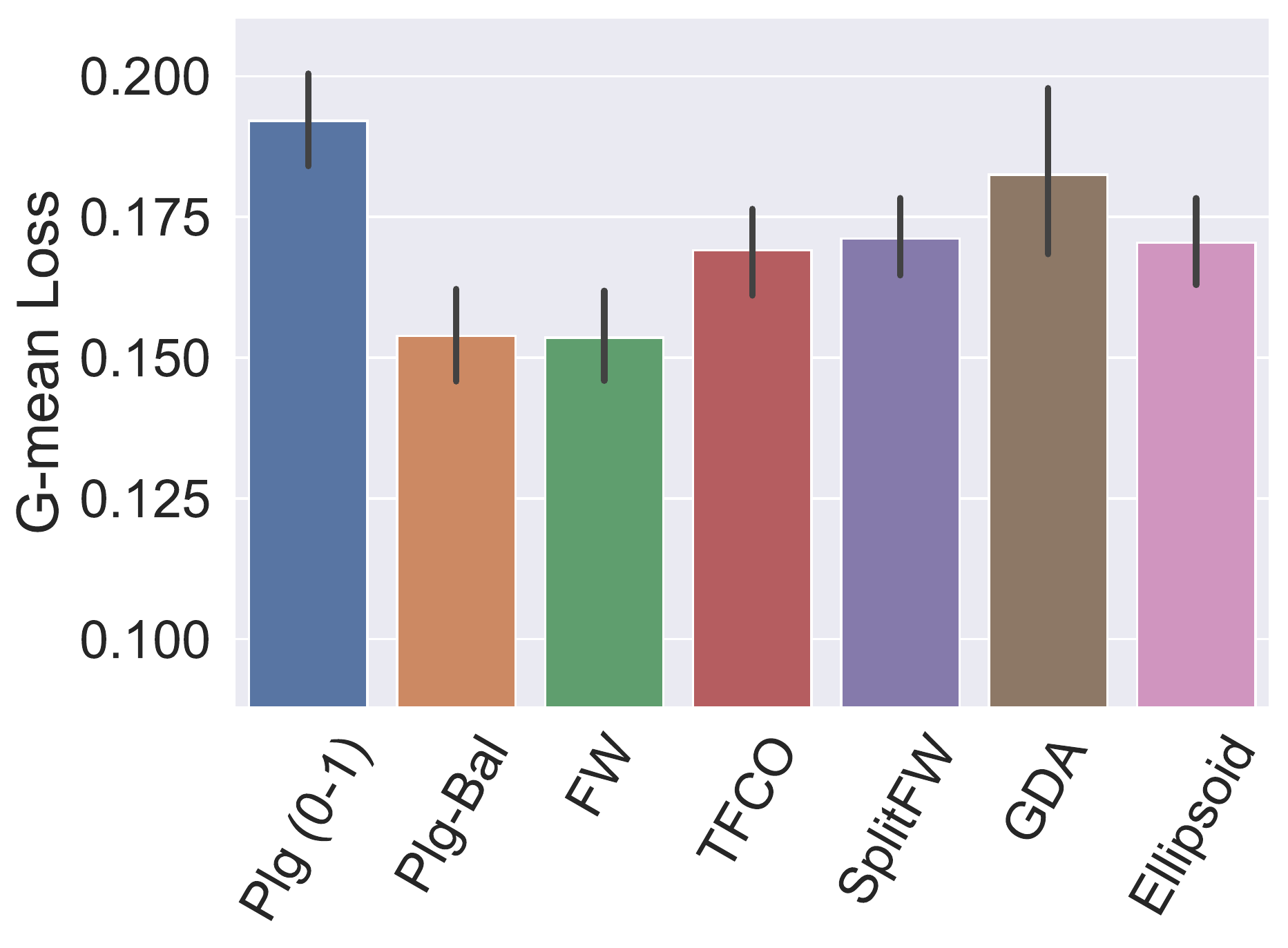}
\includegraphics[width=0.48\linewidth]{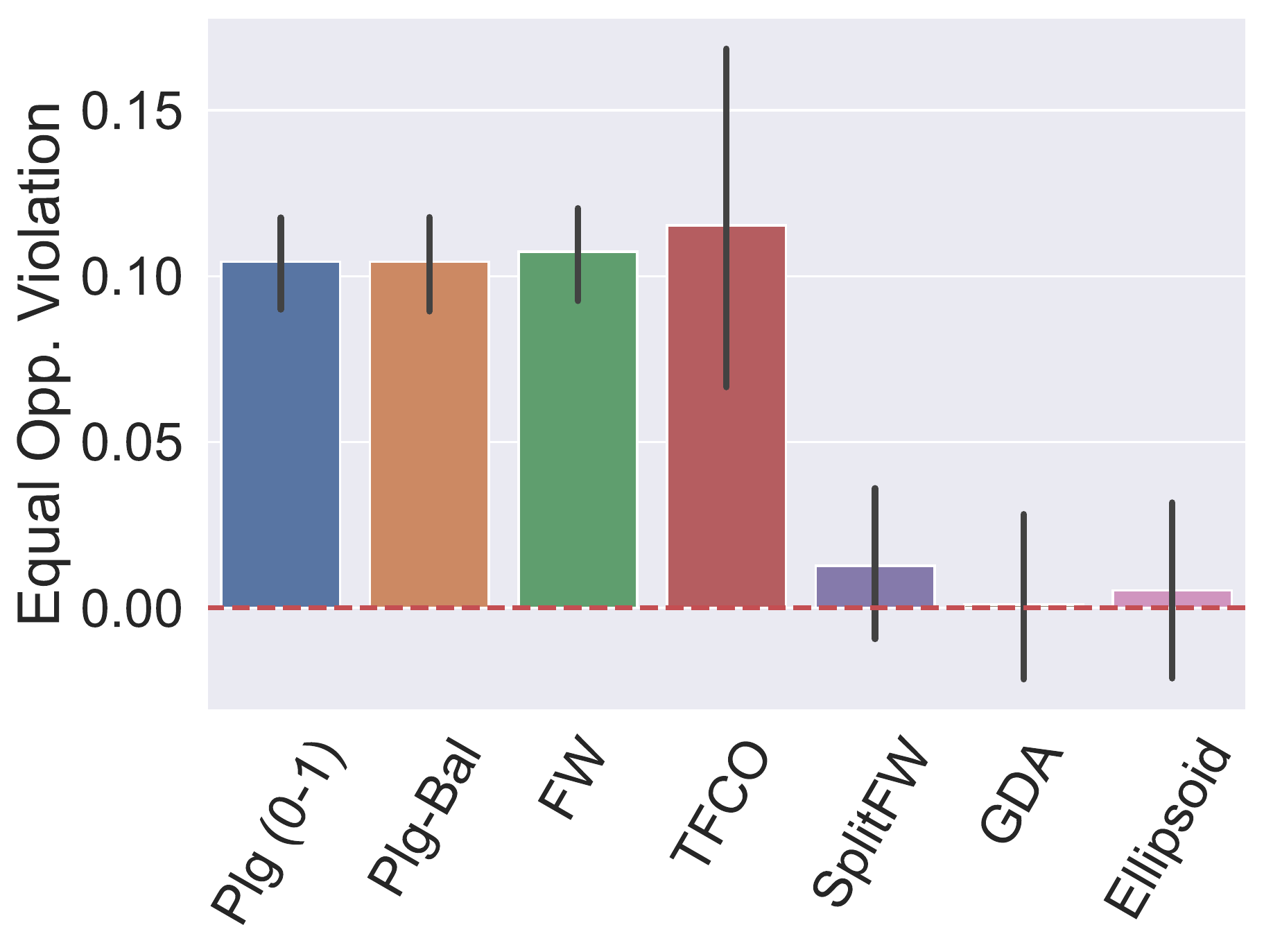}
\label{fig:app-crimes}
\caption{Crimes}
\end{subfigure}
%
\begin{subfigure}[b]{0.48\linewidth}
\centering
\vspace{3pt}
\includegraphics[width=0.48\linewidth]{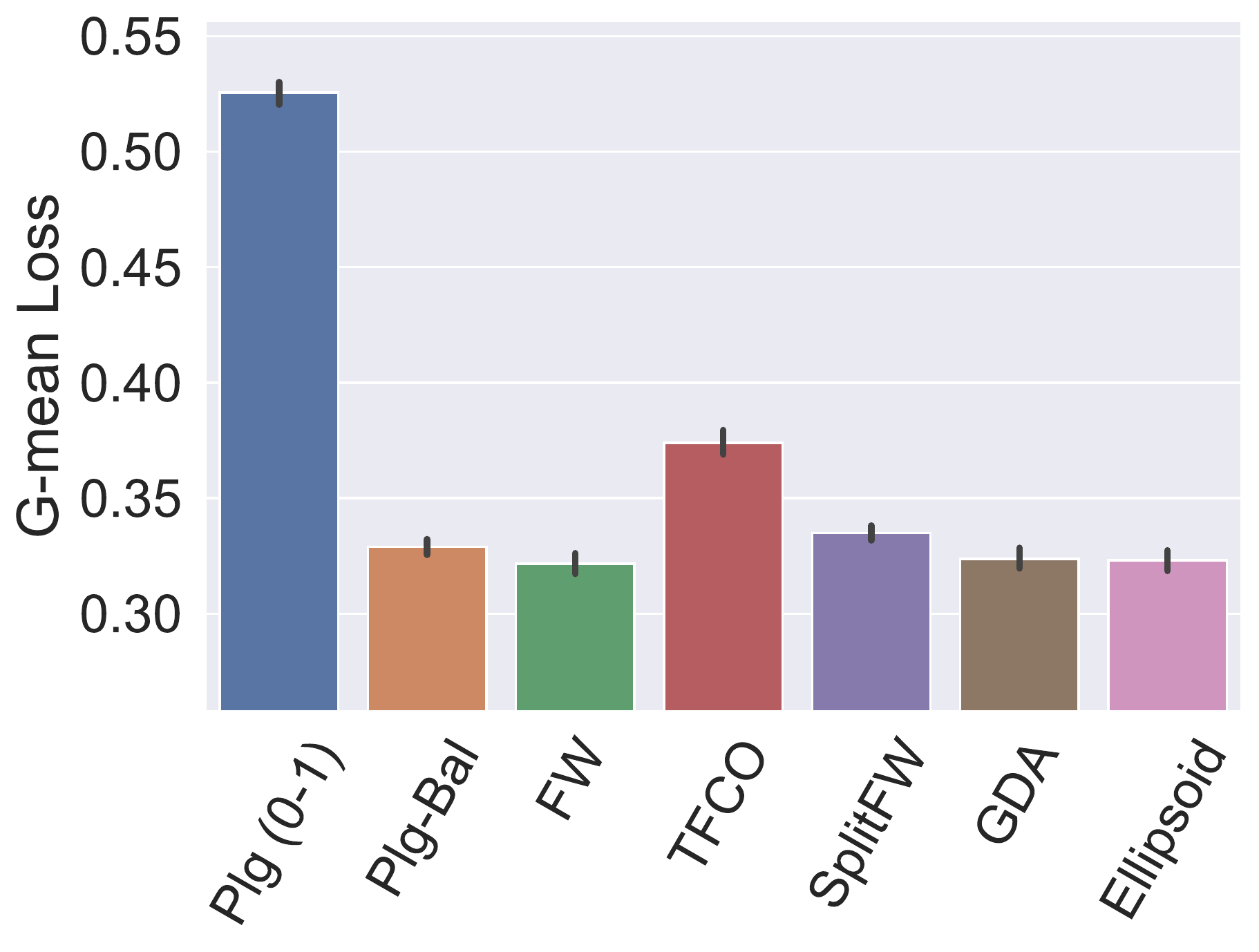}
\includegraphics[width=0.48\linewidth]{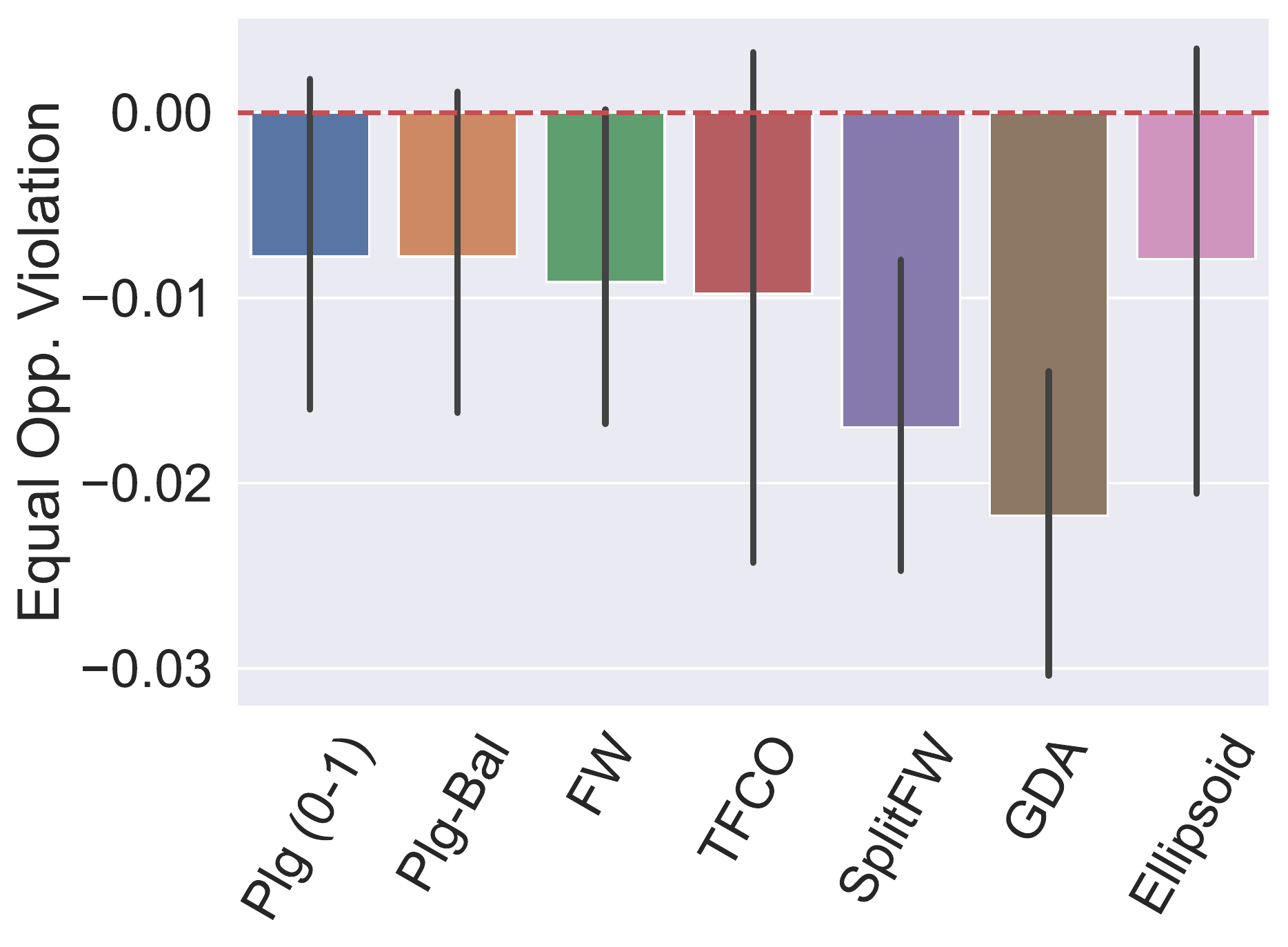}
\label{fig:app-default}
\caption{Default}
\end{subfigure}
\caption{Optimizing the G-mean loss subject to the Equal Opportunity constraint $\leq 0.01$. The plots on the left show the G-mean loss on the test set and those on the right show the equal opportunity violation (needs to be $\leq 0$) on the test set. 
\textit{Lower} G-mean values are \textit{better}.
}
\label{fig:app-gmean-eqopp-2}
\vspace{-10pt}
\end{figure}

\begin{figure}[t]
\centering
\begin{subfigure}[b]{0.48\linewidth}
\centering
\includegraphics[width=0.48\linewidth]{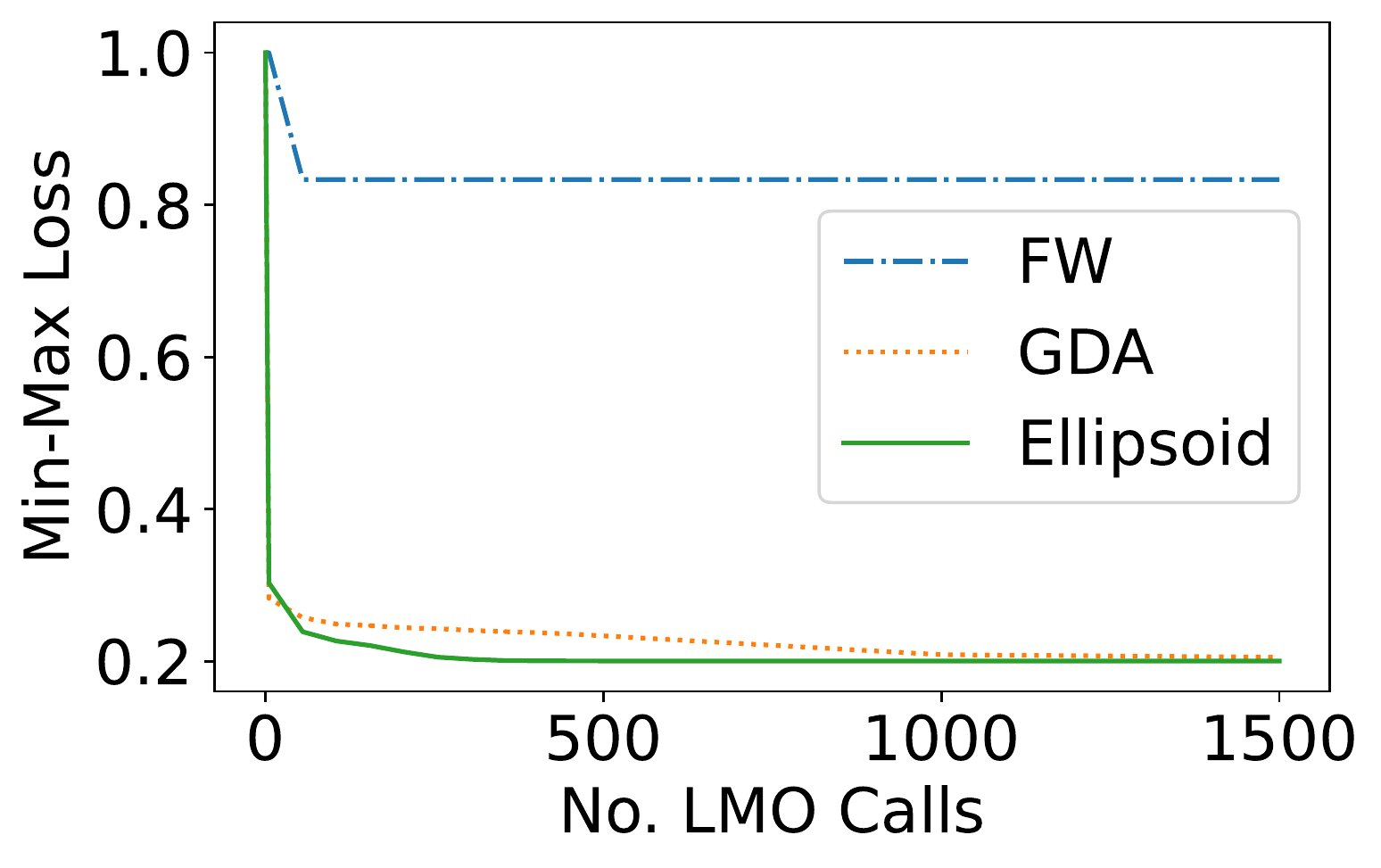}
\includegraphics[width=0.48\linewidth]{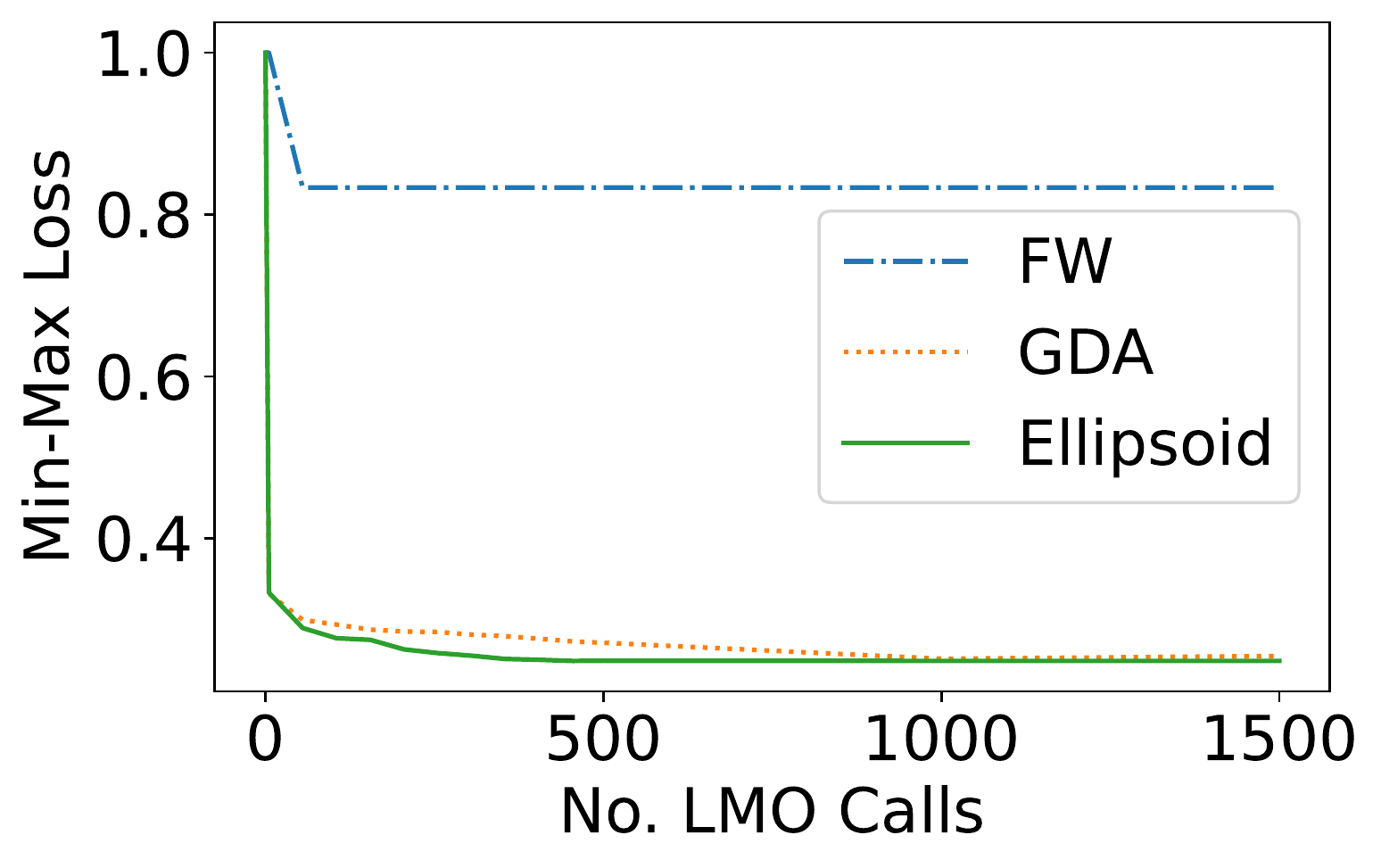}
\label{fig:lmo-satimage-train-minmax}
\caption{SatImage}
\end{subfigure}
\begin{subfigure}[b]{0.48\linewidth}
\centering
\includegraphics[width=0.48\linewidth]{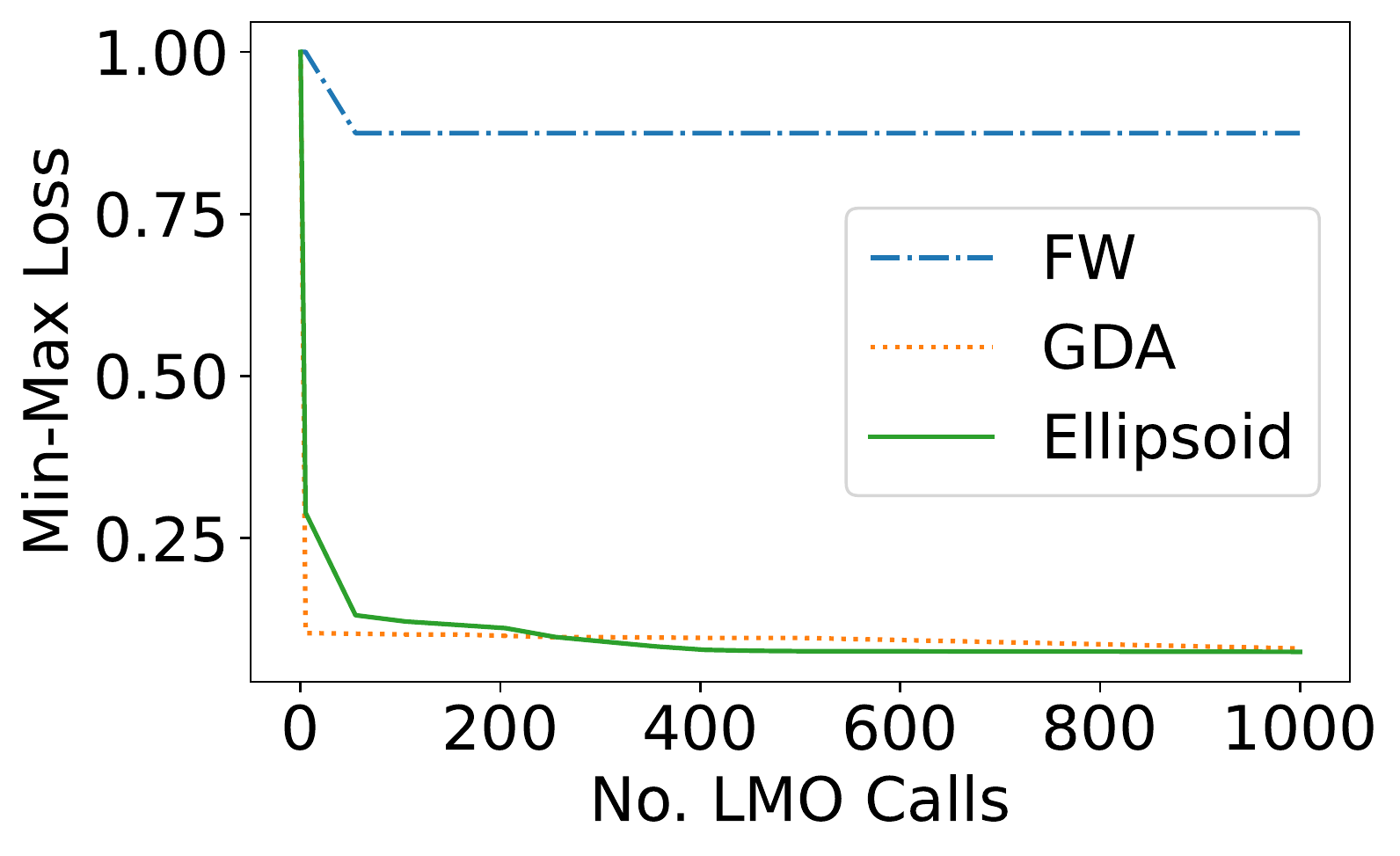}
\includegraphics[width=0.48\linewidth]{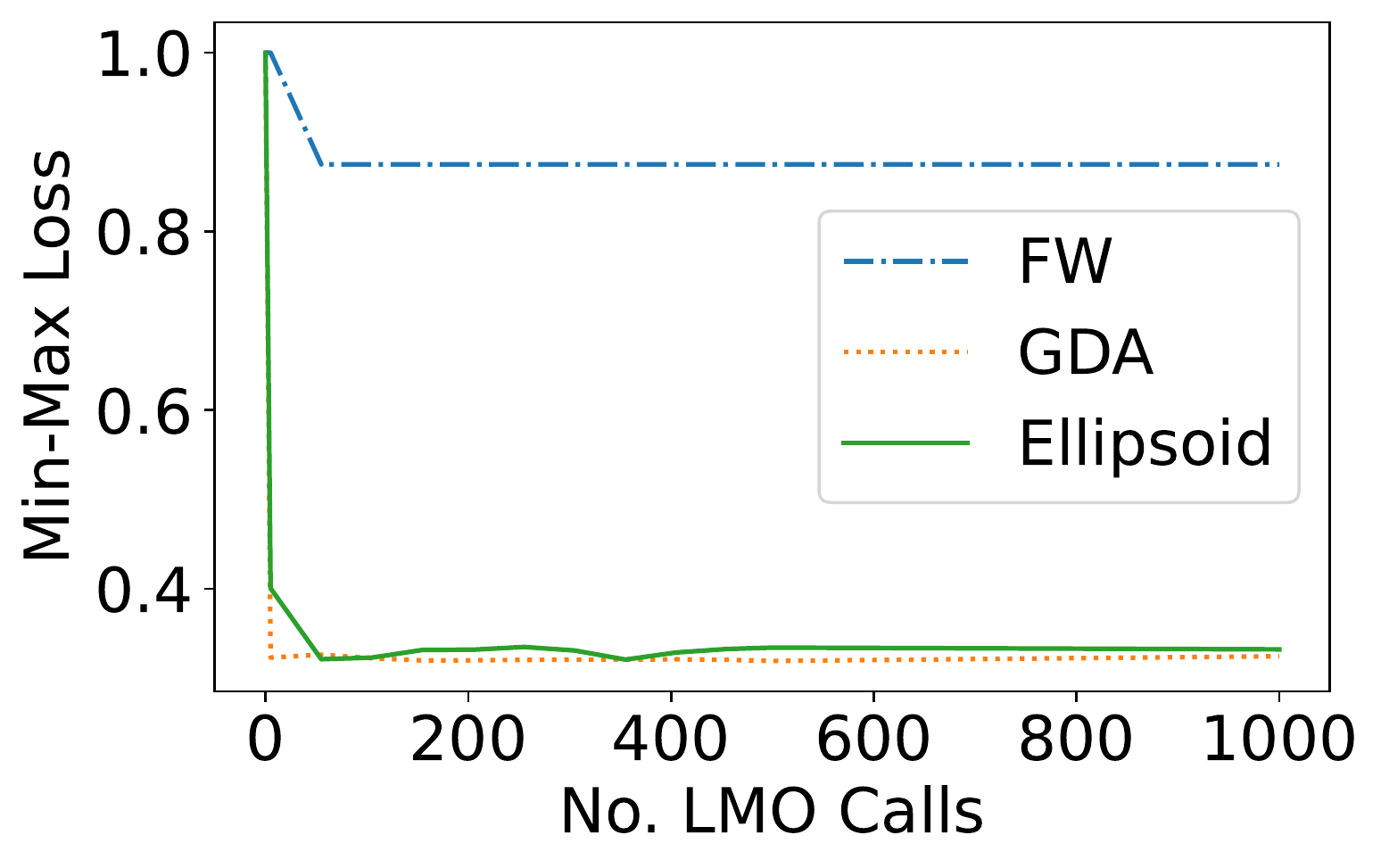}
\label{fig:lmo-macho-train-minmax}
\caption{MACHO}
\end{subfigure}
\caption{Optimizing the \emph{Min-max} loss: Comparison of performance of the Frank-Wolfe, GDA and ellipsoid methods as a function of the number of LMO calls. The plot on the left is for train data and on the right is for test data. \textit{Lower} values are \textit{better}. Because the min-max loss is non-smooth, Frank-Wolfe is seen to converge to a sub-optimal classifier.}
\label{fig:app-min-max}
\vspace{-6pt}
\end{figure}


\begin{figure}[t]
\centering
\begin{subfigure}[b]{0.48\linewidth}
\centering
\includegraphics[width=0.48\linewidth]{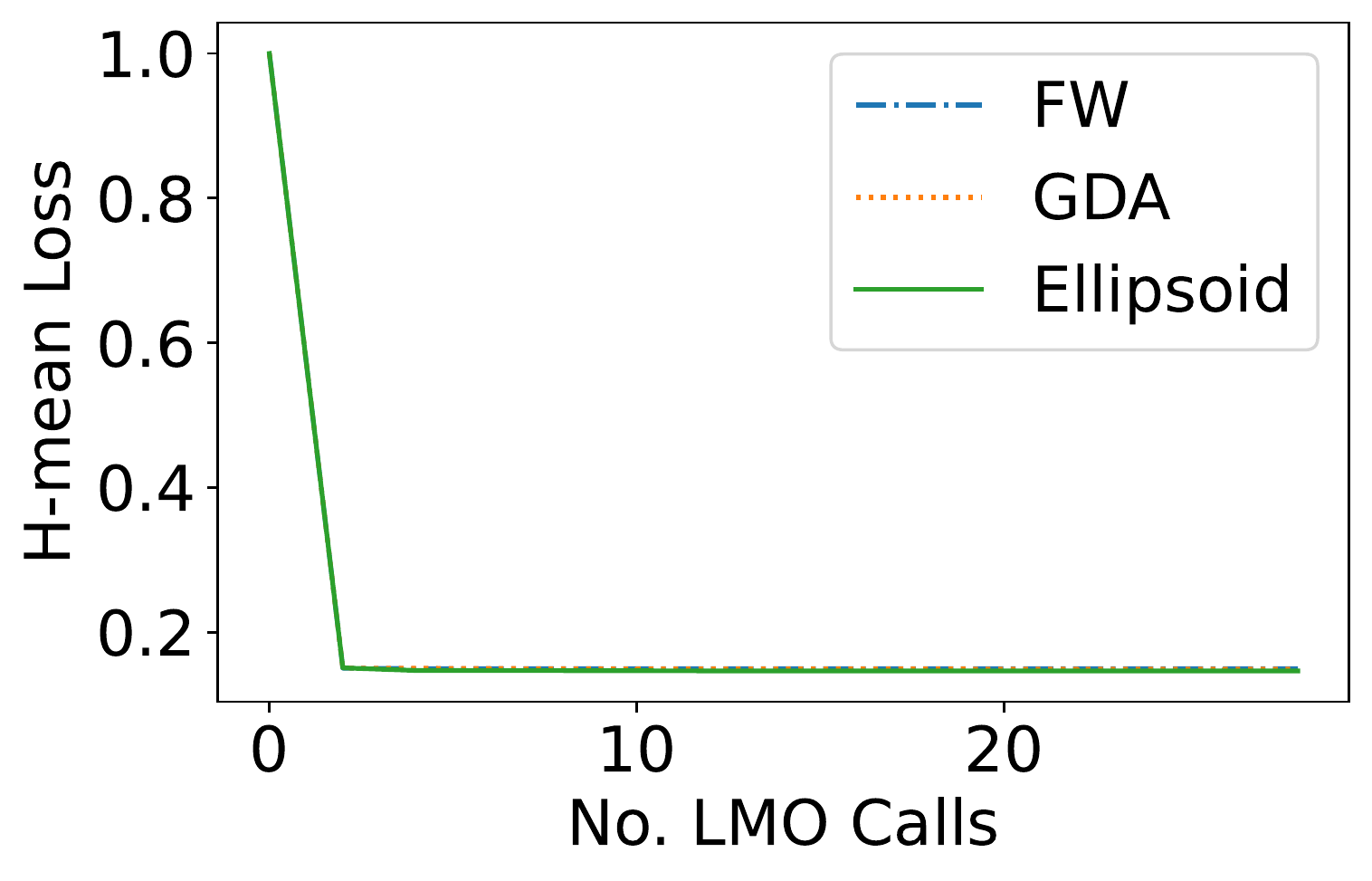}
\includegraphics[width=0.48\linewidth]{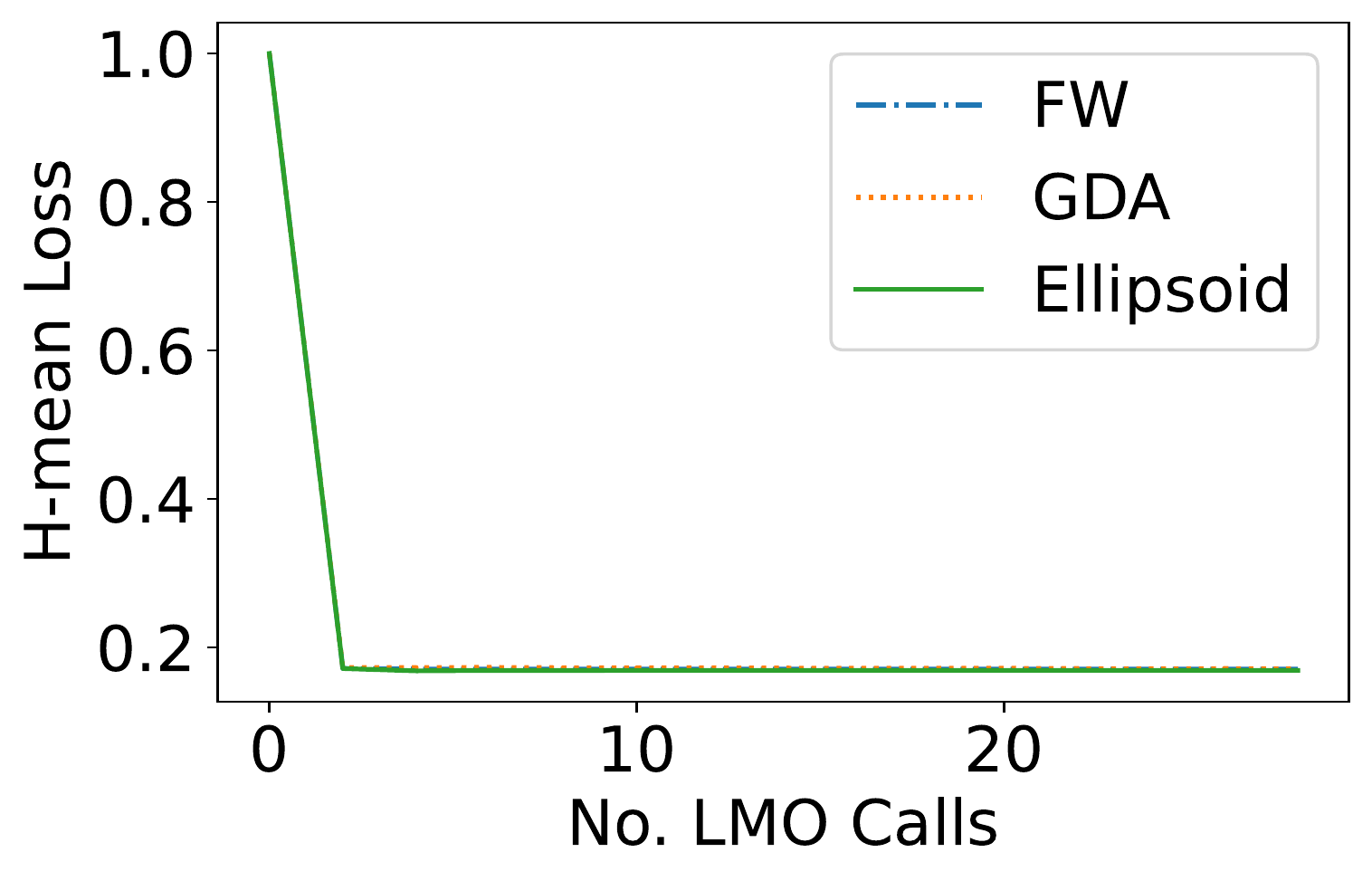}
\label{fig:lmo-satimage-train-hmean-uncon}
\caption{SatImage}
\end{subfigure}
\begin{subfigure}[b]{0.48\linewidth}
\centering
\includegraphics[width=0.48\linewidth]{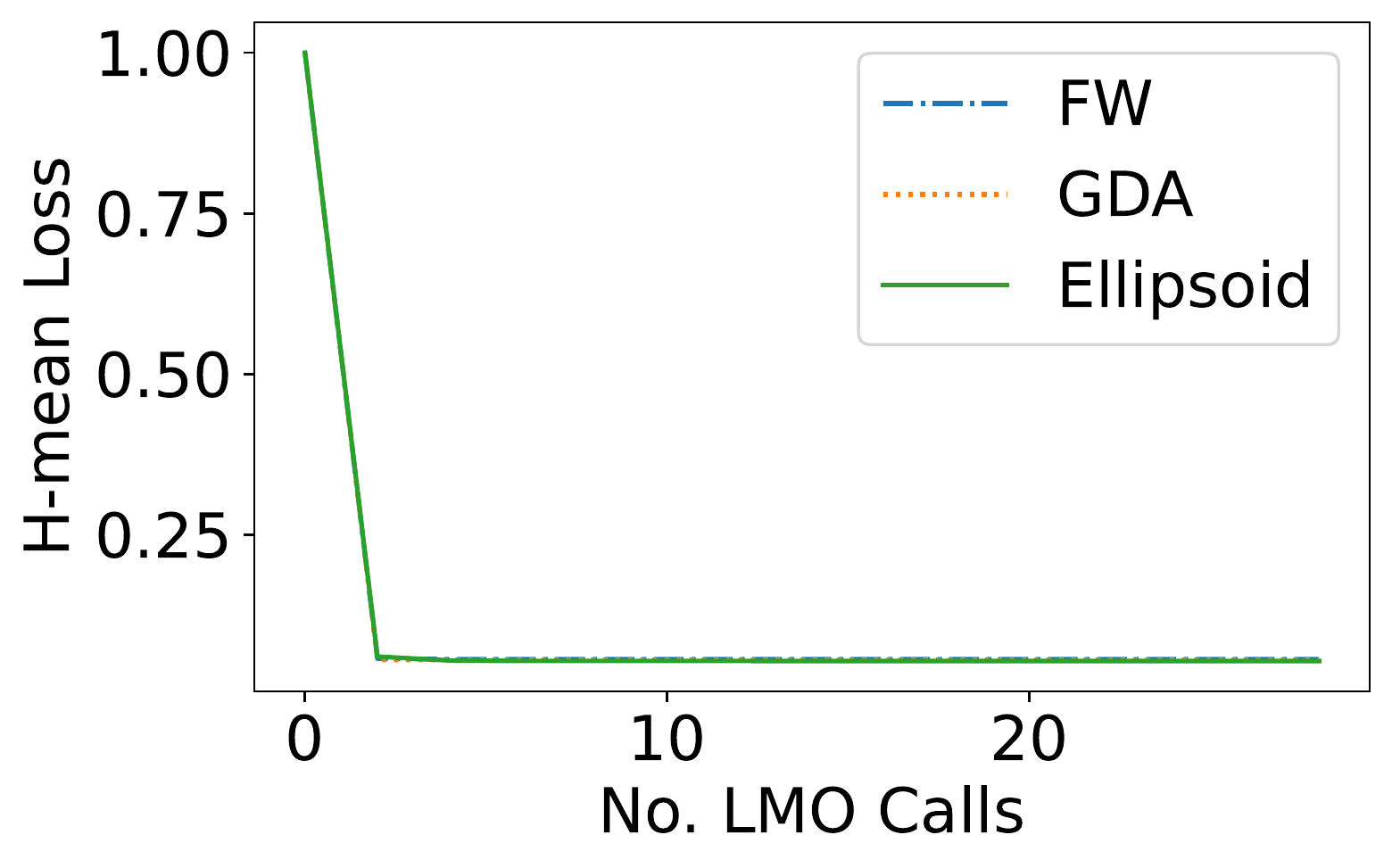}
\includegraphics[width=0.48\linewidth]{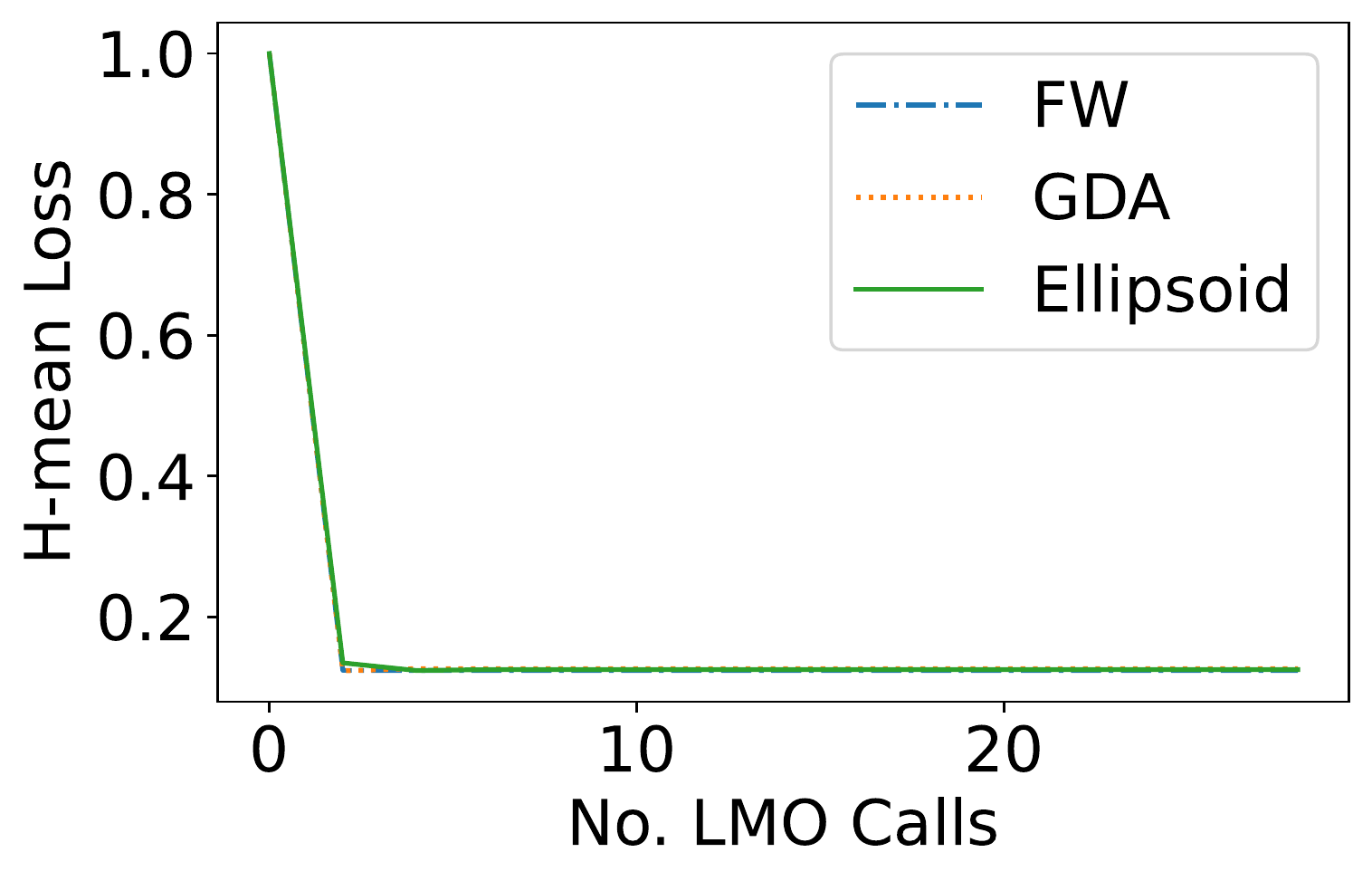}
\label{fig:lmo-macho-train-hmean-uncon}
\caption{MACHO}
\end{subfigure}
\caption{Optimizing the \emph{Hmean} loss: Comparison of performance of the Frank-Wolfe, GDA and ellipsoid methods as a function of the number of LMO calls. The plot on the left is for Train data and on the right is for Test data. \textit{Lower} values are \textit{better}.}
\label{fig:app-hmean-uncon}
\vspace{-10pt}
\end{figure}

\begin{figure}[t]
\centering
\begin{subfigure}[b]{0.48\linewidth}
\centering
\includegraphics[width=0.48\linewidth]{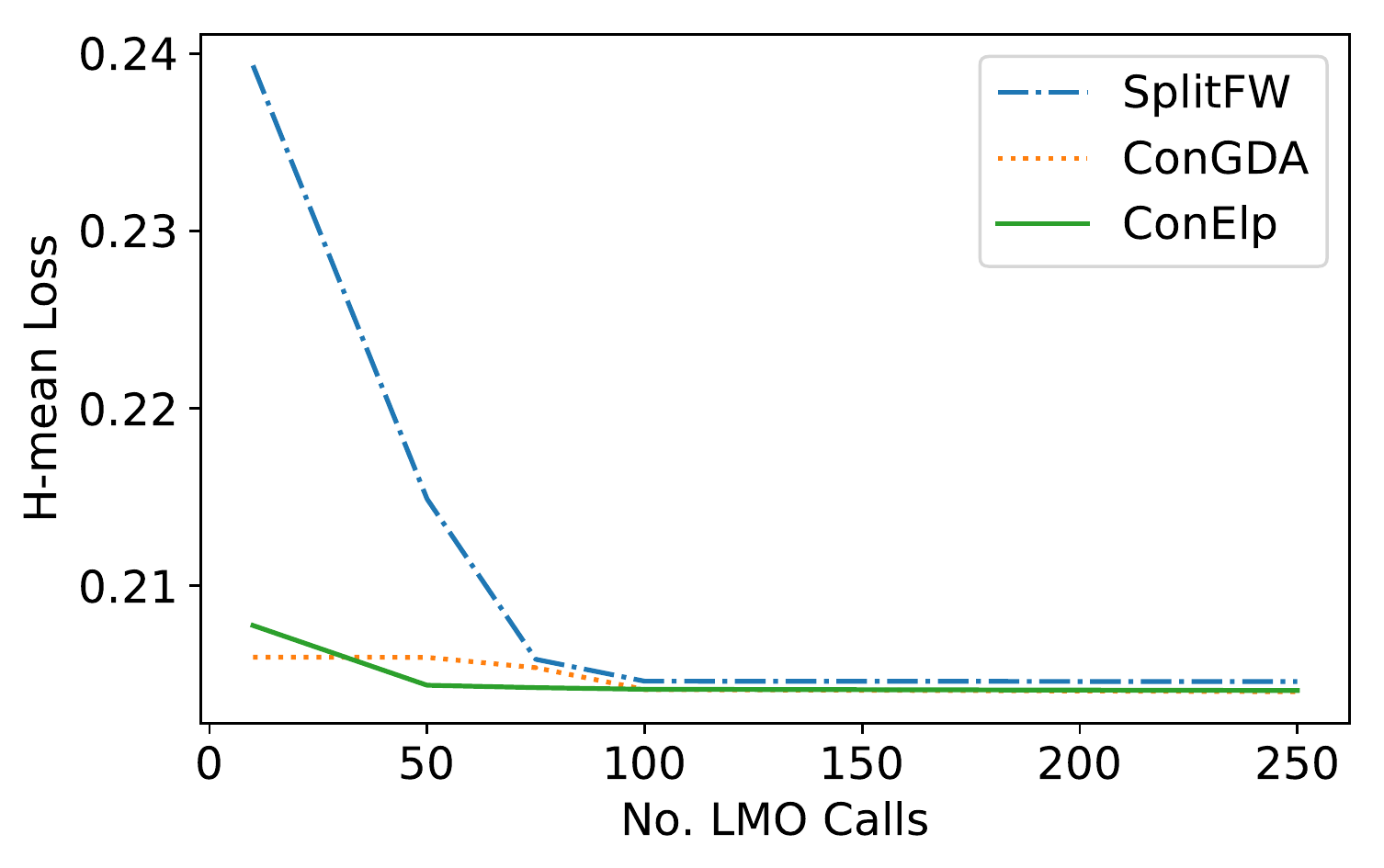}
\includegraphics[width=0.48\linewidth]{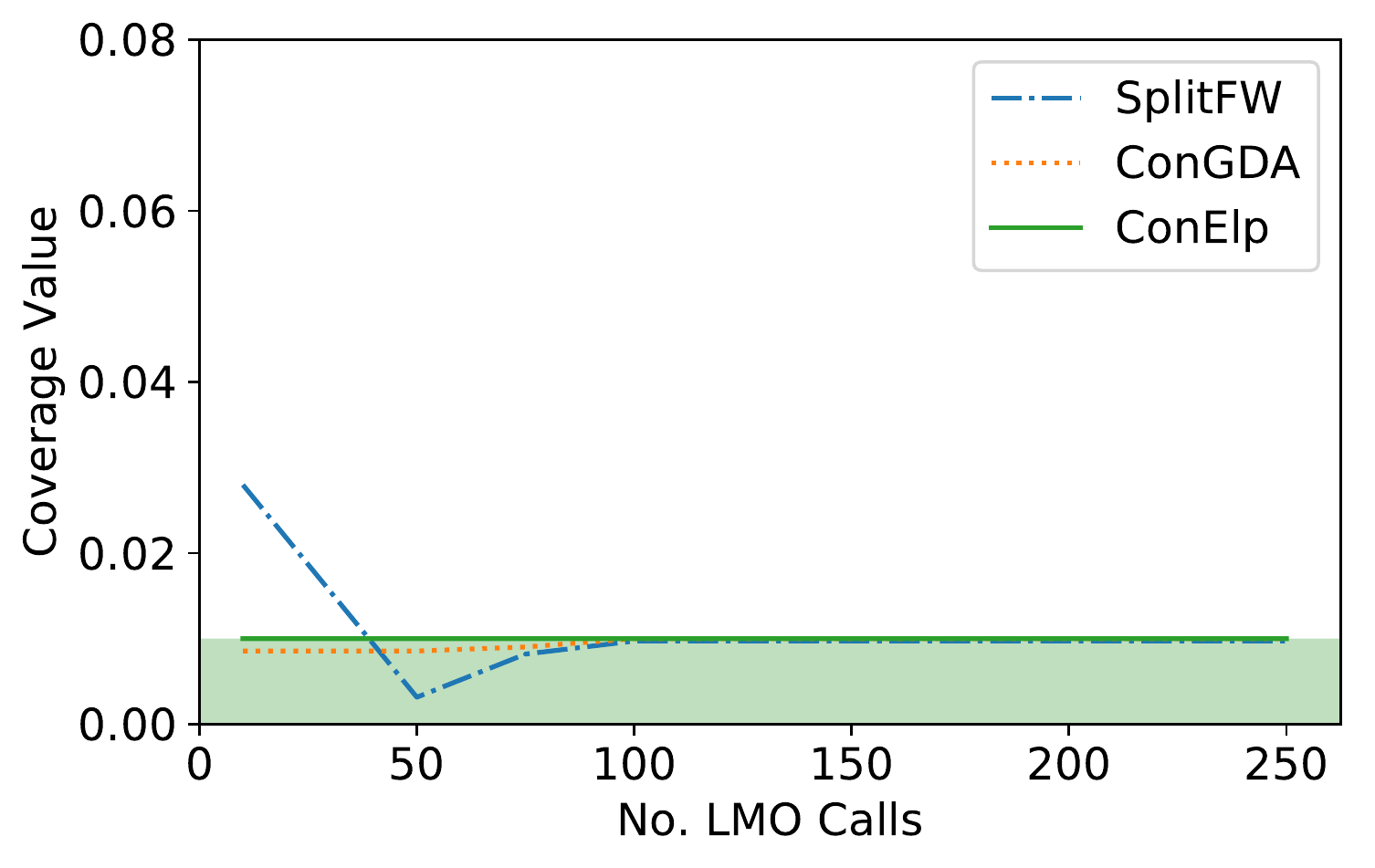}
\label{fig:app-lmo-adult-train}
\caption{Adult}
\end{subfigure}
\begin{subfigure}[b]{0.48\linewidth}
\centering
\includegraphics[width=0.48\linewidth]{Figs/LMOCalls/abalone_train_score.pdf}
\includegraphics[width=0.48\linewidth]{Figs/LMOCalls/abalone_train_cons.pdf}
\label{fig:app-lmo-macho-train}
\caption{Abalone}
\end{subfigure}
\caption{Optimizing the H-mean loss subject to the coverage constraint $\max_i|\sum_{j}C_{ji} \,-\, \pi_i| \leq 0.01$. The plots on the left show the H-mean loss on the train set and those on the right show the coverage violation $\max_i|\sum_{j}C_{ji} \,-\, \pi_i| - 0.01$ on the train set. 
\textit{Lower} H-mean value are \textit{better}, and the constraint values need to be $\leq 0$. 
}
\label{fig:app-lmo-hmean-cov}
\vspace{-15pt}
\end{figure}

\begin{table}[t]
    \centering
    \caption{Comparison of the plug-in and weighted logistic regression (WLR) based LMOs on the task of optimizing the (convex) H-mean loss. The number of iterations, i.e.\ calls to the LMO, is fixed at \emph{100}. \textit{Lower} values are \textit{better}. The results are averaged over 10 random train-test splits}
    \label{tab:cost-sen-100}
    \begin{footnotesize}
    \begin{tabular}{ccccccc}
        \hline
        \multirow{2}{*}{\textbf{Dataset}} &
        \multicolumn{2}{c}{\textbf{FW}} &
        \multicolumn{2}{c}{
        \textbf{Ellipsoid}} &
        \multicolumn{2}{c}{\textbf{GDA}} 
        \\
        &
        \textbf{Plugin} & 
        \textbf{WLR} & \textbf{Plugin} & \textbf{WLR}  & \textbf{Plugin} & \textbf{WLR} \\
        \hline
        Aba  & $0.812 \pm 0.017$ & $\textbf{0.798} \pm \textbf{0.013}$ & $\textbf{0.815} \pm \textbf{0.017}$     & $0.817 \pm 0.012$  & $0.841 \pm 0.032$   & $\textbf{0.837} \pm \textbf{0.035}$ \\
        PgB & $0.127 \pm 0.039$  & $\textbf{0.079} \pm \textbf{0.015}$ & $0.111 \pm 0.026$     & $\textbf{0.079} \pm \textbf{0.018}$  & $0.122 \pm 0.032$  & $\textbf{0.084} \pm \textbf{0.018}$ \\
        MAC    & $\textbf{0.124} \pm \textbf{0.017}$  & $0.245 \pm 0.027$ & $\textbf{0.125} \pm \textbf{0.017}$    & $0.247 \pm 0.027$ & $\textbf{0.124} \pm \textbf{0.016}$  & $0.206 \pm 0.029$ \\
        Sat    & $0.171 \pm 0.007$  & $\textbf{0.170} \pm \textbf{0.007}$ & $0.170 \pm 0.006$  & $\textbf{0.167} \pm \textbf{0.006}$ & $0.171 \pm 0.007$  & $\textbf{0.170} \pm \textbf{0.006}$ \\
        Cov & $0.466 \pm 0.001$ & $\textbf{0.450} \pm \textbf{0.001}$ & $0.466 \pm 0.001$ & $\textbf{0.451} \pm \textbf{0.001}$ & $0.463 \pm 0.001$ & $\textbf{0.447} \pm \textbf{0.001}$\\
        \hline
    \end{tabular}
    \end{footnotesize}
    \vspace{-10pt}
\end{table}


\label{app:trivial-classifier}
\begin{table}[t]
    \centering
    \caption{Performance metrics of a Random Classifiers on the Dataset. \textit{Lower Values are better.}}
    \label{tab:trivial-classifier-datasets}
    \begin{footnotesize}
    \begin{tabular}{ccccccc}
        \hline
        Dataset & H-mean Loss & micro-F1 
        \\
        \hline
        Communities \& Crime & 0.506 & 0.503 
        \\
        COMPAS & 0.501 & 0.504
        \\
        Law School & 0.501 & 0.499
        \\
        Default & 0.499 & 0.498
        \\
        Adult & 0.499 & 0.499
        \\
        Abalone & 0.940 & 0.918
        \\
        Pgblk & 0.839 & 0.794
        \\
        MACHO & 0.914 & 0.874
        \\
        SatImage & 0.835 & 0.832
        \\
        CovType & 0.858 & 0.857
        \\
        \hline
    \end{tabular}
    \end{footnotesize}
\end{table}

\subsection{Additional Experimental Results}
\label{app:expts-additional}
We report the H-mean Loss and micro-F measures of a random classifier on our datasets in Table~\ref{tab:trivial-classifier-datasets}.
We also present additional results for the experiments described in Section \ref{sec:experiments}:
    (i) \textbf{Performance on Constrained Problems} (Section \ref{sec:experiments-cons}): See Figures \ref{fig:app-hmean-cov}-- \ref{fig:app-gmean-eqopp-2}.
    (ii) \textbf{Practical Guidance on Algorithm Choice} (Section \ref{sec:expts-practical}): See Figures \ref{fig:app-min-max}--\ref{fig:app-lmo-hmean-cov}.
    (iii) \textbf{Choice of LMO: Plug-in vs. Weighted Logistic Regression} (Section \ref{sec:expts-LMO}): See Table \ref{tab:cost-sen-100}.







%% file: gda-proof.tex
Theorem \ref{thm:gda-unc} follows from Theorem \ref{thm:gda-con} under the special case of $K = 0$. The algorithms for the
constrained and unconstrained case become identical and the same bounds apply with $r = \infty$. Please see Appendix \ref{app:proof-gda-con} for proof of Theorem \ref{thm:gda-con}.

%% file: 10_SBFW_Proof.tex
\begin{thm*}[(Restated) Convergence of SplitFW algorithm]
Fix $\epsilon>0$. Let $\psi: [0,1]^d \> [0,1]$ be convex, $\beta$-smooth and $L$-Lipschitz w.r.t.\ the $\ell_2$-norm, and let $\phi_1,\ldots, \phi_K: [0,1]^d \> [-1,1]$ be convex and $L$-Lipschitz w.r.t.\ the $\ell_2$-norm. Let $\Omega$ in Algorithm \ref{alg:FW-con} be a $(\rho , \rho', \delta)$-approximate LMO for sample size $N$.
Let $\bar{h}$ be a classifier returned by Algorithm \ref{alg:FW-con} when run for $T$ iterations with some $\zeta>0$. 
Let the strict feasibility condition in Assumption \ref{assp:strict-feasibility} hold for radius $r>0$. Then,  with probability $\geq 1 - \delta$ over draw of $S \sim D^N$, after $T = \cO(1/\epsilon^2)$ iterations, the classifier $\bar{h}$ is near-optimal and near-feasible:
\[
\textbf{Optimality:}~~\psi(\C[\bar{h}]) \,\leq\, \min_{\C \in \cC, \phi_k(\C)\leq 0, \forall k}\,\psi(\C) \,+\,\cO\left(\epsilon + \sqrt{\rho^\eff}\right);
\vspace{-5pt}
\]
\[
\textbf{Feasibility:}~~\phi_k(\C[\bar{h}]) \,\leq\,  \cO\left( \epsilon + \sqrt{\rho^\eff}  \right) ,~\forall k \in [K].
\]
where $\rho^\eff = \rho+\sqrt{d}\rho'$ and the $\cO$ notation hides constant factors independent of $\rho, \rho', T, \epsilon, d, K$ for small enough $\rho, \rho'$ and large enough $T$ (or small enough $\epsilon$).
\end{thm*}

There are two key steps to the proof of this theorem.
First, we show that the use of an approximate LMO in steps 9, 11 of Algorithm~\ref{alg:FW-con} does not affect the convergence results by \cite{Gidel2018}. Specifically, they measure the sub-optimality of an iterate using a duality gap measure.  
In Lemma \ref{lem:split-fw} we show that a similar bound on the duality gap can be derived with an approximate LMO over $\cC$. 
Second, we use the strict feasibilty assumption to convert a bound on the duality gap into a bound on the sub-optimality of the in problem \eqref{eq:op-con-interesection} in Lemma \ref{lem:duality-gap-to-primal}.



We will find it useful to first define the following quantities: fat achievable set, dual functions, and the primal and dual gaps.

\begin{defn}[Fat achievable set]
The set $\cC_{\rho'}$ is defined as follows:
\[
\cC_{\rho'} =\left( \cC + B(\0,\sqrt{d} \rho') \right) \cap \Delta_d  = \{\C+\r : \C \in \cC, \|\r\|_2 \leq \sqrt{d}\rho', \C+\r \in \Delta_d \}.
\]
\end{defn}
The set $\cC_{\rho'}$ is defined so that the iterates $\widetilde \C^t$ and $\C^t$ lie within $\cC_\rho'$ with high probability. 

\begin{defn}[Dual function]
The dual function $f^{\aug}:\R^d\>\R$ is defined as
\[
f^{\aug}(\blambda) = \min_{\C \in \cC_{\rho'}, \F\in \cF} \cL^{\aug}(\C, \F, \blambda).
\]
\end{defn}

We also use $\widehat \C(\blambda), \widehat \F(\blambda)$ to denote any arbitrary minimizer of $\cL^{\aug}(.,.,\blambda)$ over $\cC_{\rho'} \times \cF$. Thus $f^{\aug}(\blambda) = \cL^{\aug}(\widehat \C(\blambda), \widehat \F(\blambda), \blambda)$.
Further, let the maximum value of the dual function be $f^{\aug *}$. By the min-max theorem, we have that
\[ 
f^{\aug *} = \max_{\blambda\in\R^d} \min_{\C\in \cC_{\rho'}, \F \in \cF} \cL^{\aug}(\C, \F, \blambda) = \min_{\C\in \cC_{\rho'}, \F \in \cF} \max_{\blambda\in\R^d} \cL^{\aug}(\C, \F, \blambda) = \min_{\C \in \cC_{\rho'} \cap \cF} 2\psi(\C).
\]
The last equality follows from the observation that if $\C \neq \F$ then $\max_{\blambda\in\R^d} \cL^{\aug}(\C, \F, \blambda)=\infty$. 
Next, let $\C^* \in \cC_{\rho'} \cap \cF$ such that 
\[
\psi(\C^*) = \min_{\C \in \cC_{\rho'} \cap \cF} \psi(\C).
\]
and let $\W^* = \argmax_{\blambda\in\R^d} f^{\aug}(\blambda) \subseteq \R^d$.

\begin{defn}[Primal and dual gaps]
For any $\C \in \cC_{\rho'}, \F \in \cF$ and $\blambda \in \R^d$, we define the primal and dual gaps as follows:
\begin{align*}
    \Delta^{(\text{p})}(\C, \F, \blambda) & = \cL^{\aug}(\C, \F, \blambda) - \min_{\C\in\cC_{\rho'}, \F \in \cF} \cL^{\aug}(\C, \F, \blambda) 
    = \cL^{\aug}(\C, \F, \blambda) - f^{\aug}(\blambda);\\
    \Delta^{(\text{d})}(\blambda) &= f^{\aug *} - f^{\aug}(\blambda) 
    = 2\psi(\C^*) - f^{\aug}(\blambda),
\end{align*}
and define the total gap as $\Delta(\C, \F, \blambda) = \Delta^{(p)}(\C, \F, \blambda) + \Delta^{(d)}(\blambda)$. 
\end{defn}

In the theorems and lemmas below, we will refer to the iterates $\C^t, \F^t, \widetilde \C^t, \widetilde \F^t$ in the Algorithm \ref{alg:FW-con}. We use the the short-hands $\Delta_t, \Delta_t^{(\text{p})}, \Delta_t^{(\text{d})}$ for representing the same primal and dual gaps evaluated at, $(\C^{t+1}, \F^{t+1}, \blambda^t)$.





We will require the use of Theorem 1 and Corollary 1 from \cite{Gidel2018}, which we restate below in our notation.  We use the following facts to transform their Theorem. The norms of vectors correspond to the $\ell_2$-norm unless specified otherwise. We also overload notation and refer to the concatenation of two vectors $\C,\F$ as $[\C,\F]$.
\begin{align*}
| \psi(\C) + \psi(\F) - \psi(\C') - \psi(\F') | 
    & \leq 2L \|[\C-\C',\F-\F']\|_2 \\
\max\left(\text{eigen-val}\left([I, -I]^\top[-I, I] \right) \right)
&= 2 \\
(\text{diam}(\cF))^2 
&\leq \text{diam}(\Delta_d)^2 \leq 2 \\
(\text{diam}(\cC_{\rho'}))^2
&\leq  \text{diam}(\Delta_d)^2 \leq 2 \\
(\text{diam}(\cC_{\rho'}\times\cF))^2
&\leq 4,
\end{align*}
where $\|M\|$ of a matrix $M$ refers to its spectral norm,  and $\text{diam}(\mathcal A)$ refers to the diameter of a set $\mathcal A$,  i.e.  the maximum $\ell_2$ distance between any two elements from the set $\mathcal A$.  





\begin{thm*} [Restated from \cite{Gidel2018}]
There exists a constant $\alpha>0$ such that 
\begin{align*}
f^{\aug *} - f^{\aug}(\blambda) 
&\geq \frac{1}{8L_\zeta } \min\left\{ \alpha^2 \textup{dist}(\blambda, \W^*)^2, \alpha L_\zeta \Zeta^2 \textup{dist}(\blambda, \W^*) \right\};  \\
|| \nabla f^{\aug}(\blambda) ||_2 
&\geq \frac{1}{8L_\zeta}\min\left\{ \alpha^2 \textup{dist}(\blambda, \W^*),  \alpha L_\zeta \Zeta^2 \right\}; \\
|| \nabla f^{\aug}(\blambda) ||_2 
&\geq \frac{\alpha}{\sqrt{8 L_\zeta }}\min\left\{\sqrt{f^{\aug *} - f^{\aug}(\blambda)},\sqrt{ \frac{L_\zeta \Zeta^2}{2}} \right\},
\end{align*}
where $L_\zeta = 2L+2\zeta$ and $\textup{dist}$ represents the standard distance function between a point and a set,  i.e.
$\text{dist}(\x,\mathcal A) = \min_{\x' \in \mathcal A} \|\x - \x'\|$.
\end{thm*}

We will fix a probability of failure $\delta$ throughout the rest of the proof, and assume that the training sample $S$ is ``good'', in which case the empirical confusion matrix output by the $\Omega$  is $\rho'$ close to the true confusion matrix of the classifier whenever it is called by Algorithm \ref{alg:FW-con}.

We then show below that, if the total gap is low then the resulting classifier is close to optimal and feasible.

\begin{lem}
\label{lem:duality-gap-to-primal}
Let the assumptions in Theorem \ref{thm:FW-con} hold. Let $g:\X\>\Delta_n$ be a randomized classifier, and $\C\in\Delta_d$ be such that $\|\C - \C[g]\|_\infty \leq \rho'$. Let $\F \in \cF, \blambda\in\R^d$ be such that $\Delta(\C,\F,\blambda) \leq \tau$ and $||\C -\F||_2^2 \leq \kappa$. We then have:
\vspace{0.3em}
\begin{align*}
\psi(C[g]) &\leq \min_{\C' \in \cC \cap \cF} \psi(\C') + \frac{\tau}{2} + (\gamma+L) \sqrt{\kappa} + L \sqrt{d} \rho' \\
\|\bphi(\C[g]) \|_\infty & \leq L(\sqrt{d}\rho' + \sqrt{\kappa}),
\end{align*}
where $\gamma=\frac{2L}{r} + \frac{\zeta r}{2L} + \frac{\tau L}{r}$.
\end{lem}
\begin{proof}
The second inequality in the lemma trivially follows from the triangle inequality and the $\ell_2$ Lipschitzness of the constraint functions $\phi_k$, i.e. for any $k\in[K]$
\begin{align*}
\phi_k(\C[g]) 
&\leq \phi_k(\C) + L\|\C - \C[g]\|_2 \\
&\leq \phi_k(\C) + L\sqrt{d}\rho' \\
&\leq \phi_k(\F) + L \|\F -\C\|_2 + L\sqrt{d}\rho' \\
&\leq L\sqrt{\kappa} + L\sqrt{d} \rho'
\end{align*}
We will prove the first inequality below. By construction, $\C \in \cC_{\rho'}$. As $\Delta(\C, \F, \blambda) \leq \tau$, we have
\begin{align}
    \Delta^{(\text{p})}(\C, \F, \blambda) &= \cL^{\aug}(\C, \F, \blambda) - \min_{\C' \in \cC_{\rho'}, \F' \in \cF} \cL^{\aug}(\C', \F', \blambda) \leq \tau \label{eq:primal-gap-bound}\\
    \Delta^{(\text{d})}(\blambda) &= 2\psi(\C^*) - \min_{\C' \in \cC_{\rho'}, \F' \in \cF} \cL^{\aug}(\C', \F', \blambda) \leq \tau \label{eq:dual-gap-bound}
\end{align}
where $\C^* \in \argmin_{\C' \in \cC_{\rho'} \cap \cF}\psi(\C')$. Setting $\C'=\F'=\C^*$ in the second term of Eqn.~\eqref{eq:primal-gap-bound}:
\begin{equation}
\label{eq:bound-1-dualitygap}
    \psi(\C) + \psi(\F) +  \blambda^T (\C - \F) + \frac{\zeta}{2}\|\C - \F\|_2^2 \leq 2\psi(\C^*) + \tau ~.
\end{equation}

The variables $\C', \F'$ in the second term of \eqref{eq:dual-gap-bound} are set as follows. Let $\C'=\C[h]$ be a strictly feasible point, i.e. $\bphi(\C') \leq -r$. Such a $h$ exists by Assumption \ref{assp:strict-feasibility}.  As the constraint functions $\phi_k$ are all $L$-Lipschitz w.r.t.\ $\ell_2$ norm, a ball of radius $\frac{r}{L}$ centered at $\C'$ is a subset of $\cF$.  Further, let $\F' = \C'+\frac{r}{L||\blambda||}\blambda$. We then have:
\begin{equation}
\label{eq:bound-2-dualitygap}
2\psi(\C^*) \leq \psi(\C') + \psi(\F') - \frac{r \|\blambda\|_2}{L} + \frac{\zeta r^2}{2 L^2} + \tau.
\end{equation}
This can be reduced to a bound on $||\blambda||_2$,
\begin{equation}
\label{eq:w-norm-bound}
\| \blambda\|_2 \leq \frac{2L}{r} + \frac{\zeta r}{2L} + \frac{\tau L}{r} = \gamma.
\end{equation}
From Cauchy-Schwarz inequality, \eqref{eq:bound-1-dualitygap} becomes:
\begin{align}
\psi(\C) + \psi(\F) 
&\leq 2\psi(\C^*) + \tau - \blambda^\top (\C - \F) - \frac{\zeta}{2}\|\C-\F\|_2^2  \nonumber \\
&\leq 2\psi(\C^*) + \tau +  \gamma \sqrt{\kappa}. \label{eq:bound-3-dualitygap}
\end{align}
As $\psi$ is $L$-Lipschitz, we have
\begin{equation}
\label{eq:bound-4-dualitygap}
\psi(\C) - \psi(\F) \leq L \|\C - \F\|_2 \leq L \sqrt{\kappa}.
\end{equation}
Adding \eqref{eq:bound-3-dualitygap} and \eqref{eq:bound-4-dualitygap} and dividing by 2, we get
\begin{equation*}
\psi(\C) \leq \min_{\C' \in \cC_{\rho'} \cap \cF} \psi(\C') + \frac{\tau}{2} +(\gamma + L)\sqrt{\kappa}.
\end{equation*}
As $\cC_{\rho'} \supseteq \cC$, and $\psi$ is $L$-Lipschitz, we have
\begin{align*}
\psi(C[\g]) 
& \leq \psi(\C) + L \|\C - C[\g]\|_2    \\
& \leq \min_{\C' \in \cC_{\rho'} \cap \cF} \psi(\C') + \frac{\tau}{2} + (\gamma+L) \sqrt{\kappa} + L \sqrt{d} \rho'  \\
& \leq \min_{\C' \in \cC \cap \cF} \psi(\C') + \frac{\tau}{2} + (\gamma+L) \sqrt{\kappa} + L \sqrt{d} \rho',
\end{align*}
which completes the proof.
\end{proof}

The lemma below bounds the duality gap $\Delta_t$ and $\|\C_t - \F_t \|^2$ based on the proof of Theorem 2 in \cite{Gidel2018}. The only difference is the approximate nature of the LMO, that simply contributes an additive factor of $\cO(\rho+\sqrt{d}\rho')$ to the convergence rate of $\cO(1/t)$. The proof is highly technical, and we skip it for brevity. The details can be inferred from \cite{Tavker+2020}, which contains the full proof using a different notation.


\begin{lem}
\label{lem:split-fw}
Let the assumptions in Theorem \ref{thm:FW-con} hold. Let $t_* \in [T]$ be such that $\bar h =  h^{t_*}$ in Algorithm \ref{alg:FW-con}. Let $\Omega$  be a $(\rho, \rho', \delta)$-approximate LMO. For large enough $T$ and $\zeta$,  with probability $1-\delta$ over draw of $S \sim D^N$ we have that
\begin{align*}
 \Delta(\C_{t_*}, \F_{t_*}, \blambda_{t_*-1}) & \leq   c_1 (\rho+\sqrt{d}\rho') + \frac{c_2}{T};
 \\
 || \C_{t_*} - \F_{t_*} ||_2^2 & \leq c_3 (\rho+\sqrt{d}\rho')  + \frac{c_4}{T},
\end{align*}
where $h_{t_*},\F_{t_*},\blambda_{t_*-1}$ are as defined in Algorithm \ref{alg:FW-con}. The constants $c_1, c_2, c_3$ and $c_4$ are independent of the dimension $d$ and number of constraints $K$, approximation constants $\rho, \rho'$ and iterations $T$. More explicitly, $c_1 = \frac{4+12\zeta}{a\zeta}$, $c_2 = 16(\beta + 2\zeta) (t_0 + 2)$, $c_3 = \frac{8+24\zeta}{\zeta}\left[ 1 + \frac{2}{a}\right]$, $c_4=8\left[32(\beta+2\zeta) \frac{(t_0 + 2)}{a} + \frac{64a(\beta+2\zeta)}{\zeta^2} \right]$, $a=\min\left[ \frac{2}{\zeta}, \frac{\alpha^{2}}{8(\beta + 2\zeta)} \right]$, and $t_0$ is a constant $> 0$. 
\end{lem}

We are now ready to prove Theorem \ref{thm:FW-con}.
\begin{proof}[Proof of Theorem \ref{thm:FW-con}]
We first note that Lemma  \ref{lem:split-fw} can be applied to Lemma \ref{lem:duality-gap-to-primal} setting $\tau = c_1 \left( \rho + \sqrt{d}\rho' \right) + \frac{c_2}{T}$ and $\kappa = c_3 \left( \rho + \sqrt{d}\rho' \right) + \frac{c_4}{T}$, with the classifier $g$ in Lemma \ref{lem:duality-gap-to-primal} set to the classifer $\overline h$ returned by Algorithm \ref{alg:FW-con}. For the sake of simplicity, the bound below focuses on the small $\rho, \rho'$ and large $T$ regime. For small enough $\rho, \rho'$ and large enough $T$, we have $(\gamma+L)\sqrt{\kappa} > \tau+L\sqrt{d}\rho'$, based on the simple argument that for a small enough positive scalar $u$ , we have $c\sqrt{u} > u$. 
Thus, from the first inequality of Lemma \ref{lem:duality-gap-to-primal},
\allowdisplaybreaks
\begin{align*}
\psi(\C[\bar{h}]) 
&\leq \min_{\C' \in \cC \cap \cF} \psi(\C') + \frac{\tau}{2} + L \sqrt{d} \rho' + (\gamma+L) \sqrt{\kappa} \\
&\leq \min_{\C' \in \cC \cap \cF} \psi(\C') + 2(\gamma+L) \sqrt{\kappa} \\
&\leq \min_{\C \in \cC, \phi_k(\C)\leq 0, \forall k}\,\psi(\C) + 2(\gamma+L)\left(\sqrt{c_3 \left( \rho + \sqrt{d}\rho' \right)+ \frac{c_4}{T}}\right) \\
&\leq \min_{\C \in \cC, \phi_k(\C)\leq 0, \forall k}\,\psi(\C) + 2(\gamma+L)\left(\sqrt{c_3 \left( \rho + \sqrt{d}\rho' \right) } + \sqrt{\frac{c_4}{T}} \right) \\
&\leq \min_{\C \in \cC, \phi_k(\C)\leq 0, \forall k}\,\psi(\C) + \cO\left( \epsilon + \sqrt{\rho+\sqrt{d}\rho'}  \right)
\end{align*}

By a similar analysis as above, from the second inequality of Lemma \ref{lem:duality-gap-to-primal}, we have for small enough $\rho, \rho'$ and large enough $T$,
\begin{align*}
\phi_k(\C[\bar{h}]) 
&\leq L\sqrt{d} \rho' + L \sqrt{\kappa} \\
&\leq 2L \sqrt{\kappa} \\
&\leq 2L\left(\sqrt{c_3 \left( \rho + \sqrt{d}\rho' \right)+ \frac{c_4}{T}}\right) \\
&\leq \cO\big( \epsilon + \sqrt{\rho+\sqrt{d}\rho'}  \big),
\end{align*}
as desired.
%
%
%
\end{proof}

%% file: con-gda-proof.tex
\begin{thm*}[(Restated) Convergence of ConGDA algorithm]
Fix $\epsilon \in (0,1)$. 
Let $\psi: [0,1]^d\>[0,1]$ and $\phi_1,\ldots,\phi_K: [0,1]^d\>[-1,1]$ be convex and $L$-Lipschitz  w.r.t.\ the $\ell_2$-norm. Let $\Omega$ in Algorithm \ref{alg:GDA-con} be a $(\rho , \rho', \delta)$-approximate LMO for sample size $N$. 
Suppose the strict feasibility condition in Assumption \ref{assp:strict-feasibility} holds for radius $r>0$. 
Let the space of Lagrange multipliers $\Lambda = \{\blambda\in \R^d\,|\,\|\blambda\|_2\leq 2L(1 + 1/r)\}$,
and $\Xi = \{\bmu\in \R_+^K\,|\,\|\bmu\|_1\leq 2/r\}$. 
Let $B_{\bphi} \geq \max_{\bxi \in \Delta_d} \|\bphi(\bxi)\|_2$.
Let $\bar{h}$ be a classifier returned by Algorithm \ref{alg:GDA-con} when run for $T$ iterations,
with step-sizes $\eta = \frac{1}{\bar{L}\sqrt{2T}}$
and $\eta' = \frac{\bar{L}}{(1 + 2\sqrt{K})\sqrt{2T}}$, where $\bar{L} = 4(1 + 1/r)L + 2/r$.
 Then with probability $\geq 1 - \delta$ over draw of $S \sim D^N$, after $T = \cO(K/\epsilon^2)$ iterations:
\[
\psi(\C[\bar{h}]) \,\leq\, \min_{\C \in \cC:\, \bphi(\C) \leq \0}\,\psi(\C) \,+\,\cO\left(L(\epsilon + \rho^\eff)\right)
\]
\vspace{-5pt}
\[
\phi_k(\C[\bar{h}]) \,\leq\, \cO\left(L(\epsilon + \rho^\eff)\right), \forall k \in [K],
\]
where $\rho^\eff = \rho+\sqrt{d}\rho'$ and the $\cO$ notation hides constant factors independent of $\rho,\rho', T, d$ and $K$.
\end{thm*}

The proof is an adaptation of the proof of convergence in \citet{Narasimhan+19_generalized} for their oracle-based optimizer (Theorem 3 in their paper), but takes into account three differences:
(i) they consider a generic objective function that is independent of $\C$, (ii) they assume that $\phi_k$s are monotonic, (iii) they perform a full optimization on $\bxi$ instead of gradient-based updates. Moreover, unlike them, we employ exponentiated gradient updates, and derive a better dependence on dimension.

We will first find it useful to first state the following lemma, 
which adapts the
 proof steps from Lemmas 2, 4 and 6 in \citet{Narasimhan+19_generalized}.
\begin{lem}
\label{lem:gda-con-helper}
Let $\psi: [0,1]^d\>[0,1]$ and $\phi_1,\ldots,\phi_K: [0,1]^d\>[-1,1]$ be convex and $L$-Lipschitz  w.r.t.\ the $\ell_2$-norm. 
Suppose the strict feasibility condition in Assumption \ref{assp:strict-feasibility} holds for radius $r>0$. 
Let $\C^* \in \underset{\C \in \cC:\, \bphi(\C) \leq \0}{\argmin}\, \psi(\C)$, and let 
$$(\blambda^*, \bmu^*) \in \underset{\blambda \in \R^d, \bmu \in \R^K_+}{\argmax}\left\{
    \min_{\C \in \cC,\,\bxi\in \Delta^d} \cL^\con(\C, \bxi, \blambda, \bmu)
    \right\}.$$
Then:
\begin{enumerate}
    \item $\displaystyle
        \psi(\C^*) = \min_{\C \in \cC,\,\bxi \in \Delta_d}\max_{\blambda \in \R^d,\, \bmu \in \R^K_+}\cL^\con(\C, \bxi, \blambda, \bmu) = \max_{\blambda \in \R^d,\, \bmu \in \R^K_+}\min_{\C \in \cC,\,\bxi \in \Delta_d}\cL^\con(\C, \bxi, \blambda, \bmu);$
    \item  $\displaystyle
\psi(\C') = \max_{\blambda \in \Lambda}\min_{\bxi \in \Delta_d}\cL(\C', \bxi, \blambda) = \min_{\bxi \in \Delta_d}\max_{\blambda \in \Lambda}\cL(\C', \bxi, \blambda),$ for any $\C' \in \cC;$
    \item $\|\bmu^*\|_1 \leq 1/r$;
    \item $\|\blambda^*\|_2 \leq L(1 + 1/r)$.
\end{enumerate}

\end{lem}
\begin{proof}
For part 1, we begin by writing out the Lagrangian from \eqref{eq:lagrangian-unc}:
\[
\cL^\con(\C, \bxi, \blambda, \bmu) 
= \psi(\bxi) + \langle \blambda, \C-\bxi \rangle + \langle \bmu, \bphi(\bxi) \rangle.
\]
Since $\cL^\con$ is convex in $\bxi$ and linear in $\blambda$ and $\bmu$, strong duality holds, and we have: 
\begin{eqnarray*}
\max_{\blambda \in \R^d,\, \bmu \in \R^K_+}\min_{\C \in \cC,\,\bxi \in \Delta_d}\cL^\con(\C, \bxi, \blambda, \bmu)
 &=&
    \min_{\C \in \cC,\,\bxi \in \Delta_d}\max_{\blambda \in \R^d,\, \bmu \in \R^K_+}\cL^\con(\C, \bxi, \blambda, \bmu)\\
&=& \min_{\C \in \cC,\,\bxi \in \Delta_d:\, \bxi = \C,\, \bphi(\bxi) \leq \0}\psi(\C)\\
&=& \min_{\C \in \cC,:\, \bphi(\C) \leq \0}\,\psi(\C)
~=\, \psi(\C^*).
\end{eqnarray*}

For part 2, we follow similar steps as part 1 except that it applies to the Lagrangian in \eqref{eq:lagrangian-unc} for the unconstrained problem.

For part 3, recall from our strict feasibility assumption that there exists  $\C' \in \cC$  such that
$\max_{k\in [K]} \phi_k(\C') \leq -r$ for some $r>0$. 
It then follows from part 1 that:
\begin{eqnarray*}
\psi(\C^*) &=& 
\min_{\C \in \cC,\, \bxi \in \Delta_d}\cL^\con(\C, \bxi, \blambda^*, \bmu^*)\\
&\leq&
\cL^\con(\C', \C', \blambda^*, \bmu^*)\\
&\leq&
\psi(\C') +  \langle \bmu^*, \bphi(\C') \rangle
~=\,
\psi(\C') - r\|\bmu^*\|_1.
\end{eqnarray*}
We thus have:
\[
\|\bmu^*\|_1 \,\leq\, (\psi(\C') - \psi(\C^*))/ r \,=\, 1/r.
\]

For part 4, 
letting $\omega(\bxi) = \psi(\bxi) + \langle \bmu^*, \bphi(\bxi) \rangle$,
we note that: 
\begin{eqnarray}
\max_{\blambda \in \R^d}\min_{\bxi \in \Delta_d}\cL^\con(\C^*, \bxi, \blambda, \bmu^*)
&=&
\min_{\bxi \in \Delta_d}\cL^\con(\C^*, \bxi, \blambda^*, \bmu^*)
\nonumber
\\
&=&
    \min_{\bxi \in \Delta_d}
    \left\{
        \psi(\bxi) + \langle \bmu^*, \bphi(\bxi) \rangle
        - \langle \blambda^*, \bxi \rangle 
    \right\}
    + \langle \blambda^*, \C^* \rangle
    \nonumber
\\
&=& 
    \min_{\bxi \in \Delta_d}
        \left\{
            \omega(\bxi) - \langle \blambda^*, \bxi \rangle 
        \right\}
    + \langle \blambda^*, \C^* \rangle
    \nonumber
\\
&=& 
     -\omega^*(\blambda^*)
    + \langle \blambda^*, \C^* \rangle,
    \label{eq:fy-rhs}
\end{eqnarray}
where $\omega^*$ denotes the Fenchel conjugate of $\omega$. 
We similarly note that:
\begin{eqnarray}
\max_{\blambda \in \R^d}\min_{\bxi \in \Delta_d}\cL^\con(\C^*, \bxi, \blambda, \bmu^*)
&=& 
    \max_{\blambda \in \R^d}
    \left\{
    \min_{\bxi \in \Delta_d}
        \left\{
            \omega(\bxi) - \langle \blambda, \bxi \rangle 
        \right\}
    + \langle \blambda, \C^* \rangle
    \right\}
    \nonumber
\\
&=& 
\max_{\blambda \in \R^d}
    \left\{
     -\omega^*(\blambda)
    + \langle \blambda, \C^* \rangle
    \right\}
    \nonumber
    \\
&=&
\omega^{**}(\C^*)
~=~
\omega(\C^*),
    \label{eq:fy-lhs}
\end{eqnarray}
where $\omega^{**}$ denotes the second Fenchel conjugate of $\omega$. 
From \eqref{eq:fy-rhs} and \eqref{eq:fy-lhs},
its clear that:
\[
\omega(\C^*) 
\,=\,
-\omega^*(\blambda^*)
    + \langle \blambda^*, \C^* \rangle.
\]
An application of the Fenchel-Young inequality then gives us that:
\[
\lambda^* = \nabla\omega(\C^*) =
\nabla\psi(\C^*) + \sum_{k=1}^K \mu^*_k \nabla\phi_k(\C^*).
\]
We can thus bound the norm of $\blambda^*$ as:
\begin{eqnarray*}
\|\lambda^*\|_2
&\leq&
\|\nabla\psi(\C^*)\|_2
+ 
\sum_{k=1}^K|\mu_k^*|\|\nabla\phi_k(\C^*)\|_2\\
&\leq&
\|\nabla\psi(\C^*)\|_2
+ 
\|\bmu^*\|_1\,\max_{k \in K}\|\nabla\phi_k(\C^*)\|_2
~=\,
L(1 + 1/r),
\end{eqnarray*}
which follows from part 2 and the fact that $\psi$ and $\phi_k$s are Lipschitz w.r.t. the $\ell_1$-norm.
\end{proof}

\begin{proof}[Proof of Theorem \ref{thm:gda-con}]
%
We begin by writing out the Lagrangian from \eqref{eq:lagrangian-unc}:
\[
\cL^\con(\C, \bxi, \blambda, \bmu) 
= \psi(\bxi) + \langle \blambda, \C-\bxi \rangle + \langle \bmu, \bphi(\bxi) \rangle
= 
\underbrace{\psi(\bxi) - \langle \blambda, \bxi \rangle  + \langle \bmu, \bphi(\C) \rangle}_{\cL_1(\bxi, \blambda, \bmu)} + 
\underbrace{\langle \blambda, \C \rangle}_{\cL_2(\C, \blambda)}.
\]

\textbf{Optimality.} 
To show optimality, note that $\cL_1$ is convex in $\xi$ and linear in $\blambda$ and $\bmu$, and $\cL_2$ is linear both in $\C$ and $\blambda$.
The use of a $(\rho , \rho', \delta)$-approximate LMO to compute $\C^t$ and $h^t$ at each iteration gives us with probability at least $1-\delta$ (over draw of $S$):
\begin{eqnarray}
\frac{1}{T}\sum_{t=1}^T\cL_2(\C^t, \blambda^t) &\leq&
\frac{1}{T}\sum_{t=1}^T\cL_2(\C[h^t], \blambda^t) \,+\, \|\blambda^t\|_1\|\C^t - \C[h^t]\|_\infty
\nonumber
\\
&\leq &
\|\blambda^t\|_\infty\frac{1}{T}\sum_{t=1}^T\min_{\C \in \cC}\left\langle \frac{\blambda^t}{\|\blambda^t\|_\infty}, \C \right\rangle \,+\,
\|\blambda^t\|_\infty\rho \,+\,
\|\blambda^t\|_1\rho'
\nonumber
\\
&\leq&
\min_{\C\in\cC}\frac{1}{T}\sum_{t=1}^T\cL_2(\C, \blambda^t) \,+\, 2L(1+1/r)\rho + 2L\sqrt{d}(1+1/r)\rho'.
\nonumber
\\
&=&
\min_{\C\in\cC}\frac{1}{T}\sum_{t=1}^T\cL_2(\C, \blambda^t) \,+\, \bar{\rho},
\label{eq:gda-con-lmo}
\end{eqnarray}
where we denote $\bar{\rho} = 2L(1+1/r)\rho + 2L\sqrt{d}(1+1/r)\rho'$.

Next, we apply the classical regret bound guarantee for online gradient \textit{descent} \citep{Zinkevich03,shalev2011online}, we have from the sequence of objectives $\cL_1(\cdot, \blambda^t, \bmu^t)$'s (where the optimization is over $\xi$). Note that 
\begin{eqnarray*}
\|\nabla_\bxi \cL_1(\bxi, \blambda^t, \bmu^t)\|_2 &\leq& 
\|\nabla_\bxi \psi(\bxi)\|_2 +
\|\blambda^t\|_2 + \|\bmu^t\|_1\max_k \|\nabla_\bxi \phi_k(\bxi)\|_2\\
&\leq&
L + 2L(1+1/r) + 2L/r =
(3 + 4/r)L \,\leq\, \bar{L}.
\end{eqnarray*}
Also note that $\max_{\bxi \in \Delta_d} \|\bxi\|_2 \leq 1$. 
So with
$\eta = \frac{1}{\bar{L}\sqrt{2T}}$, we have:
\begin{eqnarray}
\frac{1}{T}\sum_{t=1}^T\cL_1(\bxi^t, \blambda^t, \bmu^t)~\leq~
\min_{\bxi\in [0,1]^d}\,\sum_{t=1}^T\,\cL_1(\bxi, \blambda^t, \bmu^t) \,+\, 
\frac{\sqrt{2}\bar{L}}{\sqrt{T}}.
\label{eq:gda-con-xi-regret}
\end{eqnarray}
Combining \eqref{eq:gda-con-lmo} and \eqref{eq:gda-con-xi-regret}, we have with probability at least $1-\delta$ (over draw of $S$):
\begin{eqnarray}
{\frac{1}{T}\sum_{t=1}^T\cL^\con(\C^t, \bxi^t, \blambda^t, \bmu^t)}
&\leq&
\min_{\C\in\cC,\, \bxi\in[0,1]^d}\,\sum_{t=1}^T\,\cL^\con(\C, \bxi, \blambda^t, \bmu^t) \,+\, \frac{\sqrt{2}\bar{L}}{\sqrt{T}} \,+\, \bar{\rho}
\nonumber
\\
&=&
\min_{\C\in\cC,\, \bxi\in[0,1]^d}\,\cL(\C, \bxi, \bar{\blambda}, \bar{\bmu}) \,+\, \frac{\sqrt{2}\bar{L}}{\sqrt{T}} \,+\, \bar{\rho}
\nonumber
\\
&\leq&
\max_{\blambda \in \R^d, \bmu \in \R^K_+}\min_{\C\in\cC,\, \bxi\in[0,1]^d}\,\cL^\con(\C, \bxi, {\blambda}, \bmu) \,+\, \frac{\sqrt{2}\bar{L}}{\sqrt{T}} \,+\, \bar{\rho}
\nonumber\\
&=&
\min_{\C\in\cC}\left\{\max_{\blambda \in \R^d, \bmu \in \R^K_+}\,\min_{\bxi\in[0,1]^d}\,\cL^\con(\C, \bxi, {\blambda}, \bmu)\right\} \,+\, \frac{\sqrt{2}\bar{L}}{\sqrt{T}} \,+\, \bar{\rho}
\nonumber\\
&=&
\min_{\C\in\cC:\, \bphi(\C) \leq \0}\psi(\C) \,+\, \frac{\sqrt{2}\bar{L}}{\sqrt{T}} \,+\, \bar{\rho},
\label{eq:gda-con-C-xi}
\end{eqnarray}
where in the second step $\bar{\blambda} = \frac{1}{T}\sum_{t=1}^T\blambda^t$ and $\bar{\bmu} = \frac{1}{T}\sum_{t=1}^T\bmu^t$ and we use the linearity of $\cL$ in $\blambda$ and $\bmu$; in the fourth step we use strong duality to interchange the max and min; and the last step follows
from Lemma \ref{lem:gda-con-helper} (part 1).

\allowdisplaybreaks
Similarly, we apply the standard online gradient \textit{ascent} analysis on the sequence of losses $\cL^\con(\C^t, \bxi^t, \cdot, \cdot)$'s, where the optimization is over $\blambda$ and $\bmu$. 
Note that
$\|\nabla_{\blambda, \bmu}\, \cL^\con(\C^t, \bxi^t,\blambda^t, \bmu^t)\|_2 = \|\C^t - \bxi^t\|_2 +\|\bphi(\bxi^t)\|_2 \leq 1 + B_{\bphi}$
and $
\textstyle
\left\|
\begin{bmatrix}
\blambda\\
\bmu
\end{bmatrix}
\right\|_2 
\leq 2L(1 + 1/r) + 2/r \leq \bar{L}
$ (from Lemma \ref{lem:gda-con-helper}, parts 3--4). 
So with 
$\eta' = \frac{\bar{L}}{(1 + B_{\bphi})\sqrt{2T}}$, we have:
\begin{eqnarray}
\lefteqn{\frac{1}{T}\sum_{t=1}^T\cL^\con(\C^t, \bxi^t, \blambda^t, \bmu^t)}
\nonumber\\
&\geq&
\max_{\blambda \in \Lambda, \, \bmu \in \Xi}\,\sum_{t=1}^T\,\cL^\con(\C^t, \bxi^t, \blambda, \bmu) \,-\, \frac{\sqrt{2}\bar{L}(1 + B_{\bphi})}{\sqrt{T}}
\nonumber
\\
&\geq&
\max_{\blambda \in \Lambda, \, \bmu \in \Xi}
\left\{
\sum_{t=1}^T\,\cL^\con(\C[h^t], \bxi^t, \blambda, \bmu)
\,-\,  \|\blambda\|_1\|\C^t - \C[h^t]\|_\infty
\right\}
\,-\, \frac{\sqrt{2}\bar{L}(1 + B_{\bphi})}{\sqrt{T}}
\nonumber
\\
&\geq&
\max_{\blambda \in \Lambda, \, \bmu \in \Xi}\,\sum_{t=1}^T\,\cL(\C[h^t], \bxi^t,\blambda,\bmu) 
\,-\,  2L(1+1/r)\sqrt{d}\rho'
\,-\, \frac{\sqrt{2}\bar{L}(1 + B_{\bphi})}{\sqrt{T}}
\nonumber
\\
&\geq&
\max_{\blambda \in \Lambda, \, \bmu \in \Xi}\,\cL^\con(\C[\bar{h}], \bar{\bxi},\blambda, \bmu) 
\,-\,  2L(1+1/r)\sqrt{d}\rho'
\,-\,  \frac{\sqrt{2}\bar{L}(1 + B_{\bphi})}{\sqrt{T}}
\nonumber\\
&=&
\max_{\blambda \in \Lambda}\left\{
\psi(\bar{\bxi}) + \langle \blambda, \C[\bar{h}]-\bar{\bxi} \rangle \right\}
\,+\, \max_{\bmu \in \Xi}\, \langle \bmu, \bphi(\bar{\bxi}) \rangle
\,-\,  2L(1+1/r)\sqrt{d}\rho'
\,-\,  \frac{\sqrt{2}\bar{L}(1 + B_{\bphi})}{\sqrt{T}}
\label{eq:gda-con-lambda-common}\\
&\geq&
\min_{\bxi\in [0,1]^d}\left\{
\max_{\blambda \in \Lambda}\,\left\{
\psi(\bxi) + \langle \blambda, \C[\bar{h}]-\bxi \rangle \right\}
\,+\, \max_{\bmu \in \Xi}\, \langle \bmu, \bphi(\bxi) \rangle
\right\}
\,-\,  2L(1+1/r)\sqrt{d}\rho'
 \,-\, \frac{\sqrt{2}\bar{L}(1 + B_{\bphi})}{\sqrt{T}}
\nonumber
%
\\
&\geq&
\min_{\bxi\in [0,1]^d}\left\{
\max_{\blambda \in \Lambda}\,\left\{
\psi(\bxi) + \langle \blambda, \C[\bar{h}]-\bxi \rangle \right\}
\,+\,  \langle \0, \bphi(\bxi) \rangle
\right\}
\,-\,  2L(1+1/r)\sqrt{d}\rho'
 \,-\, \frac{\sqrt{2}\bar{L}(1 + B_{\bphi})}{\sqrt{T}}
\nonumber
\\
&=&
\psi(\C[\bar{h}])
\,-\,  2L(1+1/r)\sqrt{d}\rho'
\,-\, \frac{\sqrt{2}\bar{L}(1 + B_{\bphi})}{\sqrt{T}},
\label{eq:gda-con-lambda-opt}
\end{eqnarray}
where
in the third step, we use the fact that for any $\lambda \in \Lambda$,  $\|\blambda\|_\infty \leq \|\blambda\|_2 \leq L$, and the property of the LMO. 
In the fourth step, we use 
$\C[\bar{h}] = \frac{1}{T}\sum_{t=1}^T \C[h^t]$ and $\bar{\bxi} = \frac{1}{T}\sum_{t=1}^T \bxi^t$, and use the linearity of $\cL$ in $\C$, and convexity of $\cL$ in $\bxi$ and Jensen's inequality. In the last step, we apply Lemma \ref{lem:gda-con-helper} (part 2).
The last six steps hold with probability at least $1-\delta$.

Combining \eqref{eq:gda-con-C-xi} and \eqref{eq:gda-con-lambda-opt},
we get with probability at least $1-\delta$ (over draw of $S$), for any $\bmu' \in \Xi$
\begin{eqnarray*}
\psi(\C[\bar{h}])  \leq \min_{\C \in \cC:\, \bphi(\C) \leq \0}\,\psi(\C)
\,+\, \frac{\sqrt{2}\bar{L}(2 + B_{\bphi})}{\sqrt{T}}
\,+\, 2L(1+1/r)(\rho + 2\sqrt{d}\rho').
\end{eqnarray*}
Setting
$B_\bphi = \sqrt{K}$ and
$T=(K+1)/\epsilon^2$  
completes the proof of optimality.

\textbf{Feasibility.} 
Let $\C^*, \blambda^*, \bmu^*$ be as defined in Lemma \ref{lem:gda-con-helper}. 
To show feasibility, combining \eqref{eq:gda-con-C-xi} and \eqref{eq:gda-con-lambda-common}, and
interchanging the min and max, we get:
\begin{equation}
\max_{\blambda \in \Lambda, \bmu \in \Xi}\left\{
\psi(\bxi) + \langle \blambda, \C[\bar{h}]-\bar{\bxi} \rangle
\,+\, \langle \bmu, \bphi(\bar{\bxi}) \rangle
\right\}
\,\leq\,
\psi(\C^*) 
\,+\,  \tilde{\rho}
 \,+\, \frac{\sqrt{2}\bar{L}(2 + B_{\bphi})}{\sqrt{T}},
 \label{eq:gda-con-feas}
\end{equation}
where we denote $\tilde{\rho} = 2L(1+1/r)(\rho + 2\sqrt{d}\rho')$
Let $k' \in \argmax_{k\in[K]}\, \phi_k(\C[\bar{h}])$ denote the index of the most-violated among the $K$ constraints $\phi_1(\C[\bar{h}]), \ldots, \phi_K(\C[\bar{h}])$. Also let $\blambda' = \blambda^*$
and $\mu'_{k'} =  \mu^*_{k'} + \frac{1}{r}$ and $\mu'_{k} =  \mu^*_{k}, \forall k \ne k'$. 
Note that
$\blambda' \in \Lambda$ and
$\bmu' \in \Xi$. 
Substituting $(\mu', \lambda')$ into the LHS of \eqref{eq:gda-con-feas}, we have:
\begin{equation*}
\psi(\bar{\bxi}) + \langle \blambda^*, \C[\bar{h}]-\bar{\bxi} \rangle
\,+\, \langle \bmu^*, \bphi(\bxi) \rangle
\,+\, \frac{1}{r}\max_{k}\,\phi_k(\bar{\bxi})
\,\leq\,
\psi(\C^*) 
\,+\,  \tilde{\rho}
 \,+\, \frac{\sqrt{2}\bar{L}(1 + B_{\bphi})}{\sqrt{T}},
\end{equation*}
and we further get:
\begin{eqnarray*}
\lefteqn{
\min_{\C\in \cC,\, \bxi\in [0,1]^d}\,\left\{
\psi(\bxi) + \langle \blambda^*, \C-\bxi \rangle
\,+\, \langle \bmu^*, \bphi(\bxi) \rangle
\right\}
\,+\, \frac{1}{r}\max_{k}\,\phi_k(\bar{\bxi})}
\\
&\leq&
\psi(\C^*) 
\,+\,  \tilde{\rho}
 \,+\, \frac{\sqrt{2}\bar{L}(1 + B_{\bphi})}{\sqrt{T}}.
\end{eqnarray*}
Applying Lemma \ref{lem:gda-con-helper} (part 1), 
\begin{equation*}
\psi(\C^*) 
\,+\, \frac{1}{r}\max_{k}\,\phi_k(\bar{\bxi})
\,\leq\,
\psi(\C^*) 
\,+\, \tilde{\rho}
 \,+\, \frac{\sqrt{2}\bar{L}(1 + B_{\bphi})}{\sqrt{T}},
\end{equation*}
giving us for all $k$:
\begin{equation}
\phi_k(\bar{\bxi})
\,\leq\,
r
\left(
\tilde{\rho}
 \,+\, \frac{\sqrt{2}\bar{L}(1 + B_{\bphi})}{\sqrt{T}}\right).
 \label{eq:gda-con-feas-1}
\end{equation}

Next set
$\bmu' = \bmu^*$
and $$\lambda'_{j'} =  \lambda^*_{j'} + \frac{L(1+1/r)}{\left\|\C[\bar{h}] - \bar{\bxi}\right\|_2}(\C_{j'}[\bar{h}] - \bar{\xi}_{j'}).$$ Substituting $(\mu', \lambda')$ into the LHS of \eqref{eq:gda-con-feas}, we have:
\begin{eqnarray*}
\lefteqn{
\psi(\bar{\bxi}) + \langle \blambda^*, \C[\bar{h}]-\bar{\bxi} \rangle
\,+\, \langle \bmu^*, \bphi(\bxi) \rangle
\,+\,{L(1+1/r)}\left\|\C[\bar{h}] - \bar{\bxi}\right\|_2}\\
&\leq&
\psi(\C^*) 
\,+\,  \tilde{\rho}
 \,+\, \frac{\sqrt{2}\bar{L}(1 + B_{\bphi})}{\sqrt{T}},\hspace{5cm}
\end{eqnarray*}
and again taking a min over $\C$ and $\bxi$ and applying Lemma \ref{lem:gda-con-helper},
we get
\begin{eqnarray}
{\left\|\C[\bar{h}] - \bar{\bxi}\right\|_2}
&\leq&
\frac{1}{L(1+1/r)}\left(\tilde{\rho}
 \,+\, \frac{\sqrt{2}\bar{L}(1 + B_{\bphi})}{\sqrt{T}}\right).
 \label{eq:gda-con-feas-2}
\end{eqnarray}
Combining \eqref{eq:gda-con-feas-1} and \eqref{eq:gda-con-feas-2}, and using the Lipschitz property of $\phi_k$, we get for all $k$:
\begin{eqnarray*}
\phi_k(\C[\bar{h}])
&\leq&
L\left\|\C[\bar{h}] - \bar{\bxi}\right\|_2
\,+\,
r
\left(
\tilde{\rho}
 \,+\, \frac{\sqrt{2}\bar{L}(1 + B_{\bphi})}{\sqrt{T}}
 \right)\\
 &\leq&
\frac{r(2+r)}{1+r}
\left(
\tilde{\rho}
 \,+\, \frac{\sqrt{2}\bar{L}(1 + B_{\bphi})}{\sqrt{T}}
 \right)
 ~\leq~
 r
\left(
\tilde{\rho}
 \,+\, \frac{\sqrt{2}\bar{L}(1 + B_{\bphi})}{\sqrt{T}}
 \right).
\end{eqnarray*}
Setting $B_\bphi = 2\sqrt{K}$ and
$T=(K+1)/\epsilon^2$ 
completes the proof of feasibility.
\end{proof}

%% file: con-ellipsoid-proof.tex
\begin{thm*}[(Restated) Convergence of ConEllipsoid]
Fix $\epsilon \in (0,1)$. 
Let $\psi: [0,1]^d\>[0,1]$ and $\phi_1,\ldots,\phi_K: [0,1]^d\>[-1,1]$ be convex and $L$-Lipschitz w.r.t. the $\ell_2$ norm. Let $\Omega$ in Algorithm \ref{alg:GDA-con} be a $(\rho , \rho', \delta)$-approximate LMO for sample size $N$. Suppose the strict feasibility condition in Assumption \ref{assp:strict-feasibility} holds for radius $r>0$. Let the initial classifier $h^0$ satisfy this condition, i.e. $\bphi(\C[h^0]) \leq -r$ and $\C[h^0]=\C^0$. Let $\bar{d}=d+K$. Let $\bar{h}$ be the classifier returned by Algorithm \ref{alg:ellipsoid-con} when run for $T>2\bar{d}^2 \log(\frac{\bar{d}}{\epsilon})$ iterations with  initial radius $a>2(L+\frac{L+1}{r})$.
 Then with probability $\geq 1 - \delta$ over draw of $S \sim D^N$, we have
\begin{align*}
\textbf{Optimality:}~~\psi(\C[\bar{h}]) \,&\leq\, \min_{\C\in\cC:\, \phi_k(\C) \leq 0,\forall k}\,\psi(\C) + \big(4a\big)\epsilon + 4a(\rho+2\sqrt{d}\rho') ;   \\
\textbf{Feasibility:}~~\phi_k(\C[\bar{h}]) \,&\leq\, a(\rho+2\sqrt{d}\rho') ,~\forall k \in [K]
\end{align*}
\end{thm*}
\noindent
In both the constrained and unconstrained versions of the Ellipsoid Algorithm, successive ellipsoids are constructed by obtaining the L\"{o}wner-John ellipsoid (JLE), i.e., the minimum volume ellipsoid containing the intersection of the current ellipsoid and a half space obtained by drawing a cutting hyperplane through the current center. This process yields a sequence of ellipsoids with geometrically decreasing volumes. We restate the lemma from \citet{Bubeck15} that establishes this fact. 

\begin{lem}
\label{lem:ellipsoid-volumes}
Let the ellipsoid $\mathcal{E}^{0} = \{\x \in \R^d: (\x- \c_0)^\top \mathbf{H}_0^{-1} (\x-\c_0) \leq 1 )\}$, where $H_0 \in \R^{d\times d}$ is a positive definite matrix and $\c_0, \in \R^d$. Let $( \mathbf{H}, \c) = \JLE(\mathbf{H}_0,\c_0,\g)$ , where $\JLE$ refers to the subroutine \ref{alg:jle}. Let the ellipsoid $\cE=\{\x \in \R^d: (\x- \c)^\top \mathbf{H}^{-1} (\x-\c) \leq 1 )\}$. Then,
\begin{align*}
\mathcal{E} &\supset \cE^0 \cap \{\x \in \R^d: \g^\top (\x-\c_0)  \geq 0 \} \\
\vol(\mathcal{E}) &\leq \exp\left(\frac{-1}{2d}\right)\hspace{0.3em}\vol(\mathcal{E}^{0})
\end{align*}
where $\vol$ refers to the standard $d$-dimensional volume.
\end{lem}

We will define some functions and variables below that will be useful in our proofs:
\begin{align*}
\cL^\con (\C, \bxi , \blambda, \bmu) &= \psi(\bxi) + \blambda^\top (\C - \bxi) + \bmu ^\top \bphi(\bxi) \\
f^\con (\blambda, \bmu) &= \min_{\C \in \cC, \bxi \in \Delta_d} \cL^\con (\C, \bxi , \blambda, \bmu) \\
\mathcal{R}^{0} &:= \{\x \in \R^{d+K}: \vert\vert \x \vert\vert_{2} \leq a, \x_{d+i} \geq 0, \hspace{0.3em} \forall i \in \{1, 2, ..., K\}\} \\
\widehat f^\con (\blambda, \bmu)& = f^\con (\blambda, \bmu) - \infty \1([\blambda, \bmu] \notin \mathcal{R}^{0}) \\
\bxi(\blambda, \bmu) &\in \argmin_{\bxi \in \Delta_d} \psi(\bxi) - \blambda^\top \bxi + \bmu ^\top \bphi(\bxi) 
\end{align*}
The helper function $\widehat f^\con (\blambda, \bmu)$ is equal to $ f^\con (\blambda, \bmu)$ when $[\blambda, \bmu] \in \mathcal{R}^{0}$. Let $h^t$ and $\C^t$ be the iterates in Algorithm \ref{alg:ellipsoid-con}. Let $\cE^t$ denote the ellipsoid centered at $[\blambda^t,\bmu^t]$ with axes given by the eigen vectors of $\A^t$, with axes lengths squared given by the corresponding eigenvalues of $\A^t$, i.e.
\[
\cE^t = \{[\blambda, \bmu]\in \R^{d+K}: [\blambda-\blambda^t, \bmu-\bmu^t]^\top {(\A^t)}^{-1} [\blambda-\blambda^t, \bmu-\bmu^t] \leq 1 \}
\]

We abuse notation sometimes in the proof below by interchangeably using the ellipsoid $\cE^t$ and its corresponding center, axis matrix $[\blambda^t, \bmu^t],\A^t$ whenever the context is clear. For example, line 14 of Algorithm \ref{alg:ellipsoid-con} can be written compactly as $\cE^{t+1}=\JLE(\cE^t, [\C^t-\bxi^t, \bphi(\bxi^t)])$.  
 

\subsubsection{Bounding the Dual Suboptimality of $[\blambda^t, \bmu^t]$}

We first prove, that for any iteration $t \in \{0, 1, ..., T-1\}$, if $[\blambda, \bmu] \notin \mathcal{R}^{0}$, then  $\cE^{t+1} \supseteq \{\cE^{t} \cap \mathcal{R}^{0}\}$. We establish this in the following three lemmas.  

\begin{lem} 
\label{lem:ellipsoid if condition}
If at any iteration $t \in \{0, 1, ..., T-1\}, \hspace{0.3em} \vert\vert [{\blambda}^{t}, {\bmu}^{t}]  \vert\vert_{2} > a$, then $\cE^{t+1} \supseteq \{\cE^{t} \cap \mathcal{R}^{0}\}$
\end{lem}

\begin{proof}
Let $t \in \{0, 1, ..., T-1\}$, such that, $\vert\vert [{\blambda}^{t}, {\bmu}^{t}]  \vert\vert_{2} > a$. In such a case, the \textbf{if} condition (line 5) of algorithm \ref{alg:ellipsoid-con} gets invoked and we obtain the new ellipsoid $\cE^{t+1}$. Due to the $\JLE$ construction, we get that
\begin{align*}
     \cE^{t+1} &\supseteq {\cE}^{t} \cap \{\x \in \R^{d+K}: (\x-[{\blambda}^{t}, {\bmu}^{t}])^\top(-[{\blambda}^{t}, {\bmu}^{t}])) \geq 0\} \\ 
     &= {\cE}^{t}  \cap \{\x \in \R^{d+K}: \x^\top[{\blambda}^{t}, {\bmu}^{t}] \leq \vert\vert [{\blambda}^{t}, {\bmu}^{t}] \vert\vert_{2}^{2}]\} \\
     &\supseteq {\cE}^{t} \cap \mathbf{B}(\0_{d+K}, \vert\vert [\hat{\blambda}^{t}, \hat{\bmu}^{t}]
     \vert\vert_{2}) \supseteq \{\cE^{t} \cap \mathcal{R}^{0}\}
\end{align*}
Thus, $\cE^{t+1} \supseteq \{\cE^{t} \cap \mathcal{R}^{0}\}$. 
\end{proof}

\begin{lem} 
\label{lem:ellipsoid else-if condition}
If at any iteration $t \in \{0, 1, ..., T-1\}, \hspace{0.3em} \vert\vert [{\blambda}^{t}, {\bmu}^{t}]  \vert\vert_{2} \leq a$, and ${\bmu}^t \ngeq \0$, then $\cE^{t+1} \supseteq \{\cE^{t} \cap \mathcal{R}^{0}\}$
\end{lem}

\begin{proof}
Let $t \in \{0, 1, ..., T-1\}$, such that, $\vert\vert [{\blambda}^{t}, {\bmu}^{t}]  \vert\vert_{2} \leq a$, while ${\bmu}^t \ngeq \0$. In such a case, the \textbf{else-if} condition (line 8) of algorithm \ref{alg:ellipsoid-con} gets invoked and we obtain the new ellipsoid $\cE^{t+1}$. Due to the $\JLE$ construction, we get that
\begin{align*}
    \cE^{t+1} &\supseteq \cE^{t} \cap \{\x \in \R^{d+K}: (\x-[{\blambda}^{t}, {\bmu}^{t}])^\top([\0_{d}, \text{pos}(-\bmu^{t})])) \geq 0\} \\ 
     &= \cE^{t} \cap \{\x \in \R^{d+K}: \x^\top[\0_{d}, \text{pos}(-\bmu^{t})] \geq [\blambda^t, \bmu^t]^\top[\0_{d}, \text{pos}(-\bmu^{t})] \\
     &\supseteq \cE^{t} \cap  \{\x \in \R^{d+K}: \x_{d+i} \geq 0, \forall i \in {1, 2, ..., K}\} \supseteq \{\cE^{t}\cap \mathcal{R}^{0}\}
\end{align*}
Thus, $\cE^{t+1} \supseteq \{\cE^{t} \cap \mathcal{R}^{0}\}$. 
\end{proof}

\begin{lem}
\label{lem:ellipsoid center filter}
For any iteration $t \in \{0, 1, ..., T-1\}$ of Algorithm \ref{alg:ellipsoid-con}, if $[\blambda, \bmu] \notin \mathcal{R}^{0}$, then  $\cE^{t+1} \supseteq \{\cE^{t} \cap \mathcal{R}^{0}\}$
\end{lem}
\begin{proof}
The result follows directly from Lemmas \ref{lem:ellipsoid if condition} and \ref{lem:ellipsoid else-if condition}.
\end{proof}
We would also like to prove that the optimal solution, i.e., the maximizer of  $f^\con$ over $\R^d \times \R_+^K$ indeed lies inside our search space. In our setting, we show in \ref{lem:ellipsoid maximizer} that the maximizer lies inside $\mathcal{R}^{0}$
\begin{lem}
\label{lem:ellipsoid maximizer}
Let $(\blambda^*, \bmu^*)$ be a maximizer of $f^\con$ over $\R^d \times \R_+^K$. Then $[\blambda^*, \bmu^*] \in \mathcal{R}^{0}$ 
\end{lem}
\begin{proof}
From Lemma \ref{lem:gda-con-helper} (parts 3-- 4) we have that $ \|[\blambda^*,\bmu^*]\|_2 \leq L + \frac{L+1}{r} \leq a/2$. Thus:
\begin{equation*}
\max_{\blambda \in \R^d, \bmu \in \R^K_+} f^\con(\blambda, \bmu) = \max_{\blambda \in \R^d, \bmu \in \R^K_+} \widehat f^\con(\blambda, \bmu) = 
\psi(\C^*) = \min_{\C \in \cC, \bphi(\C) \leq 0} \psi(\C)
\label{eqn:dual-optimizer-bounded}
\end{equation*}
This ensures that $[\blambda^*, \bmu^*] \in \mathcal{R}^{0}$
\end{proof}
\noindent
Lemmas \ref{lem:ellipsoid center filter} and \ref{lem:ellipsoid maximizer} allow us to establish Lemma \ref{lem:ellipsoid intermediate center}, which will be required in proving Theorem \ref{thm:ellipsoid-dual-con}. 

\begin{lem}
\label{lem:ellipsoid intermediate center}
Let $\epsilon \in [0, 1]$ and $[\blambda^*, \bmu^*]$ be any maximiser of $\widehat f^\con$. Define the convex set $\cR^0_\epsilon \subseteq \cR^0 \subseteq \R^d\times \R_+^K$ as 
\[\mathcal{R}^{0}_{\epsilon} := \{[\blambda, \bmu]\in \mathcal{R}^{0}:(1-\epsilon)[\blambda^{*}, \bmu^{*}] + \epsilon[\blambda, \bmu] \} .\]
Let the number of iterations $T$ in Algorithm \ref{alg:ellipsoid-con}, be such that $ T > 2(d+K)^{2}\log\left(\frac{2}{\epsilon}\right)$. Then there exists an iteration $t^{*} \in \{0, 1, .., T-1\}$ such that $\mathcal{R}^{0}_{\epsilon} \subseteq \cE^{t^{*}}$ and $\mathcal{R}^{0}_{\epsilon} \not\subseteq \cE^{t^{*}+1}$ and $[\blambda^{t^{*}}, \bmu^{t^{*}}] \in \mathcal{R}^{0}$.
\end{lem}

\begin{proof}
From Lemma \ref{lem:ellipsoid maximizer}, $[\blambda^*, \bmu^*] \in \cR^0$ and thus $\cR^0_\epsilon \subseteq \cR^0 \subseteq \cE^0$. We also have the following by simple geometry and the classic ellipsoid volume reduction result of Lemma \ref{lem:ellipsoid-volumes}.
\begin{align*}
\vol(\mathcal{R}^{0}_{\epsilon}) &= \epsilon^{d+K} \vol(\cR^0) =  \epsilon^{d+K}2^{-K} \vol(\cE^0)\\
\vol(\cE^T) 
&\leq \exp\left( \frac{-T}{2(d+K)} \right) \vol(\cE^0) \\
&\leq \exp\left((d+K) \log\left(\frac{\epsilon}{2} \right) \right) \vol(\cE^0)
<  \vol(\mathcal{R}^{0}_{\epsilon})
\end{align*}
And hence $\cR^0_\epsilon \nsubseteq \cE^T$. Clearly, there exists an iteration $t^* \in \{0,1,\ldots,T-1\}$ such that $\cR^0_\epsilon \subseteq \cE^{t^*}$ but $\cR^0_\epsilon \nsubseteq \cE^{t^*+1}$. If $[\blambda^{t^*},\bmu^{t^*}] \notin \cR^0$, then by Lemma \ref{lem:ellipsoid center filter} we have that $\cE^{t^*+1} \supseteq \cE^{t^*} \cap \cR^0 \supseteq \cR^0_\epsilon$, giving a contradiction.  Thus $[\blambda^{t^*},\bmu^{t^*}] \in \cR^0$.

\end{proof}
\noindent
We now prove that $f^\con$ is a Lipschitz function  w.r.t. $\ell_2$ norm over the domain $\mathcal{R}^{0}$. We will exploit this fact later in the proof for Theorem \ref{thm:ellipsoid-dual-con}. 
\begin{lem}
\label{lem:con-ellipsoid-f-lipschitz}
$f^\con$ is a $\sqrt{d+K}$-Lipschitz function w.r.t. $\ell_2$ norm over the domain $\mathcal{R}^{0}$. 
\end{lem}
\begin{proof}
The difference $f^\con$ at 
$[\blambda, \bmu] \in \mathcal{R}^{0}$ and $[\blambda',\bmu']  \in \mathcal{R}^{0}$ can be bounded by:
\begin{align*}
f^\con(\blambda, \bmu) - f^\con(\blambda',\bmu') 
&= \min_{\C\in\cC, \bxi \in \Delta_d} \cL^\con(\C, \bxi, \blambda, \bmu) - \min_{\C\in\cC, \bxi \in \Delta_d} \cL^\con(\C, \bxi, \blambda', \bmu') \\
&\leq \max_{\C\in\cC, \bxi \in \Delta_d} \left(
\cL^\con(\C, \bxi, \blambda, \bmu)  -   \cL^\con(\C, \bxi, \blambda', \bmu') \right) \\
&\leq  \max_{\C\in\cC, \bxi \in \Delta_d}  
\left(   (\blambda - \blambda')^\top (\C - \bxi) + (\bmu - \bmu')^\top \bphi(\bxi) \right) \\
&\leq  \max_{\C\in\cC, \bxi \in \Delta_d} \left( \|\blambda - \blambda'\|_1 \|\C - \bxi\|_\infty + \|\bmu - \bmu'\|_1 \|\bphi(\bxi)\|_\infty \right) \\
&\leq \|\blambda - \blambda'\|_1 +  \|\bmu - \bmu'\|_1 = \| [\blambda, \bmu] - [\blambda', \bmu'] \|_{1} \\
& \leq \sqrt{d+K} \hspace{0.3em} \| [\blambda, \bmu] - [\blambda', \bmu'] \|_{2}
\end{align*}
Identically, $f^\con(\blambda',\bmu') - f^\con(\blambda, \bmu) \leq \sqrt{d+K} \hspace{0.3em} \| [\blambda, \bmu] - [\blambda', \bmu'] \|_{2}$. Thus $\vert f^\con(\blambda',\bmu') - f^\con(\blambda, \bmu) \vert \leq \sqrt{d+K} \hspace{0.3em} \| [\blambda, \bmu] - [\blambda', \bmu'] \|_{2}$ which concludes the proof. 
\end{proof}
\noindent
Recall that we only have access to $(\rho, \rho^{'}, \delta)$-approximate LMO. The sample and approximation errors induced by calls to this approximate LMO must be accounted for. It turns out, that despite having access to only an approximate LMO, we are able to achieve a desirable sub-optimality with probability $1-\delta$ over the draw of random sample $S \sim D^N$. \textbf{The rest of the analysis will only apply for this high probability event}. We now present two lemmas that will be helpful in allowing us to show provided an approximate LMO, the iterates $[\blambda^{t}, \bmu^{t}]$ approximately maximize $f^\con$ and subsequently, we will use these results to convert our dual guarantees into primal guarantees. 
\begin{lem}
\label{lem: ellipsoid LMO inequalities}
Let $ t \in \{0, 1, ..., T-1\}$. Then with probability $1-\delta$ (over draw of $S \sim D^N$) uniformly for all $t$, such that $[\blambda^{t}, \bmu^{t}] \in \mathcal{R}^{0}$, we have that: 
\begin{itemize}
    \item ${\blambda^t}^\top \C[h^t] \leq \min_{\C \in \cC} {\blambda^t}^\top \C + a  \rho$
    \item $\|\C[h^t] - \C^t \|_2 \leq \sqrt{d} \rho'$
    \item ${\blambda^t}^\top \C^t \leq \min_{\C \in \cC} {\blambda^t}^\top \C + a(\rho+\sqrt{d}\rho')$
\end{itemize}
\end{lem}
\begin{proof}
The first two inequalities are simply restatements of the definition of $(\rho,\rho',\delta)$-approximate LMO. And the third follows by putting the first two together.


\end{proof}

\begin{lem}
\label{lem:con-ellipsoid-f-supergradient}
Let $t \in \{0, 1, ..., T-1\}$ and let $[\blambda^t,\bmu^t] \in \mathcal{R}^{0}$. Then, $[\C^t - \bxi^t, \bphi(\bxi^t)]$ is a $\tau$-supergradient to $\widehat f^\con$ at $[\blambda^t, \bmu^t] \in \mathcal{R}^{0}$, with $\tau = a(\rho+2\sqrt{d}\rho') $, i.e. for  all $\blambda \in \R^d, \bmu \in \R^K$, 
\[
\widehat f^\con(\blambda, \bmu) 
\leq 
\widehat f^\con(\blambda^t, \bmu^t) + (\blambda - \blambda^t)^\top (\C^t -\bxi^t) +
(\bmu - \bmu^t)^\top (\bphi(\bxi^t)) + \tau
\]
\end{lem}
\begin{proof}
Fix $[\blambda, \bmu] \in \mathcal{R}^{0}$. We have that,
\begin{align*}
\widehat f^\con(\blambda, \bmu)  
&= \min_{\C \in \cC, \bxi \in \Delta_d} \cL(\C,\bxi, \blambda, \bmu) \\
&\leq  \cL(\C[h^t],\bxi^t, \blambda, \bmu) \\
&= \cL(\C^t,\bxi^t, \blambda, \bmu) + (\C[h^t] - \C^t)^\top \blambda \\
&\leq \cL(\C^t,\bxi^t, \blambda, \bmu) + \|\C[h^t] - \C^t\|_2 \|\blambda\|_2 \\
& \leq \cL(\C^t,\bxi^t, \blambda, \bmu) + a\sqrt{d}\rho'.
\end{align*}
Further,
\begin{align*}
\widehat f^\con(\blambda^t, \bmu^t)  
&= \min_{\C \in \cC} {\blambda^t}^\top \C + \min_{\bxi \in \Delta_d} \psi(\bxi) -  {\blambda^t}^\top \bxi + {\bmu^t}^\top \bphi(\bxi) \\
&\geq {\blambda^t}^\top \C^t - a (\rho+\sqrt{d}\rho') + \psi(\bxi^t) -  {\blambda^t}^\top \bxi^t + {\bmu^t}^\top \bphi(\bxi^t) \\
&= \cL(\C^t, \bxi^t, \blambda^t, \bmu^t) - a(\rho+\sqrt{d}\rho') \\
&= \cL(\C^t, \bxi^t, \blambda, \bmu) + (\blambda^t - \blambda)^\top (\C^t - \bxi^t) + (\bmu^t - \bmu)^\top \bphi(\bxi^t) - a(\rho+\sqrt{d}\rho')  \\
&\geq \widehat f^\con(\blambda, \bmu)  - a\sqrt{d} \rho' + (\blambda^t - \blambda)^\top (\C^t - \bxi^t) + (\bmu^t - \bmu)^\top \bphi(\bxi^t) - a(\rho+\sqrt{d}\rho'),
\end{align*}
as desired. If $[\blambda, \bmu] \notin \mathcal{R}^{0}$, the result follows trivially. 
\end{proof}

\noindent
Equipped with lemmas \ref{lem:ellipsoid intermediate center},  \ref{lem:con-ellipsoid-f-supergradient} and \ref{lem:con-ellipsoid-f-lipschitz}, we are now ready to prove that Algorithm \ref{alg:ellipsoid-con} approximately maximizes $f^\con$. The monograph by \citet{Bubeck15} presents a proof to derive the sub-optimality of the regular ellipsoid algorithm, where perfect (sub/ super) gradient access is assumed. In our setting, we only have access to approximate super-gradients. We show how to adapt the proof of \citet{Bubeck15} to our setting, in the proof for Theorem \ref{thm:ellipsoid-dual-con}. 

\begin{lem}
\label{lem: filter optimality}
Let $\tau = a( \rho + 2\sqrt{d}\rho')$. For any $t \in \{0,1,\ldots, T-1\}$, such that, $[\blambda^{t}, \bmu^{t}] \in \mathcal{R}^{0}$
\[
    \cE^{t} \setminus \cE^{t+1} \subset \{[\blambda, \bmu] \in \R^{d+K}: \widehat f^\con(\blambda, \bmu) \leq \widehat f^\con(\blambda^{t}, \bmu^{t}) + \tau \}
\]
\end{lem}
\begin{proof}
Pick $t \in \{0,1,\ldots, T-1\}$, such that $[\blambda^{t}, \bmu^{t}] \in \mathcal{R}^{0}$. We know by lemma \ref{lem:con-ellipsoid-f-supergradient} that $\g^{t} := [\C^{t} -\bxi^{t}, \bphi(\bxi^{t})]$ is a $\tau$ super-gradient to $\widehat{f}^{\con}$ at $[\blambda^{t}, \bmu^{t}]$. Thus, $\forall \hspace{0.15em} \blambda \in \R^d, \hspace{0.3em} \forall \hspace{0.15em} \bmu \in \R^K$, we have that 
\begin{equation}
\widehat f^\con(\blambda, \bmu) 
\leq 
\widehat f^\con(\blambda^{t}, \bmu^{t}) + (g^{t})^\top([\blambda, \bmu] - [\blambda^{t}, \bmu^{t}]) + \tau \label{eqn:tau-supergrad-gt}
\end{equation}
\noindent
Since $[\blambda^{t}, \bmu^{t}] \in \mathcal{R}^{0}$, the \textbf{else} condition (line 11) of Algorithm \ref{alg:ellipsoid-con} gets invoked and we get that  $\cE^{t+1} = \JLE(\cE^t, g^{t})$ and thus by Lemma \ref{lem:ellipsoid-volumes} and Equation \eqref{eqn:tau-supergrad-gt} we have the following:
\begin{align*}
     \cE^{t+1} &\supseteq \cE^t \cap \{[\blambda, \bmu] \in \R^{d+K}: (\g^{t})^\top([\blambda, \bmu] - [\blambda^{t}, \bmu^{t}]) \geq 0\} \\
     \cE^t \setminus \cE^{t+1} &\subseteq  \{[\blambda, \bmu] \in \R^{d+K}: (\g^{t})^\top([\blambda, \bmu] - [\blambda^{t}, \bmu^{t}]) < 0\} \\
     &\subseteq \{[\blambda, \bmu] \in \R^{d+K}: \widehat f^\con(\blambda, \bmu) \leq \widehat f^\con(\blambda^{t}, \bmu^{t}) + \tau \} 
\end{align*}
\noindent
where the second line follows from the argument that for any sets $A,B,C$, if $A \supset B \cap C$ then $B \setminus A \subseteq C^c$, and the last line follows from Equation \eqref{eqn:tau-supergrad-gt}.
\end{proof}

\begin{thm}
\label{thm:ellipsoid-dual-con}
Let the assumptions stated in Theorem \ref{thm:ellipsoid-con} hold. Then,
\[
\max_{0 \leq t\leq T-1} \widehat f^\con({\blambda^t, \bmu^t}) \geq \max_{\blambda, \bmu}  \widehat f^\con(\blambda, \bmu) - \big(4a\sqrt{d+K}\big)\cdot\exp\left(\frac{-T}{2(d+K)^{2}}\right) - \tau
\]
where $\tau=a(\rho+2\sqrt{d}\rho')$
\end{thm}
\begin{proof}

Due to lemma \ref{lem:ellipsoid maximizer}, we know that $\exists \hspace{0.3em}[\blambda^{*}, \bmu^{*}] \in \mathcal{R}^{0}$, where $[\blambda^{*}, \bmu^{*}]$ is a maximizer of $f^\con$ over $\R^{d} \times \R^{K}_{+}$. Set $\epsilon = 2\exp\left(\frac{-T}{2(d+K)^{2}}\right)$ which implies $T > 2(d+K)^{2}log(\frac{2}{\epsilon})$. Let $\cR^0_\epsilon \subseteq \cR^0 \subseteq \R^d\times \R_+^K$ be
\[\mathcal{R}^{0}_{\epsilon} := \{[\blambda, \bmu]\in \mathcal{R}^{0}:(1-\epsilon)[\blambda^{*}, \bmu^{*}] + \epsilon[\blambda, \bmu] \} .\]

By Lemma \ref{lem:ellipsoid intermediate center}, there exists an iteration $t^{*} \in \{0,1,\ldots, T-1\}$, such that, 
$\mathcal{R}^{0}_{\epsilon} \subseteq \cE^{t^{*}}, \hspace{0.3em}$ $\mathcal{R}^{0}_{\epsilon} \not\subseteq \cE^{t^{*}+1}$ and $[\blambda^{t^{*}}, \bmu^{t^{*}}] \in \mathcal{R}^{0}$. Pick any element $ [\blambda_{\epsilon}, \bmu_{\epsilon}] \in \mathcal{R}^{0}_{\epsilon}\setminus {\cE}^{t^{*}+1} \subseteq \cE^{t^*} \setminus {\cE}^{t^{*}+1}$. Because of the definition of $\mathcal{R}^{0}_{\epsilon}, \hspace{0.3em} \exists \hspace{0.15em} [\blambda^{'}, \bmu{'}] \in \mathcal{R}^{0}$, such that, $[\blambda_{\epsilon}, \bmu_{\epsilon}]  = (1-\epsilon)[\blambda^{*}, \bmu^{*}] + \epsilon[\blambda^{'}, \bmu{'}]$. Due to Lemma \ref{lem: filter optimality}, we have that,
\begin{align*}
    \widehat f^\con (\blambda^{t^{*}}, \bmu^{t^{*}}) &\geq \widehat f^\con (\blambda_{\epsilon}, \bmu_{\epsilon}) - \tau \\
    &= f^\con(\blambda_{\epsilon}, \bmu_{\epsilon}) - \tau \\
    &= f^\con((1-\epsilon)\blambda^{*} + \epsilon\blambda^{'}, (1-\epsilon)\bmu^{*} + \epsilon\bmu^{'}) - \tau \\ 
    &\geq (1-\epsilon)f^\con(\blambda^{*}, \bmu^{*}) + \epsilon f^\con(\blambda^{'}, \bmu^{'}) - \tau \\
    &\geq (1-\epsilon)f^\con(\blambda^{*}, \bmu^{*}) + \epsilon (f^\con(\blambda^{*}, \bmu^{*}) - 2a\sqrt{d+K}) - \tau \\ 
    &= f^\con(\blambda^{*}, \bmu^{*}) - \epsilon (2a\sqrt{d+K}) - \tau \\
    &= f^\con(\blambda^{*}, \bmu^{*}) - 4a\sqrt{d+K}\exp\left(\frac{-T}{2(d+K)^{2}}\right)  - \tau 
\end{align*}
the second inequality in the above argument is due to the concavity of $f^\con$ and the third inequality is due to the $\sqrt{d+K}$ Lipschitzness of $f^\con$ in $\cR^0$ (Lemma \ref{lem:con-ellipsoid-f-lipschitz}) and the $\ell_2$-norm diameter of the set $\cR^0$ being bounded above by $2a$. The theorem follows from the equality of $f^\con$ and $\widehat f^\con$ within $\cR^0$.
\end{proof}

\subsubsection{Converting  guarantee on $\hat{f}^\con$ to  primal optimality-feasibility guarantees}
Now, we can bound the primal sub-optimality using a standard technique from optimization theory \cite{lee2015faster}. Throughout, we will appeal to the high-probability inequalities established in \ref{lem: ellipsoid LMO inequalities}. 

\begin{lem}
\label{lem:dual-to-primal-1}
Denote $\widetilde \cC = \textup{conv}(\{\C[h^0], \C[h^1], \ldots, \C[h^{T-1}]\})$. We then have:
\begin{align*}
\min_{\C \in \widetilde \cC, \bphi(\C) \leq 0} \psi(\C)
    &\leq
    \min_{\C \in  \cC, \bphi(\C) \leq 0} \psi(\C) + \big(4a\sqrt{d+K}\big)\cdot\exp\left(\frac{-T}{2(d+K)^{2}}\right) + 2\tau 
\end{align*}
where $\tau=a(\rho+2\sqrt{d}\rho')$.
\end{lem}
\begin{proof}
Consider an alternative version of $f^\con$ defined as 
\[
 \widetilde f^\con(\blambda, \bmu) = \min_{\C \in \widetilde \C, \bxi \in[0,1]^d} \cL^\con(\C, \xi, \blambda,  \bmu).
\] 
And let $\widehat{\widetilde{f}}^\con$ be equal to $\widetilde f^\con$ if its argument $\blambda, \bmu$ is inside the $\ell_2$-norm ball of radius $a$ and $\bmu \geq 0$, and negative infinity otherwise.

 Clearly we have that $\widehat {\widetilde f}^\con(\blambda, \bmu) \geq \widehat f^\con(\blambda, \bmu)$. We can also show $\widehat {\widetilde f}^\con$ and $\widehat f^\con$ are close at the iterates $[\blambda^t, \bmu^t]$. If $[\blambda^t, \bmu^t] \notin \mathcal{R}^{0}$, then both sides are trivially equal to negative infinity. Suppose $[\blambda^t, \bmu^t] \in \mathcal{R}^{0}$, we then have:
\begin{align}
\widehat {\widetilde f}^\con({\blambda^t, \bmu^t}) =  {\widetilde f}^\con({\blambda^t, \bmu^t}) &\leq \cL(\C[h^t], \bxi^t, {\blambda^t}, \bmu^t) \nonumber \\
&= \psi(\bxi^t) - {\blambda^t}^\top \bxi^t + {\blambda^t}^\top \C[h^t] + {\bmu^t}^\top \bphi(\bxi^t) \nonumber \\
&= \min_{\bxi \in \Delta_d} \left( \psi(\bxi) - {\blambda^t}^\top \bxi +  {\bmu^t}^\top \bphi(\bxi^t) \right) + {\blambda^t}^\top \C[h^t]  \nonumber \\
&\leq \min_{\bxi \in \Delta_d} \left( \psi(\bxi) - {\blambda^t}^\top \bxi +  {\bmu^t}^\top \bphi(\bxi^t) \right) + \min_{\C \in \cC} {\blambda^t}^\top \C  + a \rho \nonumber \\
&= \min_{\bxi \in \Delta_d, \C \in \cC} \left( \psi(\bxi) + {\blambda^t}^\top (\C - \bxi) + {\bmu^t}^\top \bphi(\bxi)  \right) + a \rho  \nonumber \\
&= f^\con ({\blambda^t}, \bmu^t) + a\rho  = \widehat f^\con ({\blambda^t}, \bmu^t) + a\rho. 
\label{eqn:tilde-hat-ineq}
\end{align}

From Lemma \ref{lem:ellipsoid maximizer} and the min-max theorem, we have the following:
\begin{align}
\max_{\blambda \in \R^d, \bmu \in \R^K_+} \widehat {\widetilde f}^\con(\blambda, \bmu) &=
 \max_{\blambda \in \R^d, \bmu \in \R^K_+} \widetilde f^\con(\blambda, \bmu) \nonumber \\
 &= \max_{\blambda \in \R^d, \bmu \in \R^K_+} \min_{\C \in \widetilde \cC, \bxi \in\Delta_d} \cL^\con(\C, \xi, \blambda, \bmu) \nonumber \\
 &=  \min_{\C \in \widetilde \cC, \bxi \in\Delta_d} \max_{\blambda \in \R^d, \bmu \in \R^K_+} \cL^\con(\C, \xi, \blambda, \bmu) \nonumber \\
 &= \min_{\C \in \widetilde \cC, \bphi(\C) \leq 0} \psi(\C).
 \label{eqn:min-max-tilde-hat}
\end{align}

Recall that Algorithm \ref{alg:ellipsoid-con} is designed to find the minimum of $\psi$ over $\cC$ (subject to constraints $\bphi$). However, the exact same sequence of iterates would also apply for minimizing ove $\widetilde \cC$, and hence the sequence of iterates ${\blambda^t, \bmu^t}$ also approximately maximise $\widetilde f^\con$. Then by Theorem \ref{thm:ellipsoid-dual-con} and Equation \eqref{eqn:tilde-hat-ineq} we have,
\allowdisplaybreaks
\begin{align*}
 \max_{\blambda \in \R^d, \bmu \in \R^K_+} \widehat{\widetilde f}^\con(\blambda, \bmu)
 & \leq \max_{0 \leq t\leq T} \widehat{\widetilde f}^\con({\blambda^t, \bmu^t})  + \big(4a\sqrt{d+K}\big)\cdot\exp\left(\frac{-T}{2(d+K)^{2}}\right) + \tau  \\
& \leq \max_{0 \leq t\leq T}  \widehat f^\con({\blambda^t, \bmu^t})  + a\rho + \big(4a\sqrt{d+K}\big)\cdot\exp\left(\frac{-T}{2(d+K)^{2}}\right) + \tau  \\
& \leq \max_{\blambda \in \R^d, \bmu \in \R^K_+}  \widehat f^\con(\blambda, \bmu) + \big(4a\sqrt{d+K}\big)\cdot\exp\left(\frac{-T}{2(d+K)^{2}}\right) + 2\tau  \\
 \end{align*}

Putting the above together with Equation \eqref{eqn:min-max-tilde-hat} we get,
\begin{align*}
\min_{\C \in \widetilde \cC, \bphi(\C) \leq 0} \psi(\C)
    &\leq
    \min_{\C \in  \cC, \bphi(\C) \leq 0} \psi(\C) + \big(4a\sqrt{d+K}\big)\cdot\exp\left(\frac{-T}{2(d+K)^{2}}\right) + 2\tau
\end{align*}
where $\tau=a(\rho+2\sqrt{d}\rho')$.
which completes the proof.
\end{proof}

\begin{lem}
\label{lem:dual-to-primal-2}
Let $\balpha^* \in \underset{\balpha \in \Delta_{T}, \bphi(\sum_t \alpha_t\C^t)\leq 0 }{\argmin} \psi\left( \sum_{i=0}^{T-1} \alpha_{i} \C^{i} \right)$.
Then:
\begin{align*}
\psi\left( \sum_{i=0}^{T-1} \alpha^*_{i} \C[h^i] \right)
&\leq 
 \min_{\C \in  \widetilde \cC, \bphi(\C) \leq 0} \psi(\C) + 2\tau; \\
\phi_k \left( \sum_{i=0}^{T-1} \alpha^*_{i} \C[h^i] \right)
&\leq \tau,
\end{align*}
where $\widetilde \C = \textup{conv}(\{\C[h^0], \ldots, \C[h^{T-1}]\})$. 
\end{lem}
\begin{proof}
Let $\bbeta^* \in \underset{\bbeta \in \Delta_{T}, \bphi(\sum_t \beta_t\C[h^t])\leq 0}{\argmin} \psi\left( \sum_{i=0}^{T-1} \beta_{i} \C[h^i] \right)$
denote the  coefficients obtained by solving a similar minimization problem with the estimates $\C^t$ replaced with the true confusion matrices $C[h^t]$. 
First, we note that $\balpha^*$ and $\bbeta^*$ exist because $h_0$ (and in turn, $\C[h^0]=\C^0$) is strictly feasible. 
\allowdisplaybreaks
\begin{align*}
\psi\left( \sum_{i=0}^{T-1} \alpha^*_{i} \C[h^i] \right)
&= \psi\left( \sum_{i=0}^{T-1} \alpha^*_{i} \C^i + \sum_{i=0}^{T-1} \alpha^*_{i}(\C[h^i]-\C^i) \right) \\
&\leq  \psi\left( \sum_{i=0}^{T-1} \alpha^*_{i} \C^i \right) + L\rho' \sqrt{d} \\
&= \min_{\balpha \in \Delta_T} \psi\left(\sum_{i=0}^{T-1} \alpha_{i} \C^i\right) + L\rho' \sqrt{d}  \\
&\leq  \psi\left(\sum_{i=0}^{T-1} \beta^*_{i} \C^i\right) + L\rho'\sqrt{d}   \\
&=  \psi\left(\sum_{i=0}^{T-1} \beta^*_{i} \C[h^i] + \sum_{i=0}^{T-1} \beta^*_{i} (\C^i - \C[h^i]) \right) + L\rho'\sqrt{d} \\
&\leq  \psi\left(\sum_{i=0}^{T-1} \beta^*_{i} \C[h^i] \right) + 2L\rho'\sqrt{d} \\
&= \min_{\bbeta \in \Delta_T, \bphi(\sum_t \beta_t\C[h^t])\leq 0} \psi\left(\sum_{i=0}^{T-1} \beta_{i} \C[h^i] \right) + 2L\rho'\sqrt{d} \\
&= \min_{\C \in \widetilde \cC, \bphi(\C) \leq 0} \psi(\C) + 2L\rho'\sqrt{d}  \\
& \leq \min_{\C \in  \widetilde \cC, \bphi(\C) \leq 0} \psi(\C) + 2\tau,
\end{align*}
where the first and third inequality above are due to the Lipschitzness of $\psi$.

Using a similar argument as above, we get for all $k \in [K]$,
\begin{align*}
\phi_k\left( \sum_{i=0}^{T-1} \alpha^*_{i} \C[h^i] \right)
&= \phi_k\left( \sum_{i=0}^{T-1} \alpha^*_{i} \C^i + \sum_{i=0}^{T-1} \alpha^*_{i}(\C[h^i]-\C^i) \right) \\
&\leq  \phi_k\left( \sum_{i=0}^{T-1} \alpha^*_{i} \C^i \right) + L\rho' \sqrt{d} \\
&\leq 0 +  L\rho' \sqrt{d} \leq \tau,
\end{align*}
where the first inequality above is due to the Lipschitzness of $\phi$, and the second inequality is due to the property of $\balpha^*$ being chosen from a set such that the weighted combination of $\C^i$ is feasible.

We are now ready to prove Theorem \ref{thm:ellipsoid-con}.
\end{proof}
\begin{proof}[Proof of Theorem \ref{thm:ellipsoid-con}]
Let $\balpha^* \in \underset{\balpha \in \Delta_{T}, \bphi(\sum_t \alpha_t\C^t)\leq \0 }{\argmin} \psi\left( \sum_{i=0}^{T-1} \alpha_{i} \C^{i} \right).$  Let $\bar{d}=d+K$.
Putting Lemmas \ref{lem:dual-to-primal-1} and \ref{lem:dual-to-primal-2} together we get,
\begin{align*}
\psi(\C[\overline h]) &=  \psi\left( \sum_{i=0}^{T-1} \alpha^*_{i} \C[h^i] \right) \\
&\leq     \min_{\C \in  \cC, \bphi(\C) \leq 0} \psi(\C) + \big(4a\sqrt{\bar{d}}\big)\cdot\exp\left( \frac{-T}{2(\bar{d})^{2}} \right) + 4\tau   \\
 &=  \min_{\C \in  \cC, \bphi(\C) \leq 0} \psi(\C) + \big( 4a\sqrt{\bar{d}} \big) \cdot\exp\left( \frac{-T}{2(\bar{d})^{2}} \right) + 4\tau 
\end{align*}

We now set $T=2\bar{d}^2 \log\left(\frac{\bar{d}}{\epsilon}\right)$ to obtain
\begin{align*}
\psi(\C[\overline h]) 
&\leq  \min_{\C \in  \cC, \bphi(\C) \leq 0} \psi(\C) + \big(4a\big)\epsilon + 4\tau 
\end{align*}

The feasibility inequality then follows easily from Lemma \ref{lem:dual-to-primal-2}
\begin{align*}
\phi_k(\C[\overline h]) &=  \phi_k\left( \sum_{i=0}^{T-1} \alpha^*_{i} \C[h^i] \right) 
~~\leq \tau.
 \end{align*}
for all $k \in [K]$.
\end{proof}